\documentclass[12pt]{article}

\usepackage[utf8]{inputenc}
\usepackage[english]{babel}
\usepackage[margin=1in]{geometry}

\usepackage{blindtext}
\usepackage{appendix}
\usepackage{amssymb}
\usepackage{mathrsfs}
\usepackage{hyperref}
\hypersetup{colorlinks=true,
            linkcolor=blue,
            anchorcolor=blue,
            citecolor=blue} 
\usepackage{enumerate}
\usepackage[round]{natbib}
\usepackage{amsthm}
\usepackage{amsmath}
\usepackage{graphicx}
\usepackage{subfig}
\usepackage{float}
\usepackage{booktabs}
\usepackage{multirow}

\usepackage{algorithm, algorithmic}
\newtheorem{theorem}{Theorem}[section]

\newtheorem{lemma}[theorem]{Lemma}
\newtheorem{assumption}{Assumption}[section]
\theoremstyle{remark}
\newtheorem{remark}{Remark}[section]

\theoremstyle{definition}
\newtheorem{definition}{Definition}[section]

\def\m0{\mathbf{0}}
\def \mb {\mathbf{b}}
\def \mh {\mathbf{h}}
\def \mm {\mathbf{m}}
\def \mn {\mathbf{n}}
\def \mj {\mathbf{j}}
\def \mk {\mathbf{k}}
\def \ms {\mathbf{s}}
\def \mv {\mathbf{v}}
\def \mw {\mathbf{w}}
\def \mx {\mathbf{x}}
\def \my {\mathbf{y}}
\def \mz {\mathbf{z}}
\def \mI {\mathbf{I}}

\def \mU {\mathbf{U}}
\def \mV {\mathbf{V}}
\def \mW {\mathbf{W}}
\def \mX {\mathbf{X}}
\def \mY {\mathbf{Y}}
\def \mZ {\mathbf{Z}}

\def \mrd {\mathrm{d}}
\def \mrN {\mathrm{NN}}

\def \mcB {\mathcal{B}}
\def \mcC {\mathcal{C}}
\def \mcD {\mathcal{D}}

\def \mcH {\mathcal{H}}
\def \mcP {\mathcal{P}}
\def \mcL {\mathcal{L}}
\def \mcM {\mathcal{M}}
\def \mcN {\mathcal{N}}
\def \mcO {\mathcal{O}}
\def \mcT {\mathcal{T}}
\def \mcW {\mathcal{W}}
\def \mcX {\mathcal{X}}
\def \mcY {\mathcal{Y}}
\def \mcZ {\mathcal{Z}}

\def\Ebb{\mathbb{E}}
\def \Rbb{\mathbb{R}}
\def \Pbb{\mathbb{P}}

\def\wh{\widehat}
\def\wt{\widetilde}
\def\ov{\overline}

\newcommand{\Law}{{\operatorname{Law}}}

\newcommand{\revxsoft}{\chi_{\mathbf n,\delta_{\mathrm{ind}}}(\mx)}

\title{\textbf{
Deep Bootstrap 
}}

\author{
Jinyuan Chang
\thanks{
Joint Laboratory of Data Science and Business Intelligence, Southwestern University of Finance and Economics, Chengdu, Sichuan 611130, China.
}
\thanks{
State Key Laboratory of Mathematical Sciences, Academy of Mathematics and Systems Science, Chinese Academy of Sciences, Beijing 100190, China.
Email: changjinyuan@swufe.edu.cn}
\and
Yuling Jiao
\thanks{
School of Artificial Intelligence, 
Hubei Key Laboratory of Computational Science, 
Wuhan University, Wuhan 430072, China. Email: yulingjiaomath@whu.edu.cn
}
\and
Lican Kang
\thanks{
Institute for Math and AI, 
Hubei Key Laboratory of Computational Science, School of 
Artificial Intelligence,
Wuhan University, Wuhan 430072, China.
Email: kanglican@whu.edu.cn
}
\and
Junjie Shi
\thanks{
School of Mathematics and Statistics, Wuhan University, Wuhan 430072, China.
Email: shijunjie@whu.edu.cn
}
}

\date{}

\begin{document}

\maketitle
\begin{abstract}
In this work, we propose a novel deep bootstrap framework for nonparametric regression based on conditional diffusion models. Specifically, we construct a conditional diffusion model to learn the distribution of the response variable given the covariates. This model is then used to generate bootstrap samples by pairing the original covariates with newly synthesized responses. We reformulate nonparametric regression as conditional sample mean estimation, which is implemented directly via the learned  conditional diffusion model. Unlike traditional bootstrap methods that decouple the estimation of the conditional distribution, sampling, and nonparametric regression, our approach integrates these components into a unified generative framework.  With the expressive capacity of diffusion models, our method facilitates both efficient sampling from high-dimensional or multimodal distributions and accurate nonparametric estimation.
We establish theoretical guarantees for the proposed method. In particular, we derive an end-to-end convergence rate in total variation distance between the learned and target conditional distributions.
Our analysis combines a normalized local-polynomial density approximation with a relative-Fisher transfer argument to obtain conditional-score approximation bounds without additional boundary smoothness assumptions.
For a corresponding nondegenerate density class, and a fixed covariate design density bounded away from zero, the polynomial part of this rate is minimax optimal for the average conditional total-variation risk.
We also study the bounded-Lipschitz distance between the sampling and bootstrap laws of the resulting estimator, and formulate a separate consistency condition at the stochastic scale for confidence-interval validity. Numerical studies demonstrate the effectiveness and scalability of our approach for complex regression tasks.

\vspace{0.5cm} \noindent{\bf KEY WORDS}:
Bootstrap, Statistical inference,
Conditional   diffusion model, End-to-end convergence rate. 
\end{abstract}

\section{Introduction}
Across many modern scientific domains, inferential goals have extended far beyond point estimation to the construction of confidence intervals and other measures of uncertainty, which are indispensable for assessing estimator variability and supporting reliable scientific and operational decisions.
Such needs arise in areas ranging from biomedical research \citep{pencina2004overall,benjamini2005false} and genomics \citep{efron2001empirical,storey2003statistical}, to environmental science \citep{tebaldi2007use,north2011correlation}, econometrics \citep{stock2002testing,paparoditis2003residual}, and  computer-model calibration \citep{kennedy2001bayesian}. By providing a confidence range for parameter estimates, interval estimation mitigates the risks of real-world model deployment, facilitates the practical application of data science to interdisciplinary problems, and fosters robust decision-making \citep{kirch2025challenges}. As a result, the demands for interval estimation, and consequently for its validity and precision, have experienced a sustained increase over time and are reflected in a number of recent studies. For example, in proteomics, confidence intervals are employed to assess the association between post-translational modifications and intrinsically disordered regions of proteins, validating hypotheses derived from predictive models and facilitating large-scale functional analyses \citep{tunyasuvunakool2021highly,bludau2022structural}. In genomic research, confidence intervals are leveraged to characterize the distribution of gene expression levels, enabling robust inferences about promoter sequence effects and genetic variability \citep{vaishnav2022evolution}. In the realm of environmental science, interval estimation can be used to monitor deforestation rates of forests, yielding uncertainty-aware insights critical for climate policy formulation \citep{bullock2020satellite}. As for social sciences, confidence intervals are utilized to evaluate relationships between socioeconomic factors, bolstering the robustness of conclusions drawn from census data \citep{ding2021retiring}. 
In these contexts, analytic characterization of sampling distributions is often infeasible, underscoring the enduring importance of bootstrap methodology \citep{efron1979bootstrap}, whose model-agnostic resampling framework provides a versatile and practically reliable foundation for uncertainty quantification in contemporary statistical applications.

The Bootstrap method, originally proposed by \cite{efron1979bootstrap}, has become a foundational tool in modern statistical inference. It is widely used for estimating standard errors, constructing confidence intervals, correcting bias, and evaluating model performance.
The core idea of the Bootstrap is to approximate the sampling distribution of a statistic by repeatedly resampling, with replacement, from the observed data to generate multiple resampled datasets. This resampling-based approach is entirely data-driven and nonparametric, requiring no strong assumptions about the underlying population distribution. As such, it bypasses the derivation of asymptotic distributions, which are often difficult or infeasible to obtain in complex or nonstandard settings, and thus provides substantial flexibility and broad applicability. 
In particular, Bootstrap methods have found extensive applications in nonparametric regression analysis  
\citep{freedman1981bootstrapping,silverman1987bootstrap,hardle1988bootstrapping,hardle1991bootstrap,rutherford1991error,hall1992bootstrap,hardle1993comparing}. 
Beyond estimating confidence intervals for regression coefficients in parametric models, 
these methods have been successfully extended to nonparametric regression models, enabling rigorous uncertainty quantification for the regression function itself. 
For a comprehensive introduction  of bootstrap theory and methods, see the related books \citep{efron1994introduction, shao2012jackknife, hall2013bootstrap}.
We now briefly review the application of bootstrap methods in nonparametric regression. Consider the nonparametric regression model:
$$
\mY=f_0(\mX)+\epsilon,
$$
where $\mY \in \mathbb{R}$ is the response variable, $\mX \in \mathbb{R}^{d_\mcX}$ is the covariate vector, and $\epsilon$ is a random error term satisfying $\mathbb{E}[\epsilon | \mX] = 0$ and $\operatorname{Var}(\epsilon) < \infty$. The regression function is defined as $f_0(\mx) := \mathbb{E}[\mY | \mX = \mx]$, $\mx \in \mathbb{R}^{d_\mcX}$.
We assume access to an independent and 
identically distributed (i.i.d) dataset
$\mathcal{D}:=\{(\mX_i,\mY_i)\}_{i=1}^n$  drawn from the joint distribution $P_{\mX,\mY} = P_\mX \times P_{\mY|\mX}$ of the 
covariate-response pair $(\mX,\mY)$.
Here, $P_\mX$ denotes 
the marginal distribution of the covariate $\mX$, 
and $P_{\mY|\mX}$ denotes the conditional distribution of $\mY$ given $\mX$.
In bootstrap methods for nonparametric regression, we first obtain an estimator $\widehat{f}$ of the underlying regression $f_0$ using the original dataset $\mathcal{D}$. This estimator can be constructed using a variety of classical nonparametric methods, such as kernel estimation, local polynomial, splines, wavelets, or nearest-neighbor regression \citep{gyorfi2002distribution,wasserman2006all,biau2015lectures}, as well as modern approaches based on deep neural networks
(DNNs)
\citep{bauer2019deep, schmidt2020nonparametric, kohler2021rate, farrell2021deep, jiao2023deep, bhattacharya2024deep}.
We note that while classical nonparametric methods often face performance limitations in high-dimensional settings, DNN–based methods have demonstrated superior adaptability and accuracy in such complex scenarios.
Next, we construct a bootstrap dataset $\mathcal{D}^* := \{(\mX_i^*, \mY_i^*)\}_{i=1}^n \sim P^*_{\mX,\mY} := P^*_\mX \times P^*_{\mY|\mX}$. 
Here, $\mathcal{D}^*$  consists of independent replicates of the bootstrap covariate-response pair 
$(\mX^*, \mY^*)$.  The distribution 
$P^*_{\mX,\mY}$ denotes the joint distribution of   
$(\mX^*, \mY^*)$ and is   an approximation to the true 
data-generating distribution $P_{\mX,\mY}$. Specifically, $P^*_\mX$ represents the marginal distribution of the bootstrap covariate $\mX^*$, while $P^*_{\mY|\mX}$ corresponds to the conditional distribution of the bootstrap response $\mY^*$ given $\mX^*$. 
Various techniques have been proposed to  construct such bootstrap datasets, including the residual Bootstrap \citep{freedman1981bootstrapping}, the wild Bootstrap \citep{hardle1991bootstrap, hardle1993comparing}, the paired (or naive resampling) Bootstrap,  the smooth Bootstrap \citep{silverman1987bootstrap},
and the conditional Bootstrap \citep{rutherford1991error}. 
A common approach, for instance, is to fix the covariates as $\mX_i^* = \mX_i$, and sample $\mY_i^* \sim \widehat{P}_{\mY|\mX=\mX_i}$, where $\widehat{P}_{\mY|\mX}$ is an estimator of the conditional distribution of $\mY$ given $\mX$, such as one obtained via conditional kernel density estimation.
Applying the same estimation procedure used for $\widehat{f}$ to the bootstrap dataset $\mathcal{D}^*$, we obtain a bootstrap estimator $\widehat{f}^*$.
This conditional bootstrap framework allows for the derivation  of the conditional distribution   $\widehat{P}_{\mY|\mX}$, and thus facilitates the construction of confidence intervals  for $f_0(\mx)$. However, generating samples from $\widehat{P}_{\mY|\mX}$ can be computationally challenging, particularly when the conditional distribution is high-dimensional or   multi-modal   \citep{dunson2020hastings}. 

Existing bootstrap methods for nonparametric regression encompass a variety of resampling techniques and nonparametric estimation approaches. However, these methods face significant challenges, particularly in high-dimensional settings where traditional nonparametric estimators often exhibit degraded performance. Specifically, smooth bootstrap and conditional bootstrap methods require the estimation and sampling from (conditional) distributions, processes that can be hindered by the complexities associated with high-dimensionality and multimodality. As a result, current bootstrap methods for nonparametric regression lack a unified framework that can simultaneously address both resampling and nonparametric estimation in high-dimensional   regression settings.
To address these challenges, we propose a deep bootstrap method that unifies nonparametric estimation and conditional distribution learning through a conditional diffusion model \citep{song2020score}. By integrating this deep generative modeling framework into the bootstrap procedure for nonparametric regression, our method ensures both efficient sampling and accurate nonparametric estimation, even in high-dimensional and complex settings.
Our approach proceeds in two main stages. First, we construct a conditional diffusion model  based on the original dataset $\mathcal{D}$ to learn the conditional distribution $\widehat{P}_{\mY|\mX}$, which is the estimator of the true conditional distribution $P_{\mY|\mX}$.
Under certain conditions, we derive an end-to-end convergence rate in total variation distance between $\widehat{P}_{\mY|\mX}$ and $P_{\mY|\mX}$.
Using this estimated distribution, we can obtain the nonparametric regression estimator as 
$
\widehat{f}(\mx) = \frac{1}{n} \sum_{j=1}^n \widehat{\mY}_{\mx}^{(j)},
$
where 
$\{\widehat{\mY}_{\mx}^{(j)}\}_{j=1}^n$ are i.i.d. samples drawn from $\widehat{P}_{\mY|\mX=\mx}$ by running the sampling dynamics of the constructed conditional diffusion model.
Here, we use the law of large numbers and the fact that  the underlying regression function can be expressed as $f_0(\mx):=\mathbb{E}(\mY_{\mx})$, where $\mY_{\mx} \sim P_{\mY|\mX=\mx}$. 
Next, we generate a bootstrap dataset $\mathcal{D}^* = \{(\mX_i^*, \mY_i^*)\}_{i=1}^n$, where each response $\mY_i^*$ is sampled from the learned conditional distribution $\widehat{P}_{\mY|\mX = \mX_i}$ via the conditional diffusion model, while covariates $\mX_i^*$ are taken from the original dataset. Based on $\mathcal{D}^*$, we compute the bootstrap estimator $\widehat{f}^*(\mx)$ using the same nonparametric technique employed to obtain $\widehat{f}(\mx)$.
By repeating this bootstrap procedure multiple times, we can construct empirical confidence intervals for the target regression function $f_0(\mx)$. Under certain conditions, we establish the convergence property of our deep bootstrap method and provide coverage guarantees for the resulting confidence interval.
\subsection{Contributions}
Our contributions are shown as follows:
\begin{itemize}
\item  
We propose a novel deep bootstrap method for nonparametric regression by introducing conditional diffusion models. Unlike conventional bootstrap approaches, which separately address resampling and nonparametric estimation, our method uses conditional diffusion models to jointly generate samples via learned stochastic dynamics while simultaneously conducting nonparametric estimation. This integrated approach effectively mitigates the challenges inherent in distribution estimation and sampling, particularly in high-dimensional or multimodal settings, thus significantly enhancing the applicability of bootstrap methods. Numerical experiments validate the accuracy, efficiency, and scalability of our method across diverse regression scenarios.
\item 
We derive a pointwise-in-time relative-Fisher bound from uniform initial relative density error and combine it with neural-network approximation of the score of a normalized local-polynomial density proxy after Gaussian smoothing. This yields conditional score approximation bounds without additional density smoothness near the support boundary. Together with statistical and sampling error analyses, it gives an end-to-end convergence rate in total variation distance.
For a corresponding nondegenerate density class and a fixed covariate design density bounded away from zero, a matching minimax lower bound establishes optimality of the polynomial rate under average conditional total-variation loss.
This result not only affirms the validity of our method but also offers broader contributions to the theoretical understanding of diffusion models. 
\item
Building on the result given above, we further establish an unscaled bounded-Lipschitz bound comparing the sampling and bootstrap laws of the regression estimator. Under additional pointwise bootstrap consistency at the stochastic scale and quantile regularity conditions, we obtain asymptotically valid confidence intervals.
\end{itemize}

\subsection{Related Work}
In this section, we review related works with a primary focus on bootstrap and diffusion models.

~~\\
\noindent
\textbf{Bootstrap.}
In statistics, the Bootstrap is a fundamental tool for assessing uncertainty and conducting inference in a data-driven manner, and it has been widely adopted across various domains.
Bootstrap methods provide important statistical inference techniques for nonparametric regression \citep{efron1994introduction,shao2012jackknife,hall2013bootstrap}. Next, we review several key Bootstrap approaches in this context.
In   nonparametric regression, the objective is to estimate the conditional mean function $f_0(\mx) = \mathbb{E}[\mY|\mX = \mx]$. Estimators such as kernel,
splines, wavelets, 
local polynomial regressions,  nearest-neighbor regression models  \citep{gyorfi2002distribution,wasserman2006all,biau2015lectures}
and DNNs \citep{bauer2019deep, schmidt2020nonparametric, kohler2021rate, farrell2021deep, jiao2023deep, bhattacharya2024deep} offer  flexibility for modeling nonlinear relationships.
Besides the nonparametric estimator, to address the challenges of statistical inference, a variety of Bootstrap methods have been developed for nonparametric regression, resulting in the emergence of diverse resampling techniques. 
One of the earliest approaches is the  residual Bootstrap \citep{freedman1981bootstrapping}, which involves computing residuals from a fitted nonparametric estimator $\widehat{f}(\mx)$, resampling these residuals ($\hat{\epsilon}_i:=\mY_i-\widehat{f}(\mX_i)$), and adding them back to the fitted values: 
$$
\mY_i^* = \widehat{f}(\mX_i) + \epsilon_i^*, \quad \epsilon_i^* \sim \{\hat{\epsilon}_1, \dots, \hat{\epsilon}_n\}.
$$
In the wild Bootstrap \citep{hardle1993comparing,hardle1991bootstrap}, it modifies the residuals with random weights to preserve conditional variance: 
$$
\mY_i^* = \widehat{f}(\mX_i) + \hat{\epsilon}_i \cdot \xi_i,
$$
where $\xi_i$ are i.i.d. random variables with mean zero and unit variance (e.g., Rademacher or standard normal). The wild Bootstrap is more adaptable in the presence of non-constant variance but still assumes a reliable estimate of $\widehat{f}(\mx)$ and its residuals.
The  paired Bootstrap  takes a different route by naive resampling original observations $\{(\mX_i, \mY_i)\}_{i=1}^n$. This procedure implicitly retains the joint distribution of $(\mX, \mY)$ and avoids estimating residuals altogether. However, it may perform poorly when the marginal distribution of $\mX$ is not of primary inferential interest, or when the structure of $P_{\mY|\mX}$ is more nuanced. 
To better approximate the covariate distribution and cover sparse regions in $\mX$, the  smooth Bootstrap \citep{silverman1987bootstrap} was developed. This method augments the paired Bootstrap by perturbing the covariates using a smooth kernel density estimate, generating  $\mX_i^*$ by sampling from the estimated density.
Beyond these   strategies, conditional Bootstrap methods are  proposed for settings where a reliable estimate of the conditional distribution $P_{\mY|\mX}$ is obtainable. In this framework, one first  obtain an estimator $\widehat{P}_{\mY|\mX}$ of   $P_{\mY|\mX}$  using a suitable statistical or machine learning model such as kernel density estimators, then generates new responses $\mY_i^* \sim \widehat{P}_{\mY|\mX = \mX_i}$ \citep{rutherford1991error}. This technique accommodates heteroscedasticity and complex conditional relationships, provided that $\widehat{P}_{\mY|\mX}$ is accurate. 
As data becomes 
higher-dimensional and more complex, traditional bootstrap methods encounter practical limitations, including inefficiencies in sparse settings, inability to capture multimodal behavior, and degraded performance in high dimensions.  To address these issues, we propose a deep bootstrap method by incorporating conditional diffusion models to directly model  $P_{\mY|\mX}$, synthesize new response variables, and construct the nonparametric estimator of the underlying regression function.  Conditional diffusion models flexibly capture complex relationships between covariates and responses, including multimodal and high-dimensional distributions, making them particularly effective in challenging high-dimensional regression settings.

~~\\
\noindent
\textbf{Diffusion Models.}
Diffusion models  \citep{sohl2015deep, song2019generative, ho2020denoising, song2020improved, song2020score, nichol2021improved, yang2023diffusion, chen2024overview}  have recently emerged as a powerful and versatile class of generative models, establishing themselves as a central component of modern generative artificial intelligence. These models have demonstrated state-of-the-art performance across a broad spectrum of applications, including image synthesis, text-to-image generation, and scientific domains \citep{dhariwal2021diffusion, ho2022cascaded, rombach2022high, saharia2022photorealistic, zhang2023adding, han2022card, li2022diffusion}. Their ability to generate high-fidelity and diverse samples with strong theoretical guarantees has positioned diffusion models as a leading framework in the landscape of deep generative modeling. 
Diffusion models are fundamentally grounded in  stochastic differential equations (SDEs). This mathematical formulation enables them to model complex,  high-dimensional data distributions in a  probabilistic framework.
Generally, diffusion models consist of two key stages: the  forward process and the  backward  process. 
The forward process defines a  SDE  that progressively perturbs the input data by injecting Gaussian noise over a continuous time horizon. This transformation gradually maps the   data distribution into a simple and tractable prior, typically a standard  Gaussian.
Conversely, the backward process is characterized by a reverse-time SDE that aims to denoise samples drawn from the prior distribution. By reversing the diffusion trajectory, this process reconstructs samples that approximate the original data distribution.

Formally, the  forward process   is described by an It\^{o} SDE: 
\begin{align}\label{sdef}
    d\mathbf{x}^F_t = f(t, \mathbf{x}^F_t)\,dt + g(t)\,d\mathbf{w}_t, \quad \mathbf{x}^F_0 \sim p_{\text{data}}(\mathbf{x}),
\end{align}
where $\mathbf{x}^F_t \in \mathbb{R}^{d_\mcX}$ denotes the state at time $t \in [0, T]$,   $\mathbf{w}_t$ is a standard $d_\mcX$-dimensional Brown motion,
and $p_{\text{data}}(\mathbf{x})$ is the target distribution. In the context of conditional diffusion models, this target corresponds to a conditional distribution.
The drift term $f: [0, T] \times \mathbb{R}^{d_\mcX} \rightarrow \mathbb{R}^{d_\mcX}$ and diffusion coefficient $g: [0, T] \rightarrow \mathbb{R}$ govern the deterministic and stochastic components of the system, respectively.
This forward process progressively corrupts the original data by adding Gaussian noise, eventually mapping the data distribution $p_{\text{data}}$ to a tractable prior distribution (e.g., a standard Gaussian) as $t \rightarrow \infty$.
As studied in \cite{anderson1982reverse,haussmann1986time}, the corresponding reverse-time process is also governed by an SDE. Specifically, the reverse-time SDE takes the form:
\begin{align}\label{sdeb}
    d\mathbf{x}^R_t = \left[f(t, \mathbf{x}^R_t) - g(t)^2 \nabla_{\mathbf{x}^R_t} \log p_t(\mathbf{x}^R_t)\right]dt + g(t)\,d 
    \mathbf{w}^R_t,
\end{align}
where $\mathbf{w}^R_t$ is a standard Brown motion run backward in time, 
$p_t$ denotes the   marginal distribution   of $\mathbf{x}^R_t$.
and $\nabla_{\mathbf{x}^R_t} \log p_t(\mathbf{x}^R_t)$ refers to the score function of   $p_t$.
Let  $t\rightarrow T-t$,  
we can reformulate the reverse-time SDE \eqref{sdeb} as a 
forward-time equation:
\begin{align}\label{sde:bw}
d\mathbf{y}_t=
\left[g^2(T-t)\nabla \log p_{T-t}(\mathbf{y}_t)-f(T-t,\mathbf{y}_t)\right]dt
+g(T-t)d\mathbf{w}_t,~ t \in [0,T],
\end{align}
where $\mathbf{y}_t=\mathbf{x}_{T-t}$.
We note that the SDEs in  \eqref{sdef}–\eqref{sde:bw} are defined over an infinite time horizon as $T \to \infty$,  which poses practical challenges for computation when applying numerical solvers.
To address this, one can apply  an exponential time reparameterization $t \mapsto \exp(-t)$ \citep{albergo2023stochastic}, which effectively maps the infinite horizon onto a compact interval $[0, 1]$. 
In this work, we adopt 
variance-preserving (VP) conditional diffusion models that are defined directly over the unit interval $t \in [0, 1]$, in contrast to the original formulation in \citep{song2020score}, which operates over an infinite time horizon.
Furthermore, in diffusion models, the  score function  $\nabla_{\mx} \log p_t(\mx)$  is typically approximated using DNNs. This approximation enables practical sampling by numerically solving the corresponding SDE with the learned score network. 
This score estimation is not only critical for generating high-quality samples but also plays a foundational role in the theoretical analysis of diffusion models. Existing theoretical work can be broadly classified into two categories based on how the score estimation error is handled. 
The first category assumes a well-controlled  approximation of the score function and derives guarantees under this assumption \citep{chen2023improved, conforti2023score, lee2022convergence, lee2023convergence, benton2023linear, li2023towards, gao2023wasserstein}.
The second line of research avoids assuming access to accurate scores and instead conducts end-to-end error analysis, directly bounding the discrepancy between the generated and true data distributions without requiring precise control of the score error \citep{oko2023diffusion,chen2023score,jiao2024latent,fu2024unveil,jiao2025model,tang2025conditional,azangulov2024convergence}. 
Among these works, \cite{chen2023score, jiao2025model} assume that the score function is Lipschitz continuous and design 
Lipschitz-continuous neural networks to approximate it. However, this technique leads to suboptimal convergence rates.
To address this limitation, \cite{oko2023diffusion, jiao2024latent, fu2024unveil} consider a more general setting where the score function is H\"older continuous. They develop a more refined approximation technique by separately approximating the numerator and denominator of the score expression, ultimately deriving sharp convergence rates.
\citet{tang2025conditional} establish sharp rates for conditional density estimation on cubes under H\"older smoothness and a positive density lower bound, without additional density smoothness near the support boundary. Their approximation proof uses local relative density error comparisons.
In the unconditional setting, \citet{azangulov2024convergence} derive a score error bound of relative-Fisher form, controlled by Wasserstein distance and integrated over time, and use it in neural-network approximation for manifold-supported distributions.

\subsection{Preliminary}\label{sec:prel}

\noindent
\textbf{Notations.}
We introduce the notations used throughout this paper.
Let $[\mcM]:=\{1,\cdots,\mcM\}$ represent the set of integers ranging from $1$ to $\mcM$. 
Let $\mathbb{N}^+$ denote the set of positive integers. $d_\mcX, d_\mcY$ denote the dimension of $\mx,\my$ or $\mX,\mY$.
For matrices $A, B \in \mathbb{R}^{d\times d}$, we assert $A\preccurlyeq B$ when the matrix $B - A$ is positive semi-definite. The identity matrix in $\mathbb{R}^{d\times d}$ is denoted as $\mathbf{I}_d$.
The $\ell^2$-norm of a vector $\mathbf{x}=\{x_1,\ldots,x_d\}^{\top}\in\mathbb{R}^d$ is defined by $\Vert\mathbf{x}\Vert:=\sqrt{\sum_{i=1}^{d}x_i^2}$. 
The $L^{\infty}(K)$-norm, denoted as  
$\Vert{f}\Vert_{L^{\infty}(K)}:=\sup_{\mathbf{x}\in K}|f(\mathbf{x})|$,
captures the supremum of the absolute values of a function over a  set $K \subset \mathbb{R}^d$. For a vector function $\mathbf{v}:\mathbb{R}^{d}\rightarrow\mathbb{R}^{d}$, the $L^{\infty}(K)$-norm is defined as $\Vert \mathbf{v} \Vert_{L^{\infty}(K)} := \sup_{\mathbf{x}\in K}\Vert \mathbf{v}(\mathbf{x})\Vert$. The asymptotic notation $f(\mathbf{x}) = \mathcal{O}\left(g(\mathbf{x})\right)$ is employed to signify that $f(\mathbf{x})\leq Cg(\mathbf{x})$ for some constant $C > 0$. Additionally, the notation $\wt{\mathcal{O}}(\cdot)$ is utilized to discount logarithmic factors in the asymptotic analysis.

\begin{definition}[ReLU DNNs]\label{def: relufnns}
A class of deep neural networks NN$(L,M,J,\kappa)$ with depth $L$, width $M$,  
sparsity level $J$,  
weight bound $\kappa$,
  is defined as
    \begin{equation*}
\begin{aligned}
            &~~~~{\rm{NN}}(L,M,J,\kappa) \\
            &= \Big\{\mathbf{b}(t,\my,\mx):= (\mathbf{W}_L{\rm{ReLU}}(\cdot) + \mathbf{b}_L)\circ\cdots\circ(\mathbf{W}_2{\rm{ReLU}}(\cdot) + \mathbf{b}_2)\circ(\mathbf{W}_1(\cdot) + \mathbf{b}_1)(
            [t, \my^{\top},\mx^{\top}]^{\top}):\\
&
 ~~~~~~~~~~~~~~\mathbf{W}_{i+1} \in \mathbb{R}^{d_{i+1}\times d_i},  \mathbf{b}_{i+1} \in \mathbb{R}^{d_{i+1}}, i=0,1,\ldots,L-1,
M:=\max\{d_0,\ldots,d_L\},\\
&~~~~~~~~~~~~~~\mathop{\max}_{1\leq{i}\leq{L}}\{\Vert{\mathbf{b}}_i\Vert_{\infty}, \Vert{\mathbf{W}_i}\Vert_{\infty}\}\leq{\kappa},
~ \sum_{i=1}^L\left(\Vert{\mathbf{W}_i}\Vert_0 + \Vert{\mathbf{b}_i}\Vert_0\right)\leq J\Big\}.
\end{aligned}
\end{equation*}
\end{definition}

\begin{definition}[Bounded-Lipschitz distance]
For probability measures $\mu$ and $\nu$ on $\mathbb R^d$, define
\[
d_{\mathrm{BL}}(\mu,\nu):=
\sup_{\|h\|_\infty\leq1,\ \mathrm{Lip}(h)\leq1}
\left|\int h\,\mathrm d\mu-\int h\,\mathrm d\nu\right|.
\]
This distance metrizes weak convergence on $\mathbb R^d$.
\end{definition}

\begin{definition}[Wasserstein distance]
Let $\mu$ and $\nu$ be two  probability measures defined on $\mathbb{R}^d$ with finite first moments. The first-order Wasserstein distance is defined as:
\begin{equation*}
    \mcW_1(\mu,\nu) :=
    \inf_{\gamma \in \mathcal D(\mu,\nu)}
    \int_{\mathbb R^d}\int_{\mathbb R^d}\|\mx-\my\|\,
    \gamma(\mathrm{d}\mx,\mathrm{d}\my).
\end{equation*}
where $\mathcal{D}(\mu, \nu)$ denotes the set of probability measures $\gamma$ on $\mathbb{R}^{2d}$ such that their respective marginal distributions are $\mu$ and $\nu$.
\end{definition}

\noindent\textbf{Bounded-support comparison.}
Let $\mu$ and $\nu$ be probability measures supported on a common Borel set $\Omega\subset\mathbb R^d$ with diameter $D:=\sup_{x,y\in\Omega}\|x-y\|<\infty$. Throughout this paper we use the normalization
$\mathrm{TV}(\mu,\nu):=\sup_{A\in\mathcal B(\mathbb R^d)}|\mu(A)-\nu(A)|$. Then
\begin{equation*}
    \mcW_1(\mu,\nu)\le D\,\mathrm{TV}(\mu,\nu).
\end{equation*}
Indeed, a maximal coupling $(U,V)$ of $(\mu,\nu)$ satisfies
$\Pbb(U\ne V)=\mathrm{TV}(\mu,\nu)$. Hence
$\Ebb\|U-V\|\le D\Pbb(U\ne V)=D\,\mathrm{TV}(\mu,\nu)$, and taking the
infimum over couplings proves the displayed comparison.

\begin{definition}[Covering number]
    Let $\rho$ be a pseudo-metric on $\mathcal{U}$ and $S\subseteq\mathcal{U}$. For any $\delta > 0$, a set $A\subseteq\mathcal{U}$ is called a $\delta$-covering  of $S$ if for any $\mathbf{x}\in S$ there exists $\mathbf{y}\in A$ such that $\rho(\mathbf{x},\mathbf{y})\leq\delta$. The $\delta$-covering number of $S$, denoted by $\mathcal{N}(\delta,S,\rho)$, is the minimum cardinality of any $\delta$-covering of $S$.
\end{definition}

\begin{definition}[($\beta$, $R$)-H{\"o}lder Class] Let $\beta = r + s > 0$ be a degree of smoothness, where 
$r$ is an integer and $s\in (0, 1]$. For a function $f: \Omega\rightarrow \mathbb{R}$, its H{\"o}lder norm is defined as
$$
\Vert f \Vert_{\mathcal{H}^{\beta}} := \max_{\Vert\boldsymbol{\alpha}\Vert_1\leq r}\Vert \partial^{\boldsymbol{\alpha}}f \Vert_{\infty} + \max_{\Vert\boldsymbol{\alpha}\Vert_1 = r}\sup_{\mathbf{x}\neq\mathbf{y}}\frac{|\partial^{\boldsymbol{\alpha}} f(\mathbf{x}) - \partial^{\boldsymbol{\alpha}}f(\mathbf{y})|}{\Vert \mathbf{x} - \mathbf{y} \Vert^s},
$$
where $\boldsymbol{\alpha}$ is a multi-index. We say a function $f$ is $\beta$-H{\"o}lder, if and only if $\Vert f \Vert_{\mathcal{H}^{\beta}} < \infty$. The ($\beta$, $R$)-H{\"o}lder class for some constant $R > 0$ is defined as
$$
\mathcal{H}^{\beta}(\Omega, R) := \Big\{f: \Omega\rightarrow\mathbb{R} \Big| \Vert f \Vert_{\mathcal{H}^{\beta}} \leq R
 \Big\}.
$$
\end{definition}


\subsection{Outlines}
The remainder of this paper is organized as follows.
In Section \ref{sec: method}, we present our proposed deep bootstrap method. 
The   theoretical analysis of our proposed method is presented in Section \ref{sec:th}.
In Section 
\ref{sec:ne}, we give the numerical experiments.
We conclude in Section \ref{sec:con}. 
In the Appendix, we provide detailed proofs for all lemmas and theorems.

\section{Method}\label{sec: method}
In this section, we introduce our proposed deep bootstrap method for nonparametric regression. The first step, presented in Section \ref{sec:cdm}, involves constructing a conditional diffusion model.
Using this model, we can estimate the target conditional distribution, and further perform nonparametric regression by computing the sample mean of the generated data. Subsequently, we generate bootstrap samples through the conditional diffusion model to formulate a bootstrap procedure for statistical inference, as detailed in Section \ref{sec:bcdm}.

\subsection{Conditional Diffusion Model}\label{sec:cdm}
In this  section, we formulate the VP conditional diffusion model over the unit time interval 
$t \in [0,1]$ to learn the   conditional distribution.
It is important to emphasize that the model is developed for a general setting in which the variable  $\my \in \mathbb{R}^{d_\mcY}$ is a 
$d_\mcY$-dimensional vector.

\noindent \textbf{Diffusion Process.} The forward diffusion process is defined as:
\begin{equation} \label{eq: forward_sde}
\mrd \my_t^F = -\frac{\my^F_t}{1-t} \mrd t + \sqrt{\frac{2}{1-t}}\mrd\mw_t, ~~~ \my_0^F\sim p_0(\my|\mx).
\end{equation}
It determines a diffusion process that starts with conditional distribution $p_0(\my|\mx)$ at time $t = 0$ and evolves toward the standard Gaussian distribution $p_1(\my|\mx) = \mcN(\mathbf{0},\mI_{d_\mcY})$ at time $t = 1$. Conditioning on some observation $\mx\in\mcX$, the transition probability distribution from $\my_0^F$ to $\my_t^F$ is given by $\my_t^F|\my_0^F \sim \mcN(m_t\my_0^F, \sigma_t^2\mI_{d_\mcY})$, where $m_t:=1-t$, $\sigma_t:=\sqrt{t(2-t)}$. The process describes how data are transformed into noises during unit-time
interval.

The SDE above can be reversed if we know the score of the
distribution at each intermediate time step, $\nabla\log p_t(\my|\mx)$. The reverse SDE reads
\begin{equation} \label{eq: reversed_sde}
\mrd \my_t^R = \left[-\frac{\my_t^R}{1-t} - \frac{2}{1-t}\nabla\log p_t(\my_t^R|\mx)\right] \mrd t + \sqrt{\frac{2}{1-t}} \mrd \mw_t^R,
\end{equation}
where $t$ flows from 1 to 0 and $\mw_t^R$ is a time reverse Brownian motion. 
For convenience, we take transformation $t\rightarrow 1 - t$, and then we can rewrite \eqref{eq: reversed_sde} as a time forward version:
\begin{equation}\label{eq: forward_sde2}
    \mrd \my_t = \left[\frac{\my_t}{t} + \frac{2}{t}\nabla\log p_{1-t}(\my_t|\mx)\right] \mrd t + \sqrt{\frac{2}{t}} \mrd \mw_t, ~~ \my_0\sim\mcN(0,\mI_{d_\mcY}), ~~ \my_1 \sim p_0(\my|\mx).
\end{equation}
It is obvious that $\my_t\sim p_{1-t}(\my|\mx)$.

~~\\
\noindent\textbf{Score Matching.}
We apply  score matching techniques \citep{hyvarinen2005estimation,vincent2011connection} to  estimate the target conditional score function of SDE \eqref{eq: forward_sde2}, defined as 
$\mb^*(t,\my,\mx):=\nabla\log{p}_{t}(\my|\mx)$.
It can be verified that  $\mathbf{b}^*$ minimizes the loss function $\mcL(\mb)$ over all measurable functions, where 
\begin{equation*}
        \mcL(\mb) := \frac{1}{1-2T}\int_{T}^{1-T}\frac{1}{1-t}\cdot\Ebb_{(\mx,\my_t^F)}\Vert{\mb(t, \my_t^F,\mx) - \nabla\log{p_t(\my_t^F|\mx)}}\Vert^2 \mrd t,
    \end{equation*}
with $0<T<1$.
We notice that the deliberate selection of  $T$ is motivated by the necessity to preclude the score function from exhibiting a blow-up at  $T=0$, concomitantly facilitating the stabilization of the training process for the model. 
This time truncation strategy is not unique to our work; it has also been used in both the training and theoretical analysis of diffusion models
\citep{song2020improved,vahdat2021score,chen2023score,chen2023improved,oko2023diffusion}.
By  \cite{hyvarinen2005estimation, vincent2011connection}, we alternatively adopt a denoising score matching objective. With a slight abuse of notation, we denote this objective as
\begin{equation*} 
        \begin{aligned}
            \mcL(\mb) &= \frac{1}{1-2T}\int_{T}^{1-T}\frac{1}{1-t}\left(\Ebb_{(\mx,\my_0^F)}\Ebb_{\my_t^F\mid(\my_0^F,\mx)}\left\Vert\mb(t, \my_t^F,\mx) + \frac{\my_t^F - m_t\my_0^F}{\sigma_t^2}\right\Vert^2\right)\mrd t\\
            &= \frac{1}{1-2T}\int_{T}^{1-T}\frac{1}{1-t}\left(\Ebb_{(\mx,\my_0^F)}\Ebb_{\mcN(\mz;\mathbf{0},\mathbf{I}_{d_\mcY})}\left\Vert\mb(t, m_t\my_0^F + \sigma_t\mz,\mx) + \frac{\mz}{\sigma_t}\right\Vert^2\right)\mrd t.
        \end{aligned}
    \end{equation*}
Given $n$ i.i.d.  samples $\{(\mx_i,\my_{0,i}^F)\}_{i=1}^{n}$ from $p_\mX(\mx)p_0(\my|\mx)$, and $m$ i.i.d. samples $\{(t_j, \mz_j)\}_{j=1}^{m}$  from $\mathrm{U}[T,1-T]$ and $\mcN(\mz;\mathbf{0},\mathbf{I}_{d_\mcY})$, we can employ the empirical risk minimizer (ERM) to estimate the conditional score function $\mb^*$. This ERM, denoted by $\wh{\mb}$, is obtained by
\begin{align}\label{eq: escore}
\wh{\mb}\in{\mathop{\rm{argmin}}_{\mb\in\mathrm{NN}}}
~\wh{\mcL}(\mb):= \frac{1}{mn}\sum_{j=1}^{m}\sum_{i=1}^{n}\frac{1}{1-t_j}\left\Vert\mb(t_j, m_{t_j}\my_{0,i}^F + \sigma_{t_j}\mz_{j},\mx_i) + \frac{\mz_j}{\sigma_{t_j}}\right\Vert^2,
\end{align}
where $\text{NN}$ refers to ReLU DNNs defined in 
Definition \ref{def: relufnns}.

~~\\
\noindent\textbf{Frozen-drift Discretization.} Given the estimated score function $\wh{\mb}$, as defined in \eqref{eq: escore}, we can formulate an SDE  initializing from the prior distribution: 
\begin{equation}\label{eq: sampling_SDE}
\mrd\wh{\my}_t = \left[\frac{\wh{\my}_t}{t} + \frac{2\wh{\mb}(1-t,\wh{\my}_t,\mx)}{t} \right] \mrd t + \sqrt{\frac{2}{t}}\mrd\mw_t, ~ \wh{\my}_0\sim\mcN(0,\mathbf{I}_{d_\mcY}). 
\end{equation}
Now, we employ a discrete-time approximation for the sampling dynamics \eqref{eq: sampling_SDE}. Let $K\in\mathbb N^+$ and use the uniform grid
$$
t_i:=T+i\Delta,\qquad \Delta:=\frac{1-2T}{K},\qquad i=0,1,\ldots,K.
$$
Thus $T=t_0<t_1<\cdots<t_K=1-T$. We consider the following frozen-drift scheme:
\begin{equation}\label{eq: EI_scheme}
\mrd \wt{\my}_t = \left[\frac{\wt{\my}_{t_i}}{t_i} + \frac{2\wh{\mb}(1-t_i,\wt{\my}_{t_i},\mx)}{t_i}\right] \mrd t + \sqrt{\frac{2}{t}}\mrd\mw_t,~ \wt{\my}_{t_0} \sim\mcN(\mathbf{0},\mathbf{I}_{d_\mcY}), ~ t\in[t_i, t_{i+1}),
\end{equation}
for $i=0,1,\cdots,K-1$. Subsequently, we can utilize dynamics \eqref{eq: EI_scheme} to generate new samples.
The proposed conditional diffusion model is summarized in the following algorithm.
\begin{algorithm}[H]
\caption{
The Proposed Conditional Diffusion Model
}
\label{sampling_algorithm}
\begin{algorithmic}
\STATE 
1. {\bf Input:}
$T$, $K$, $
\{(\mx_i,\my_{0,i}^F)\}_{i=1}^{n} \sim p(\mx)p_0(\my|\mx)$,  
$\{(t_j, \mz_j)\}_{j=1}^{m} \sim \mathrm{U}[T,1-T]$ and $\mcN(\mz;\mathbf{0},\mathbf{I}_{d_\mcY})$
\STATE 
2. {\bf Score estimation:} 
Obtain  $\widehat{\mathbf{b}}$ by  \eqref{eq: escore}.
\STATE 
3. {\bf Sampling procedure:} 
\STATE 
Sample 
$\wt{\mathbf{y}}_{t_0} \sim {\mathcal{N}}(\mathbf{0},\mathbf{I}_{d_\mcY})$;
        \FORALL{$i=0,1,\ldots,K-1$}{
            \STATE Sample $\mathbf{\epsilon}_{t_i}\sim{\mathcal{N}}(\mathbf{0},\mathbf{I}_{d_\mcY})$;
            \STATE 
            $\mathbf{b}(\wt{\mathbf{y}}_{t_i}) =  \frac{\wt{\my}_{t_i}}{t_i} + \frac{2\wh{\mb}(1-t_i,\wt{\my}_{t_i},\mx)}{t_i} $;
            \STATE $\wt{\mathbf{y}}_{t_{i+1}} = \wt{\mathbf{y}}_{t_i} + \frac{1-2T}{K}\mathbf{b}(\wt{\mathbf{y}}_{t_i}) + \sqrt{2\ln (t_{i+1}/t_i)}\mathbf{\epsilon}_{t_i}$;
        }
        \ENDFOR
\STATE
4. {\bf Output:}
$  \wt{\mathbf{y}}_{t_K} $.        
    \end{algorithmic}
\end{algorithm}

\subsection{Bootstrap via Conditional Diffusion Model}\label{sec:bcdm}
In this  section, we  utilize the conditional diffusion model proposed in the previous  section to construct our deep bootstrap method. Throughout this development, we focus exclusively on the univariate case where $d_\mcY=1$. However, for notational consistency and to facilitate potential generalization, we may retain the use of $d_\mcY$ in contexts where no ambiguity arises. Given an i.i.d. dataset $\mathcal{D} = \{(\mX_i, \mY_i)\}_{i=1}^n$, we formulate a conditional diffusion model as described in Section \ref{sec:cdm} to learn the target conditional distribution $P_{\mY|\mX}$, and denote the resulting distribution as $\widehat{P}_{\mY|\mX}$. Using this model, we generate a synthetic dataset 
$\{\widehat{\mY}_{\mx}^{(j)}\}_{j=1}^n \sim \widehat{P}_{\mY|\mX=\mx}$ by running sampling dynamic 
\eqref{eq: EI_scheme} $n$ times, where the samples are approximately distributed according to $P_{\mY|\mX=\mx}$.
Recall that $\mY_\mx \sim P_{\mY|\mX=\mx}$ and  the underlying regression function is  $f_0(\mx)=\mathbb{E}(\mY_\mx)$.
We estimate the underlying regression function using the sample mean of the generated data:
\begin{align}\label{mean}
\widehat{f}(\mx) := 
\frac{1}{n}\sum_{j=1}^n \widehat{\mY}_{\mx}^{(j)}.
\end{align}
Furthermore, given the original covariates $\{\mX_i\}_{i=1}^n$,  we can construct a  bootstrap dataset $\mcD^*:=\{\mX_i,\wh{Y}_{\mX_i}\}_{i=1}^n$ by generating responses from the conditional diffusion model.
Based on this bootstrap dataset $\mathcal{D}^*$, we re-train a conditional diffusion model to learn an updated distribution $\wh{P}^*_{\mY|\mX}$, which is approximated to
$\widehat{P}_{\mY|\mX}$. Using $\wh{P}^*_{\mY|\mX}$, we generate dataset
$ 
\{\wh{\mY}^{*,(j)}_{\mx}\}_{j=1}^n \sim  \wh{P}^*_{\mY|\mX=\mx}
$, and define  the corresponding bootstrap estimator as
$
\wh{f}^*(\mx):= \frac{1}{n}\sum_{j=1}^n\wh{\mY}^{*,(j)}_{\mx}
$.
The procedure can be summarized as follows:
\begin{enumerate}[(I)]
\item \label{step1}
 Obtain $\wh{f}(\mx)$:\\ $\mathcal{D}=\{\mX_i,\mY_i\}_{i=1}^n\rightarrow 
\mbox{
Conditional diffusion model (Algorithm \ref{sampling_algorithm})
}
\rightarrow
\{\wh{\mY}_{\mx}^{(j)}\}_{j=1}^n \sim  \wh{P}_{\mY|\mX=\mx}
\rightarrow
\wh{f}(\mx):=
\frac{1}{n}\sum_{j=1}^n\wh{\mY}_{\mx}^{(j)};
$ 
\item \label{step2}
Obtain $\wh{f}^*(\mx)$:\\
$\mcD^*:=\{\mX_i,\wh{\mY}_{\mX_i}\}_{i=1}^n \rightarrow
\mbox{
Conditional diffusion model (Algorithm \ref{sampling_algorithm})
}
\rightarrow
\{\wh{\mY}^{*,(j)}_{\mx}\}_{j=1}^n \sim  \wh{P}^*_{\mY|\mX=\mx}
\rightarrow
\wh{f}^*(\mx):=
\frac{1}{n}\sum_{j=1}^n\wh{\mY}^{*,(j)}_{\mx}
$.
\end{enumerate} 
Consequently, by repeating the above Step \eqref{step2} $\mathcal{B}$  ($\mcB>1$)  times, we can construct a confidence interval for the underlying regression function $f_0(\mx)$.
Let $\widehat{f}_b^*(\mx)$ denote the $b$-th bootstrap estimator of $\widehat{f}(\mx)$, for $b \in \{1, \ldots, \mathcal{B}\}$. We define   the centered statistic as
$
\widehat{R}_b^*(\mx) := \widehat{f}_b^*(\mx) - \widehat{f}(\mx),
$
and let $\widehat{H}^*(r)$, $r \in \mathbb{R}$, denote its empirical cumulative distribution function (CDF):
$$
\widehat{H}^*(r) := \frac{1}{\mathcal{B}} \sum_{b=1}^{\mathcal{B}} \mI\left( \widehat{R}_b^*(\mx) \leq r \right),
$$
where $\mI(\cdot)$ denotes the indicator function. 
{For the theoretical analysis below, we distinguish this empirical distribution function from the conditional bootstrap distribution function
\[
H^*(r):=\Pbb^*\left(\widehat{R}^*(\mx)\leq r \mid \mathscr F_{n,\mx}\right),\qquad r\in\Rbb,
\]
where $\mathscr F_{n,\mx}:=\sigma(\mcD,\mcT,\mcZ,\mx,\wh f(\mx))$, so that the original fitted model and the realized original estimate are held fixed.
The empirical CDF $\widehat H^*$ is the finite-$\mcB$ approximation of $H^*$.} Then, the $(1 - \alpha)$ ($\alpha \in (0,1)$) confidence interval for $f_0(\mx)$ is given by
\begin{align}\label{interval}
\left[ \widehat{f}(\mx) -  (\widehat{H}^* )^{-1}(1 - \alpha/2), \; \widehat{f}(\mx) -  (\widehat{H}^* )^{-1}(\alpha/2) \right],
\end{align}
where $ (\widehat{H}^* )^{-1}(\cdot)$ denotes the quantile function (i.e., the inverse CDF) of $\widehat{H}^*(\cdot)$.

In summary, our proposed deep bootstrap method is outlined in the following algorithm.
\begin{algorithm}[H]
\caption{
Bootstrap  via Conditional Diffusion Models
}
\label{boot_algorithm}
\begin{algorithmic}
\STATE 
1. {\bf Input:}
$\mathcal{B}$, $T$, $K$, $
\{(\mx_i,\my_{i})\}_{i=1}^{n} \sim 
P_\mX \times P_{\mY|\mX}$,  
$\{(t_j, \mz_j)\}_{j=1}^{m} \sim \mathrm{U}[T,1-T]$ and $\mcN(\mz;\mathbf{0},\mathbf{I}_{d_\mcY})$
\STATE
2. 
Obtain the nonparametric estimator $\wh{f}(\mx)$ by   \eqref{step1}.
\STATE 
3. 
Obtain Bootstrap estimators 
$\{\wh{f}_b^*(\mx)\}_{b=1}^{\mathcal{B}}$:
\FOR{$b=1,\ldots,\mathcal{B}$}{
            \STATE  Compute $\wh{f}_b^*(\mx)$ by using Step \eqref{step2}
        }
        \ENDFOR
\STATE
 
4. {\bf Output:} The asymptotic $1-\alpha$ confidence interval \eqref{interval}.       
    \end{algorithmic}
\end{algorithm}

\section{Theory}\label{sec:th}
In this section, we establish the theoretical foundations of the proposed method. Section \ref{sec:thcdm} analyzes the conditional diffusion model presented in Algorithm  \ref{sampling_algorithm} and derives an end-to-end convergence rate in total variation distance. Building on this result, Section \ref{sec:thboot} analyzes the deep bootstrap procedure described in Algorithm \ref{boot_algorithm} using bounded-Lipschitz distance for the laws of the resulting statistics.

\subsection{Convergence of Conditional Diffusion Model}\label{sec:thcdm}
In this section, we obtain the convergence rate  for the  conditional diffusion model. We denote $ p_{1-t}(\cdot|\mx)$, $\wh{p}_{1-t}(\cdot|\mx)$,  and $\wt{p}_{1-t}(\cdot|\mx)$ as the distributions of $\my_t$, $\wh{\my}_t$, and $\wt{\my}_t$ correspondingly.
 
Let $\mcD:=\{(\mx_i,\my_{0,i}^F)\}_{i=1}^n$,  $\mcT:=\{t_j\}_{j=1}^{m}$, and $\mcZ := \{\mz_j\}_{j=1}^{m}$. We define
\begin{align*}
\ov{\mcL}_{\mcD}(\mb):=\frac{1}{n}\sum_{i=1}^{n}\ell_{\mb}(\mx_i,\my_{0,i}^F), 
~
\text{and}~
\wh{\mcL}_{\mcD,\mcT, \mcZ}(\mb):= \frac{1}{n}\sum_{i=1}^{n}\wh{\ell}_{\mb}(\mx_i,\my_{0,i}^F),
\end{align*}
where
$$
\ell_{\mb}(\mx, \my_0^F) :=\frac{1}{1-2T}\int_{T}^{1-T}\frac{1}{1-t}\cdot\Ebb_{\mz}\left\Vert\mb(t, m_t\my_0^F + \sigma_t\mz,\mx) + \frac{\mz}{\sigma_t}\right\Vert^2 \mrd t,
$$
and
$$
\wh{\ell}_{\mb}(\mx,\my_0^F):= \frac{1}{m}\sum_{j=1}^{m}\frac{1}{1-t_j}\left\Vert \mb(t_j, m_{t_j}\my_0^F + \sigma_{t_j}\mz_j, \mx) + \frac{\mz_j}{\sigma_{t_j}}\right\Vert^2.
$$
Then,  for any $\mb\in\mathrm{NN}$,  it yields that
\begin{equation*}
\begin{aligned}
&~~~~\mcL(\wh{\mb}) - \mcL(\mb^*) \\
&= \mcL(\wh{\mb}) - 2\ov{\mcL}_{\mcD}(\wh{\mb}) + \mcL(\mb^*) + 2\left(\ov{\mcL}_{\mcD}(\wh{\mb}) - \mcL(\mb^*)\right)\\
&=\mcL(\wh{\mb}) - 2\ov{\mcL}_{\mcD}(\wh{\mb}) + \mcL(\mb^*) + 2\left(\ov{\mcL}_{\mcD}(\wh{\mb}) - \wh{\mcL}_{\mcD,\mcT,\mcZ}(\wh{\mb})\right) + 2\left(\wh{\mcL}_{\mcD,\mcT,\mcZ}(\wh{\mb}) - \mcL(\mb^*)\right)\\
&\leq \mcL(\wh{\mb}) - 2\ov{\mcL}_{\mcD}(\wh{\mb}) + \mcL(\mb^*) + 2\left(\ov{\mcL}_{\mcD}(\wh{\mb}) - \wh{\mcL}_{\mcD,\mcT,\mcZ}(\wh{\mb})\right) + 2\left(\wh{\mcL}_{\mcD,\mcT,\mcZ}(\mb) - \mcL(\mb^*)\right).
\end{aligned}
\end{equation*}
Taking expectations, followed by taking the  infimum over $\mb \in \mathrm{NN}$ on both sides of the above inequality,  it holds that
\begin{equation*}
\begin{aligned}
&~~~~\Ebb_{\mcD,\mcT,\mcZ}\left(\frac{1}{1-2T}\int_{T}^{1-T}\frac{1}{1-t}\cdot\Ebb_{(\mx,\my_t^F)}\left\Vert{\wh{\mb}(t, \my_t^F,\mx) - \nabla\log{p_t(\my_t^F|\mx)}}\right\Vert^2 \mrd t\right)\\
&=\Ebb_{\mcD,\mcT,\mcZ}\mcL(\wh{\mb}) - \mcL(\mb^*)\\
& \leq \Ebb_{\mcD,\mcT,\mcZ}\left(\mcL(\wh{\mb}) - 2\ov{\mcL}_{\mcD}(\wh{\mb}) + \mcL(\mb^*)\right) + 2\Ebb_{\mcD,\mcT,\mcZ}\left(\ov{\mcL}_{\mcD}(\wh{\mb}) - \wh{\mcL}_{\mcD,\mcT,\mcZ}(\wh{\mb})\right)\\
&~~~~+ 2\mathop{\mathrm{inf}}_{\mb\in{\mathrm{NN}}}(\mcL(\mb) - \mcL(\mb^*)).
\end{aligned}
\end{equation*}
In the above inequality, 
the terms $$
\Ebb_{\mcD,\mcT,\mcZ}\left(\mcL(\wh{\mb}) - 2\ov{\mcL}_{\mcD}(\wh{\mb}) + \mcL(\mb^*)\right) + 2\Ebb_{\mcD,\mcT,\mcZ}\left(\ov{\mcL}_{\mcD}(\wh{\mb}) - \wh{\mcL}_{\mcD,\mcT,\mcZ}(\wh{\mb})\right)
$$
and 
$$
\mathop{\mathrm{inf}}_{\mb\in{\mathrm{NN}}}(\mcL(\mb) - \mcL(\mb^*))
$$  
denote the statistical error and approximation error, respectively. 

Using tools from empirical process theory and deep approximation theory, we derive upper bounds for the associated errors. To proceed, we first introduce the following necessary assumptions.

\begin{assumption}[Bounded Conditional Density]\label{ass: bounded_density}
Assume $\Pbb\left(\mX\in[-1,1]^{d_\mcX}\right)=1$. For every $\mx\in[-1,1]^{d_{\mcX}}$, the conditional density $p_0(\cdot\mid\mx)$ is supported on $[-1,1]^{d_{\mcY}}$, and there exist constants $0<C_l\leq C_u<\infty$ such that
\[
C_l\leq p_0(\my|\mx)\leq C_u,
\qquad
(\my,\mx)\in[-1,1]^{d_{\mcY}}\times[-1,1]^{d_{\mcX}}.
\]
\end{assumption}

\begin{assumption}[H{\"o}lder Continuity]\label{ass: Holder}
The conditional density $p_0(\my|\mx) \in \mcH^{\beta}([-1, 1]^{d_{\mcY}} \times [-1, 1]^{d_{\mcX}}, R)$ for a H{\"o}lder index $\beta > 0$ and a constant $R>2^{-d_\mcY}$, which is imposed to yield a nontrivial H\"older class.
\end{assumption}

\begin{remark}
These foundational assumptions are critical to advancing the theoretical framework of diffusion models. In particular, the bounded support assumption has been extensively documented in prior studies, as exemplified by the works of \cite{lee2023convergence, li2023towards, oko2023diffusion}. Similarly, within the field of nonparametric regression, imposing boundedness assumptions on the response variable is a common practice in the literature \cite{bauer2019deep, gyorfi2002distribution, kohler2021rate, farrell2021deep}. This assumption carries substantial technical significance, and it holds potential for further generalization to unbounded scenarios through the inclusion of exponential tail properties. Assumption \ref{ass: Holder} enforces the H{\"o}lder continuity of the conditional density, a fundamental condition in nonparametric estimation that mirrors Assumption 3.1 in \cite{fu2024unveil}.
\end{remark}

Under Assumptions \ref{ass: bounded_density}--\ref{ass: Holder}, we provide the following lemma, which bounds the approximation error. 

\begin{lemma}[Approximation Error]\label{lem: approximation_error}
Suppose that Assumptions \ref{ass: bounded_density}--\ref{ass: Holder} hold. Let $\mcM \gg 1, C_T>0$, and $T = \mcM^{-C_T}$. Then we can choose a \rm{ReLU} neural network $\mb\in\mathrm{NN}(L,M,J,\kappa)$ that satisfies
$$
\mb(t,\my,\mx) = \mb(t, \mb_{\mathrm{clip}}(\my, -C_0, C_0), \mx),
$$
for a constant $C_0 = \mathcal{O}(\sqrt{\log\mcM})$, and 
$$
\|\mb(t,\cdot,\mx) \|_{\infty} \lesssim \frac{\sqrt{\log \mcM}}{\sigma_t},
$$
and has the following structure:
$$
L = \mathcal{O}(\log^4 \mcM), M = \mathcal{O}(\mcM^{d_\mcX + d_\mcY}\log^7 \mcM), J = \mathcal{O}(\mcM^{d_\mcX + d_\mcY}\log^9 \mcM), \kappa = \exp\left(\mathcal{O}(\log^4 \mcM)\right).
$$
Moreover, for any $t\in[\mcM^{-C_T}, 1 - \mcM^{-C_T}]$ and $\mx\in[-1,1]^{d_{\mcX}}$, it holds that
$$
\int_{\Rbb^{d_\mcY}} \|\mb(t,\my,\mx) - \mb^*(t,\my,\mx)\|^2 p_t(\my|\mx)\mrd\my \lesssim \frac{\mcM^{-2\beta}\log \mcM}{\sigma_t^2},
$$
where $\sigma_t = \sqrt{t(2-t)}$. 
\end{lemma}

Now, we restrict  the ReLU neural network class $\mathrm{NN}(L,M,J,\kappa)$ to 
$$
\begin{aligned}
\mathcal{C}:= \Big\{\mb\in &\mathrm{NN}(L,M,J,\kappa) \Big| \Vert\mb(t,\cdot,\mx)\Vert_\infty \lesssim \frac{\sqrt{\log \mcM}}{\sigma_t}, \\
& \mb(t,\my,\mx) = \mb(t, \mb_{\mathrm{clip}}(\my, -C_0, C_0),\mx) ~\text{ for }~ \Vert\my\Vert_\infty > C_0, ~\text{where}~ C_0 = \mathcal{O}(\sqrt{\log \mcM}), \\
& L = \mathcal{O}(\log^4 \mcM), M = \mathcal{O}(\mcM^{d_\mcX + d_\mcY}\log^7 \mcM), J = \mathcal{O}(\mcM^{d_\mcX + d_\mcY}\log^9 \mcM), \kappa = \exp\left(\mathcal{O}(\log^4 \mcM)\right)
\Big\}.
\end{aligned}
$$

\begin{remark}
The proof of Lemma \ref{lem: approximation_error} begins with the joint local-polynomial approximation $p_{\mcM}$ of $p_0$. It then introduces the normalized conditional density
\[
    q_0(\my,\mx)
    = p_{\mcM}(\my,\mx)+2^{-d_{\mcY}}\left\{1-\int_{[-1,1]^{d_{\mcY}}}p_{\mcM}(\mathbf u,\mx)\mrd\mathbf u\right\},
\]
for which
\[
    \int_{[-1,1]^{d_\mcY}} q_0(\my,\mx)\mrd\my=1,\qquad
    \|q_0-p_0\|_\infty\lesssim\mcM^{-\beta},\qquad
    \frac{C_l}{2}\leq q_0\leq2C_u.
\]
Let $q_t(\my,\mx)$ be the Gaussian forward smoothing of $q_0(\my,\mx)$. Seeing $\nabla\log q_t(\my,\mx)$ as a proxy for $\nabla\log p_t(\my\mid\mx)$, we first approximate $\nabla\log q_t(\my,\mx)$. To elaborate, we write
\[
    \nabla\log q_t(\my,\mx) = \frac{\nabla q_t(\my|\mx)}{q_t(\my|\mx)}.
\]
To obtain the approximation, we treat the numerator and denominator separately. Since the approximation techniques for both terms are analogous, we focus on approximating $q_t(\my,\mx)$.   

The first step involves truncating the integral domain in $q_t(\my,\mx)$. We denote $m_t:=1-t$, $\sigma_t:= \sqrt{t(2-t)}$. Recall that
$$
q_t(\my,\mx) := \int_{\Rbb^{d_\mcY}}q_0(\my_1,\mx)\mathbf{I}_{\{\Vert\my_1\Vert_{\infty}\leq 1\}}\cdot
\underbrace{\frac{1}{\sigma_t^{d_\mcY}(2\pi)^{d_\mcY/2}}\exp\left(-\frac{\Vert\my - m_t\my_1\Vert^2}{2\sigma_t^2}\right)}_{\text{Transition Kernel}} \mrd\my_1.
$$
Since $q_0(\my,\mx)$ is supported on $[-1,1]^{d_\mcY}$ and is lower and upper bounded, Lemma~\ref{lem: integral_clipping} also applies. Thus there exists a constant $C>0$ such that for any $\my\in\Rbb^{d_\mcY}$,
$$
\left|q_t(\my,\mx) - \int_{A_{\my}}q_0(\my_1,\mx)\mathbf{I}_{\{\Vert\my_1\Vert_{\infty}\leq 1\}}\cdot
\frac{1}{\sigma_t^{d_\mcY}(2\pi)^{d_\mcY/2}}\exp\left(-\frac{\Vert\my - m_t\my_1\Vert^2}{2\sigma_t^2}\right) \mrd\my_1\right| \lesssim \epsilon,
$$
where $A_{\my} = \prod_{i=1}^{d_\mcY}a_{i,\my}\cap[-1,1]^{d_\mcY}$ with $a_{i,\my} = \left[\frac{y_i - C\sigma_t\sqrt{\log\epsilon^{-1}}}{m_t}, \frac{y_i + C\sigma_t\sqrt{\log\epsilon^{-1}}}{m_t}\right]$. 

Next, we approximate the exponential transition kernel
$\frac{1}{\sigma_t^{d_\mcY}(2\pi)^{d_\mcY/2}}\exp\left(-\frac{\Vert\my - m_t\my_1\Vert^2}{2\sigma_t^2}\right) 
$ using a Taylor polynomial denoted as $f_{\rm{Taylor}}^{\rm{kernel}}(t,\my, \my_1)$. 
This yields the approximation
\begin{equation*}
\int_{A_{\my}} q_0(\my_1,\mx) f_{\rm{Taylor}}^{\rm{kernel}}(t,\my,\my_1)\mathbf{I}_{\{\Vert\my_1\Vert_{\infty}\leq 1\}} \mrd\my_1.
\end{equation*}
For fixed $(t,\my,\mx)$, the product $q_0(\my_1,\mx)f_{\rm{Taylor}}^{\rm{kernel}}(t,\my,\my_1)$ is a polynomial in $\my_1$ on each response cell. Thus, the integral can be evaluated explicitly by integrating separately over the intersection of each response cell with $A_{\my}$.

In the third step, we utilize a DNN to approximate the resulting integral.
Integrating these analyses, we can construct a ReLU neural network $\mb_q(t,\my,\mx)$ that approximates the proxy conditional score function $\nabla\log q_t(\my,\mx)$ over $[T,1-T]$ as
\[
    \int_{\Rbb^{d_\mcY}} q_t(\my,\mx)\|\mathbf b_q(t,\my,\mx)-\nabla\log q_t(\my,\mx)\|^2 \mrd\my
    \lesssim \frac{\mcM^{-2\beta}\log\mcM}{\sigma_t^2}.
\]

Ultimately, writing $w=q_0/p_0$ and using a relative-Fisher transfer approach gives
\[
\frac{q_t(\my,\mx)}{p_t(\my\mid\mx)}
=\Ebb[w(\mathbf Y_0,\mx)\mid\mathbf Y_t=\my,\mathbf X=\mx],
\]
which yields $p_t\leq2q_t$ and
\[
\int_{\Rbb^{d_\mcY}} p_t(\my\mid\mx)
\|\nabla\log q_t(\my,\mx)-\nabla\log p_t(\my\mid\mx)\|^2\mrd\my
\lesssim\frac{\mcM^{-2\beta}}{\sigma_t^2}.
\]
Related score approximation arguments use local relative density error comparisons \citep[Supplement, Section~D.1.2, equation~(40)]{tang2025conditional}, or transfer Wasserstein error to time-integrated relative-Fisher error \citep[Proposition~5.6 and Appendix~D.1]{azangulov2024convergence}. Here, the normalized proxy and conditional second moments yield a full-space, pointwise-in-time relative-Fisher bound from uniform initial relative density error. Combining this bound with the neural-network approximation of the proxy score proves Lemma~\ref{lem: approximation_error}.
\end{remark}

\begin{lemma}[Statistical Error]\label{lem: statistical_error}
Suppose that Assumptions \ref{ass: bounded_density} hold. Let $\mcM \gg 1$, $C_T > 0$, and $T = \mcM^{-C_T}$. The score estimator $\widehat{\mb}$ defined in \eqref{eq: escore} with the neural network structures belonging to class $\mathcal{C}$, satisfies
\begin{equation*}
\mathbb{E}_{\mcD,\mcT,\mcZ}\left(\mcL(\widehat{\mb}) - 2\overline{\mcL}_{\mcD}(\widehat{\mb}) + \mcL(\mb^{*})\right) \lesssim
\frac{\mcM^{d_\mcX + d_\mcY}\log^{15} \mcM\left(\log^4 \mcM + \log n\right)}{n}
\end{equation*}
and
\begin{equation*}
\mathbb{E}_{\mcD,\mcT,\mcZ}\left(\overline{\mcL}_{\mcD}(\widehat{\mb}) - \widehat{\mcL}_{\mcD,\mcT,\mcZ}(\widehat{\mb})\right) \lesssim 
 \mcM^{C_T}(\log m + \log \mcM) \cdot \frac{\mcM^{\frac{d_\mcX + d_\mcY}{2}}\log^{\frac{13}{2}}\mcM(\log^2 \mcM + \sqrt{\log m})}{\sqrt{m}}.
\end{equation*}
\end{lemma}

By combining the approximation and statistical errors presented in Lemmas \ref{lem: approximation_error}-\ref{lem: statistical_error}, we can now derive the upper bound for the conditional score estimation, as detailed in the following theorem.
\begin{theorem}[Error Bound for Conditional Score Estimation]
\label{thm: generalization}
Suppose that Assumptions \ref{ass: bounded_density}--\ref{ass: Holder} hold. Denote by $\mcP_\beta(C_l,C_u,R)$ the collection of conditional densities satisfying Assumptions \ref{ass: bounded_density}--\ref{ass: Holder} with these constants.
By choosing $\mcM = \lfloor n^{\frac{1}{d_\mcX + d_\mcY + 2\beta}} \rfloor + 1 \lesssim n^{\frac{1}{d_\mcX + d_\mcY + 2\beta}}$, $m\asymp \mcM^{d_\mcX+d_\mcY+4\beta+2C_T}$, and $C_T=2\beta/(\beta\wedge1)$ in 
Lemmas \ref{lem: approximation_error}-\ref{lem: statistical_error}, the score estimator $\widehat{\mb}$ defined in \eqref{eq: escore} with the neural network structure belonging to class $\mathcal{C}$, satisfies 
$$
\begin{aligned}
\sup_{p_0\in\mcP_\beta(C_l,C_u,R)}
&\Ebb_{\mcD,\mcT,\mcZ}\left(\frac{1}{1-2T}\int_{T}^{1-T}\frac{1}{t}\cdot\Ebb_{\my_t,\mx}\Vert\widehat{\mb}(1-t,\my_t,\mx) - \nabla\log p_{1-t}(\my_t|\mx)\Vert^2 \mrd t\right)\\
&\lesssim n^{-\frac{2\beta}{d_\mcX + d_\mcY + 2\beta}} \log^{19}n.
\end{aligned}
$$
The implicit constant is uniform over the fixed density class and all covariate laws supported on $[-1,1]^{d_{\mcX}}$.
\end{theorem}

Now we bound $\Ebb_{\mcD,\mcT,\mcZ}\Ebb_{\mx}[\mathrm{TV}(\wt{p}_T(\cdot|\mx),p_0(\cdot|\mx))]$. We decompose this error term into two distinct components, as outlined by the following inequality:
\begin{equation*}
    \Ebb_{\mcD,\mcT,\mcZ}\Ebb_{\mx}[\mathrm{TV}(\wt p_T(\cdot|\mx),p_0(\cdot|\mx))]
    \le \Ebb_{\mcD,\mcT,\mcZ}\Ebb_{\mx}[\mathrm{TV}(\wt p_T(\cdot|\mx),p_T(\cdot|\mx))] + \Ebb_{\mx}[\mathrm{TV}(p_T(\cdot|\mx),p_0(\cdot|\mx))].
\end{equation*}

\subsubsection{Bound $\Ebb_{\mcD,\mcT,\mcZ}\Ebb_{\mx}[\mathrm{TV}(\wt{p}_T(\cdot|\mx),p_T(\cdot|\mx))]$}\label{subsec: bound1}

In this subsection, we bound $\Ebb_{\mcD,\mcT,\mcZ}\Ebb_{\mx}[\mathrm{TV}(\wt{p}_T(\cdot|\mx),p_T(\cdot|\mx))]$. In the proof, we utilize the passage from the continuous score risk to its deterministic-grid counterpart. Let $\Lambda_n=(M\kappa)^L$. Lemma~\ref{lem:uniform_grid} proves a time-continuity bound
\[
    \left|\sum_{i=0}^{K-1} F_{p_0,\mb}(t_i)\Delta-\int_T^{1-T}F_{p_0,\mb}(t)\mrd t\right|
    \leq C(1+\Lambda_n)T^{-3}\log\mcM\,\Delta,
\]
for the weighted score risk, uniformly over the fixed density class $\mcP_\beta(C_l,C_u,R)$ and every network in $\mcC$. Here $F_{p_0,\mb}(t)$ is defined by
\[
    F_{p_0,\mb}(t)=\frac1t\Ebb_{\mx}\int_{\Rbb^{d_{\mcY}}}
    \|\mb(1-t,\my,\mx)-\nabla\log p_{1-t}(\my\mid\mx)\|^2p_{1-t}(\my\mid\mx)\mrd\my.
\]
Choose a sufficiently small fixed $c>0$, depending only on the class constants, and set
\[
    \Delta_n=\frac{c\mcM^{-2\beta}T^3}{(1+\Lambda_n)\log\mcM},\qquad
    K=\left\lceil\frac{1-2T}{\Delta_n}\right\rceil.
\]
This grid is independent of the unknown density and the training outcome. For $\Delta\le\Delta_n$, the Riemann-sum error is at most $C(1+\Lambda_n)T^{-3}\log\mcM\,\Delta_n\leq\mcM^{-2\beta}$, after fixing $c$ sufficiently small. Combining this with Theorem~\ref{thm: generalization} gives, for any uniform grid with $\Delta\leq\Delta_n$,
\begin{equation} \label{eq: discrete_generalization}
\begin{aligned}
&~~~\Ebb_{\mcD,\mcT,\mcZ} \left(\sum_{i=0}^{K-1}\frac{t_{i+1}-t_i}{t_i}\cdot \Ebb_{\my_{t_i},\mx} \Vert \wh{\mb}(1-t_i, \my_{t_i},\mx) - \nabla\log p_{1-t_i}(\my_{t_i}|\mx)\Vert^2   \right)\\
&\leq (1-2T)\cdot \Ebb_{\mcD,\mcT,\mcZ} \left[\frac{1}{1-2T}\int_{T}^{1-T}\frac{1}{t}\cdot\Ebb_{\my_t,\mx}\Vert\wh{\mb}(1-t,\my_t,\mx) - \nabla\log p_{1-t}(\my_t|\mx)\Vert^2 \mathrm{d}t\right] + \mcM^{-2\beta}\\
&\lesssim n^{-\frac{2\beta}{d_\mcX + d_\mcY + 2\beta}}\log^{19}n.
\end{aligned}
\end{equation}
Using \eqref{eq: discrete_generalization}, we obtain the following total-variation bound.
\begin{theorem}\label{thm: sampling_error} 
Suppose that assumptions of Theorem \ref{thm: generalization} hold. Then,  by choosing the mesh size $\Delta = \mcO\left(\min\{\Delta_n, \mcM^{-(C_T+2\beta)}\}\right)$, we have
\begin{equation}\label{eq: sub_bound2}
\Ebb_{\mcD,\mcT,\mcZ}\Ebb_{\mx}[\mathrm{TV}(\wt{p}_T(\cdot|\mx),p_T(\cdot|\mx))] \lesssim n^{-\frac{\beta}{d_\mcX + d_\mcY + 2\beta}}\log^{\frac{19}{2}}n.
\end{equation}
\end{theorem}

\subsubsection{Bound $\Ebb_{\mx}[\mathrm{TV}(p_T(\cdot|\mx), p_0(\cdot|\mx))]$}\label{subsec: bound2}
In this subsection, we present the upper bound for $\Ebb_{\mx}[\mathrm{TV}(p_T(\cdot|\mx), p_0(\cdot|\mx))]$. This error originates from the process of early stopping. We have the following lemma.
\begin{lemma}\label{lem: early_stopping}
Suppose Assumptions~\ref{ass: bounded_density}--\ref{ass: Holder} hold. Let $\mcM = \lfloor n^{\frac{1}{d_\mcX + d_\mcY + 2\beta}} \rfloor + 1 \lesssim n^{\frac{1}{d_\mcX + d_\mcY + 2\beta}}$, $T = \mcM^{-C_T}$,
and $C_T=2\beta/(\beta\wedge1)$, we have
\begin{equation}\label{eq: sub_bound1}
    \Ebb_{\mx}[\mathrm{TV}(p_T(\cdot|\mx), p_0(\cdot|\mx))]
    \le\sup_{\mx\in[-1,1]^{d_\mcX}}\mathrm{TV}(p_T(\cdot|\mx), p_0(\cdot|\mx))
    \lesssim n^{-\frac{\beta}{d_\mcX + d_\mcY + 2\beta}}.
\end{equation}
\end{lemma}


\subsubsection{Bound $\Ebb_{\mcD,\mcT,\mcZ}\Ebb_{\mx}[\mathrm{TV}(\wt{p}_T(\cdot|\mx),p_0(\cdot|\mx))]$}
\begin{theorem}[Convergence Rate of Conditional Diffusion Model] \label{thm: convergence_rate}
Suppose that Assumptions \ref{ass: bounded_density}--\ref{ass: Holder} hold, and the conditional score estimator $\mb$ is structured as introduced in Theorem \ref{thm: generalization}. By choosing $\mcM$, $m$, $C_T$,  introduced in Theorem \ref{thm: generalization} and $\Delta = \mcO\left(\min\{\Delta_n, \mcM^{-(C_T+2\beta)}\}\right)$, we have for all sufficiently large $n$ that
\begin{equation*}
    \sup_{p_0\in\mcP_\beta(C_l,C_u,R)}
    \Ebb_{\mcD,\mcT,\mcZ}\Ebb_{\mx}[\mathrm{TV}(\wt{p}_T(\cdot|\mx),p_0(\cdot|\mx))] \lesssim n^{-\frac{\beta}{d_\mcX + d_\mcY + 2\beta}} \log^{\frac{19}{2}}n.
\end{equation*}
The implicit constant and the threshold for sufficiently large $n$ are independent of $p_0$, and the same network class and deterministic grid are used for every $p_0\in\mcP_\beta(C_l,C_u,R)$.
\end{theorem}

\begin{remark}
Theorem \ref{thm: convergence_rate} establishes an end-to-end total-variation convergence rate for conditional diffusion models. This advances the literature by improving upon earlier analyses that do not provide end-to-end guarantees, including \cite{chen2023improved, conforti2023score, lee2022convergence, lee2023convergence, benton2023linear, li2023towards, gao2023wasserstein}.
Moreover, fix $0<C_l<2^{-d_{\mcY}}<C_u<\infty$ and $R>2^{-d_{\mcY}}$. If the fixed covariate density is bounded away from zero on $[-1,1]^{d_\mcX}$, then Theorem~\ref{thm: convergence_rate} is minimax optimal under average conditional total-variation loss, up to logarithmic
factors. Indeed, Appendix~\ref{sec:minimax} establishes
\[
    \inf_{\widehat p}\sup_{p_0\in\mcP_\beta(C_l,C_u,R)}
    \Ebb_{p_0,p_\mX}^{(n)}\Ebb_\mx\left[\operatorname{TV}(\widehat p(\cdot|\mx),p_0(\cdot|\mx))\right]
    \gtrsim n^{-\frac{\beta}{d_\mcX+d_\mcY+2\beta}},
\]
where $\Ebb_{p_0,p_{\mX}}^{(n)}$ denotes expectation over the training sample $\mcD$, consisting of $n$ independent observations with joint density $p_0(\my|\mx)p_{\mX}(\mx)$, and $\widehat p$ is an estimator based on $\mcD$. Averaging over the auxiliary training randomness gives the estimator $\widehat p_{\mcD}^{\,\mathrm{av}} :=\Ebb_{\mcT,\mcZ}[\wt p_T\mid\mcD]$, which depends only on $\mcD$. By convexity of total variation, this averaging does not increase risk, so the same uniform upper bound holds for $\widehat p_{\mcD}^{\,\mathrm{av}}$. The design-density lower bound $p_\mX(\mx)\ge c_\mX>0$ is needed only for the minimax lower bound. Detailed proofs are left to Appendix~\ref{sec:minimax}.
In comparison to existing end-to-end results, our convergence bound represents a significant refinement over works such as \cite{chen2023score, jiao2025model}, which rely on Lipschitz-continuous score networks and consequently yield suboptimal rates. Our theoretical guarantees are comparable to those in \cite{oko2023diffusion,jiao2024latent,fu2024unveil,tang2025conditional}; however, it is important to note that \cite{fu2024unveil} does not incorporate the numerical discretization error of the reverse process, which is explicitly addressed in our analysis.
Also, our relative-Fisher transfer argument yields conditional score approximation bounds without the additional density smoothness near the support boundary imposed in \cite{oko2023diffusion,jiao2024latent}, while retaining the minimax optimal polynomial convergence rate for the density class considered here.
\end{remark}

\subsection{Convergence of Bootstrap}\label{sec:thboot}
In this section, we study convergence of the deep bootstrap method in bounded-Lipschitz distance. First, we bound the distance between the centered estimator $\widehat{R}(\mx):= \widehat{f}(\mx)-f_0(\mx)$ and its bootstrap counterpart $\widehat{R}^*(\mx):=\widehat{f}^*(\mx)- \widehat{f}(\mx)$, as shown in Theorem \ref{thm: bootstrap_consistency}. We further establish the coverage guarantee under additional assumptions for the confidence interval constructed using the deep bootstrap procedure.

For the theoretical analysis in this section, we impose the following postprocessing procedure, which is only a technical device for the proof, to turn the distributions under study into compactly supported ones. Set $B=1$ and define
\begin{equation}\label{eq: bootstrap_postprocessing}
    g_B(\my,\mathbf U)
    := \my\mathbf1_{\{\|\my\|_\infty\le B\}}+\mathbf U\mathbf1_{\{\|\my\|_\infty>B\}},\qquad
    \mathbf U\sim\operatorname{Unif}([-B,B]^{d_\mcY}).
\end{equation}
For any terminal conditional law $P(\cdot\mid\mx)$, let $P^B(\cdot\mid\mx)$ denote the law of $g_B(\my,\mathbf U)$, where $\my\sim P(\cdot\mid\mx)$ and $\mathbf U$ is independent of $\my$. A fresh independent $\mathbf U$ is used for every generated response. This convention defines $p_T^B$, $\wh p_T^B$, $\wt p_T^B$, and their bootstrap counterparts. If $P$ has conditional density $p$, then $P^B$ has density
\[
    p^B(\my\mid\mx)
    = \mathbf1_{\{\|\my\|_\infty\le B\}}\left[p(\my\mid\mx)+(2B)^{-d_\mcY}
    \int_{\{\|\mathbf v\|_\infty>B\}}p(\mathbf v\mid\mx)\,\mrd\mathbf v\right].
\]
Thus the postprocessing preserves absolute continuity. The common postprocessing kernel contracts TV as $\mathrm{TV}(P^B,Q^B)\leq\mathrm{TV}(P,Q)$. Since $B=1$ and $p_0$ is supported on $[-1,1]^{d_\mcY}$, $p_0^B=p_0$, and hence
\[
    \mathrm{TV}(\wt p_T^B(\cdot\mid\mx),p_0(\cdot\mid\mx))
    \leq \mathrm{TV}(\wt p_T(\cdot\mid\mx),p_0(\cdot\mid\mx)).
\]
Consequently, Theorem~\ref{thm: convergence_rate} also bounds the expected TV error of the postprocessed original distribution at the same rate. While the corresponding sample means become
\[
    \wh f_B(\mx)
    :=\frac1n\sum_{j=1}^n g_B(\wh{\mY}_{\mx}^{(j)},\mathbf U_j),\qquad
    \wh f_B^*(\mx)
    :=\frac1n\sum_{j=1}^n g_B(\wh{\mY}_{\mx}^{*,(j)},\mathbf U_j^*),
\]
we suppress the subscript $B$ throughout this section, and still write $\wh f(\mx)$ and $\wh f^*(\mx)$. Throughout this section and its proofs, $\Ebb_{\mcD,\mcT,\mcZ}\Ebb_{\mx}$ denotes expectation over $\mcD,\mcT,\mcZ,\mx$ and the additional sampling and postprocessing randomness used to construct $\wh f(\mx)$, under their joint law.
\begin{theorem}[Convergence Rate of Bootstrap]\label{thm: bootstrap_consistency}
Assume that the conditions of Theorem \ref{thm: convergence_rate} hold, and the bootstrap training target $\wt p_T^B(\cdot\mid\mx)$ belongs almost surely to some $\mcP_\beta(C_l^\star,C_u^\star,R^\star)$, with deterministic constants independent of $n$ and the original training realization. We denote by
$\wh{R}(\mx):=\wh{f}(\mx)-f_0(\mx)$ and $\wh{R}^*(\mx):=\wh{f}^*(\mx)-\wh{f}(\mx)$. Then, we have  
\begin{align*}
    &\Ebb_{\mcD,\mcT,\mcZ}\Ebb_{\mx}\left[d_{\mathrm{BL}}\left(\Law(\wh R(\mx)\mid\mx),\Law^*(\wh R^*(\mx)\mid\mathscr F_{n,\mx})\right)\right]\\
    &~~~~~~~\le \Ebb_{\mcD,\mcT,\mcZ}\Ebb_{\mx}\left[\mcW_1\left(\Law(\wh{R}(\mx)\mid\mx),\Law^*(\wh{R}^*(\mx)\mid\mathscr F_{n,\mx})\right)\right]
    \lesssim n^{-\frac{\beta}{d_\mcX + d_\mcY + 2\beta}} \log^{\frac{19}{2}}n,
\end{align*}
where $\Law(U\mid\mx)$ denotes the conditional law of a random variable $U$ given $\mx$, and $\Law^*(U^*\mid\mathscr F_{n,\mx})$ denotes the conditional bootstrap law of $U^*$ given $\mathscr F_{n,\mx}$.
\end{theorem}
\begin{remark}\label{rem: bootstrap_analogue}
In the proof of Theorem~\ref{thm: bootstrap_consistency}, we use a bootstrap analogue of Theorem~\ref{thm: convergence_rate}. Indeed, under the additional uniform regularity assumption on the bootstrap target given above, it follows from the same approximation and sampling arguments. The statistical-error bound can be established by adapting the original proof through conditional Bernstein inequalities and a uniform comparison of empirical and population design risks to accommodate the reused covariates.
\end{remark}

Theorem~\ref{thm: bootstrap_consistency} establishes unscaled Wasserstein convergence averaged over the test-covariate distribution and the original estimation randomness. However, the unscaled Wasserstein consistency alone is not sufficient to justify the validity of the bootstrap confidence interval, since the Wasserstein distance is not invariant under rescaling. Indeed, two distributions may both degenerate to a point mass at zero while having different stochastic scales. 
Therefore, to obtain a valid confidence interval, we impose the following additional scale and quantile regularity conditions. 

\begin{assumption}[Bounded-Lipschitz bootstrap consistency at the stochastic scale]\label{ass: stochastic_scale}
For the fixed covariate point $\mx$, there exists a positive deterministic sequence $a_n\downarrow 0$ such that
\begin{equation*}
    \frac{\wh{R}(\mx)}{a_n}\xrightarrow{d}Z_\mx,
\end{equation*}
where $Z_\mx$ is a nondegenerate real random variable with distribution function $F_\mx$. Moreover, the conditional bootstrap distribution is bounded-Lipschitz consistent at this stochastic scale, in the sense that
\begin{equation*}
    d_{\mathrm{BL}}\left(\Law^*(a_n^{-1}\wh R^*(\mx)\mid\mathscr F_{n,\mx}),\Law(a_n^{-1}\wh R(\mx)\mid\mx)\right)
    \xrightarrow{\Pbb}0.
\end{equation*}
\end{assumption}

\begin{assumption}[Quantile regularity]\label{ass: quantile_regular}
Let $q_\mx(u)=F_\mx^{-1}(u)=\inf\{t:F_\mx(t)\geq u\}$ for $0<u<1$.  The quantile function $F_\mx^{-1}$ is continuous at $u_L=\alpha/2$ and $u_U=1-\alpha/2$.  Moreover, $F_\mx$ is continuous at $q_\mx(u_L)$ and $q_\mx(u_U)$.
\end{assumption}

Assumption \ref{ass: stochastic_scale} identifies $a_n$ as the local stochastic scale of $\wh R(\mx)$ and requires conditional weak convergence of the normalized bootstrap law to the sampling law in the bounded-Lipschitz metric, a standard formulation of bootstrap consistency; see \cite{van2000asymptotic,shao2012jackknife,kosorok2008introduction}.
Also, this requirement is closely related to the notion of asymptotic pertinence in bootstrap prediction intervals; see \cite{politis2015model}. Matching the distributions of normalized estimation errors is explicitly included in the definition of asymptotic pertinence in the deep learning framework of \cite{wu2024deep} (condition C3), and is also studied for nonparametric autoregressions in \cite{politis2023multi}. These works demonstrate the importance of matching the estimation error at its stochastic scale. 
To establish asymptotic pertinence, \cite{wang2026model} imposes a high-level assumption of matching asymptotic normality and negligible bias in the original and bootstrap worlds.
Assumption~\ref{ass: stochastic_scale} plays the same role here: it isolates the additional scale-level bootstrap consistency needed to justify the confidence interval.
Assumption \ref{ass: quantile_regular} is only imposed at the two quantile levels used by the confidence interval; for example, it is implied by local strict increase and continuity of $F_\mx$ around the two target quantiles. The next theorem shows that, under the additional scale and quantile regularity conditions, the bootstrap confidence interval achieves the desired asymptotic coverage probability.

\begin{theorem}[Coverage Guarantee of Confidence Interval]\label{thm: bootstrap_confidence_interval}
Suppose that Assumptions \ref{ass: stochastic_scale}--\ref{ass: quantile_regular} hold for each covariate point $\mx\in[-1,1]^{d_\mcX}$ and $\mcB,n\to\infty$. Then, the bootstrap confidence interval defined in \eqref{interval} satisfies
$$
\Pbb\left(
f_0(\mx)\in \left[\wh{f}(\mx)-(\wh{H}^*)^{-1}(1-\alpha/2),\wh{f}(\mx)-(\wh{H}^*)^{-1}(\alpha/2)\right]
\right)\longrightarrow 1-\alpha.
$$
\end{theorem}
 
\section{Numerical Experiments}\label{sec:ne} 

In this section, we validate the efficiency of the proposed deep bootstrap method through a series of numerical experiments.
Throughout the experiments, we consider the 
nonparametric regression model with a standard normal noise:
\begin{equation*}
    \mY=f_0(\mX)+\epsilon,\quad \epsilon\sim \mcN(0,1),   
\end{equation*}
where the covariate vector is multi-dimensional and the response variable is of 1 dimension. To comprehensively assess the regression capability of our deep bootstrap method across different scenarios, we examine bounded and unbounded domains for the covariate vector respectively. Notably, unlike the assumptions made in the theoretical analysis, we do not impose a bounded domain on the response variable in our experimental setup. This choice enables us to evaluate the robustness and generalizability of the proposed method under more realistic and diverse conditions.

Several metrics are adopted to evaluate the performance of the experimental results. Firstly,  we employ the coverage probability (CP). Note that the CP we use in the experiments are defined to be
\begin{equation*}
    \text{CP}:=\frac{\#\{\mx\in \text{Test}:f_0(\mx)\in \left[\alpha_1(\mx),\alpha_2(\mx)\right] \}}{N_\text{Test}},
\end{equation*}
where $N_\text{Test}$ denotes the size of the test set, and $\alpha_1(\mx),\alpha_2(\mx)$ represent the endpoints of the confidence interval defined in (\ref{interval}). This metric records the empirical coverage fraction across test covariates for a fitted model. We also report $\mathrm{MSE}_{\mathrm{org}}$, the squared regression error averaged over the test set, and $\mathrm{MSE}_{\mathrm b}$, the bootstrap second moment around the original estimate. The latter summarizes bootstrap variability and centering. Finally, we report the average interval length. Table~\ref{tab:metric defs} gives the definitions.
\begin{table}[H]
\centering
\renewcommand{\arraystretch}{2.3}
\tabcolsep=0.5cm
\caption{Several metrics}
\begin{tabular}{cc}
\toprule
Metrics          & Mathematical formulation \\ \midrule
CP               & $\displaystyle\frac{1}{N_\text{Test}}\sum\limits_{\mx\in \text{Test}}\mI\Big( f_0(\mx)\in[\alpha_1(\mx),\alpha_2(\mx)] \Big)$     \\
$\text{MSE}_\text{org}$       & $\displaystyle\frac{1}{N_\text{Test}}\sum\limits_{\mx\in \text{Test}}\left| \wh{f}(\mx)-f_0(\mx) \right|^2$     \\
$\text{MSE}_\text{b}$ & $\displaystyle\frac{1}{N_\text{Test}}\sum\limits_{\mx\in \text{Test}}\frac{1}{\mcB}\sum\limits_{b=1}^\mcB\left| \wh{f}^*_b(\mx)-\wh{f}(\mx) \right|^2$     \\
Interval Length & $\displaystyle\frac{1}{N_\text{Test}}\sum\limits_{\mx\in \text{Test}}\left( \alpha_2(\mx)-\alpha_1(\mx) \right)$     \\[10pt] \bottomrule
\end{tabular}
\label{tab:metric defs}
\end{table}

During the implementation, we fix the confidence level $1-\alpha$ to be 0.95 (i.e. $\alpha=0.05$). The test sample size is held constant as 2500, with the total number of train and test samples varied only through the proportion of the test sample size. The number of epochs for training the base model is determined adaptively based on the observed decay of training error. This epoch count is then maintained throughout the bootstrap model-training procedure. For the bootstrap ensemble size $\mcB$, which serves as a hyperparameter, we would normally adjust it flexibly according to model performance. However, to facilitate comparative analysis across experiments, we fix $\mcB$ at an empirically determined value of 200 throughout all the experiments. To see all the details of experiment settings, please refer to Table \ref{tab:specified experiment config}.

Experiments are carried out with dimension of the covariate vector chosen to be 5 and 10. In the 5-dimensional scenario, the following two targets of bounded and unbounded domains are considered. Both targets are designed to be maximally complex and representative.
\begin{enumerate}[(I)]
    \item $\mX \sim U([0,1]^{d_\mcX}),\ \epsilon \sim 
    \mcN(0,1)$ with $f_0$ defined by
    \begin{equation*}
        f_0(\mx)=(2x_1-x_2+1)^2+|x_3-5|+\exp(x_4+\frac{x_5}{2}).
    \end{equation*}
    \item $\mX \sim \mcN(\mathbf{0},\textbf{I}_{d_\mcX}),\ \epsilon \sim \mcN(0,1)$ with $f_0$ defined by
    \begin{equation*}
        f_0(\mx)=(2x_1-1)^2-x_2^3+\exp(\frac{1}{10}(x_3+x_4+x_5)).
    \end{equation*}
\end{enumerate}

\begin{table}[H]
\centering
\renewcommand{\arraystretch}{1.5}
\tabcolsep=0.7cm
\caption{Numerical results of $d_\mcX=5$
}
\begin{tabular}{ccccc}
\toprule
          & CP     & $\text{MSE}_\text{org}$ & $\text{MSE}_\text{b}$ & Interval Length \\ \midrule
(I) & 0.9480 & 0.0007     & 0.0011           & 0.1166           \\
(II) & 0.9452 & 0.0137     & 0.0542           & 0.2895           \\ \bottomrule
\end{tabular}
\label{tab:5-dim result}
\end{table}

In the 10-dimensional scenario, we first consider a simple linear target, followed by two relatively complex instances defined on bounded and unbounded domains respectively. The latter two functions serve to demonstrate the applicability of our method to generic complex functions as well as  functions exhibiting strong oscillatory behavior.
\begin{enumerate}[(I)]
    \item $\mX \sim \mcN(0,\textbf{I}_{d_\mcX}),\ \epsilon \sim \mcN(0,1)$ with linear target $f_0(\mx)=W\mx+b$, whose weight $W$ and bias $b$ are randomly drawn from the uniform distribution over [-1,1].
    \item $\mX \sim U([0,1]^{d_\mcX}),\ \epsilon \sim \mcN(0,1)$ with $f_0$ defined by
\begin{align*}
f_0(\mx)=
&3x_1+4\left(x_2-\frac{1}{2}\right)^2-x_3^2+2\sin \left(\pi(x_4+2x_5)\right)\\
&-5\left|x_6-\frac{1}{2}\right|+\exp\left(\frac{1}{10}(x_7+x_8+x_9+x_{10})\right).
\end{align*}
    \item $\mX \sim \mcN(0,\textbf{I}_{d_\mcX}),\ \epsilon \sim \mcN(0,1)$ with $f_0$ defined by
    \begin{equation*}
        f_0(\mx)=\frac{1}{2}\sum\limits_{i=0}^1\left[ \sin(2x_{5i+1}+x_{5i+2})+\frac{1}{2}\left( \cos(x_{5i+3})+x_{5i+4}^2 \right)x_{5i+5} \right].
    \end{equation*}
\end{enumerate}

\begin{table}[H]
\centering
\renewcommand{\arraystretch}{1.5}
\tabcolsep=0.7cm
\caption{Numerical results of $d_\mcX=10$
}
\begin{tabular}{ccccc}
\hline
      & CP     & $\text{MSE}_\text{org}$    & $\text{MSE}_\text{b}$    & Interval Length \\ \hline
(I)   & 0.9464 & 0.0006 & 0.0011 & 0.1185          \\
(II)  & 0.9380 & 0.0015 & 0.0019 & 0.1581          \\
(III) & 0.9424 & 0.0054 & 0.0092 & 0.2777          \\ \hline
\end{tabular}
\label{tab:10-dim result}
\end{table}

Based on the experimental results presented in Table \ref{tab:5-dim result} and \ref{tab:10-dim result}, we observe that the key index CP consistently approaches the ideal value of 0.95 across various target function configurations, achieving low MSEs simultaneously. These results provide empirical evidence for the performance of the deep bootstrap method in the considered simulation settings. While our training procedure and hyperparameter settings may not achieve the global optimality, we empirically claim that the learning difficulty increases with higher dimensionality of the target covariate vector as well as greater functional complexity. These factors correlate with reduced precision in the CP. We therefore recommend practitioners to conduct adaptive hyperparameter exploration based on the dimension, specific functional forms and computational cost to achieve an optimal calculation.

\begin{table}[H]
\centering
\renewcommand{\arraystretch}{1.5}
\tabcolsep=0.35cm
\caption{Specified configuration}
\begin{tabular}{cccccc}
\toprule
                              & \multicolumn{2}{c}{$d_\mcX=5$} & \multicolumn{3}{c}{$d_\mcX=10$}              \\ \cline{2-6} 
                              & (I)         & (II)        & (I)         & (II)        & (III)       \\ \midrule
dimension                     & 5           & 5           & 10          & 10          & 10          \\
$\alpha$                      & 0.05        & 0.05        & 0.05        & 0.05        & 0.05        \\
test set size                 & 2500        & 2500        & 2500        & 2500        & 2500        \\
test set ratio                & 0.02        & 0.02        & 0.025       & 0.0125       & 0.0125        \\
hidden layers                 & {[}48,48{]} & {[}56,56{]} & {[}48,48{]} & {[}56,56{]} & {[}64,64{]} \\
$\mcB$                             & 200         & 200         & 200         & 200         & 200         \\
\bottomrule
\end{tabular}
\label{tab:specified experiment config}
\end{table}

\section{Conclusion}\label{sec:con}
This work presents a deep bootstrap method for statistical inference in nonparametric regression, built upon conditional diffusion models. By harnessing the expressive power of diffusion models, our approach goes beyond traditional bootstrap methods by unifying nonparametric estimation and resampling, achieving efficiency in high-dimensional and potentially multimodal settings.
We provide an end-to-end total-variation bound for the conditional diffusion model and an unscaled bounded-Lipschitz bound for the deep bootstrap statistic. For the nondegenerate density class with $0<C_l<2^{-d_{\mcY}}<C_u<\infty$, $R>2^{-d_{\mcY}}$, and every fixed covariate design density bounded away from zero, the polynomial part of the total-variation rate matches the minimax lower bound.
A key ingredient in our diffusion model analysis is a relative-Fisher transfer argument that controls the conditional score approximation error without requiring additional smoothness of the density near the support boundary.
Under a separate pointwise consistency condition for the normalized bootstrap law, the resulting confidence interval has asymptotically correct coverage. 
Together, these results offer a strong theoretical basis for integrating generative modeling into bootstrap-based inference frameworks.
Future research may explore extending the deep bootstrap framework to more general settings, such as time series analysis and broader classes of nonparametric models. Additional directions include applying the method to diverse domains within statistics and machine learning, and incorporating recent advances in diffusion models to further improve empirical performance and robustness.

\section*{Appendix}\label{sec:ap}
\appendix
In this appendix, we provide detailed proofs for the theoretical results presented in the paper. Section  \ref{sec:ae} establishes an approximation error bound for the score network, while Section  \ref{sec:se} derives the corresponding statistical error bound. In Section \ref{sec:bb}, we bound  $\Ebb_{\mcD, \mcT,\mcZ}\Ebb_{\mx}\left[\mathrm{TV}(\wt{p}_T(\cdot|\mx), p_0(\cdot|\mx))\right]$. The matching minimax lower bound is proved in Section \ref{sec:minimax}. Section \ref{sec:prbc} presents the convergence analysis of the proposed deep bootstrap procedure. Finally, auxiliary lemmas used throughout the proofs are deferred to Section \ref{sec:al}.

\section{Approximation Error}\label{sec:ae}
We prove Lemma \ref{lem: approximation_error} by first constructing a normalized local-polynomial proxy $q_0$, then approximating the score of its Gaussian smoothing $q_t$ on the entire time interval, and finally transferring this approximation to the true score by a relative-Fisher bound. Throughout this section,
\[
    m_t:=1-t, \qquad
    \sigma_t:=\sqrt{t(2-t)}, \qquad
    m_t^2+\sigma_t^2=1.
\]

\subsection{Approximating $p_0$ by a normalized local-polynomial density}\label{subsec: c1}
Write $\beta=r+s$, where $r\in\mathbb N_0$ and $s\in(0,1]$. For $\mathbf m=(m_1,m_2,\cdots,m_{d_\mcY})^\top\in\{1,\ldots,\mcM\}^{d_{\mcY}}$ and $\mathbf n=(n_1,n_2,\cdots,n_{d_\mcX})^\top\in\{1,\ldots,\mcM\}^{d_{\mcX}}$, define
\begin{gather*}
    I_{\mathbf m}:=\prod_{i=1}^{d_{\mcY}}I_{m_i}, \qquad
    I_{m_i}:=
    \begin{cases}
    [-1,-1+\frac{2}{\mcM}],&m_i=1,\\
    (-1+\frac{2(m_i-1)}{\mcM},-1+\frac{2m_i}{\mcM}],&2\leq m_i\leq\mcM,
    \end{cases}\\
    \my_{\mathbf m}:=-\mathbf 1+\frac{2\mathbf m}{\mcM}, \qquad
    \mx_{\mathbf n}:=-\mathbf 1+\frac{2\mathbf n}{\mcM}.
\end{gather*}
Let $\delta_{\mathrm{ind}}:=\mcM^{-2}$ and, for $j\in\{1,\ldots,\mcM\}$, set
\[
    a_1:=-1-\delta_{\mathrm{ind}},\quad
    a_j:=-1+\frac{2(j-1)}{\mcM}\quad(2\leq j\leq\mcM),\quad
    b_j:=-1+\frac{2j}{\mcM},
\]
\[
    H_{\delta_{\mathrm{ind}}}(u)
    :=\mathrm b_{\mathrm{clip}}\left(\frac{u}{\delta_{\mathrm{ind}}},0,1\right), \quad
    \pi_{j,\delta_{\mathrm{ind}}}(u)
    :=H_{\delta_{\mathrm{ind}}}(u-a_j)-H_{\delta_{\mathrm{ind}}}(u-b_j),
\]
and
\[
    \chi_{\mathbf n,\delta_{\mathrm{ind}}}(\mx)
    :=\prod_{j=1}^{d_{\mcX}}\pi_{n_j,\delta_{\mathrm{ind}}}(x_j).
\]
Note for every $x\in[-1,1]$, it holds
\[
\begin{aligned}
    \sum_{n=1}^{\mcM}\pi_{n,\delta_{\mathrm{ind}}}(x)
    & =H_{\delta_{\mathrm{ind}}}(x-a_1)
    -H_{\delta_{\mathrm{ind}}}(x-b_{\mcM})\\
    & =H_{\delta_{\mathrm{ind}}}(x+1+\delta_{\mathrm{ind}})
    -H_{\delta_{\mathrm{ind}}}(x-1)=1,
\end{aligned}
\]
and hence
\[
    \chi_{\mathbf n,\delta_{\mathrm{ind}}}(\mx)\geq0,
    \qquad
    \sum_{\mathbf n\in\{1,\ldots,\mcM\}^{d_{\mcX}}}
    \chi_{\mathbf n,\delta_{\mathrm{ind}}}(\mx)=1.
\]
If $\chi_{\mathbf n,\delta_{\mathrm{ind}}}(\mx)>0$, then
\[
    \|\mx-\mx_{\mathbf n}\|_\infty
    \leq\frac{2}{\mcM}+\delta_{\mathrm{ind}}
    \leq\frac{3}{\mcM}.
\]
For $|\boldsymbol\alpha|_1+|\boldsymbol\gamma|_1\leq r$, put
\[
    c_{\mathbf m,\mathbf n,\boldsymbol\alpha,\boldsymbol\gamma}:=
    \frac{
    \partial_{\my}^{\boldsymbol\alpha}
    \partial_{\mx}^{\boldsymbol\gamma}
    p_0(\my_{\mathbf m}\mid\mx_{\mathbf n})
    }{\boldsymbol\alpha!\boldsymbol\gamma!}.
\]
By Assumption \ref{ass: Holder},
\[
    \left|c_{\mathbf m,\mathbf n,\boldsymbol\alpha,\boldsymbol\gamma}\right|
    \leq\frac{R}{\boldsymbol\alpha!\boldsymbol\gamma!}.
\]
Define
\[
\begin{aligned}
    p_{\mcM}(\my,\mx):= 
    \sum_{\mathbf m\in\{1,\ldots,\mcM\}^{d_{\mcY}}}
    \sum_{\mathbf n\in\{1,\ldots,\mcM\}^{d_{\mcX}}}
    \mathbf I_{I_{\mathbf m}}(\my)
    \chi_{\mathbf n,\delta_{\mathrm{ind}}}(\mx)
    \cdot
    \sum_{|\boldsymbol\alpha|_1+|\boldsymbol\gamma|_1\leq r}
    c_{\mathbf m,\mathbf n,\boldsymbol\alpha,\boldsymbol\gamma}
    (\my-\my_{\mathbf m})^{\boldsymbol\alpha}
    (\mx-\mx_{\mathbf n})^{\boldsymbol\gamma}.
\end{aligned}
\]
For every $(\my,\mx)$, let $\mathbf m$ be the unique index such that $\my\in I_{\mathbf m}$. Taylor's formula and the preceding partition of unity thus give
\[
\begin{aligned}
    &~~~~|p_{\mcM}(\my,\mx)-p_0(\my\mid\mx)|\\
    &\leq
    \sum_{\mathbf n}
    \chi_{\mathbf n,\delta_{\mathrm{ind}}}(\mx)
    \left|
    p_0(\my\mid\mx)
    -\sum_{|\boldsymbol\alpha|_1+|\boldsymbol\gamma|_1\leq r}
    c_{\mathbf m,\mathbf n,\boldsymbol\alpha,\boldsymbol\gamma}
    (\my-\my_{\mathbf m})^{\boldsymbol\alpha}
    (\mx-\mx_{\mathbf n})^{\boldsymbol\gamma}
    \right|\\
    &\lesssim
    R\sum_{\mathbf n}\chi_{\mathbf n,\delta_{\mathrm{ind}}}(\mx)
    \left(\|\my-\my_{\mathbf m}\|^2
    +\|\mx-\mx_{\mathbf n}\|^2\right)^{\beta/2}\\
    &\leq
    R\left(
    \frac{4d_{\mcY}}{\mcM^2}
    +\frac{9d_{\mcX}}{\mcM^2}
    \right)^{\beta/2}
    \lesssim\mcM^{-\beta}.
\end{aligned}
\]
\begin{lemma}[Normalized local-polynomial proxy]\label{lem: approximate_pdata1}\label{lem: normalized_proxy}
Under Assumptions \ref{ass: bounded_density}--\ref{ass: Holder}, there exists a function $q_0(\my,\mx)$ supported on $[-1,1]^{d_{\mcY}}$ such that, uniformly in $(\my,\mx)\in[-1,1]^{d_{\mcY}}\times[-1,1]^{d_{\mcX}}$,
\[
    \int_{[-1,1]^{d_{\mcY}}}q_0(\my,\mx)\mrd\my=1,\quad
    \|q_0-p_0\|_\infty\lesssim\mcM^{-\beta},\quad
    \frac{C_l}{2}\leq q_0\leq2C_u
\]
for all sufficiently large $\mcM$. Moreover, $q_0$ can be written in the same local-polynomial and soft-gate form as $p_{\mcM}$, with uniformly bounded coefficients and $\mcO(\mcM^{d_{\mcX}+d_{\mcY}})$ summands.
\end{lemma}
\begin{proof}
For $\my\in[-1,1]^{d_{\mcY}}$, define
\[
    q_0(\my,\mx)
    := p_{\mcM}(\my,\mx)+2^{-d_{\mcY}}\left( 1-\int_{[-1,1]^{d_{\mcY}}}p_{\mcM}(\mathbf u,\mx)\mrd\mathbf u \right),
\]
and put $q_0(\my,\mx)=0$ outside $[-1,1]^{d_{\mcY}}$. Then
\[
    \int_{[-1,1]^{d_{\mcY}}}q_0(\my,\mx)\mrd\my
    = \int_{[-1,1]^{d_{\mcY}}} p_{\mcM}(\my,\mx)\mrd\my + 2^{d_{\mcY}}\cdot 2^{-d_{\mcY}} \left(1-\int_{[-1,1]^{d_{\mcY}}} p_{\mcM}(\mathbf u,\mx)\mrd\mathbf u\right)=1.
\]
Moreover, it holds
\[
\begin{aligned}
    |q_0(\my,\mx)-p_0(\my\mid\mx)|
    &\leq |p_{\mcM}(\my,\mx)-p_0(\my\mid\mx)|+2^{-d_{\mcY}}\int_{[-1,1]^{d_{\mcY}}} |p_{\mcM}(\mathbf u,\mx)-p_0(\mathbf u\mid\mx)|\mrd\mathbf u\\
    &\lesssim \mcM^{-\beta}.
\end{aligned}
\]
For sufficiently large $\mcM$, the last display and Assumption \ref{ass: bounded_density} imply
\[
    \frac{C_l}{2}\leq q_0(\my,\mx)\leq2C_u.
\]
\indent To prove the asserted representation, define for $|\boldsymbol\gamma|_1\leq r$,
\[
    A_{\mathbf n,\boldsymbol\gamma}
    := \sum_{\mathbf m}\sum_{|\boldsymbol\alpha|_1+|\boldsymbol\gamma|_1\leq r} c_{\mathbf m,\mathbf n,\boldsymbol\alpha,\boldsymbol\gamma}
    \int_{I_{\mathbf m}} (\mathbf u-\my_{\mathbf m})^{\boldsymbol\alpha}\mrd\mathbf u.
\]
Fubini's theorem and the finite sums yield
\[
\begin{aligned}
    \int p_{\mcM}(\mathbf u,\mx)\mrd\mathbf u
    &=
    \sum_{\mathbf n}
    \chi_{\mathbf n,\delta_{\mathrm{ind}}}(\mx)
    \sum_{|\boldsymbol\gamma|_1\leq r}
    A_{\mathbf n,\boldsymbol\gamma}
    (\mx-\mx_{\mathbf n})^{\boldsymbol\gamma}.
\end{aligned}
\]
Set $b_{\mathbf n,\boldsymbol\gamma}:=2^{-d_{\mcY}}\left(\mathbf I_{\{\boldsymbol\gamma=\mathbf0\}}-A_{\mathbf n,\boldsymbol\gamma}\right)$. Since $\sum_{\mathbf m}\mathbf I_{I_{\mathbf m}}(\my)=\sum_{\mathbf n}\chi_{\mathbf n,\delta_{\mathrm{ind}}}(\mx)=1$, we obtain
\[
    2^{-d_{\mcY}}\left(1-\int p_{\mcM}(\mathbf u,\mx)\mrd\mathbf u\right)
    = \sum_{\mathbf m,\mathbf n}
    \mathbf I_{I_{\mathbf m}}(\my)
    \chi_{\mathbf n,\delta_{\mathrm{ind}}}(\mx)
    \sum_{|\boldsymbol\gamma|_1\leq r}
    b_{\mathbf n,\boldsymbol\gamma}
    (\my-\my_{\mathbf m})^{\mathbf0}
    (\mx-\mx_{\mathbf n})^{\boldsymbol\gamma}.
\]
Define $\wt c_{\mm,\mn,\boldsymbol\alpha,\boldsymbol\gamma}:=c_{\mm,\mn,\boldsymbol\alpha,\boldsymbol\gamma}+\mI_{\{\boldsymbol\alpha=\mathbf0\}} b_{\mn,\boldsymbol\gamma}$, then
\[
\begin{aligned}
    q_0(\my,\mx)
    &= p_\mcM(\my,\mx)+2^{-d_\mcY}\left(1-\int_{[-1,1]^{d_\mcY}}p_\mcM(\mathbf u,\mx)\mrd\mathbf u\right)\\
    &= \sum_{\mathbf m,\mathbf n}\mathbf I_{I_{\mathbf m}}(\my)\chi_{\mathbf n,\delta_{\mathrm{ind}}}(\mx) \sum_{|\boldsymbol\alpha|_1+|\boldsymbol\gamma|_1\leq r}
    c_{\mm,\mn,\boldsymbol\alpha,\boldsymbol\gamma}
    (\my-\my_{\mm})^{\boldsymbol\alpha}
    (\mx-\mx_{\mn})^{\boldsymbol\gamma}\\
    &~~~~~~~~ + \sum_{\mathbf m,\mathbf n}\mathbf I_{I_{\mathbf m}}(\my)\chi_{\mathbf n,\delta_{\mathrm{ind}}}(\mx) \sum_{|\boldsymbol\gamma|_1\leq r}
    b_{\mn,\boldsymbol\gamma}(\my-\my_{\mm})^{\mathbf{0}}
    (\mx-\mx_{\mn})^{\boldsymbol\gamma}\\
    &= \sum_{\mathbf m,\mathbf n}\mathbf I_{I_{\mathbf m}}(\my)\chi_{\mathbf n,\delta_{\mathrm{ind}}}(\mx) \sum_{|\boldsymbol\alpha|_1+|\boldsymbol\gamma|_1\leq r}
    \wt c_{\mathbf m,\mathbf n,\boldsymbol\alpha,\boldsymbol\gamma}
    (\my-\my_{\mathbf m})^{\boldsymbol\alpha}
    (\mx-\mx_{\mathbf n})^{\boldsymbol\gamma},
\end{aligned}
\]
which is exactly of the same local-polynomial and soft-gate form as $p_\mcM$. Furthermore,
\[
\begin{aligned}
    |A_{\mathbf n,\boldsymbol\gamma}|
    &\leq
    R\sum_{|\boldsymbol\alpha|_1+|\boldsymbol\gamma|_1\leq r}
    \sum_{\mathbf m}
    \int_{I_{\mathbf m}}
    |\mathbf u-\my_{\mathbf m}|^{\boldsymbol\alpha}\mrd\mathbf u\\
    &\leq
    R\sum_{|\boldsymbol\alpha|_1+|\boldsymbol\gamma|_1\leq r}
    \sum_{\mathbf m}
    \int_{I_{\mathbf m}}
    \|\mathbf u-\my_{\mathbf m}\|^{|\boldsymbol\alpha|_1}\mrd\mathbf u\\
    &\leq
    R\sum_{|\boldsymbol\alpha|_1+|\boldsymbol\gamma|_1\leq r}
    \left(\frac{2\sqrt{d_{\mcY}}}{\mcM}\right)^{|\boldsymbol\alpha|_1}
    \sum_{\mathbf m}|I_{\mathbf m}|\\
    &\le 2^{d_{\mcY}}R
    \sum_{|\boldsymbol\alpha|_1+|\boldsymbol\gamma|_1\leq r}
    \left(2\sqrt{d_{\mcY}}\right)^{|\boldsymbol\alpha|_1}
    \leq C,
\end{aligned}
\]
where the constant $C$ depends only on $R,r$ and $d_\mcY$. This implies immediately
\[
    |b_{\mathbf n,\boldsymbol\gamma}|\leq C,\qquad
    |\wt c_{\mathbf m,\mathbf n,\boldsymbol\alpha,\boldsymbol\gamma}| \leq C.
\]
Lastly, the number of summands is of order
\[
    \mcO\left(\mcM^{d_{\mcY}}\mcM^{d_{\mcX}}\cdot\binom{r+d_\mcX+d_\mcY}{d_\mcX+d_\mcY}\right)
    = \mcO(\mcM^{d_{\mcX}+d_{\mcY}}).
\]
The proof is complete.
\end{proof}
Define
\[
    q_t(\my,\mx)
    := \int_{[-1,1]^{d_{\mcY}}}q_0(\my_1,\mx)\cdot
    \frac{1}{\sigma_t^{d_{\mcY}}(2\pi)^{d_{\mcY}/2}}\exp\left(-\frac{\|\my-m_t\my_1\|^2}{2\sigma_t^2}\right) \mrd\my_1.
\]
In the following sections, we first construct neural networks to approximate $\nabla\log q_t(\my,\mx)$, and then bound the gap between $\nabla\log q_t(\my,\mx)$ and $\nabla\log p_t(\my\mid\mx)$.
\subsection{Approximating $q_t(\my|\mx)$ and $\sigma_t\nabla q_t(\my|\mx)$ by local polynomial integrals}\label{subsec: c2}
In this section, we approximate $q_t(\my|\mx)$ and $\sigma_t\nabla q_t(\my|\mx)$ via local polynomial integrals. Absorbing fixed powers of $2$ into the coefficients, we can rewrite the representation in Lemma \ref{lem: approximate_pdata1} as
\[
\begin{aligned}
    q_0(\my_1,\mx)
    =&\sum_{\mathbf m\in[\mcM]^{d_{\mcY}},\mathbf n\in[\mcM]^{d_{\mcX}}}
    \sum_{|\boldsymbol\alpha|_1+|\boldsymbol\gamma|_1<\beta}
    c_{\mathbf m,\mathbf n,\boldsymbol\alpha,\boldsymbol\gamma}
    \revxsoft
    \prod_{j=1}^{d_{\mcX}}
    \left(\frac{x_j+1}{2}-\frac{n_j}{\mcM}\right)^{\gamma_j}\\
    &\quad\cdot
    \prod_{i=1}^{d_{\mcY}}
    \mathbf I_{\{y_{1,i}\in I_{m_i}\}}
    \left(\frac{y_{1,i}+1}{2}-\frac{m_i}{\mcM}\right)^{\alpha_i},
\end{aligned}
\]
where $[\mcM]:=\{1,\ldots,\mcM\}$ and $|c_{\mathbf m,\mathbf n,\boldsymbol\alpha,\boldsymbol\gamma}|\leq C$. Define
\[
    g_1(t,\my,\mx):=q_t(\my,\mx)
    = \int_{[-1,1]^{d_{\mcY}}}q_0(\my_1,\mx)\cdot
    \frac{1}{\sigma_t^{d_{\mcY}}(2\pi)^{d_{\mcY}/2}}\exp\left(-\frac{\|\my-m_t\my_1\|^2}{2\sigma_t^2}\right) \mrd\my_1.
\]
For $0<\epsilon<1$, put
\[
    A_{\my}:=\prod_{i=1}^{d_{\mcY}}
    \left[ \frac{y_i-C\sigma_t\sqrt{\log\epsilon^{-1}}}{m_t}, \frac{y_i+C\sigma_t\sqrt{\log\epsilon^{-1}}}{m_t} \right] \cap [-1,1]^{d_{\mcY}},
\]
where we enlarge $C$, if necessary, to make it also no smaller than the constant in Lemma~\ref{lem: integral_clipping}. Then we define $g_2(t,\my,\mx)$ as
\[
    g_2(t,\my,\mx):=\int_{A_\my}q_0(\my_1,\mx)\cdot
    \frac{1}{\sigma_t^{d_{\mcY}}(2\pi)^{d_{\mcY}/2}}\exp\left(-\frac{\|\my-m_t\my_1\|^2}{2\sigma_t^2}\right) \mrd\my_1.
\]
Since $q_0$ is also upper bounded as $q_0\leq 2C_u$, replacing $p_0$ in Lemma~\ref{lem: integral_clipping} with $q_0$ yields
\[
    |g_2(t,\my,\mx)-g_1(t,\my,\mx)|\lesssim\epsilon.
\]
\indent Note that $g_2(t,\my,\mx)$ includes an integral involving the exponential function, which is challenging to handle.  To address this difficulty, we use polynomials to approximate the exponential function. 
For any $1\leq i \leq d_\mcY$ and $\my_1\in A_{\my}$, we know that $\frac{|y_i - m_ty_{1,i}|}{\sigma_t} \leq C\sqrt{\log\epsilon^{-1}}$. Thus, by Taylor expansions, we have
$$
    \left|
    \exp\left(-\frac{(y_i - m_ty_{1,i})^2}{2\sigma_t^2}\right) - \sum_{l=0}^{k-1}\frac{1}{l!}\left(-\frac{(y_i-m_ty_{1,i})^2}{2\sigma_t^2}\right)^l
    \right| \leq \frac{C^{2k}\log^k\epsilon^{-1}}{k!2^k}, ~ \forall \my_1\in A_\my.
$$
By setting $k \geq \frac{3}{2}C^2 u\log\epsilon^{-1}$ and using the inequality $k! \geq (k/3)^k$ when $k \geq 3$, we have
$$
    \left|
    \exp\left(-\frac{(y_i - m_t y_{1,i})^2}{2\sigma_t^2}\right) - \sum_{l=0}^{k-1}\frac{1}{l!}\left(-\frac{(y_i-m_ty_{1,i})^2}{2\sigma_t^2}\right)^l
    \right| \leq \epsilon^{\frac{3}{2}C^2 u\log u}.
$$
Thus, we can set
$$
    u = \max\left\{e, \frac{2}{3C^2}\left(1 + \frac{\log d_\mcY}{\log \epsilon^{-1}}\right)\right\}
$$
such that
$$
    \epsilon^{\frac{3}{2}C^2u\log u} \leq \frac{\epsilon}{d_\mcY},
$$
where $k = \mathcal{O}\left(\log \epsilon^{-1}\right)$. 
By multiplying over the $d_\mcY$ dimensions indexed by $i$, we have
$$
    \left|
    \exp\left(-\frac{\Vert\my - m_t\my_1\Vert^2}{2\sigma_t^2}\right) - \prod_{i=1}^{d_\mcY}\sum_{l=0}^{k-1}\frac{1}{l!}\left(-\frac{(y_i-m_t y_{1,i})^2}{2\sigma_t^2}\right)^l
    \right| \leq d_\mcY\left(1 + \frac{\epsilon}{d_\mcY}\right)^{d_\mcY-1}\cdot\frac{\epsilon}{d_\mcY} \lesssim \epsilon.
$$
Therefore, we only need to approximate
\[
    g_3(t,\my,\mx):=\int_{A_\my}q_0(\my_1,\mx)\cdot
    \frac{1}{\sigma_t^{d_{\mcY}}(2\pi)^{d_{\mcY}/2}}
    \prod_{i=1}^{d_\mcY}\sum_{l=0}^{k-1}\frac{1}{l!}\left(-\frac{(y_i-m_t y_{1,i})^2}{2\sigma_t^2}\right)^l \mrd\my_1.
\]
Notice the difference between $g_2(t,\my,\mx)$ and $g_3(t,\my,\mx)$ can be bounded as
$$
\begin{aligned}
    |g_3(t,\my,\mx) - g_2(t,\my,\mx)|
    &\lesssim \epsilon \cdot \frac{1}{\sigma_t^{d_\mcY}(2\pi)^{d_\mcY/2}} \int_{A_{\my}} q_0(\my_1,\mx) \mrd\my_1 \\
    & \lesssim \epsilon \cdot \frac{1}{\sigma_t^{d_\mcY}(2\pi)^{d_\mcY/2}} \int_{A_{\my}} \left(p_0(\my_1|\mx) + {\mcM}^{-\beta}\right) \mrd\my_1 \\
    & \lesssim \epsilon \cdot (C_u + \mcM^{-\beta}) \left(\frac{2C\sigma_t\sqrt{\log\epsilon^{-1}}}{m_t}\right)^{d_\mcY} \cdot \frac{1}{\sigma_t^{d_\mcY}} \\
    & \lesssim \frac{1}{m_t^{d_\mcY}} \cdot \epsilon \log^{\frac{d_\mcY}{2}} \epsilon^{-1}.
\end{aligned}
$$
Also, it holds that
$$
\begin{aligned}
    |g_3(t,\my,\mx) - g_2(t,\my,\mx)|
    &\lesssim \epsilon \cdot \frac{1}{\sigma_t^{d_\mcY}(2\pi)^{d_\mcY/2}} \int_{A_{\my}} q_0(\my_1,\mx) \mrd\my_1 \\
    & \lesssim \epsilon \cdot \frac{1}{\sigma_t^{d_\mcY}(2\pi)^{d_\mcY/2}} \int_{A_{\my}} \left(p_0(\my_1|\mx) + {\mcM}^{-\beta}\right) \mrd\my_1 \\
    & \lesssim \epsilon \cdot (C_u + \mcM^{-\beta}) \cdot 2^{d_\mcY} \cdot \frac{1}{\sigma_t^{d_\mcY}}\\
    & \lesssim \frac{1}{\sigma_t^{d_\mcY}} \cdot \epsilon \log^{\frac{d_\mcY}{2}} \epsilon^{-1}.
\end{aligned}
$$
Since $\min\{1/m_t^{d_\mcY}, 1/\sigma_t^{d_\mcY}\}$ is bounded by $\mcO(1)$, we have
$$
\begin{aligned}
    |g_3(t,\my,\mx) - g_2(t,\my,\mx)| & \lesssim \min\left\{\frac{1}{m_t^{d_\mcY}}, \frac{1}{\sigma_t^{d_\mcY}}\right\} \cdot \epsilon\log^{\frac{d_\mcY}{2}}\epsilon^{-1} \\
    & \lesssim \epsilon\log^{\frac{d_\mcY}{2}}\epsilon^{-1}.
\end{aligned}
$$
This implies that we only need to approximate $g_3(t,\my,\mx)$ using a sufficiently accurate ReLU neural network.\\
\indent Substituting the local-polynomial representation of $q_0$ into $g_3(t,\my,\mx)$ yields
\begin{equation}\label{eq: diffused_local_poly1}
\begin{aligned}
    g_3(t,\my,\mx)
    =&\sum_{\mm\in[\mcM]^{d_\mcY},\mn\in[\mcM]^{d_\mcX}}
    \sum_{|\boldsymbol\alpha|_1+|\boldsymbol\gamma|_1<\beta}
    c_{\mathbf m,\mathbf n,\boldsymbol\alpha,\boldsymbol\gamma}
    \revxsoft
    \prod_{j=1}^{d_{\mcX}}
    \left(\frac{x_j+1}{2}-\frac{n_j}{\mcM}\right)^{\gamma_j}\\
    &\quad\cdot\prod_{i=1}^{d_{\mcY}}
    \frac1{\sigma_t(2\pi)^{1/2}}
    \int_{a_{i,\my}\cap I_{m_i}}
    \left(\frac{y_{1,i}+1}{2}-\frac{m_i}{\mcM}\right)^{\alpha_i}\cdot
    \sum_{l=0}^{k-1}\frac1{l!}
    \left(-\frac{(y_i-m_ty_{1,i})^2}{2\sigma_t^2}\right)^l
    \mrd y_{1,i},
\end{aligned}
\end{equation}
where
\[
    a_{i,\my}:=
    \left[
    \frac{y_i-C\sigma_t\sqrt{\log\epsilon^{-1}}}{m_t},
    \frac{y_i+C\sigma_t\sqrt{\log\epsilon^{-1}}}{m_t}
    \right].
\]
Similarly, we can define $\mh_1(t,\my,\mx)$ as $\sigma_t\nabla q_t(\my|\mx)$:
\[
\begin{aligned}
    \mh_1(t,\my,\mx)
    := \sigma_t\nabla q_t(\my,\mx)
    = \int_{[-1,1]^{d_\mcY}} q_0(\my_1,\mx)
    \frac{m_t\my_1-\my}{\sigma_t^{d_{\mcY}+1}(2\pi)^{d_{\mcY}/2}}
    \exp\left(-\frac{\|\my-m_t\my_1\|^2}{2\sigma_t^2}\right) \mrd\my_1.
\end{aligned}
\]
Also, we define $\mh_2(t,\my,\mx)$ and $\mh_3(t,\my,\mx)$ as follows:
\begin{gather*}
    \mh_2(t,\my,\mx) :=  \int_{A_{\my}} q_0(\my_1,\mx)\cdot
    \frac{m_t\my_1-\my}{\sigma_t^{d_\mcY+1}(2\pi)^{d_\mcY/2}}\exp\left(-\frac{\Vert\my - m_t\my_1\Vert^2}{2\sigma_t^2}\right) \mrd\my_1,\\
    \mh_3(t,\my,\mx) := \int_{A_{\my}} q_0(\my_1,\mx)\cdot
    \frac{m_t\my_1-\my}{\sigma_t^{d_\mcY + 1}(2\pi)^{d_\mcY/2}}
    \prod_{i=1}^{d_\mcY}\sum_{l=0}^{k-1}\frac{1}{l!}\left(-\frac{(y_i-m_t y_{1,i})^2}{2\sigma_t^2}\right)^l \mrd\my_1.
\end{gather*}
Then,  we have
$$
    \Vert \mh_2(t,\my,\mx) - \mh_1(t,\my,\mx)\Vert \lesssim \epsilon,
$$
and
$$
    \Vert \mh_3(t,\my,\mx) - \mh_2(t,\my,\mx)\Vert \lesssim \epsilon\log^{\frac{d_\mcY+1}{2}}\epsilon^{-1}.
$$
The $j$-th component of $\mh_3(t,\my,\mx)$ can be expressed as the combination of the following functional forms
\begin{equation}\label{eq: diffused_local_poly2}
\begin{aligned}
        \relax[\mh_3]_j
        =&\sum_{\mm\in[\mcM]^{d_\mcY},\mn\in[\mcM]^{d_\mcX}}
        \sum_{|\boldsymbol\alpha|_1+|\boldsymbol\gamma|_1<\beta}
        c_{\mathbf m,\mathbf n,\boldsymbol\alpha,\boldsymbol\gamma}
        \revxsoft
        \prod_{a=1}^{d_{\mcX}}
        \left(\frac{x_a+1}{2}-\frac{n_a}{\mcM}\right)^{\gamma_a}\\
        &\quad\cdot
        \prod_{i\ne j}^{d_\mcY}\frac1{\sigma_t(2\pi)^{1/2}}
        \int_{a_{i,\my}\cap I_{m_i}}
        \left(\frac{y_{1,i}+1}{2}-\frac{m_i}{\mcM}\right)^{\alpha_i}
        \sum_{l=0}^{k-1}\frac1{l!}
        \left(-\frac{(y_i-m_ty_{1,i})^2}{2\sigma_t^2}\right)^l\mrd y_{1,i}\\
        &\quad\cdot\frac1{\sigma_t(2\pi)^{1/2}}
        \int_{a_{j,\my}\cap I_{m_j}}
        \left(\frac{y_{1,j}+1}{2}-\frac{m_j}{\mcM}\right)^{\alpha_j}
        \left(\frac{m_ty_{1,j}-y_j}{\sigma_t}\right)
        \sum_{l=0}^{k-1}\frac1{l!}
        \left(-\frac{(y_j-m_ty_{1,j})^2}{2\sigma_t^2}\right)^l\mrd y_{1,j},
\end{aligned}
\end{equation}

\subsection{Approximating the local polynomial integrals via ReLU neural networks}\label{subsec: c3}
In this section, we use ReLU neural networks to approximate the local polynomial integrals. Specifically, we focus on approximation of \eqref{eq: diffused_local_poly1} and \eqref{eq: diffused_local_poly2} under the restriction $\Vert\my\Vert_{\infty} \leq C_0$, where $C_0 > 0$ is a constant.
For convenience, we define
$$
f(t,y,m,\alpha,l):= \frac{1}{\sigma_t(2\pi)^{1/2}} \int_{a_{y} \cap I_m} \mathbf{I}_{\{|y_1| \leq 1\}}
\left(\frac{y_1 + 1}{2} - \frac{m}{\mcM}\right)^{\alpha} 
\frac{1}{l!}\left(-\frac{(y-m_t y_1)^2}{2\sigma_t^2}\right)^l \mrd y_1,
$$
where $a_y = \Big[\frac{y - C\sigma_t\sqrt{\log\epsilon^{-1}}}{m_t}, \frac{y + C\sigma_t\sqrt{\log\epsilon^{-1}}}{m_t} \Big]$. 
Then,
\eqref{eq: diffused_local_poly1} can be expressed as
\begin{equation*}
\begin{aligned}
    \eqref{eq: diffused_local_poly1} = \sum_{\mathbf{m}\in[\mcM]^{d_\mcY}, \mathbf{n}\in[\mcM]^{d_\mcX}}\sum_{\Vert\boldsymbol{\alpha}\Vert_1 + \Vert\boldsymbol{\gamma}\Vert_1 < \beta}
    & c_{\mathbf{m},\mathbf{n},\boldsymbol{\alpha},\boldsymbol{\gamma}}
    \revxsoft
    \prod_{i=1}^{d_\mcX}\left(\frac{x_i + 1}{2} - \frac{n_i}{\mcM}\right)^{\gamma_i}\\
    &~~~~~~\cdot
    \prod_{i=1}^{d_\mcY}\sum_{l=0}^{k-1}f(t,y_i,m_i,\alpha_i,l).
\end{aligned}
\end{equation*}
We first use the ReLU neural network to approximate $f(t,y,m,\alpha,l)$ for $|y|\leq C_0$.
\begin{lemma}\label{lem: approximate_f}
Given $\mcM \gg 1$, $C_0 > 0$, let $T = \mcM^{-C_T}$, where $C_T > 0$ is a constant. For any $\epsilon_0 > 0$, $m \leq \mcM$, $\alpha < \beta$, and $l \leq k - 1$, there exists a ReLU neural network $\mathrm{b}_f \in \mrN(L, M, J, \kappa)$ with
$$
\begin{aligned}
&L = \mathcal{O}\left(\log^2\epsilon^{-1}_0 + \log^2 \mcM + \log^2C_0 + \log^4 \epsilon^{-1}\right),\\ &M = \mathcal{O}\left(\log^3\epsilon^{-1}_0 + \log^3 \mcM + \log^3C_0 + \log^6 \epsilon^{-1}\right), \\
&J = \mathcal{O}\left(\log^4\epsilon^{-1}_0 + \log^4 \mcM + \log^4C_0 + \log^8 \epsilon^{-1}\right), \\ &\kappa = \exp\left(\mathcal{O}\left(\log^2\epsilon^{-1}_0 + \log^2 \mcM + \log^2C_0 + \log^4 \epsilon^{-1} \right)\right).
\end{aligned}
$$
such that 
$$
|\mathrm{b}_f(t, y, m, \alpha, l) - f(t, y, m, \alpha, l)| \lesssim \epsilon_0, ~ y\in[-C_0, C_0], ~ t\in[\mcM^{-C_T}, 1-\mcM^{-C_T}].
$$
\end{lemma}

\begin{proof}
By the definition of $f(t, y,m,\alpha,l)$, we take the transformation $z = \frac{y - m_t y_1}{\sigma_t}$, then
$$
\begin{aligned}
f(t,y,m,\alpha,l) & = \frac{1}{l!(2\pi)^{1/2}m_t}\int_{D}\sum_{j=0}^{\alpha} C_{\alpha}^{j}\left(\frac{y + m_t}{2m_t} - \frac{m}{\mcM}\right)^{\alpha-j}\left(-\frac{\sigma_t z}{2m_t}\right)^j\left(-\frac{z^2}{2}\right)^l \mrd z \\
& = \frac{1}{l!(2\pi)^{1/2}m_t^{\alpha+1}} \sum_{j=0}^{\alpha} C_{\alpha}^j \left(\frac{y + m_t}{2} - \frac{m\cdot m_t}{\mcM}\right)^{\alpha-j}  \frac{1}{(-2)^{j+l}}\cdot\sigma_t^j \int_D z^{j + 2l} \mrd z \\
& = \frac{1}{l!(2\pi)^{1/2}m_t^{\alpha+1}} \sum_{j=0}^{\alpha} C_{\alpha}^j \left(y + m_t - \frac{2m\cdot m_t}{\mcM}\right)^{\alpha-j}  \frac{(-1)^{j+l}}{2^{\alpha+l}}\cdot\sigma_t^j \frac{D_U^{j + 2l + 1}(y) - D_L^{j + 2l + 1}(y)}{j + 2l + 1},
\end{aligned}
$$
where $D = [D_L(t,y), D_U(t,y)]$ with
$$
D_L(t,y) = \mathrm{b}_{\mathrm{clip}}\left(\frac{1}{\sigma_t}\cdot\mathrm{b}_{\mathrm{clip}}\left(y - m_t\left(\frac{2m}{\mcM} - 1\right), y-m_t, y + m_t \right), -C\sqrt{\log\epsilon^{-1}}, C\sqrt{\log\epsilon^{-1}} \right),
$$
and 
$$
D_U(t,y) = \mathrm{b}_{\mathrm{clip}}\left(\frac{1}{\sigma_t}\cdot\mathrm{b}_{\mathrm{clip}}\left(y - m_t\left(\frac{2(m-1)}{\mcM} - 1\right), y-m_t, y + m_t \right), -C\sqrt{\log\epsilon^{-1}}, C\sqrt{\log\epsilon^{-1}} \right).
$$
Thus, we only need to approximate the function of the form
$$
f_{m,j,\alpha,l}(t, y) = \frac{1}{m_t^{\alpha+1}} \cdot \left(y + m_t - \frac{2m\cdot m_t}{\mcM} \right) ^{\alpha-j} \cdot \sigma_t^j \cdot \left(D_{U}^{j + 2l + 1}(t,y) - D_L^{j + 2l + 1}(t,y)\right).
$$
Throughout the constructions in Appendix~\ref{sec:ae}, a product with no factors is implemented by the constant network $1$, and a product with one factor is implemented by the identity network. These two cases are exact, and Lemma~\ref{lem: product} is invoked only for products with at least two factors. The exact constant and identity networks require only a bounded number of layers and parameters, so they satisfy the network size bounds derived below.

We first consider the approximation of $D_U^{j+2l+1}(t,y)$. We can choose neural networks $\mathrm{b}_{m,1},\cdots,\mathrm{b}_{m,4}$ with $L = \mcO(1)$, $M = \mcO(1)$, $J = \mcO(1)$ and $\kappa = \mcO(1)$, such that
$$
\mathrm{b}_{m,1}(t,y) = y - m_t\left(\frac{2(m-1)}{\mcM}-1\right), ~~ \mathrm{b}_{m,2}(t,y) = y - m_t\left(\frac{2m}{\mcM}-1\right),
$$
$$
\mathrm{b}_{m,3}(t,y) = y - m_t, ~~ \mathrm{b}_{m,4}(t,y) = y + m_t.
$$
Then we can construct a ReLU neural network
$$
\begin{aligned}
\mathrm{b}_{m,j,l}^{(1)}(t,y) & = \mathrm{b}_{\mathrm{prod},1}(\underbrace{\cdot, \cdots, \cdot}_{j + 2l + 1 ~ \text{times}}) \circ \mathrm{b}_{\mathrm{clip}}(\cdot, -C\sqrt{\log\epsilon^{-1}}, C\sqrt{\log\epsilon^{-1}}) \\
 & \circ \mathrm{b}_{\mathrm{prod},2}\left(
\mathrm{b}_{\mathrm{clip}}\left(\mathrm{b}_{m,1}(t,y), \mathrm{b}_{m,3}(t,y), \mathrm{b}_{m,4}(t,y)\right), 
\mathrm{b}_{\mathrm{rec}}(\cdot) \circ \mathrm{b}_{\mathrm{root}}(\cdot) \circ \mathrm{b}_{\mathrm{prod},3}(t,2-t)
\right).
\end{aligned}
$$
$D_L^{j + 2l + 1}(t,y)$ 
is  approximated by 
$$
\begin{aligned}
\mathrm{b}_{m,j,l}^{(2)}(t,y) & = \mathrm{b}_{\mathrm{prod},1}(\underbrace{\cdot, \cdots, \cdot}_{j + 2l + 1 ~ \text{times}}) \circ \mathrm{b}_{\mathrm{clip}}(\cdot, -C\sqrt{\log\epsilon^{-1}}, C\sqrt{\log\epsilon^{-1}}) \\
 & \circ \mathrm{b}_{\mathrm{prod},2}\left(
\mathrm{b}_{\mathrm{clip}}\left(\mathrm{b}_{m,2}(t,y), \mathrm{b}_{m,3}(t,y), \mathrm{b}_{m,4}(t,y)\right), 
\mathrm{b}_{\mathrm{rec}}(\cdot) \circ \mathrm{b}_{\mathrm{root}}(\cdot) \circ \mathrm{b}_{\mathrm{prod},3}(t,2-t)
\right).
\end{aligned}
$$
Therefore, $D_U^{j + 2l + 1}(t,y) - D_L^{j + 2l + 1}(t,y)$ can be approximated by
$$
\mathrm{b}_{m,j,l}^{(3)}(t,y) = \mathrm{b}_{m,j,l}^{(1)}(t,y) - \mathrm{b}_{m,j,l}^{(2)}(t,y).
$$
Next, we approximate $\left(y+m_t-\frac{2m m_t}{\mcM}\right)^{\alpha-j}
=(\mathrm b_{m,2}(t,y))^{\alpha-j}$. Since $\left|y+m_t-\frac{2m m_t}{\mcM}\right|\le C_0+3$, we apply Lemma~\ref{lem: product} with $C=C_0+3$ and
$d=\alpha-j$ when $\alpha-j\ge2$. With the exact conventions above for $\alpha-j=0,1$, define
\[
    \mathrm b_{m,j,\alpha}^{(4)}(t,y)
    = \mathrm b_{\mathrm{prod},4}
    (\underbrace{\cdot,\ldots,\cdot}_{\alpha-j\ \mathrm{times}})
    \circ\mathrm b_{m,2}(t,y).
\]
In particular, $\mathrm b_{m,j,\alpha}^{(4)}\equiv1$ if $\alpha=j$, and $\mathrm b_{m,j,\alpha}^{(4)}=\mathrm b_{m,2}$ if $\alpha-j=1$. For $\sigma_t^j$, define
\[
    \mathrm b_j^{(5)}(t)
    = \mathrm b_{\mathrm{prod},5}
    (\underbrace{\cdot,\ldots,\cdot}_{j\ \mathrm{times}})
    \circ\mathrm b_{\mathrm{root}}(\cdot)
    \circ\mathrm b_{\mathrm{prod},3}(t,2-t),
\]
where $\mathrm b_0^{(5)}\equiv1$, and the outer product network is the identity when $j=1$. For $\tfrac{1}{m_t^{\alpha+1}}$, define
\[
    \mathrm b_\alpha^{(6)}(t)
    := \mathrm b_{\mathrm{prod},6}
    (\underbrace{\cdot,\ldots,\cdot}_{\alpha+1\ \mathrm{times}})
    \circ\mathrm b_{\mathrm{rec}}(1-t),
\]
where the outer product network is the identity when $\alpha=0$. Finally, we can choose a ReLU neural network 
$$
\mathrm{b}_{m,j,\alpha,l}^{(7)}(t,y) = \mathrm{b}_{\mathrm{prod},7}(\mathrm{b}_{m,j,l}^{(3)}(t,y), \mathrm{b}_{m,j,\alpha}^{(4)}(t,y), \mathrm{b}_{j}^{(5)}(t), \mathrm{b}_{\alpha}^{(6)}(t))
$$
to approximate $f_{m,j,\alpha,l}(t,y)$.

Next, we derive the error bound between $\mathrm{b}_{m,j,\alpha,l}^{(7)}(t,y)$ and $f_{m,j,\alpha,l}(t,y)$. For convenience, we denote $\epsilon^{(i)} (i=1,\cdots,7)$ as the approximation error of $\mathrm{b}_{m,j,l}^{(1)}, \mathrm{b}_{m,j,l}^{(2)}, \mathrm{b}_{m,j,l}^{(3)}, \mathrm{b}_{m,j,\alpha}^{(4)},\mathrm{b}_j^{(5)}, \mathrm{b}_{\alpha}^{(6)}$, $\mathrm{b}_{m,j,\alpha,l}^{(7)}$ and denote $\epsilon_{\mathrm{prod},i} (i=1,\cdots,7)$ as the approximation error of $\mathrm{b}_{\mathrm{prod},i} (i=1,\cdots,7)$. We also denote $\epsilon_{\mathrm{root}}$ and $\epsilon_{\mathrm{rec}}$ as the approximation error of $\mathrm{b}_{\mathrm{root}}$ and $\mathrm{b}_{\mathrm{rec}}$ respectively. Since $\left|y + m_t - \frac{2m\cdot m_t}{\mcM}\right|^{\alpha-j} \leq \left(C_0 + 3\right)^{\alpha-j}$, $|D_L(t,y)| \leq C\sqrt{\log\epsilon^{-1}}$, $|D_U(t,y)| \leq C\sqrt{\log\epsilon^{-1}}$, $\sigma_t^j \leq 1$ and $\frac{1}{m_t^{\alpha+1}} \leq \frac{1}{T^{\alpha+1}}$, we set
\[
    C_1 = \max\left\{ (C_0+3)^{\alpha-j}, 1, 2(C\sqrt{\log\epsilon^{-1}})^{j+2l+1}, \frac{1}{T^{\alpha+1}} \right\}.
\]
By Lemma \ref{lem: product}, we have
$$
\begin{aligned}
\epsilon^{(7)} &= \max_{t,y}|\mathrm{b}_{m,j,\alpha,l}^{(7)}(t, y) - f_{m,j,\alpha,l}(t,y)| \\
&\leq \epsilon_{\mathrm{prod},7} + 4C_1^3\cdot \max\{\epsilon^{(3)}, \epsilon^{(4)}, \epsilon^{(5)}, \epsilon^{(6)}\}.
\end{aligned}
$$
By Lemmas \ref{lem: reciprocal}-\ref{lem: square_root},
we can bound $\epsilon^{(1)}$ and $\epsilon^{(2)}$. Let $C_2 = \max\left\{C_0 + 1, \frac{1}{\sqrt{T}}\right\}$, then we have
$$
\epsilon^{(1)} \leq \epsilon_{\mathrm{prod},1} + (j + 2l + 1)(C\sqrt{\log\epsilon^{-1}})^{j + 2l}\left[\epsilon_{\mathrm{prod},2} + 2C_2\left(\epsilon_{\mathrm{rec}} + \frac{\epsilon_{\mathrm{root}} + \frac{\epsilon_{\mathrm{prod},3}}{\sqrt{\epsilon_{\mathrm{root}}}}}{\epsilon_{\mathrm{rec}}^2}\right)\right]
$$
and
$$
\epsilon^{(2)} \leq \epsilon_{\mathrm{prod},1} + (j + 2l + 1)(C\sqrt{\log\epsilon^{-1}})^{j + 2l}\left[\epsilon_{\mathrm{prod},2} + 2C_2\left(\epsilon_{\mathrm{rec}} + \frac{\epsilon_{\mathrm{root}} + \frac{\epsilon_{\mathrm{prod},3}}{\sqrt{\epsilon_{\mathrm{root}}}}}{\epsilon_{\mathrm{rec}}^2}\right)\right].
$$
Since
$$
\epsilon^{(3)} \leq \epsilon^{(1)} + \epsilon^{(2)}, 
$$ 
we obtain
$$
\epsilon^{(3)} \leq 2\epsilon_{\mathrm{prod},1} + 2(j + 2l + 1)(C\sqrt{\log\epsilon^{-1}})^{j + 2l}\left[\epsilon_{\mathrm{prod},2} + 2C_2\left(\epsilon_{\mathrm{rec}} + \frac{\epsilon_{\mathrm{root}} + \frac{\epsilon_{\mathrm{prod},3}}{\sqrt{\epsilon_{\mathrm{root}}}}}{\epsilon_{\mathrm{rec}}^2}\right)\right].
$$
We can also bound $\epsilon^{(4)}, \epsilon^{(5)}$ and $\epsilon^{(6)}$ as
\[
    \epsilon^{(4)}\le\epsilon_{\mathrm{prod},4},
    \qquad
    \epsilon^{(5)}
    \le\epsilon_{\mathrm{prod},5}
    +j\left(
        \epsilon_{\mathrm{root}}
        +\frac{\epsilon_{\mathrm{prod},3}}
        {\sqrt{\epsilon_{\mathrm{root}}}}
    \right).
\]
and
$$
\epsilon^{(6)} \leq \epsilon_{\mathrm{prod},6} + \frac{\alpha + 1}{T^{\alpha}} \epsilon_{\mathrm{rec}}.
$$

It suffices to consider $0<\epsilon_0\le1$, since a network constructed with tolerance $1$ also satisfies any larger requested tolerance. To ensure that $\epsilon^{(7)} \leq \epsilon_0$, we take $\epsilon_{\mathrm{prod},7} = \frac{\epsilon_0}{2}$ and $\max\{\epsilon^{(3)}, \epsilon^{(4)}, \epsilon^{(5)}, \epsilon^{(6)}\} \leq \frac{\epsilon_0}{8C_1^3}=:\epsilon^*$. In detail, we take
$$
\epsilon_{\mathrm{prod},1} = \frac{\epsilon^*}{4}, \epsilon_{\mathrm{prod},2} = \frac{\epsilon^*}{8(j + 2l + 1)(C\sqrt{\log\epsilon^{-1}})^{j + 2l}}, \epsilon_{\mathrm{prod},4} = \epsilon^*, \epsilon_{\mathrm{prod},5} = \frac{\epsilon^*}{2}, \epsilon_{\mathrm{prod},6} = \frac{\epsilon^*}{2}.
$$
Moreover, we take
\[
    \epsilon_{\mathrm{prod},3} = \epsilon_{\mathrm{root}}^{\frac{3}{2}},\qquad
    \epsilon_{\mathrm{root}}
    = \min\left\{ \frac{\epsilon_{\mathrm{rec}}^3}{2},\frac{\epsilon^*}{4(j\vee1)} \right\}.
\]
$$
\epsilon_{\mathrm{rec}} = \min\left\{\frac{\epsilon^*}{32C_2(j + 2l + 1)(C\sqrt{\log\epsilon^{-1}})^{j + 2l}},\frac{\epsilon^*T^{\alpha}}{2(\alpha+1)} \right\}.
$$
These choices give
$\epsilon_{\mathrm{root}}+\epsilon_{\mathrm{prod},3}/\sqrt{\epsilon_{\mathrm{root}}}=2\epsilon_{\mathrm{root}}$ and $\epsilon_{\mathrm{rec}}+2\epsilon_{\mathrm{root}}/\epsilon_{\mathrm{rec}}^2\le2\epsilon_{\mathrm{rec}}$. Consequently, the preceding bound on $\epsilon^{(3)}$ yields $\epsilon^{(3)}\le\epsilon^*/2+\epsilon^*/4+\epsilon^*/4=\epsilon^*$. Also, $\epsilon^{(4)}\le\epsilon^*$ and $\epsilon^{(6)}\le\epsilon^*/2+\epsilon^*/2=\epsilon^*$. For $j\ge1$, we have
$\epsilon^{(5)}\le\epsilon^*/2+2j\epsilon_{\mathrm{root}}\le\epsilon^*$, whereas $\epsilon^{(5)}=0$ for $j=0$ by the exact constant construction. Thus
$\max\{\epsilon^{(3)},\epsilon^{(4)},\epsilon^{(5)},\epsilon^{(6)}\}\le\epsilon^*$. Note that $j \leq \alpha \leq C_{\beta} = \mcO(1)$, $l \leq k - 1 = \mcO(\log \epsilon^{-1})$.  Subsequently, we can obtain the network structures of $\mathrm{b}_{m,j,l}^{(1)}$, $\mathrm{b}_{m,j,l}^{(2)}$, $\mathrm{b}_{m,j,l}^{(3)}$, $\mathrm{b}_{m,j,\alpha}^{(4)}$,
$\mathrm{b}_j^{(5)}$, $\mathrm{b}_{\alpha}^{(6)}$, $\mathrm{b}_{m,j,\alpha,l}^{(7)}$. For $\mathrm{b}_{m,j,l}^{(1)}$, $\mathrm{b}_{m,j,l}^{(2)}$ and $\mathrm{b}_{m,j,l}^{(3)}$, we have
$$
\begin{aligned}
&L = \mathcal{O}\left(\log^2\epsilon^{-1}_0 + \log^2 \mcM + \log^2C_0 + \log^4 \epsilon^{-1}\right),\\ &M = \mathcal{O}\left(\log^3\epsilon^{-1}_0 + \log^3 \mcM + \log^3C_0 + \log^6 \epsilon^{-1}\right), \\
&J = \mathcal{O}\left(\log^4\epsilon^{-1}_0 + \log^4 \mcM + \log^4C_0 + \log^8 \epsilon^{-1}\right), \\ &\kappa = \exp\left(\mathcal{O}\left(\log^2\epsilon^{-1}_0 + \log^2 \mcM + \log^2C_0 + \log^4 \epsilon^{-1} \right)\right).
\end{aligned}
$$
For $\mathrm{b}_{m,j,\alpha}^{(4)}$, we have
$$
\begin{aligned}
L = \mathcal{O}\left(\log\epsilon^{-1}_0 + \log \mcM + \log C_0 + \log^2 \epsilon^{-1}\right), ~ M = \mathcal{O}(1), \\
J = \mathcal{O}\left(\log\epsilon^{-1}_0 + \log \mcM + \log C_0 + \log^2 \epsilon^{-1} \right), ~ \kappa = \exp\left(\mathcal{O}(\log C_0)\right).
\end{aligned}
$$
For $\mathrm{b}_{j}^{(5)}$, we have
$$
\begin{aligned}
&L = \mathcal{O}\left(\log^2\epsilon^{-1}_0 + \log^2 \mcM + \log^2C_0 + \log^4 \epsilon^{-1}\right),\\ &M = \mathcal{O}\left(\log^3\epsilon^{-1}_0 + \log^3 \mcM + \log^3C_0 + \log^6 \epsilon^{-1}\right), \\
&J = \mathcal{O}\left(\log^4\epsilon^{-1}_0 + \log^4 \mcM + \log^4 C_0 + \log^8 \epsilon^{-1}\right), \\ 
&\kappa = \exp\left(\mathcal{O}\left(\log\epsilon^{-1}_0 + \log \mcM + \log C_0 + \log^2 \epsilon^{-1} \right)\right).
\end{aligned}
$$
For $\mathrm{b}_{\alpha}^{(6)}$, we have
$$
\begin{aligned}
&L = \mathcal{O}\left(\log^2\epsilon^{-1}_0 + \log^2 \mcM + \log^2C_0 + \log^4 \epsilon^{-1}\right),\\ &M = \mathcal{O}\left(\log^3\epsilon^{-1}_0 + \log^3 \mcM + \log^3C_0 + \log^6 \epsilon^{-1}\right), \\
&J = \mathcal{O}\left(\log^4\epsilon^{-1}_0 + \log^4 \mcM + \log^4C_0 + \log^8 \epsilon^{-1}\right), \\ &\kappa = \exp\left(\mathcal{O}\left(\log^2\epsilon^{-1}_0 + \log^2 \mcM + \log^2C_0 + \log^4 \epsilon^{-1} \right)\right).
\end{aligned}
$$
Therefore, by Lemma \ref{lem: concatenation}, 
we finally obtain the network structure of $\mathrm{b}_{m,j,\alpha,l}^{(7)}$ as follows:
$$
\begin{aligned}
&L = \mathcal{O}\left(\log^2\epsilon^{-1}_0 + \log^2 \mcM + \log^2C_0 + \log^4 \epsilon^{-1}\right),\\ &M = \mathcal{O}\left(\log^3\epsilon^{-1}_0 + \log^3 \mcM + \log^3C_0 + \log^6 \epsilon^{-1}\right), \\
&J = \mathcal{O}\left(\log^4\epsilon^{-1}_0 + \log^4 \mcM + \log^4C_0 + \log^8 \epsilon^{-1}\right), \\ &\kappa = \exp\left(\mathcal{O}\left(\log^2\epsilon^{-1}_0 + \log^2 \mcM + \log^2C_0 + \log^4 \epsilon^{-1} \right)\right).
\end{aligned}
$$

Now, we consider the following ReLU network
$$
\mathrm{b}_{f}(t,y,m,\alpha,l) = \frac{1}{l!(2\pi)^{1/2}}\sum_{j=0}^{\alpha} C_{\alpha}^j \frac{(-1)^{j+l}}{2^{\alpha + l}}\frac{1}{j + 2l + 1} \mathrm{b}_{m,j,\alpha,l}^{(7)}(t,y).
$$
Then, we have
$$
\begin{aligned}
|\mathrm{b}_f(t,y,m,\alpha,l) - f(t,y,m,\alpha,l)| & \leq \frac{1}{l!(2\pi)^{1/2}}\left(\sum_{j=0}^{\alpha} C_{\alpha}^j\frac{1}{2^{\alpha + l}(j + 2l + 1)}
\right) \cdot \epsilon_0 \\
& \leq \frac{\epsilon_0}{l!2^l(2\pi)^{1/2}} \lesssim \epsilon_0.
\end{aligned}
$$
By Lemma \ref{lem: parallelization}, the parameters of the network $\mathrm{b}_f$ satisfy
$$
\begin{aligned}
&L = \mathcal{O}\left(\log^2\epsilon^{-1}_0 + \log^2 \mcM + \log^2C_0 + \log^4 \epsilon^{-1}\right),\\ &M = \mathcal{O}\left(\log^3\epsilon^{-1}_0 + \log^3 \mcM + \log^3C_0 + \log^6 \epsilon^{-1}\right), \\
&J = \mathcal{O}\left(\log^4\epsilon^{-1}_0 + \log^4 \mcM + \log^4C_0 + \log^8 \epsilon^{-1}\right), \\ &\kappa = \exp\left(\mathcal{O}\left(\log^2\epsilon^{-1}_0 + \log^2 \mcM + \log^2C_0 + \log^4 \epsilon^{-1} \right)\right).
\end{aligned}
$$
The proof is complete.
\end{proof}

With Lemma \ref{lem: approximate_f}, we can construct ReLU neural networks to approximate $g_3$ and $\mh_3$ expressed as \eqref{eq: diffused_local_poly1} and \eqref{eq: diffused_local_poly2}.
\begin{lemma}\label{lem: approximate_poly1}
Given $\mcM \gg 1$, $C_0 > 0$, let $T = \mcM^{-C_T}$, where $C_T > 0$ is a constant. For any $\epsilon_0 > 0$, there exists a ReLU neural network $\mathrm{b}_1\in\mathrm{NN}(L,M,J,\kappa)$ with
$$
L = \mathcal{O}\left(\log^2\epsilon^{-1}_0  + \log^2 \mcM + \log^2 C_0+ \log^4 \epsilon^{-1} \right), $$
$$
M = \mathcal{O}\left(\mcM^{d_\mcX + d_\mcY}\log \epsilon^{-1} (\log^3\epsilon^{-1}_0  + \log^3 \mcM + \log^3 C_0 + \log^6 \epsilon^{-1})\right),  $$
$$
J = \mathcal{O}\left(\mcM^{d_\mcX + d_\mcY}\log \epsilon^{-1} (\log^4\epsilon^{-1}_0 + \log^4 \mcM + \log^4 C_0 + \log^8 \epsilon^{-1}) \right), $$
$$
\kappa = \exp\left(\mathcal{O}\left(\log^2\epsilon^{-1}_0 + \log^2 \mcM + \log^2C_0 + \log^4 \epsilon^{-1} \right)\right).
$$
such that for the target $g_3(t,\my,\mx)$, uniformly over $\mx\in[-1,1]^{d_\mcX}$, $\my\in[-C_0,C_0]^{d_\mcY}$ and $t\in[\mcM^{-C_T},1-\mcM^{-C_T}]$, the following holds:
$$
\begin{aligned}
\Bigg|\mathrm{b}_1(t,\my,\mx)-& \sum_{\mathbf{m}\in[\mcM]^{d_\mcY}, \mathbf{n}\in[\mcM]^{d_\mcX}}\sum_{\Vert\boldsymbol{\alpha}\Vert_1 + \Vert\boldsymbol{\gamma}\Vert_1 < \beta} c_{\mathbf{m},\mathbf{n},\boldsymbol{\alpha},\boldsymbol{\gamma}}\\
& \cdot {\revxsoft}
\prod_{i=1}^{d_\mcX}\left(\frac{x_i + 1}{2} - \frac{n_i}{\mcM}\right)^{\gamma_i} \cdot\prod_{i=1}^{d_\mcY}\sum_{l=0}^{k-1}f(t,y_i,m_i,\alpha_i,l) \Bigg| \lesssim \epsilon_0.
\end{aligned}
$$
\end{lemma}

\begin{proof}
We first construct a ReLU neural network $\mathrm{b}_{\mathrm{sum}}\in\mathrm{NN}(L,M,J,\kappa)$, which  satisfies
$$
\mathrm{b}_{\mathrm{sum}}(t,y,m,\alpha) = \sum_{l=0}^{k-1}\mathrm{b}_f(t,y,m,\alpha,l),
$$
with network parameters
$$
L = \mathcal{O}\left(\log^2\epsilon^{-1}_f + \log^2\mcM + \log^2 C_0 + \log^4 \epsilon^{-1} \right), 
$$
$$
M = \mathcal{O}\left(\log \epsilon^{-1}(\log^3\epsilon^{-1}_f + \log^3 \mcM + \log^3C_0 +  \log^6 \epsilon^{-1})\right), 
$$
$$
J = \mathcal{O}\left(\log \epsilon^{-1}(\log^4\epsilon^{-1}_f + \log^4 \mcM + \log^4C_0 + \log^8 \epsilon^{-1}) \right), 
$$
$$
\kappa = \exp\left(\mathcal{O}\left(\log^2\epsilon^{-1}_f + \log^2 \mcM + \log^2C_0 + \log^4 \epsilon^{-1} \right)\right),
$$
where $\epsilon_f$ is the approximation error of $\mathrm{b}_f$. 

Next, we can construct a neural network $\mathrm{b}_{n,1}$ with $L = \mcO(1)$, $M = \mcO(1)$, $J = \mcO(1)$ and $\kappa = \mcO(1)$, such that
$$
\mathrm{b}_{n,1}(x) = \frac{x+1}{2} - \frac{n}{\mcM}.
$$
Using the preceding convention for zero or one factor and Lemma~\ref{lem: product} otherwise, we construct a neural network
$$
\mathrm{b}_{\mathrm{prod}}(x,n,\gamma):=\mathrm{b}_{\mathrm{prod},1}(\underbrace{\cdot,\cdots,\cdot}_{\gamma \text{ times }})\circ \mathrm{b}_{n,1}(x)
$$ 
to approximate $\left(\frac{x + 1}{2} - \frac{n}{\mcM}\right)^{\gamma}$. We then approximate the soft-gate indicator $\chi_{\mathbf n,\delta_{\mathrm{ind}}}(\mx)$. For $n\in\{1,\ldots,\mcM\}$, define the exact one-dimensional ReLU network
\[
\mathrm b_{\pi,n}(x)
:=
\mathrm b_{\mathrm{clip}}\!\left(\frac{x-a_n}{\delta_{\mathrm{ind}}},0,1\right)
-\mathrm b_{\mathrm{clip}}\!\left(\frac{x-b_n}{\delta_{\mathrm{ind}}},0,1\right)
=\pi_{n,\delta_{\mathrm{ind}}}(x).
\]
Since $0\leq\mathrm b_{\pi,n}(x)\leq1$, Lemma \ref{lem: product} gives a network
\[
\mathrm b_{\mathrm{xind}}(\mx,\mathbf n)
:=\mathrm b_{\mathrm{prod},0}\bigl(
\mathrm b_{\pi,n_1}(x_1),\ldots,
\mathrm b_{\pi,n_{d_\mcX}}(x_{d_\mcX})
\bigr)
\]
such that
\[
\sup_{\mx\in[-1,1]^{d_\mcX}}
\left|
\mathrm b_{\mathrm{xind}}(\mx,\mathbf n)
-\chi_{\mathbf n,\delta_{\mathrm{ind}}}(\mx)
\right|
\leq\epsilon_{\mathrm{ind}}.
\]
Because $\delta_{\mathrm{ind}}=\mcM^{-2}$, the exact affine weights in $\mathrm b_{\pi,n}$ are bounded by $\mcM^2$, while Lemma \ref{lem: product} gives
\[
L_{\mathrm{xind}}=\mcO(\log\epsilon_{\mathrm{ind}}^{-1}),
\quad
M_{\mathrm{xind}}=\mcO(1),
\quad
J_{\mathrm{xind}}=\mcO(\log\epsilon_{\mathrm{ind}}^{-1}),
\quad
\log\kappa_{\mathrm{xind}}
\lesssim\log\mcM.
\]
Therefore, we can construct our desired neural network as follows 
$$
\mathrm{b}_1(t,\my,\mx) := \sum_{\mathbf{m}\in[\mcM]^{d_\mcY}, \mathbf{n}\in[\mcM]^{d_\mcX}}\sum_{\Vert\boldsymbol{\alpha}\Vert_1 + \Vert\boldsymbol{\gamma}\Vert_1 < \beta} c_{\mathbf{m},\mathbf{n},\boldsymbol{\alpha},\boldsymbol{\gamma}} \mathrm{b}_{\mathbf{m},\mathbf{n},\boldsymbol{\alpha},\boldsymbol{\gamma}}(t,\my,\mx),
$$
where
$$
\begin{aligned}
\mathrm{b}_{\mathbf{m},\mathbf{n},\boldsymbol{\alpha},\boldsymbol{\gamma}}(t,\my,\mx):= \mathrm{b}_{\mathrm{prod},2}\bigg(
&\mathrm{b}_{\mathrm{xind}}(\mx,\mathbf n),\\
&\mathrm{b}_{\mathrm{prod}}(x_1,n_1,\gamma_1),\cdots,\mathrm{b}_{\mathrm{prod}}(x_{d_\mcX},n_{d_\mcX},\gamma_{d_\mcX}),\\
& \mathrm{b}_{\mathrm{sum}}(t,y_1,m_1,\alpha_1),\cdots,\mathrm{b}_{\mathrm{sum}}(t,y_{d_\mcY},m_{d_\mcY},\alpha_{d_\mcY})
\bigg).
\end{aligned}
$$
The approximation error can be written as
$$
\epsilon_{\mathbf{m},\mathbf{n},\boldsymbol{\alpha}, \boldsymbol{\gamma}} \leq  \epsilon_{\mathrm{prod},2} + 
(d_\mcX + d_\mcY + 1) C_3^{d_\mcX+d_\mcY}
\max\{\epsilon_{\mathrm{prod},1}, \epsilon_{\mathrm{sum}}, \epsilon_{\mathrm{ind}}\},
$$
where
$$
C_3 = \max\left\{ 1, \max_{1\leq i \leq d_\mcX} 2^{\gamma_i}, \sup_{\substack{
t\in[T,1-T],\,|y_i|\leq C_0\\
1\leq i\leq d_\mcY}} \sum_{l = 0}^{k-1} |f(t,y_i,m_i,\alpha_i,l)| \right\}.
$$
The term $|f(t,y_i,m_i,\alpha_i,l)|$ ($1 \leq i \leq d_\mcY$) can be bounded as follows:
$$
\begin{aligned}
|f(t,y_i,m_i,\alpha_i,l)| & \leq \frac{2}{2^{\alpha_i+l}l!(2\pi)^{1/2}m_t^{\alpha_i + 1}}\sum_{j=0}^{\alpha_i} C_{\alpha_i}^j(C_0 + 3)^{\alpha_i-j} \sigma_t^j (C\sqrt{\log\epsilon^{-1}})^{j + 2l + 1} \\
& \leq \frac{2(C\sqrt{\log\epsilon^{-1}})^{\alpha_i + 2(k-1) + 1}}{2^{\alpha_i + l}l!(2\pi)^{1/2}m_t^{\alpha_i+1}}\sum_{j=0}^{\alpha_i} C_{\alpha_i}^j(C_0 + 3)^{\alpha_i-j} \\
& \leq \frac{2(C\sqrt{\log\epsilon^{-1}})^{\alpha_i + 2k - 1}}{2^{\alpha_i} l! (2\pi)^{1/2}m_t^{\alpha_i+1}} \cdot \left(C_0 + 3 + 1\right)^{\alpha_i}.
\end{aligned}
$$
It implies that
$$
\begin{aligned}
\sum_{l=0}^{k-1} |f(t,y_i,m_i,\alpha_i,l)| & \leq \frac{2(C\sqrt{\log\epsilon^{-1}})^{\alpha_i + 2k -1}}{2^{\alpha_i}(2\pi)^{1/2}m_t^{\alpha_i+1}} \cdot \left(C_0 + 4 \right)^{\alpha_i} \cdot \sum_{l=0}^{k-1}\frac{1}{l!} \\
& \leq \frac{2e(C\sqrt{\log\epsilon^{-1}})^{\alpha_i + 2k -1}}{2^{\alpha_i}(2\pi)^{1/2}T^{\alpha_i+1}} \cdot \left(C_0 + 4 \right)^{\alpha_i}.
\end{aligned}
$$
Therefore, $C_3$ satisfies
$$
\begin{aligned}
\log C_3 &= \mathcal{O}(\log C_0 + (\log \epsilon^{-1})\cdot(\log\log \epsilon^{-1}) + \log T^{-1}) \\
&\leq \mathcal{O}(\log C_0 + \log^2 \epsilon^{-1} + \log \mcM),
\end{aligned}
$$

We denote $\epsilon^* = \frac{\epsilon_0}{(d_\mcX + d_\mcY)^{\beta} \mcM^{d_\mcX + d_\mcY}}$. By taking
\begin{equation*}
\begin{aligned}
    &\epsilon_{\mathrm{prod},2} = \frac{\epsilon^*}{2},\quad
    \epsilon_f = \frac{\epsilon^*}{2k(d_\mcX + d_\mcY +1) C_3^{d_\mcX + d_\mcY}},\\
    &\epsilon_{\mathrm{prod},1} = \frac{\epsilon^*}{2 {(d_\mcX + d_\mcY +1) C_3^{d_\mcX + d_\mcY}}},\quad
    \epsilon_\mathrm{ind} = \frac{\epsilon^*}{2(d_\mcX + d_\mcY +1) C_3^{d_\mcX + d_\mcY}},
\end{aligned}
\end{equation*}
we ensure that $\epsilon_{\mathrm{sum}} \leq k\epsilon_f \leq \frac{\epsilon^*}{2{(d_\mcX + d_\mcY +1) C_3^{d_\mcX + d_\mcY}}}$. Therefore, $\epsilon_{\mathbf{m},\mathbf{n},\boldsymbol{\alpha}, \boldsymbol{\gamma}} \leq \epsilon^*$.
Note that $k = \mathcal{O}(\log \epsilon^{-1})$. By substituting $\epsilon_{\mathrm{prod},1}$, $\epsilon_{\mathrm{prod},2}$, $\epsilon_f$ and $\epsilon_\mathrm{ind}$ into the network parameters of $\mathrm{b}_{\mathrm{prod}}$, $\mathrm{b}_{\mathrm{xind}}$ and $\mathrm{b}_{\mathrm{sum}}$, we obtain the network parameters of $\mathrm{b}_{\mathbf{m},\mathbf{n},\boldsymbol{\alpha},\boldsymbol{\gamma}}$, which satisfy
$$
L = \mathcal{O}\left(\log^2\epsilon^{-1}_0 + \log^2 \mcM +\log^2C_0 + \log^4 \epsilon^{-1} \right), 
$$
$$
M = \mathcal{O}\left(\log \epsilon^{-1}(\log^3\epsilon^{-1}_0 + \log^3 \mcM + \log^3C_0 + \log^6 \epsilon^{-1}) \right),
$$
$$
J = \mathcal{O}\left(\log \epsilon^{-1} (\log^4\epsilon^{-1}_0 + \log^4 \mcM + \log^4C_0 + \log^8 \epsilon^{-1} \right)), 
$$
$$
\kappa = \exp\left(\mathcal{O}\left(\log^2\epsilon^{-1}_0 + \log^2 \mcM + \log^2C_0  + \log^4 \epsilon^{-1} \right)\right).
$$
Therefore, the network parameters of $\mathrm{b}_1$ satisfy
$$
L = \mathcal{O}\left(\log^2\epsilon^{-1}_0  + \log^2 \mcM + \log^2 C_0+ \log^4 \epsilon^{-1} \right), $$
$$
M = \mathcal{O}\left(\mcM^{d_\mcX + d_\mcY}\log \epsilon^{-1} (\log^3\epsilon^{-1}_0  + \log^3 \mcM + \log^3 C_0 + \log^6 \epsilon^{-1})\right),  $$
$$
J = \mathcal{O}\left(\mcM^{d_\mcX + d_\mcY}\log \epsilon^{-1} (\log^4\epsilon^{-1}_0 + \log^4 \mcM + \log^4 C_0 + \log^8 \epsilon^{-1}) \right), $$
$$
\kappa = \exp\left(\mathcal{O}\left(\log^2\epsilon^{-1}_0 + \log^2 \mcM + \log^2C_0 + \log^4 \epsilon^{-1} \right)\right).
$$
The approximation error between $\mathrm{b}_1(t,\my,\mx)$ and $g_3(t,\my,\mx)$ satisfies
$$
\begin{aligned}
&\Bigg|\mathrm{b}_1(t,\my,\mx) - \sum_{\mathbf{m}\in[\mcM]^{d_\mcY}, \mathbf{n}\in[\mcM]^{d_\mcX}}\sum_{\Vert\boldsymbol{\alpha}\Vert_1 + \Vert\boldsymbol{\gamma}\Vert_1 < \beta} c_{\mathbf{m},\mathbf{n},\boldsymbol{\alpha},\boldsymbol{\gamma}}  \\
& ~~~~~~~~~~~~ \cdot {\revxsoft}\prod_{i=1}^{d_\mcX}\left(\frac{x_i + 1}{2} - \frac{n_i}{\mcM}\right)^{\gamma_i}
\cdot\prod_{i=1}^{d_\mcY}\sum_{l=0}^{k-1}f(t,y_i,m_i,\alpha_i,l) \Bigg| \\
 &\leq \sum_{\mathbf{m}\in[\mcM]^{d_\mcY},\mathbf{n}\in[\mcM]^{d_\mcX}}\sum_{\Vert\boldsymbol{\alpha}\Vert_1 + \Vert\boldsymbol{\gamma}\Vert_1 < \beta}|c_{\mathbf{m},\mathbf{n},\boldsymbol{\alpha},\boldsymbol{\gamma}}|\epsilon_{\mathbf{m},\mathbf{n}, \boldsymbol{\alpha},\boldsymbol{\gamma}} \\
& \lesssim \sum_{\mathbf{m}\in[\mcM]^{d_\mcY}, \mathbf{n}\in[\mcM]^{d_\mcX}}\sum_{\Vert\boldsymbol{\alpha}\Vert_1 + \Vert\boldsymbol{\gamma}\Vert_1 < \beta}\frac{1}{\boldsymbol{\alpha}!\boldsymbol{\gamma}!} \cdot \epsilon^* \\
& \lesssim (d_\mcX + d_\mcY)^{\beta} \mcM^{d_\mcX + d_\mcY} \cdot \frac{\epsilon_0}{(d_\mcX + d_\mcY)^{\beta} \mcM^{d_\mcX + d_\mcY}} \lesssim \epsilon_0.
\end{aligned}
$$
The proof is complete.
\end{proof}

Similarly, the $j$-th element of \eqref{eq: diffused_local_poly2} can be rewritten as
$$
\begin{aligned}
[\eqref{eq: diffused_local_poly2}]_j &= \sum_{\mathbf{m}\in[\mcM]^{d_\mcY}, \mathbf{n}\in[\mcM]^{d_\mcX}}\sum_{\Vert\boldsymbol{\alpha}\Vert_1 + \Vert\boldsymbol{\gamma}\Vert_1 < \beta}
c_{\mathbf{m},\mathbf{n},\boldsymbol{\alpha},\boldsymbol{\gamma}}
\revxsoft
\prod_{i=1}^{d_\mcX}\left(\frac{x_i + 1}{2} - \frac{n_i}{\mcM}\right)^{\gamma_i}\\ &~~~~~~~~~~~~~~\cdot \left(\prod_{i=1,i\neq j}^{d_\mcY}\sum_{l=0}^{k-1}f(t,y_i,m_i,\alpha_i,l)\right) \cdot \left(\sum_{l=0}^{k-1} f_1(t,y_j,m_j,\alpha_j,l)\right),
\end{aligned}
$$
where
$$
\begin{aligned}
f_1(t,y,m,\alpha,l)&:= \frac{1}{\sigma_t(2\pi)^{1/2}} \int_{a_y\cap \left(\frac{2(m-1)}{\mcM}-1, \frac{2m}{\mcM}-1 \right]} \mathbf{I}_{\{|y_1| \leq 1\}} \left(\frac{m_t y_1-y}{\sigma_t}\right) \\
& ~~~~~~~~~~~ \cdot\left(\frac{y_1 + 1}{2} - \frac{m}{\mcM}\right)^{\alpha} 
\frac{1}{l!}\left(-\frac{(y-m_t y_1)^2}{2\sigma_t^2}\right)^l \mrd y_1.
\end{aligned}
$$
In the same way, we can choose a ReLU neural network $\mb_2(t,\my,\mx)\in\Rbb^{d_\mcY}$ to approximate \eqref{eq: diffused_local_poly2}. We omit the proof and give the following lemma.
\begin{lemma}\label{lem: approximate_poly2}
Given $\mcM \gg 1$, $C_0 > 0$, let $T = \mcM^{-C_T}$, where $C_T > 0$ is a constant. For any $\epsilon_0 > 0$, there exist ReLU neural networks $\mathrm{b}_{2,j}\in\mathrm{NN}(L,M,J,\kappa)$, $1\leq j \leq d_\mcY$, with
$$
L_j = \mathcal{O}\left(\log^2\epsilon^{-1}_0  + \log^2 \mcM + \log^2 C_0+ \log^4 \epsilon^{-1} \right), $$
$$
M_j = \mathcal{O}\left(\mcM^{d_\mcX + d_\mcY}\log \epsilon^{-1} (\log^3\epsilon^{-1}_0  + \log^3 \mcM + \log^3 C_0 + \log^6 \epsilon^{-1})\right),  $$
$$
J_j = \mathcal{O}\left(\mcM^{d_\mcX + d_\mcY}\log \epsilon^{-1} (\log^4\epsilon^{-1}_0 + \log^4 \mcM + \log^4 C_0 + \log^8 \epsilon^{-1}) \right), $$
$$
\kappa_j = \exp\left(\mathcal{O}\left(\log^2\epsilon^{-1}_0 + \log^2 \mcM + \log^2C_0 + \log^4 \epsilon^{-1} \right)\right).
$$
such that for the $j$-th component of the target $\mh_3(t,\my,\mx)$, for any $\mx\in[-1,1]^{d_\mcX}$, $\my\in[-C_0,C_0]^{d_\mcY}$ and $t\in[\mcM^{-C_T}, 1-\mcM^{-C_T}]$, the following holds:
$$
\begin{aligned}
\Bigg|\mathrm{b}_{2,j}(t,\my,\mx) &- 
\sum_{\mathbf{m}\in[\mcM]^{d_\mcY}, \mathbf{n}\in[\mcM]^{d_\mcX}}\sum_{\Vert\boldsymbol{\alpha}\Vert_1 + \Vert\boldsymbol{\gamma}\Vert_1 < \beta}
c_{\mathbf{m},\mathbf{n},\boldsymbol{\alpha},\boldsymbol{\gamma}} \cdot\revxsoft
\prod_{i=1}^{d_\mcX}\left(\frac{x_i + 1}{2} - \frac{n_i}{\mcM}\right)^{\gamma_i} \\ &~~~~~~~~~~~~~~\cdot \left(\prod_{i=1,i\neq j}^{d_\mcY}\sum_{l=0}^{k-1}f(t,y_i,m_i,\alpha_i,l)\right) \cdot \left(\sum_{l=0}^{k-1} f_1(t,y_j,m_j,\alpha_j,l)\right)
\Bigg| \lesssim \epsilon_0.
\end{aligned}
$$
\end{lemma}

Therefore, according to Lemma \ref{lem: approximate_poly2}, there exists a network $\mb_2 := [\mathrm{b}_{2,1}, \cdots, \mathrm{b}_{2, d_\mcY}]^{\top} \in \mathbb{R}^{d_\mcY}$ that approximates $\mh_3(t,\my,\mx)$. The $\ell_2$-norm error between $\mathbf{b}_2(t,\my,\mx)$ and $\mh_3(t,\my,\mx)$ is bounded by $\mathcal{O}(\epsilon_0)$.

\subsection{Error bound of approximating $\nabla\log p_t(\my|\mx)$ with ReLU neural network}\label{subsec: c4}

\begin{lemma}[Relative-Fisher transfer]\label{lem: relative_fisher}
Suppose Assumptions \ref{ass: bounded_density}--\ref{ass: Holder} hold. Uniformly over $0<t<1$ and $\mx\in[-1,1]^{d_{\mcX}}$,
\[
    \int_{\Rbb^{d_{\mcY}}}p_t(\my\mid\mx)
    \|\nabla\log q_t(\my,\mx)-\nabla\log p_t(\my\mid\mx)\|^2\mrd\my
    \lesssim\frac{\mcM^{-2\beta}}{\sigma_t^2}.
\]
Moreover, for all sufficiently large $\mcM$,
\[
p_t(\my\mid\mx)\leq2q_t(\my,\mx).
\]
\end{lemma}
\begin{proof}
Fix $\mx\in[-1,1]^{d_{\mcX}}$. Let $\mathbf U$ have density $p_0(\cdot\mid\mx)$, $\mz\sim\mcN(\mathbf0,\mI_{d_{\mcY}})$ be independent of $\mathbf U$, and put
\[
    \mathbf Y_t=m_t\mathbf U+\sigma_t\mz, \qquad
    w(\mathbf U,\mx):=\frac{q_0(\mathbf U,\mx)}{p_0(\mathbf U\mid\mx)}.
\]
Lemma \ref{lem: approximate_pdata1} and Assumption \ref{ass: bounded_density} imply
\[
\begin{aligned}
    |w(\mathbf U,\mx)-1|
    &= \frac{|q_0(\mathbf U,\mx)-p_0(\mathbf U\mid\mx)|}{p_0(\mathbf U\mid\mx)}\\
    &\leq C_l^{-1}\|q_0-p_0\|_\infty
    \lesssim\mcM^{-\beta}.
\end{aligned}
\]
For every $\my\in\Rbb^{d_{\mcY}}$, Bayes' formula gives
\[
\begin{aligned}
    \frac{q_t(\my,\mx)}{p_t(\my\mid\mx)}
    &=
    \frac{\int q_0(\mathbf u,\mx)\cdot
    \frac{1}{(2\pi\sigma_t^2)^{d_{\mcY}/2}}\exp\left(-\frac{\|\my-m_t\mathbf u\|^2}{2\sigma_t^2}\right) \mrd\mathbf u}
    {\int p_0(\mathbf u\mid\mx)\cdot
    \frac{1}{(2\pi\sigma_t^2)^{d_{\mcY}/2}}\exp\left(-\frac{\|\my-m_t\mathbf u\|^2}{2\sigma_t^2}\right)\mrd\mathbf u}\\
    &=
    \frac{\int w(\mathbf u,\mx)p_0(\mathbf u\mid\mx)\cdot
    \frac{1}{(2\pi\sigma_t^2)^{d_{\mcY}/2}}\exp\left(-\frac{\|\my-m_t\mathbf u\|^2}{2\sigma_t^2}\right) \mrd\mathbf u}
    {p_t(\my\mid\mx)}\\
    &=\Ebb[w(\mathbf U,\mx)\mid\mathbf Y_t=\my].
\end{aligned}
\]
Consequently, we have
\[
    \left|\frac{q_t(\my,\mx)}{p_t(\my\mid\mx)}-1\right|
    \le \Ebb\left[\left|w(\mathbf U,\mx)-1\right|\mid \mY_t=\my\right]
    \lesssim \mcM^{-\beta},
\]
which implies
\[
    1-\wt{C}\mcM^{-\beta}
    \leq \frac{q_t(\my,\mx)}{p_t(\my\mid\mx)}
    \leq 1+\wt{C}\mcM^{-\beta}.
\]
For sufficiently large $\mcM$, it holds
\[
    \frac12\leq\frac{q_t}{p_t}\leq\frac32,
    \qquad
    p_t\leq2q_t,
\]
and
\[
    \left| w(\mathbf U,\mx) - \Ebb[w(\mathbf U,\mx)\mid\mathbf Y_t=\my] \right|
    \leq 2\|w(\cdot,\mx)-1\|_\infty
    \lesssim \mcM^{-\beta}.
\]
\indent To justify differentiating under the integral sign, we compute, for $1\leq i\leq d_{\mcY}$ and $\|\mathbf v-\my\|\leq1$,
\[
\begin{aligned}
    &~~~~\left|\partial_{v_i} \left( q_0(\mathbf u,\mx)\cdot\frac{1}{(2\pi\sigma_t^2)^{d_{\mcY}/2}}\exp\left(-\frac{\|\mathbf v-m_t\mathbf u\|^2}{2\sigma_t^2}\right) \right)\right|\\
    &= \left| -q_0(\mathbf u,\mx)\cdot\frac{v_i-m_tu_i}{\sigma_t^2}\cdot\frac{1}{(2\pi\sigma_t^2)^{d_{\mcY}/2}}\exp\left(-\frac{\|\mathbf v-m_t\mathbf u\|^2}{2\sigma_t^2}\right) \right|\\
    &\leq 2C_u\mI_{[-1,1]^{d_{\mcY}}}(\mathbf u)
    \frac{|v_i-y_i|+|y_i|+|m_t u_i|}{\sigma_t^2(2\pi\sigma_t^2)^{d_{\mcY}/2}}
    \exp\left( -\frac{\|\mathbf v-m_t\mathbf u\|^2}{2\sigma_t^2} \right)\\
    &\leq 2C_u\mI_{[-1,1]^{d_{\mcY}}}(\mathbf u)
    \frac{|y_i|+2}{\sigma_t^2(2\pi\sigma_t^2)^{d_{\mcY}/2}}
    \exp\left( \frac{1}{2\sigma_t^2}-\frac{\|\my-m_t\mathbf u\|^2}{4\sigma_t^2} \right),
\end{aligned}
\]
where in the last inequality, we used
\begin{equation*}
    \|\my-m_t\mathbf u\|^2
    \le 2\|\my-\mathbf v\|^2 + 2\|\mathbf v-m_t\mathbf u\|^2\\
    \le 2+2\|\mathbf v-m_t\mathbf u\|^2,
\end{equation*}
which implies $\|\mathbf v-m_t\mathbf u\|^2\ge\frac12\|\my-m_t\mathbf u\|^2-1$. The right-hand side of the above expression is independent of $\mv$ and integrable in $\mathbf u$ because its support is $[-1,1]^{d_{\mcY}}$. Replacing $2C_u$ by $C_u$ gives the corresponding domination for $p_0$. Hence differentiation under both integral signs is justified, and
\[
\begin{aligned}
    \nabla q_t(\my,\mx)
    &= -\int q_0(\mathbf u,\mx)\frac{\my-m_t\mathbf u}{\sigma_t^2}\frac{1}{(2\pi\sigma_t^2)^{d_{\mcY}/2}}\exp\left(-\frac{\|\my-m_t\mathbf u\|^2}{2\sigma_t^2} \right) \mrd\mathbf u\\
    &= -\frac{1}{\sigma_t}\int w(\mathbf U,\mx)\mz\cdot p_0(\mathbf u\mid\mx)\frac{1}{(2\pi\sigma_t^2)^{d_{\mcY}/2}}\exp\left(-\frac{\|\my-m_t\mathbf u\|^2}{2\sigma_t^2} \right) \mrd\mathbf u\\
    &= -\frac{p_t(\my\mid\mx)}{\sigma_t}
    \Ebb[w(\mathbf U,\mx)\mz\mid\mathbf Y_t=\my],
\end{aligned}
\]
together with
\[
\begin{aligned}
    \nabla p_t(\my\mid\mx)
    &= -\int p_0(\mathbf u\mid\mx)\frac{\my-m_t\mathbf u}{\sigma_t^2}\frac{1}{(2\pi\sigma_t^2)^{d_{\mcY}/2}}\exp\left(-\frac{\|\my-m_t\mathbf u\|^2}{2\sigma_t^2} \right) \mrd\mathbf u\\
    &= -\frac{1}{\sigma_t}\int \mz\cdot p_0(\mathbf u\mid\mx)\frac{1}{(2\pi\sigma_t^2)^{d_{\mcY}/2}}\exp\left(-\frac{\|\my-m_t\mathbf u\|^2}{2\sigma_t^2} \right) \mrd\mathbf u\\
    &= -\frac{p_t(\my\mid\mx)}{\sigma_t}
    \Ebb[\mz\mid\mathbf Y_t=\my].
\end{aligned}
\]
It follows that
\[
\begin{aligned}
    &~~~~\nabla\log q_t(\my,\mx)-\nabla\log p_t(\my\mid\mx)\\
    &= \frac{\nabla q_t(\my,\mx)}{q_t(\my,\mx)} - \frac{\nabla p_t(\my\mid\mx)}{p_t(\my\mid\mx)}\\
    &= -\frac1{\sigma_t}
    \frac{\Ebb[w(\mathbf U,\mx)\mz\mid\mY_t=\my]}
    {\Ebb[w(\mathbf U,\mx)\mid\mY_t=\my]}
    +\frac1{\sigma_t}\Ebb[\mz\mid\mY_t=\my]\\
    &= -\frac1{\sigma_t\Ebb[w(\mathbf U,\mx)\mid\mathbf Y_t=\my]}
    \Ebb\left[\big( w(\mathbf U,\mx)-\Ebb[w(\mathbf U,\mx)\mid\mathbf Y_t=\my] \big) \mz \,\middle|\, \mY_t=\my \right].
\end{aligned}
\]
The conditional Cauchy--Schwarz inequality then yields
\[
\begin{aligned}
    &~~~~\|\nabla\log q_t(\my,\mx)-\nabla\log p_t(\my\mid\mx)\|^2\\
    &\leq \frac4{\sigma_t^2}
    \Ebb\left[
    \big|w(\mathbf U,\mx)-\Ebb[w(\mathbf U,\mx)\mid\mathbf Y_t=\my]\big|^2
    \middle|\mathbf Y_t=\my
    \right]
    \Ebb[\|\mz\|^2\mid\mathbf Y_t=\my]\\
    &\lesssim \frac{\mcM^{-2\beta}}{\sigma_t^2}
    \Ebb[\|\mz\|^2\mid\mathbf Y_t=\my],
\end{aligned}
\]
where in the first inequality we used $\Ebb[w(\mathbf U,\mx)\mid\mathbf Y_t=\my]=q_t(\my,\mx)/p_t(\my\mid\mx)\geq \frac{1}{2}$. Multiplying by $p_t(\my\mid\mx)$, integrating in $\my$, and using the tower property gives
\[
\begin{aligned}
    &~~~~\int p_t(\my\mid\mx)
    \|\nabla\log q_t(\my,\mx)-\nabla\log p_t(\my\mid\mx)\|^2\mrd\my\\
    &\lesssim \frac{\mcM^{-2\beta}}{\sigma_t^2}
    \int p_t(\my\mid\mx) \Ebb[\|\mz\|^2\mid\mathbf Y_t=\my] \mrd\my\\
    &= \frac{\mcM^{-2\beta}}{\sigma_t^2}\Ebb\left[\|\mz\|^2\right]
    = \frac{d_{\mcY}\mcM^{-2\beta}}{\sigma_t^2}.
\end{aligned}
\]
The proof is complete.
\end{proof}

\begin{lemma}[Full-interval proxy-score network]\label{lem: proxy_score_full_interval}
Let $C_T>0$ be fixed and $T=\mcM^{-C_T}$. There exists a ReLU neural network $\mathbf b_q\in\mathrm{NN}(L,M,J,\kappa)$ satisfying
\[
    L=\mcO(\log^4\mcM),\quad
    M=\mcO(\mcM^{d_{\mcX}+d_{\mcY}}\log^7\mcM),
\]
\[
    J=\mcO(\mcM^{d_{\mcX}+d_{\mcY}}\log^9\mcM),\quad
    \kappa=\exp\{\mcO(\log^4\mcM)\},
\]
such that, uniformly over $t\in[T,1-T]$ and $\mx\in[-1,1]^{d_{\mcX}}$,
\[
    \int q_t(\my,\mx)\|\mathbf b_q(t,\my,\mx)-\nabla\log q_t(\my,\mx)\|^2 \mrd\my \lesssim \frac{\mcM^{-2\beta}\log\mcM}{\sigma_t^2}.
\]
Moreover,
\[
    \|\mathbf b_q(t,\cdot,\mx)\|_\infty
    \lesssim \frac{\sqrt{\log\mcM}}{\sigma_t}.
\]
\end{lemma}
\begin{proof}
For any $\eta>0$, decompose the target into three parts as
\[
\begin{aligned}
    &~~~~\int q_t(\my,\mx)\|\mathbf b_q(t,\my,\mx)-\nabla\log q_t(\my,\mx)\|^2\mrd\my\\
    &=\underbrace{\int_{\|\my\|_\infty>m_t+C\sigma_t\sqrt{\log\eta^{-1}}}
    q_t(\my,\mx)\|\mathbf b_q(t,\my,\mx)-\nabla\log q_t(\my,\mx)\|^2\mrd\my}_{(\mathrm I)}\\
    &~~~~+\underbrace{\int_{\|\my\|_\infty\leq m_t+C\sigma_t\sqrt{\log\eta^{-1}}}
    q_t(\my,\mx)\mathbf I_{\{q_t(\my,\mx)\leq\eta\}}
    \|\mathbf b_q(t,\my,\mx)-\nabla\log q_t(\my,\mx)\|^2\mrd\my}_{(\mathrm{II})}\\
    &~~~~+\underbrace{\int_{\|\my\|_\infty\leq m_t+C\sigma_t\sqrt{\log\eta^{-1}}}
    q_t(\my,\mx)\mathbf I_{\{q_t(\my,\mx)>\eta\}}
    \|\mathbf b_q(t,\my,\mx)-\nabla\log q_t(\my,\mx)\|^2\mrd\my}_{(\mathrm{III})}.
\end{aligned}
\]
Because $q_0$ is a density supported on $[-1,1]^{d_{\mcY}}$ and $C_l/2\le q_0\le2C_u$, Lemmas~\ref{lem: bound_for_clipping} and~\ref{lem: derivatives_boundness} apply to $q_t$. By Lemma~\ref{lem: bound_for_clipping}, there exists a constant $C>0$ such that, for every $0<\eta\le e^{-1}$,
\begin{equation}\label{eq: appendix_eq1}
\begin{aligned}
    & ~~~~~ \underbrace{ \int_{\Vert\my\Vert_\infty > m_t + C\sigma_t\sqrt{\log\eta^{-1}}} q_t(\my,\mx) \Vert \mathbf{b}_q(t,\my,\mx) - \nabla\log q_t(\my,\mx)\Vert^2 \mrd \my }_{(\mathrm{I})} \\
    &\lesssim \frac{\eta}{\sigma_t} + \sigma_t\eta\Vert \mathbf{b}_q(t,\cdot,\mx)\Vert_{\infty}^2.
\end{aligned}
\end{equation}
Since $\Vert\nabla\log q_t(\my|\mx)\Vert \lesssim \frac{\sqrt{\log\eta^{-1}}} {\sigma_t}$ for $\Vert\my\Vert_{\infty} \leq m_t + C\sigma_t\sqrt{\log\eta^{-1}}$ according to Lemma~\ref{lem: derivatives_boundness}, 
we can choose $\mathbf{b}_q$ such that $\Vert\mathbf{b}_q(t,\cdot,\mx)\Vert_{\infty} \lesssim \frac{\sqrt{\log\eta^{-1}}}{\sigma_t}$. Taking $\eta=\mcM^{-(2\beta+1)}$, and \eqref{eq: appendix_eq1} is bounded by 
\[
    (\mathrm I)\lesssim \frac{\eta\log\eta^{-1}}{\sigma_t}
    \lesssim \frac{\eta\log\eta^{-1}}{\sigma_t^2}
    \lesssim \frac{\mcM^{-(2\beta + 1)}\log \mcM}{\sigma_t^2}
    \lesssim \frac{\mcM^{-2\beta}\log \mcM}{\sigma_t^2}.
\]
Moreover, by Lemma~\ref{lem: bound_for_clipping}, we can bound $(\mathrm{II})$ as follows
\[
    (\mathrm{II})
    \lesssim \frac{\eta}{\sigma_t^2}(\log\eta^{-1})^{\frac{d_{\mcY}+2}{2}}+\frac{\eta\log\eta^{-1}}{\sigma_t^2}(\log\eta^{-1})^{\frac{d_{\mcY}}{2}}
    \lesssim \frac{\mcM^{-(2\beta+1)}
    \log^{\frac{d_{\mcY}+2}{2}}\mcM}{\sigma_t^2}
    \lesssim \frac{\mcM^{-2\beta}\log\mcM}{\sigma_t^2}.
\]
\indent Now we focus on bounding the term $(\mathrm{III})$. Set $C_0:= 1+C\sqrt{(2\beta+1)\log\mcM}=\mcO(\sqrt{\log\mcM})$. Then on the region of $(\rm{III})$, $\|\my\|_\infty\le C_0$ and the conditions of Lemmas \ref{lem: approximate_poly1}--\ref{lem: approximate_poly2} are satisfied. Thus there exists networks $\mathrm b_1$ and $\mathbf b_2$ such that on $[T,1-T]\times[-C_0,C_0]^{d_{\mcY}}\times[-1,1]^{d_{\mcX}}$,
\[
    |\mathrm b_1-g_3|\lesssim\epsilon_0,\qquad
    \|\mathbf b_2-\mh_3\|\lesssim\epsilon_0.
\]
Section \ref{subsec: c2} gives
\[
\begin{aligned}
    |\mathrm b_1-q_t|
    &\leq|\mathrm b_1-g_3|
    +|g_3-g_2|
    +|g_2-g_1|\\
    &\lesssim \epsilon_0 + \epsilon\log^{\frac{d_{\mcY}}{2}}\epsilon^{-1} + \epsilon\\
    &\lesssim \epsilon_0 + \epsilon\log^{\frac{d_{\mcY}}{2}}\epsilon^{-1},
\end{aligned}
\]
and
\[
\begin{aligned}
    \|\mathbf b_2-\sigma_t\nabla q_t\|
    &\leq \|\mathbf b_2-\mh_3\|
    +\|\mh_3-\mh_2\|
    +\|\mh_2-\mh_1\|\\
    &\lesssim \epsilon_0 + \epsilon\log^{\frac{d_{\mcY}+1}{2}}\epsilon^{-1} + \epsilon\\
    &\lesssim \epsilon_0 + \epsilon\log^{\frac{d_{\mcY}+1}{2}}\epsilon^{-1}.
\end{aligned}
\]
On the region in $(\mathrm{III})$, we also have $q_t(\my,\mx)\geq\eta$ and $\Vert\nabla\log q_t(\my|\mx)\Vert \lesssim \frac{\sqrt{\log\eta^{-1}}}{\sigma_t}$. Thus there exists a constant $C_4>0$ such that $\|\sigma_t\nabla\log q_t(\my|\mx)\| \le C_4\sqrt{\log\mcM}$. We first construct a neural network to approximate $\sigma_t^{-1}$ as
\[
    {\rm b}_\sigma(t):=\mathrm b_{\mathrm{rec},2}\left(\mathrm b_{\mathrm{root}}(\mathrm b_{\mathrm{prod},3}(t,2-t))\right),\qquad
    \epsilon_\sigma:=\sup_{t\in[T,1-T]}\left|{\rm b}_\sigma(t)-\frac{1}{\sigma_t}\right|\le\epsilon_{\rm{rec},2}+\frac{\epsilon_{\rm{root}}+\frac{\epsilon_{\rm{prod},3}}{\sqrt{\epsilon_{\rm{root}}}}}{\epsilon_{\rm{rec},2}^2}.
\]
And define
\[
    \mathrm b_3(t,\my,\mx) := \mathrm b_1(t,\my,\mx)\vee\eta=\operatorname{ReLU}(\mathrm b_1(t,\my,\mx)-\eta)+\eta,
\]
which satisfies 
\[
    \mathrm b_3(t,\my,\mx) \geq \eta,\qquad
    |\mathrm b_3(t,\my,\mx)-q_t(\my,\mx)|\leq|\mathrm b_1(t,\my,\mx)-q_t(\my,\mx)| \lesssim \epsilon_0+\epsilon\log^{\frac{d_\mcY}{2}}\epsilon^{-1}.
\]
Subsequently, define for $1\leq i\leq d_{\mcY}$,
\[
    [\mathbf b_4(t,\my,\mx)]_i
    :=\mathrm b_{\mathrm{clip}}\left(
    \mathrm b_{\mathrm{prod},1}([\mathbf b_2(t,\my,\mx)]_i,\mathrm b_{\mathrm{rec},1}(\mathrm b_3(t,\my,\mx))),
    -C_4\sqrt{\log\mcM},C_4\sqrt{\log\mcM}\right),
\]
\[
    [\mathbf b_q(t,\my,\mx)]_i
    :=\mathrm b_{\mathrm{clip}}\left(
    \mathrm b_{\mathrm{prod},2}([\mathbf b_4(t,\my,\mx)]_i,{\rm b}_\sigma(t)),
    -C_4\sqrt{\log\mcM}({\rm b}_\sigma(t)+\epsilon_\sigma),
    C_4\sqrt{\log\mcM}({\rm b}_\sigma(t)+\epsilon_\sigma) \right).
\]
Here we used Lemma~\ref{lem: clipping}, which can be extended to hold for arbitrary non-fixed $\mathbf a$ and $\mathbf b$ by simply constructing accordingly
\[
\begin{aligned}
    [\mathrm b_{\mathrm{clip}}(\mx,\mathbf a,\mathbf b)]_i
    &:= \operatorname{ReLU}(x_i-a_i)-\operatorname{ReLU}(x_i-b_i)
    +\operatorname{ReLU}(a_i)-\operatorname{ReLU}(-a_i).
\end{aligned}
\]
We use $\mathbf b_q(t,\my,\mx)$ to approximate our target $\nabla\log q_t(\my,\mx)$ and denote the approximation error by $\epsilon_q$. Analyzing the error bound gives
\[
    \epsilon_q
    \lesssim \epsilon_{\rm{prod},2}+2C_5\max\left\{ \max_{1\le i\le d_{\mcY}}\epsilon_{4,i},\epsilon_{\rm{rec},2}+\frac{\epsilon_{\rm{root}}+\frac{\epsilon_{\rm{prod},3}}{\sqrt{\epsilon_{\rm{root}}}}}{\epsilon_{\rm{rec},2}^2} \right\},
\]
where, using the bound on $\sigma_t\nabla\log q_t$ established above on $(\mathrm{III})$, we take
\[
    C_5
    :=\max\left\{
        \sup_{t\in[T,1-T]}\frac1{\sigma_t},
        C_4\sqrt{\log\mcM}
    \right\}
    \lesssim
    \max\left\{\mcM^{C_T/2},\sqrt{\log\mcM}\right\}.
\]
Then we further bound error $\epsilon_{4,i}$ as follows
\[
    \epsilon_{4,i}
    \lesssim \epsilon_{\rm{prod},1}+2C_6\max\left\{ |[\mb_2(t,\my,\mx)]_i-[\sigma_t\nabla q_t(\my,\mx)]_i|,\epsilon_{\rm{rec},1}+\frac{\left|\mathrm{b}_3(t,\my,\mx)-q_t(\my,\mx)\right|}{\epsilon_{\rm{rec},1}^2} \right\},
\]
where
\[
\begin{aligned}
    C_6
    &= \max\left\{\sup_{t,\mx,\|\my\|_\infty\le C_0}|[\sigma_t\nabla q_t(\my,\mx)]_i|,\sup_{t,\mx,\|\my\|_\infty\le C_0}\frac{1}{q_t(\my,\mx)\vee\eta}\right\}\\
    &\lesssim \max\left\{ \sup_t\min\left\{\frac{1}{m_t^{d_\mcY}},\frac{1}{\sigma_t^{d_\mcY}}\right\},\mcM^{2\beta+1} \right\} \lesssim \mcM^{2\beta+1}.
\end{aligned}
\]
By taking
\[
    \epsilon_{\rm{prod},1}=\frac{\mcM^{-\beta}}{8C_5},\quad
    \epsilon_{\rm{prod},2}=\frac{\mcM^{-\beta}}{2},\quad
    \epsilon_{\rm{prod},3}=\epsilon_{\rm{root}}^{\frac{3}{2}},\quad
    \epsilon_{\rm{rec},1}=\frac{\mcM^{-\beta}}{32C_5C_6},\quad
    \epsilon_{\rm{rec},2}=\frac{\mcM^{-\beta}}{8C_5},
\]
and
\[
    \epsilon=\epsilon_0^2,\quad
    \epsilon_0=\min\left\{\frac{\mcM^{-\beta}}{32C_5C_6},\frac{1}{2}\left(\frac{\mcM^{-\beta}}{32C_5C_6}\right)^3\right\},\quad
    \epsilon_{\rm{root}}=\frac{1}{2}\left(\frac{\mcM^{-\beta}}{8C_5}\right)^3,\quad
\]
we obtain
\begin{equation*}
    \epsilon_{4,i}\lesssim\frac{\mcM^{-\beta}}{C_5},\qquad
    \epsilon_q\lesssim\mcM^{-\beta}.
\end{equation*}
Subsequently, we can obtain the network parameters of $[\mathbf{b}_q]_i$, $1\leq i \leq d_\mcY$:
$$
    L_i = \mcO(\log^4 \mcM), M_i = \mcO(\mcM^{d_\mcX+d_\mcY}\log^7 \mcM), J_i = \mcO(\mcM^{d_\mcX+d_\mcY}\log^9 \mcM), \kappa_i = \exp\left(\mcO(\log^4 \mcM)\right).
$$
Combining $[\mathbf{b}_q]_i$, $1\leq i \leq d_\mcY$, and by Lemma \ref{lem: parallelization}, we construct the ReLU neural network 
$$
    \mathbf{b}_q:=\left[ [\mathbf{b}_q]_1, [\mathbf{b}_q]_2, \cdots, [\mathbf{b}_q]_{d_\mcY} \right]^{\top},
$$
with network parameters
$$
    L = \mathcal{O}(\log^4 \mcM), M = \mathcal{O}(\mcM^{d_\mcX+d_\mcY}\log^7 \mcM), J = \mathcal{O}(\mcM^{d_\mcX+d_\mcY}\log^9 \mcM), \kappa = \exp\left(\mathcal{O}(\log^4 \mcM)\right),
$$
and it satisfies
\[
    \|\mathbf{b}_q(t,\cdot,\mx)\|_{\infty}
    \lesssim \sqrt{\log \mcM}\left(\sigma_t^{-1}+2\epsilon_\sigma\right)
    \lesssim \frac{\sqrt{\log\mcM}}{\sigma_t},
\]
where the last inequality follows from $2\epsilon_\sigma\leq1\leq\sigma_t^{-1}$. Thus, $(\rm{III})$ is bounded by
\begin{equation*}
    (\mathrm{III})
    \lesssim d_\mcY\mcM^{-2\beta} \lesssim \mcM^{-2\beta}.
\end{equation*}
Combining this with the bound of $(\rm{I})$ and $(\rm{II})$ that we obtained before completes the proof.
\end{proof}

\begin{proof}[Proof of Lemma~\ref{lem: approximation_error}]
Take $\mathbf b_q$ from Lemma~\ref{lem: proxy_score_full_interval}, and set
$\eta=\mcM^{-(2\beta+1)}$ and $C_0=1+C\sqrt{(2\beta+1)\log\mcM}$ as in its proof. Define $\mb(t,\my,\mx)=\mathbf b_q(t,\mathbf b_{\mathrm{clip}}(\my,-C_0,C_0),\mx)$. Then $\mb=\mathbf b_q$ on $\{\|\my\|_\infty\le C_0\}$, and
$\|\mb(t,\cdot,\mx)\|_\infty\lesssim\tfrac{\sqrt{\log\mcM}}{\sigma_t}$. Since $m_t\le1$ and $\sigma_t\le1$, we have $C_0\ge m_t+C\sigma_t\sqrt{\log\eta^{-1}}$. Consequently, Lemma~\ref{lem: bound_for_clipping}, applied to $q_t$, gives
\[
\begin{aligned}
    &~~~~\int_{\|\my\|_\infty>C_0} q_t(\my,\mx)\|\mb(t,\my,\mx)-\nabla\log q_t(\my,\mx)\|^2\,\mrd\my\\
    &\le 2\|\mb(t,\cdot,\mx)\|_\infty^2 \int_{\|\my\|_\infty>C_0}q_t(\my,\mx)\,\mrd\my + 2\int_{\|\my\|_\infty>C_0} q_t(\my,\mx)\|\nabla\log q_t(\my,\mx)\|^2\,\mrd\my\\
    &\lesssim \frac{\log\mcM}{\sigma_t^2}\,\sigma_t\eta + \frac{\eta}{\sigma_t}
    \lesssim \frac{\eta\log\mcM}{\sigma_t}
    \le \frac{\mcM^{-2\beta}\log\mcM}{\sigma_t^2}.
\end{aligned}
\]
Combining this estimate with $\mb=\mathbf b_q$ on $\{\|\my\|_\infty\le C_0\}$ and Lemma~\ref{lem: proxy_score_full_interval}, we obtain
\[
    \int_{\Rbb^{d_{\mcY}}} q_t(\my,\mx) \|\mb(t,\my,\mx)-\nabla\log q_t(\my,\mx)\|^2 \,\mrd\my
    \lesssim \frac{\mcM^{-2\beta}\log\mcM}{\sigma_t^2}.
\]
By Lemma~\ref{lem: relative_fisher}, we have $p_t(\my\mid\mx)\le2q_t(\my,\mx)$ and the relative-Fisher error bound. Hence
\[
\begin{aligned}
    &~~~~\int_{\Rbb^{d_\mcY}} p_t(\my\mid\mx)
    \|\mb(t,\my,\mx)-\nabla\log p_t(\my\mid\mx)\|^2\mrd\my\\
    &\leq 2\int p_t(\my\mid\mx) \|\mb(t,\my,\mx)-\nabla\log q_t(\my,\mx)\|^2\mrd\my+2\int p_t(\my\mid\mx)
    \|\nabla\log q_t(\my,\mx)-\nabla\log p_t(\my\mid\mx)\|^2\mrd\my\\
    &\leq 4\int q_t(\my,\mx)\|\mb(t,\my,\mx)-\nabla\log q_t(\my,\mx)\|^2\mrd\my+2\int p_t(\my\mid\mx)
    \|\nabla\log q_t(\my,\mx)-\nabla\log p_t(\my\mid\mx)\|^2\mrd\my\\
    &\lesssim\frac{\mcM^{-2\beta}\log\mcM}{\sigma_t^2}
    +\frac{\mcM^{-2\beta}}{\sigma_t^2}
    \lesssim\frac{\mcM^{-2\beta}\log\mcM}{\sigma_t^2}.
\end{aligned}
\]
All bounds are uniform over $t\in[T,1-T]$ and $\mx\in[-1,1]^{d_{\mcX}}$. The input clipping adds only a bounded number of layers and parameters, with weights bounded by $C_0\vee1$, and therefore preserves the stated orders of $L,M,J,\kappa$.
\end{proof}

\section{Statistical Error}\label{sec:se}
In this section, we bound the statistical error and prove Lemma \ref{lem: statistical_error}. Then, combining Lemma \ref{lem: approximation_error} and Lemma \ref{lem: statistical_error}, we prove Theorem \ref{thm: generalization}. We begin by providing an upper bound for $\ell_{\mb}(\mx,\my_0^F)$.
\\\\
\noindent \textbf{Upper bound for $\ell_{\mb}(\mx,\my_0^F)$.}
By the definition of $\ell_{\mb}(\mx,\my_0^F)$, it holds that
$$
\Ebb_{\mz}\left\Vert \mb(t, m_t\my_0^F + \sigma_t\mz,\mx) + \frac{\mz}{\sigma_t}\right\Vert^2 \lesssim 
\frac{\log \mcM}{\sigma_t^2} + \frac{\Ebb\Vert\mz\Vert^2}{\sigma_t^2}
\lesssim \frac{\log \mcM}{t}.
$$
Thus, we have
$$
\ell_{\mb}(\mx,\my_0^F) \lesssim \frac{\log \mcM}{1-{2\mcM}^{-C_T}}\int_T^{1-T}\frac{1}{t(1-t)}\mrd t \lesssim \log^2 \mcM
$$
for sufficiently large $\mcM$.
\\\\
\noindent\textbf{Lipschitz continuity for $\ell_{\mb}(\mx,\my_0^F)$.}
Note that we restrict ReLU neural networks into class $\mcC$. For any $\mb_1$, $\mb_2 \in \mcC$, by the construction structure of $\mb_1$, $\mb_2$,
we have
\begin{equation*}
\begin{aligned}
&~~~\left|\ell_{\mb_1}(\mx,\my_0^F) - \ell_{\mb_2}((\mx,\my_0^F))\right|\\
&\leq\frac{1}{1-2T}\int_{T}^{1-T}
\frac{1}{1-t}\cdot\Ebb_{\mz}\Vert \mb_1 - \mb_2\Vert\left\Vert \mb_1 + \mb_2 + \frac{2\mz}{\sigma_t}\right\Vert \mrd t\\
&\leq\frac{1}{1-2T}\int_{T}^{1-T}\frac{1}{1-t}\left(\Ebb_{\mz}\Vert \mb_1 - \mb_2 \Vert^2\right)^{\frac{1}{2}}\left(\Ebb_{\mz}\left\Vert \mb_1 + \mb_2 + \frac{2\mz}{\sigma_t}\right\Vert^2\right)^{\frac{1}{2}}\mrd t\\
&\lesssim \frac{1}{1-2T}\int_{T}^{1-T}\frac{1}{1-t}\left(\Ebb_{\mz}\Vert \mb_1 - \mb_2 \Vert^2\right)^{\frac{1}{2}}\left(\frac{\log \mcM}{\sigma_t^2} + \frac{d_\mcY}{\sigma_t^2}\right)^{\frac{1}{2}}\mrd t\\
&\lesssim \frac{1}{1-2T}\left(\int_{T}^{1-T}\frac{1}{1-t}\cdot\Ebb_{\mz}\left\Vert \mb_1 - \mb_2 \right\Vert^2 \mrd t\right)^{\frac{1}{2}}\left(\log \mcM \cdot \int_{T}^{1-T}\frac{1}{t(1-t)(2-t)}\mrd t\right)^{\frac{1}{2}}\\
&\lesssim  \log \mcM \cdot \left(\int_{T}^{1-T}\frac{1}{1-t}\cdot\Ebb_{\mz}\Vert \mb_1 - \mb_2 \Vert^2 \mrd t\right)^{\frac{1}{2}}\\
&\lesssim  \log^{\frac{3}{2}} \mcM \cdot
\Vert \mb_1 - \mb_2 \Vert_{L^{\infty}([T,1-T]\times\Rbb^{d_\mcY}\times [-1,1]^{d_\mcX})}\\
& \lesssim \log^{\frac{3}{2}} \mcM \cdot \Vert \mb_1 - \mb_2 \Vert_{L^{\infty}([T,1-T]\times[-\mcO(1)\sqrt{\log \mcM}, \mcO(1) \sqrt{\log \mcM}]^{d_\mcY} \times [-1,1]^{d_\mcX})}.
\end{aligned}
\end{equation*}
\\\\
\noindent\textbf{Covering number evaluation.}
Let $\mcN_\delta$ be the minimum cardinality of a $\delta$-cover of $\mcC$ whose centers belong to $\mcC$, under the uniform norm on $[T,1-T]\times\Rbb^{d_\mcY}\times[-1,1]^{d_\mcX}$. An ambient $\delta/2$-cover yields such a cover by choosing one element of $\mcC$ from each nonempty covering ball. Using the input-clipping property and the network covering bound, we obtain
\begin{equation*}
    \begin{aligned}
    \log\mcN_\delta &= \log\mcN\left(\mcC, \delta,\Vert\cdot\Vert_{L^{\infty}([T,1-T]\times\Rbb^{d_\mcY} \times [-1,1]^{d_\mcX})}\right) \\
    &= \log\mcN\left(\mcC, \delta,\Vert\cdot\Vert_{L^{\infty}([T,1-T]\times[-\mcO(1)\sqrt{\log \mcM}, \mcO(1) \sqrt{\log \mcM}]^{d_\mcY}\times[-1,1]^{d_\mcX})}\right)\\
    &\lesssim JL\log{\left(\frac{LM\kappa (1-{\mcM}^{-C_T} \vee \mcO(1)\sqrt{\log \mcM})}{\delta}\right)}\\
    & \lesssim {\mcM}^{d_\mcX + d_\mcY}\log^{13}\mcM\left(\log^4 \mcM + \log \frac{1}{\delta} \right),
    \end{aligned}
\end{equation*}
where we used $L=\mcO(\log^4 \mcM), M=\mcO({\mcM}^{d_\mcX + d_\mcY}\log^7 \mcM), J=\mcO(\mcM^{d_\mcX + d_\mcY}\log^9 \mcM), \kappa=\exp\left(\mcO(\log^4 \mcM)\right)$ and $T=\mcM^{-C_T}$. The details of this derivation can be found in \cite[Lemma 5.3]{CJLZ2022nonparametric}.

\begin{lemma}[Bernstein-type inequality for the denoising excess loss]
\label{lem: bernstein_denoising_loss}
Suppose that Assumption~\ref{ass: bounded_density} holds, $T=\mcM^{-C_T}$ for a fixed constant $C_T>0$, and $\mb\in\mcC$. Then, for all sufficiently large $\mcM$,
\begin{equation}\label{eq: bernstein_mean_identity}
\begin{aligned}
&\Ebb_{\mx,\my_0^F}\left[\ell_{\mb}(\mx,\my_0^F)-\ell_{\mb^*}(\mx,\my_0^F)\right]\\
&\quad=\frac{1}{1-2T}\int_T^{1-T}\frac{1}{1-t}
\Ebb_{\mx,\my_t^F}\left\|\mb(t,\my_t^F,\mx)-\mb^*(t,\my_t^F,\mx)\right\|^2\mrd t\geq0.
\end{aligned}
\end{equation}
Moreover,
\begin{equation}\label{eq: bernstein_loss_envelope}
\left|\ell_{\mb}(\mx,\my_0^F)-\ell_{\mb^*}(\mx,\my_0^F)\right|\lesssim\log^2\mcM
\end{equation}
uniformly over $(\mx,\my_0^F)\in[-1,1]^{d_\mcX}\times[-1,1]^{d_\mcY}$, and
\begin{equation}\label{eq: bernstein_denoising_loss}
    \Ebb_{\mx,\my_0^F}\left[\left( \ell_{\mb}(\mx,\my_0^F)-\ell_{\mb^*}(\mx,\my_0^F)\right)^2\right]\lesssim\log^2\mcM\,
    \Ebb_{\mx,\my_0^F}\left[\ell_{\mb}(\mx,\my_0^F)-\ell_{\mb^*}(\mx,\my_0^F)\right].
\end{equation}
\end{lemma}

\begin{proof}
For clarity, throughout this proof $(\mx,\my_0^F)$ is held fixed whenever only the expectation with respect to $(t,\mz)$ is displayed, where $t$ is uniform on $[T,1-T]$, $\mz\sim\mcN(\mathbf0,\mI_{d_\mcY})$, and
\[
\my_t^F=m_t\my_0^F+\sigma_t\mz,\qquad m_t=1-t,\qquad \sigma_t^2=t(2-t).
\]
By this structure, note we have
\begin{equation}\label{eq: denoising_orthogonality}
\Ebb\left[-\frac{\mz}{\sigma_t}\,\middle|\,t,\my_t^F,\mx\right]
=\nabla\log p_t(\my_t^F\mid\mx)=\mb^*(t,\my_t^F,\mx).
\end{equation}
We introduce the following two temporary abbreviations
\[
D_{\mb}(t):=\mb(t,\my_t^F,\mx)-\mb^*(t,\my_t^F,\mx),\qquad
R_t:=\mb^*(t,\my_t^F,\mx)+\frac{\mz}{\sigma_t}.
\]
Then for each fixed $(\mx,\my_0^F)$,
\begin{equation}\label{eq: conditional_denoising_decomposition}
    \ell_{\mb}(\mx,\my_0^F)-\ell_{\mb^*}(\mx,\my_0^F)
    =\frac{1}{1-2T}\int_T^{1-T}\frac{1}{1-t}\Ebb_{\mz}
    \left[\|D_{\mb}(t)\|^2+2\langle D_{\mb}(t),R_t\rangle\right]\mrd t.
\end{equation}
Here $D_{\mb}(t)$ is measurable with respect to $(t,\my_t^F,\mx)$, and \eqref{eq: denoising_orthogonality} implies $\Ebb[R_t\mid t,\my_t^F,\mx]=\mathbf0$. Therefore, taking expectation over $(\mx,\my_0^F,t,\mz)$ and we get
\[
\Ebb\left[\frac{\langle D_{\mb}(t),R_t\rangle}{1-t}\right]
=\Ebb\left[\frac{\left\langle D_{\mb}(t),\Ebb[R_t\mid t,\my_t^F,\mx]\right\rangle}{1-t}\right]=0.
\]
Taking expectations in \eqref{eq: conditional_denoising_decomposition} proves \eqref{eq: bernstein_mean_identity}.

We next establish the deterministic bounds required for the second-moment inequality. The envelope condition defining $\mcC$ gives
\begin{equation}\label{eq: bernstein_network_bound}
\|\mb(t,\my_t^F,\mx)\|^2\lesssim\frac{\log\mcM}{\sigma_t^2}.
\end{equation}
By Lemma \ref{lem: derivatives_boundness}, it holds
\begin{equation}\label{eq: bernstein_score_bound}
    \|\mb^*(t,\my_t^F,\mx)\|^2
    \lesssim \frac{1}{\sigma_t^2}\left(\frac{(\|\my_t^F\|_\infty-m_t)_+}{\sigma_t}\vee1\right)^2
    \lesssim \frac{1+\|\mz\|_\infty^2}{\sigma_t^2}.    
\end{equation}
Combining \eqref{eq: bernstein_network_bound} and \eqref{eq: bernstein_score_bound}, we obtain uniformly in $(\mx,\my_0^F)$
\begin{equation}\label{eq: bernstein_D_integral}
\begin{aligned}
    \frac{1}{1-2T}\int_T^{1-T}\frac{1}{1-t}\Ebb_{\mz}\|D_{\mb}(t)\|^2\mrd t
    &\lesssim (\log\mcM+1)\frac{1}{1-2T}\int_T^{1-T}\frac{\mrd t}{t(1-t)(2-t)}\\
    &\lesssim (\log\mcM+1)\log\frac1T
    \lesssim \log^2\mcM.
\end{aligned}
\end{equation}

Finally we prove the Bernstein-type inequality \eqref{eq: bernstein_denoising_loss}. By the same approach we obtain
\begin{equation*}
    \frac{1}{1-2T}\int_T^{1-T}\frac{1}{1-t}\Ebb_{\mz}\|R_t\|^2\mrd t
    \lesssim \frac{1}{1-2T}\int_T^{1-T}\frac{\mrd t}{t(1-t)(2-t)}
    \lesssim \log\mcM.
\end{equation*}
Applying the weighted Cauchy--Schwarz gives
\begin{equation}\label{eq: bernstein_cross_bound}
\begin{aligned}
&~~~~\left|\frac{1}{1-2T}\int_T^{1-T}\frac{1}{1-t}
\Ebb_{\mz}\langle D_{\mb}(t),R_t\rangle\mrd t\right|^2\\
&\leq
\left\{\frac{1}{1-2T}\int_T^{1-T}\frac{1}{1-t}\Ebb_{\mz}\|D_{\mb}(t)\|^2\mrd t\right\}
\left\{\frac{1}{1-2T}\int_T^{1-T}\frac{1}{1-t}\Ebb_{\mz}\|R_t\|^2\mrd t\right\}\\
&\lesssim \log^2\mcM\left\{\frac{1}{1-2T}\int_T^{1-T}\frac{1}{1-t}\Ebb_{\mz}\|D_{\mb}(t)\|^2\mrd t\right\}.
\end{aligned}
\end{equation}
Thus applying $(a+2c)^2\leq2a^2+8c^2$ to \eqref{eq: conditional_denoising_decomposition}, followed by \eqref{eq: bernstein_D_integral} and \eqref{eq: bernstein_cross_bound}, gives pointwise in $(\mx,\my_0^F)$ that
\begin{align*}
&~~~~\left(\ell_{\mb}(\mx,\my_0^F)-\ell_{\mb^*}(\mx,\my_0^F)\right)^2\\
&\lesssim \left|\frac{1}{1-2T}\int_T^{1-T}\frac{1}{1-t}\Ebb_{\mz}\|D_{\mb}(t)\|^2\mrd t\right|^2 + \left|\frac{1}{1-2T}\int_T^{1-T}\frac{1}{1-t} \Ebb_{\mz}\langle D_{\mb}(t),R_t\rangle\mrd t\right|^2\\
&\lesssim \log^2\mcM \left\{\frac{1}{1-2T}\int_T^{1-T}\frac{1}{1-t}\Ebb_{\mz}\|D_{\mb}(t)\|^2\mrd t\right\}.
\end{align*}
Taking expectation over $(\mx,\my_0^F)$ and using \eqref{eq: bernstein_mean_identity} proves \eqref{eq: bernstein_denoising_loss}. The corresponding variance conclusion follows from
\[
\operatorname{Var}\left(\ell_{\mb}(\mx,\my_0^F)-\ell_{\mb^*}(\mx,\my_0^F)\right)
\leq\Ebb\left[\left(\ell_{\mb}(\mx,\my_0^F)-\ell_{\mb^*}(\mx,\my_0^F)\right)^2\right].
\]
\end{proof}

\begin{proof}[Proof of Lemma \ref{lem: statistical_error}]
Let $\ell(\mb,\mx,\my_0^F) = \ell_{\mb}(\mx,\my_0^F) - \ell_{\mb^*} (\mx,\my_0^F)$ and $\mcD^{\prime} = \{(\mx_1^{\prime},\my_{0,1}^{F,\prime}), \cdots, (\mx_n^{\prime},\my_{0,n}^{F,\prime})\}$ be an independent copy of $\mcD$, 
then for any $\mb_1$, $\mb_2\in\mcC$, there exists a constant $C_8 > 0$ such that
$$
\begin{aligned}
|\ell(\mb_1,\mx,\my_0^F) - \ell(\mb_2,\mx,\my_0^F)| & = |\ell_{\mb_1}(\mx,\my_0^F) - \ell_{\mb_2}(\mx,\my_0^F)| \\
&\leq C_8\log^{\frac{3}{2}} \mcM \Vert \mb_1 - \mb_2 \Vert_{L^{\infty}([T,1-T]\times\Rbb^{d_\mcY}\times[-1,1]^{d_\mcX})}. 
\end{aligned}
$$
We first estimate $\Ebb_{\mcD,\mcT,\mcZ}\left(\mcL(\wh{\mb}) - 2\ov{\mcL}_{\mcD}(\wh{\mb}) + \mcL(\mb^*)\right)$. 
It follows that
\begin{equation*}
    \begin{aligned}
&\Ebb_{\mcD,\mcT,\mcZ}\left(\mcL(\wh{\mb}) - 2\ov{\mcL}_{\mcD}(\wh{\mb}) + \mcL(\mb^*)\right)\\ 
&= \Ebb_{\mcD,\mcT,\mcZ}\left(\Ebb_{\mcD^{\prime}}\left[\frac{1}{n}\sum_{i=1}^{n}\left(\ell_{\wh{\mb}}(\mx_i^{\prime},\my_{0,i}^{F,\prime}) - \ell_{\mb^*}(\mx_i^{\prime},\my_{0,i}^{F,\prime})\right)\right]-\frac{2}{n}\sum_{i=1}^{n}\left(\ell_{\wh{\mb}}(\mx_i,\my_{0,i}^F) - \ell_{\mb^*}(\mx_i,\my_{0,i}^F)\right)
        \right)\\
        &=\Ebb_{\mcD,\mcT,\mcZ}\left[\frac{1}{n}\sum_{i=1}^{n}G(\wh{\mb},\mx_i,\my_{0,i}^F)\right],
    \end{aligned}
\end{equation*}
where
$$
G(\wh{\mb},\mx,\my_{0}^F) := \Ebb_{\mcD^{\prime}}\left[\ell(\wh{\mb},\mx_i^{\prime},\my_{0,i}^{F,\prime}) - 2\ell(\wh{\mb},\mx,\my_{0}^F)\right].
$$

Choose a deterministic $\delta$-cover $\mcC_\delta\subseteq\mcC$ with cardinality $\mcN_\delta$. Then, for every $\mb\in\mcC$, there exists $\mb_\delta\in\mcC_\delta$ such that
$$
\begin{aligned}
|\ell(\mb,\mx,\my_0^F) - \ell(\mb_{\delta},\mx,\my_0^F)| &\leq C_8\log^{\frac{3}{2}} \mcM \Vert \mb - \mb_{\delta}\Vert_{L^{\infty}([T,1-T]\times\Rbb^{d_\mcY}\times[-1,1]^{d_\mcX})} \\
&\leq C_8\delta\log^{\frac{3}{2}} \mcM.
\end{aligned}
$$
Therefore, for $1 \leq i \leq n$, we have
$$
G(\wh{\mb},\mx_i,\my_{0,i}^F)\leq G(\mb_{\delta},\mx_i,\my_{0,i}^F) + 3C_8\delta\log^{\frac{3}{2}} \mcM.
$$
Since $|\ell(\mb_{\delta},\mx_i,\my_{0,i}^{F})|\lesssim \log^2 \mcM$, there exists a constant $C_9 > 0$ such that
$|\ell(\mb_{\delta},\mx_i,\my_{0,i}^F)| \leq C_9\log^2 \mcM$. By \eqref{eq: bernstein_loss_envelope}, $|\ell(\mb_{\delta}, \mx_i,\my_{0,i}^F) - \Ebb\ell(\mb_{\delta},\mx_i,\my_{0,i}^F)|\leq 2C_9\log^2 \mcM$. We denote $V^2 = \mathrm{Var}[\ell(\mb_{\delta},\mx_i,\my_{0,i}^F)]$, then by Lemma \ref{lem: bernstein_denoising_loss}, after enlarging the existing constant $C_8$ if necessary, we have
\begin{equation*}
\begin{aligned}
V^2&\leq\Ebb_{\mcD}[\ell(\mb_{\delta},\mx_i,\my_{0,i}^F)^2]\\
&\leq C_8^2\log^2\mcM\cdot\Ebb_{\mcD}[\ell(\mb_{\delta},\mx_i,\my_{0,i}^F)].
\end{aligned}
\end{equation*}
Consequently,
\[
\Ebb_{\mcD}[\ell(\mb_{\delta},\mx_i,\my_{0,i}^F)]\geq\frac{V^2}{C_8^2\log^2\mcM}.
\]
By Bernstein's inequality, for any $t > 0$, we have
\begin{equation*}
    \begin{aligned}
    &~\Pbb_{\mcD,\mcT,\mcZ}\left[\frac{1}{n}\sum_{i=1}^{n}G(\mb_{\delta},\mx_i,\my_{0,i}^F) > t\right]\\
    =&~\Pbb_{\mcD,\mcT,\mcZ}\left(\Ebb_{\mcD^{\prime}}\left[\frac{1}{n}\sum_{i=1}^{n}\ell(\mb_{\delta}, \mx_i^{\prime},\my_{0,i}^{F,\prime})\right] - \frac{1}{n}\sum_{i=1}^{n}\ell(\mb_{\delta},\mx_i,\my_{0,i}^F) > \frac{t}{2} + \Ebb_{\mcD^{\prime}}\left[\frac{1}{2n}\sum_{i=1}^{n}\ell(\mb_{\delta}, \mx_i^{\prime},\my_{0,i}^{F,\prime})\right]\right)\\
    =&~\Pbb_{\mcD,\mcT,\mcZ}\left(\Ebb_{\mcD}\left[\frac{1}{n}\sum_{i=1}^{n}\ell(\mb_{\delta}, \mx_i,\my_{0,i}^F)\right] - \frac{1}{n}\sum_{i=1}^{n}\ell(\mb_{\delta},\mx_i,\my_{0,i}^F) > \frac{t}{2} + \Ebb_{\mcD}\left[\frac{1}{2n}\sum_{i=1}^{n}\ell(\mb_{\delta}, \mx_i,\my_{0,i}^F)\right]\right)\\
    \leq&~\Pbb_{\mcD,\mcT,\mcZ}\left(\Ebb_{\mcD}\left[\frac{1}{n}\sum_{i=1}^{n}\ell(\mb_{\delta}, \mx_i,\my_{0,i}^F)\right] - \frac{1}{n}\sum_{i=1}^{n}\ell(\mb_{\delta},\mx_i,\my_{0,i}^F) > \frac{t}{2} + \frac{V^2}{2C_8^2\log^2 \mcM}\right)\\
    \leq&~\exp\left(-\frac{nu^2}{2V^2 + \frac{4uC_9\log^2 \mcM}{3}}\right)\\
    \leq&\exp\left(-\frac{nt}{8\left(C_8^2 + \frac{C_9}{3}\right)\log^2 \mcM}\right),
    \end{aligned}
\end{equation*}
where $u = \frac{t}{2} + \frac{V^2}{2C_8^2\log^2 \mcM}$, and we use $u\geq\frac{t}{2}$ and $V^2\leq 2uC_8^2\log^2 \mcM $. 
Hence,  for any $t > 3C_8\delta\log^{\frac{3}{2}} \mcM$, we have
\begin{equation*}
    \begin{aligned}
        \Pbb_{\mcD,\mcT,\mcZ}\left[\frac{1}{n}\sum_{i=1}^{n}G(\wh{\mb}, \mx_i,\my_{0,i}^F) > t\right]
        &\leq\Pbb_{\mcD,\mcT,\mcZ}\left[\mathop{\mathrm{sup}}_{\mb\in\mcC}\frac{1}{n}\sum_{i=1}^{n}G(\mb, \mx_i,\my_{0,i}^F) > t\right]\\
        &\leq\Pbb_{\mcD,\mcT,\mcZ}\left[\mathop{\mathrm{max}}_{\mb_{\delta}\in \mcC_{\delta}}\frac{1}{n}\sum_{i=1}^{n}G(\mb_{\delta}, \mx_i,\my_{0,i}^F) > t- 3C_8\delta\log^{\frac{3}{2}}\mcM\right]\\
        &\leq\mcN_{\delta}\mathop{\mathrm{max}}_{\mb_{\delta}\in \mcC_{\delta}}\Pbb_{\mcD,\mcT,\mcZ}\left[\frac{1}{n}\sum_{i=1}^{n}G(\mb_{\delta}, \mx_i,\my_{0.i}^F) > t- 3C_8\delta\log^{\frac{3}{2}} \mcM\right]\\
        &\leq\mcN_{\delta}\exp\left(-\frac{n(t-3C_8\delta\log^{\frac{3}{2}} \mcM)}{8\left(C_8^2 + \frac{C_9}{3}\right)\log^2 \mcM}\right).
    \end{aligned}
\end{equation*}
By setting $a=\left[3C_8\log ^{\frac{3}{2}}\mcM + 8\left(C_8^2 + \frac{C_9}{3} \right)\log^2 \mcM\log\mcN_{\delta} \right] \delta$ and $\delta = \frac{1}{n}$, then we obtain
\begin{equation}\label{eq: statsitical_error1}
    \begin{aligned}
        & ~~~~ \Ebb_{\mcD,\mcT,\mcZ}\left[\frac{1}{n}\sum_{i=1}^{n}G(\wh{\mb},\mx_i,\my_{0,i}^F)\right]
        \\&\leq \int_0^{+\infty}\Pbb\left(\frac{1}{n}\sum_{i=1}^{n}G(\wh{\mb},\mx_i,\my_{0,i}^F) > t\right) \mathrm{d}t \\ 
        &\leq a + \mcN_{\delta}\int_{a}^{\infty}\exp\left(-\frac{n(t-3C_8\delta\log^{\frac{3}{2}} \mcM)}{8\left(C_8^2 + \frac{C_9}{3}\right)\log^2 \mcM}\right)\mathrm{d}t\\
        & \leq \left[3C_8\log^{\frac{3}{2}} \mcM + 8\left(C_8^2 + \frac{C_9}{3} \right)\log^2 \mcM\log\mcN_{\delta} \right] \delta + \frac{8\left(C_8^2 + \frac{C_9}{3} \right)\log^2 \mcM}{n} \\
        & \lesssim {\mcM}^{d_\mcX + d_\mcY}\log^{15} \mcM\left(\log^4 \mcM + \log\frac{1}{\delta}\right) \delta + \frac{\log^2 \mcM}{n} \\
        & \lesssim \frac{{\mcM}^{d_\mcX + d_\mcY}\log^{15} \mcM\left(\log^4 \mcM + \log n\right)}{n}.
    \end{aligned}
\end{equation}

Next, we estimate $\Ebb_{\mcD,\mcT,\mcZ}\left(\ov{\mcL}_{\mcD}(\wh{\mb}) - \wh{\mcL}_{\mcD,\mcT,\mcZ}(\wh{\mb})\right)$.
Recall that
$$
\ov{\mcL}_{\mcD}(\wh{\mb}) - \wh{\mcL}_{\mcD,\mcT,\mcZ}(\wh{\mb}) = \frac{1}{n}\sum_{i=1}^{n}\left(\ell_{\wh{\mb}}(\mx_i,\my_{0,i}^F) - \wh{\ell}_{\wh{\mb}}(\mx_i,\my_{0,i}^F)\right).
$$
We decompose $\frac{1}{n}\sum_{i=1}^{n}\left(\ell_{\wh{\mb}}(\mx_i,\my_{0,i}^F) - \wh{\ell}_{\wh{\mb}}(\mx_i,\my_{0,i}^F)\right)$ into 
the following three terms:
$$
\begin{aligned}
&~~~~ \frac{1}{n}\sum_{i=1}^{n}\left(\ell_{\wh{\mb}}(\mx_i,\my_{0,i}^F) - \wh{\ell}_{\wh{\mb}}(\mx_i,\my_{0,i}^F)\right) \\
&= 
\underbrace{\frac{1}{n}\sum_{i=1}^{n}(\ell_{\wh{\mb}}(\mx_i,\my_{0,i}^F) - \ell_{\wh{\mb}}^{\mathrm{trunc}}(\mx_i,\my_{0,i}^F)
)}_{(A)} \\
& ~~~~ +\underbrace{\frac{1}{n}\sum_{i=1}^{n}(\ell_{\wh{\mb}}^{\mathrm{trunc}}(\mx_i,\my_{0,i}^F) - \wh{\ell}_{\wh{\mb}}^{\mathrm{trunc}}(\mx_i,\my_{0,i}^F)
)}_{(B)} \\
& ~~~~ +\underbrace{\frac{1}{n}\sum_{i=1}^{n}
(\wh{\ell}_{\wh{\mb}}^{\mathrm{trunc}}(\mx_i,\my_{0,i}^F) - \wh{\ell}_{\wh{\mb}}(\mx_i,\my_{0,i}^F)
)}_{(C)},
\end{aligned}
$$
where 
$$
\ell_{\wh{\mb}}^{\mathrm{trunc}}(\mx,\my_{0}^F) := \Ebb_{\mz}\left(\frac{1}{1-2T}\int_{T}^{1-T}\frac{1}{1-t}\left\Vert\wh{\mb}(t,m_t\my_{0}^F + \sigma_t\mz, \mx) + \frac{\mz}{\sigma_t}\right\Vert^2\mrd t\cdot\mathbf{I}_{\{\Vert\mz\Vert_{\infty} \leq r\}}\right),
$$
and
$$
\wh{\ell}_{\wh{\mb}}^{\mathrm{trunc}}(\mx,\my_{0}^F) := \frac{1}{m}\sum_{j=1}^{m}\frac{1}{1-t_j}\left\Vert\wh{\mb}(t_j,m_{t_j}\my_{0}^F + \sigma_{t_j}\mz_j,\mx) + \frac{\mz_j}{\sigma_{t_j}}\right\Vert^2\mathbf{I}_{\{\Vert\mz_j\Vert_{\infty} \leq r\}}.
$$
We estimate these three terms separately.
Firstly, 
\begin{equation*}
    \begin{aligned}
        (A) &= \frac{1}{n}\sum_{i=1}^{n}\Ebb_{\mz}\left(\frac{1}{1-2T}\int_{T}^{1-T}\frac{1}{1-t}\cdot\left\Vert\wh{\mb}(t,m_t\my_{0,i}^F+\sigma_t\mz,\mx_i) + \frac{\mz}{\sigma_t}\right\Vert^2\mrd t\cdot\mathbf{I}_{\{\Vert\mz\Vert_{\infty} > r\}}\right)\\
        &\lesssim \left(\log \mcM\cdot\Pbb(\Vert\mz\Vert_{\infty} > r) + \Ebb\left(
        \Vert\mz\Vert^2\mathbf{I}_{\{\Vert\mz\Vert_{\infty}>r\}}\right)\right) \cdot \frac{1}{1-2T}\int_{T}^{1-T}\frac{1}{t(1-t)}\mrd t\\
        &\lesssim \log \mcM\left(\log \mcM \cdot \Pbb(\Vert\mz\Vert_{\infty} > r) + [\Ebb
        (\Vert\mz\Vert^4)]^{\frac{1}{2}}\cdot\Pbb({\Vert\mz\Vert_{\infty}>r})^{\frac{1}{2}} \right)\\
        &\lesssim \log^2 \mcM \cdot \Pbb({\Vert\mz\Vert_{\infty}>r})^{\frac{1}{2}}\\
        &\lesssim \log^2 \mcM\exp\left(-\frac{r^2}{4}\right).
    \end{aligned}
\end{equation*}
Therefore, there exists a constant $C_{10} > 0$ such that
$$
\Ebb_{\mcD,\mcT,\mcZ}\left(\frac{1}{n}\sum_{i=1}^{n}(\ell_{\wh{\mb}}(\mx_i,\my_{0,i}^F) - \ell_{\wh{\mb}}^{\mathrm{trunc}}(\mx_i,\my_{0,i}^F)
)\right) \leq C_{10}\log^2 \mcM\exp\left(-\frac{r^2}{4}\right).
$$

Next, we bound the second term. Let
$$
h_{\mb}(t,\mx,\my_{0}^F,\mz) := \frac{1}{1-t}\left\Vert \mb(t,m_t\my_{0}^F + \sigma_t\mz, \mx) + \frac{\mz}{\sigma_t}\right\Vert^2\mathbf{I}_{\{\Vert\mz\Vert_{\infty}\leq r\}},
$$ 
then there exists a constant $C_{11} > 0$ such that
$$
0 \leq h_{\mb}(t,\mx,\my_0^F\mz)\lesssim \frac{r^2 + \log \mcM}{T} \leq C_{11} {\mcM}^{C_T}(r^2 + \log \mcM):= E_{\mcM}(r).
$$
For any $\delta_1 > 0$, $\mb\in\mcC$, there exists a $\mb_{\delta_1}\in \mcC_{\delta_1}$ and a constant $C_{12} > 0$ such that
\begin{equation*}
\begin{aligned}
|h_{\mb}(t,\mx,\my_0^F,\mz) - h_{\mb_{\delta_1}}(t,\mx,\my_0^F, \mz)|
&\lesssim \frac{\delta_1}{1-t}\left\Vert \mb + \mb_{\delta_1} + \frac{2\mz}{\sigma_t}\right\Vert\mathbf{I}_{\{\Vert\mz\Vert_{\infty}\leq r\}}\\
& \lesssim \frac{r + \sqrt{\log \mcM}}{T} \cdot \delta_1 \\
&\leq C_{12} {\mcM}^{C_T}(r + \sqrt{\log \mcM})\delta_1.
\end{aligned}
\end{equation*}
Then, for fixed pair $(\mx_i,\my_{0,i}^F), 1\leq i \leq n$, we have
\begin{equation*}
\begin{aligned}
&\frac{1}{1-2T}\int_{T}^{1-T}\Ebb_{\mz}h_{\wh{\mb}}(t,\mx_i, \my_{0,i}^F, \mz)\mrd t - \frac{1}{m}\sum_{j=1}^{m}h_{\wh{\mb}}(t_j,\mx_i,\my_{0,i}^F,\mz_j)\\
\leq&\mathop{\mathrm{sup}}_{\mb\in\mcC}\left(\frac{1}{1-2T}\int_{T}^{1-T}\Ebb_{\mz}h_{\mb}(t,\mx_i,\my_{0,i}^F, \mz)\mrd t - \frac{1}{m}\sum_{j=1}^{m}h_{\mb}(t_j,\mx_i,\my_{0,i}^F,\mz_j)\right)\\
\leq&\mathop{\mathrm{max}}_{\mb_{\delta}\in \mcC_{\delta_1}}\left(\frac{1}{1-2T}\int_{T}^{1-T}\Ebb_{\mz}h_{\mb_{\delta_1}}(t,\mx_i,\my_{0,i}^F, \mz)\mrd t - \frac{1}{m}\sum_{j=1}^{m}h_{\mb_{\delta_1}}(t_j,\mx_i,\my_{0,i}^F,\mz_j)\right) \\
&~~~~~~~~~ + 2C_{12}{\mcM}^{C_T}(r + \sqrt{\log \mcM})\delta_1.
\end{aligned}
\end{equation*}
Let $b = 2C_{12}{\mcM}^{C_T}(r + \sqrt{\log \mcM})\delta_1$. For $t > b$,  by Hoeffding's inequality, we obtain
$$
\begin{aligned}
&~\Pbb_{\mcT,\mcZ}\left(\frac{1}{1-2T}\int_{T}^{1-T}\Ebb_{\mz}h_{\wh{\mb}}(t,\mx_i, \my_{0,i}^F, \mz)\mrd t - \frac{1}{m}\sum_{j=1}^{m}h_{\wh{\mb}}(t_j,\mx_i,\my_{0,i}^F,\mz_j) > t
\right)\\
\leq&~\Pbb_{\mcT,\mcZ}\left(\mathop{\mathrm{max}}_{\mb_{\delta_1}\in \mcC_{\delta_1}}\left(\frac{1}{1-2T}\int_{T}^{1-T}\Ebb_{\mz}h_{\mb_{\delta_1}}(t,\mx_i,\my_{0,i}^F, \mz)\mrd t - \frac{1}{m}\sum_{j=1}^{m}h_{\mb_{\delta_1}}(t_j,\mx_i,\my_{0,i}^F,\mz_j)\right) > t - b
\right)\\
\leq&~\mcN_{\delta_1}\mathop{\mathrm{max}}_{\mb_{\delta_1}\in \mcC_{\delta_1}}\Pbb_{\mcT,\mcZ}\left(\frac{1}{1-2T}\int_{T}^{1-T}\Ebb_{\mz}h_{\mb_{\delta_1}}(t,\mx_i,\my_{0,i}^F, \mz)\mrd t - \frac{1}{m}\sum_{j=1}^{m}h_{\mb_{\delta_1}}(t_j,\mx_i,\my_{0,i}^F,\mz_j) > t - b
\right)\\
\leq&~\mcN_{\delta_1}\exp{\left(-\frac{2m(t-b)^2}{E_{\mcM}^2(r)}\right)}.
\end{aligned}
$$
Therefore, by taking  expectation  over 
$\mcT,\mcZ$,  for any 
$c_0 > 0$, we deduce that $(B)$ satisfies
$$
\begin{aligned}
\Ebb_{\mcT,\mcZ}\left(\ell_{\wh{\mb}}^{\mathrm{trunc}}(\mx_i,\my_{0,i}^F) - \wh{\ell}_{\wh{\mb}}^{\mathrm{trunc}}(\mx_i,\my_{0,i}^F)\right)  &\leq\int_{0}^{+\infty}\Pbb_{\mcT,\mcZ}\left(
\ell_{\wh{\mb}}^{\mathrm{trunc}}(\mx_i,\my_{0,i}^F) - \wh{\ell}_{\wh{\mb}}^{\mathrm{trunc}}(\mx_i,\my_{0,i}^F) > t
\right)\mrd t\\
&\leq b + c_0 + \mcN_{\delta_1}\int_{c_0}^{+\infty}\exp{\left(-\frac{2mt^2}{E_{\mcM}^2(r)}\right)}\mrd t\\
&\leq b + c_0 + \frac{\sqrt{\pi}}{2}\mcN_{\delta_1}\exp{\left(-\frac{2mc_0^2}{E_{\mcM}^2(r)}\right)}\frac{E_{\mcM}(r)}{\sqrt{2m}}.
\end{aligned}
$$
Thus, we have
$$
\begin{aligned}
&~\Ebb_{\mcD,\mcT,\mcZ}\left(\frac{1}{n}\sum_{i=1}^{n}\left(\ell_{\wh{\mb}}^{\mathrm{trunc}}(\mx_i,\my_{0,i}^F) - \wh{\ell}_{\wh{\mb}}^{\mathrm{trunc}}(\mx_i,\my_{0,i}^F)\right)\right) \\
=&~\frac{1}{n}\sum_{i=1}^{n}\Ebb_{\mcD}\left[\Ebb_{\mcT,\mcZ}\left(\ell_{\wh{\mb}}^{\mathrm{trunc}}(\mx_i,\my_{0,i}^F) - \wh{\ell}_{\wh{\mb}}^{\mathrm{trunc}}(\mx_i,\my_{0,i}^F)\right)\right]\\
\leq&~b + c_0 + \frac{\sqrt{\pi}}{2}\mcN_{\delta_1}\exp{\left(-\frac{2mc_0^2}{E_{\mcM}^2(r)}\right)}\frac{E_{\mcM}(r)}{\sqrt{2m}}.
\end{aligned}
$$

The last term can be expressed as
\[
(C) = -\frac{1}{mn}\sum_{i=1}^{n}\sum_{j=1}^{m}\frac{1}{1-t_j}
\left\| \wh{\mb}(t_j,m_{t_j}\my_{0,i}^F+\sigma_{t_j}\mz_j,\mx_i)+\frac{\mz_j}{\sigma_{t_j}} \right\|^2 \mI_{\{\|\mz_j\|_\infty>r\}}
\le 0.
\]
which implies
$$
\Ebb_{\mcD,\mcT,\mcZ}\left(
\frac{1}{n}\sum_{i=1}^{n}\left(\wh{\ell}_{\wh{\mb}}^{\mathrm{trunc}}(\mx_i,\my_{0,i}^F) - \wh{\ell}_{\wh{\mb}}(\mx_i,\my_{0,i}^F)\right)
\right)\leq 0.
$$
Combining the above inequalities, we have
\begin{equation*}
\Ebb_{\mcD,\mcT,\mcZ}\left[\frac{1}{n}\sum_{i=1}^{n}\left(\ell_{\wh{\mb}}(\mx_i,\my_{0,i}^F) - \wh{\ell}_{\wh{\mb}}(\mx_i,\my_{0,i}^F)\right)\right]
\leq b_0 + c_0 + \frac{\sqrt{\pi}}{2}\mcN_{\delta_1}\exp{\left(-\frac{2mc_0^2}{E_{\mcM}^2(r)}\right)}\frac{E_{\mcM}(r)}{\sqrt{2m}},
\end{equation*}
where $b_0 = C_{10}\log^2 \mcM\exp\left(-\frac{r^2}{4}\right) + 2C_{12}{\mcM}^{C_T}(r + \sqrt{\log \mcM})\delta_1$. 
By setting $c_0 = E_{\mcM}(r)\sqrt{\frac{\log{\mcN_{\delta_1}}}{2m}}$, $r = 2\sqrt{\log{m}}$, and $\delta_1 = \frac{1}{m}$, we obtain
\begin{equation}\label{eq: statsitical_error2}
\begin{aligned}
& ~~~~ \Ebb_{\mcD,\mcT,\mcZ}\left[\frac{1}{n}\sum_{i=1}^{n}\left(\ell_{\wh{\mb}}(\mx_i,\my_{0,i}^F) - \wh{\ell}_{\wh{\mb}}(\mx_i,\my_{0,i}^F)\right)\right]\\
& \leq b_0 + E_{\mcM}(r)\cdot\frac{\sqrt{\log\mcN_{1/m}} + 1}{\sqrt{2m}}\\
& \lesssim \frac{\log^2 \mcM + {\mcM}^{C_T}(\sqrt{\log m} + \sqrt{\log \mcM})}{m} \\
& ~~~~~~~~~~ + {\mcM}^{C_T}(\log m + \log \mcM) \cdot \frac{{\mcM}^{\frac{d_\mcX + d_\mcY}{2}}\log^{\frac{13}{2}}\mcM(\log^2 \mcM + \sqrt{\log m})}{\sqrt{m}} \\
& \lesssim {\mcM}^{C_T}(\log m + \log \mcM) \cdot \frac{{\mcM}^{\frac{d_\mcX + d_\mcY}{2}}\log^{\frac{13}{2}}\mcM(\log^2 \mcM + \sqrt{\log m})}{\sqrt{m}}.
\end{aligned}
\end{equation}
The proof is complete.
\end{proof}

Combining Lemma \ref{lem: approximation_error} and Lemma \ref{lem: statistical_error}, we can prove Theorem \ref{thm: generalization}.

\begin{proof}[Proof of Theorem \ref{thm: generalization}]
By choosing $\mcM = \lfloor n^{\frac{1}{d_\mcX + d_\mcY + 2\beta}} \rfloor + 1 \lesssim n^{\frac{1}{d_\mcX + d_\mcY + 2\beta}}$ in Lemma \ref{lem: approximation_error}, the approximation error can be bounded as
$$
\begin{aligned}
\inf_{\mb\in\mcC}\left(\mcL(\mb) - \mcL(\mb^*)\right) 
& \lesssim \frac{1}{1-2T}\int_T^{1-T}\frac{{\mcM}^{-2\beta}\log \mcM}{(1-t)\sigma_t^2} \mrd t \\
&\lesssim {\mcM}^{-2\beta} \log^2 \mcM \\
&\lesssim n^{-\frac{2\beta}{d_\mcX + d_\mcY +  2\beta}} \log^2 n.
\end{aligned}
$$
Substituting $\mcM = \lfloor n^{\frac{1}{d_\mcX + d_\mcY + 2\beta}} \rfloor + 1 \lesssim n^{\frac{1}{d_\mcX + d_\mcY + 2\beta}}$, $m\asymp\mcM^{d_\mcX+d_\mcY+4\beta+2C_T}$ and $C_T=2\beta/(\beta\wedge1)$ into \eqref{eq: statsitical_error1} and \eqref{eq: statsitical_error2}, and noting that $\mcM^{C_T+(d_\mcX+d_\mcY)/2}m^{-1/2}\asymp\mcM^{-2\beta}$, we obtain
$$
\eqref{eq: statsitical_error1} \lesssim n^{-\frac{2\beta}{d_\mcX + d_\mcY + 2\beta}}\log^{19}n, ~~ \eqref{eq: statsitical_error2} \lesssim n^{-\frac{2\beta}{d_\mcX + d_\mcY + 2\beta}}\log^{\frac{19}{2}}n,
$$
which implies that the statistical error can be bounded as
$$
\Ebb_{\mcD,\mcT,\mcZ}\left(\mcL(\wh{\mb}) - 2\ov{\mcL}_{\mcD}(\wh{\mb}) + \mcL(\mb^*)\right) + 2\Ebb_{\mcD,\mcT,\mcZ}\left(\ov{\mcL}_{\mcD}(\wh{\mb}) - \wh{\mcL}_{\mcD,\mcT,\mcZ}(\wh{\mb})\right) \lesssim n^{-\frac{2\beta}{d_\mcX + d_\mcY + 2\beta}}\log^{19}n.
$$
Therefore, we finally obtain
$$
\begin{aligned}
& ~~~ \Ebb_{\mcD,\mcT,\mcZ}\left(\frac{1}{1-2T}\int_{T}^{1-T}\frac{1}{t}\cdot\Ebb_{\my_t,\mx}\Vert\wh{\mb}(1-t,\my_t,\mx) - \nabla\log p_{1-t}(\my_t|\mx)\Vert^2 \mrd t\right) \\
&= \Ebb_{\mcD,\mcT,\mcZ}\left(\frac{1}{1-2T}\int_{T}^{1-T}\frac{1}{1-t}\cdot\Ebb_{\my_t^F,\mx}\Vert\wh{\mb}(t,\my_t^F,\mx) - \nabla\log p_t(\my_t^F|\mx)\Vert^2 \mrd t\right) \\
&\lesssim n^{-\frac{2\beta}{d_\mcX + d_\mcY + 2\beta}} \log^{19}n.
\end{aligned}
$$
Since the preceding estimates hold uniform over $p_0\in\mcP_\beta(C_l,C_u,R)$ for all sufficiently large $n$, taking the supremum over this class yields the claimed bound. The proof is complete.
\end{proof}

\section{Bound $\Ebb_{\mcD, \mcT,\mcZ}\Ebb_{\mx}\left[\mathrm{TV}(\wt{p}_T(\cdot|\mx), p_0(\cdot|\mx))\right]$}\label{sec:bb}

In this section, we bound $\Ebb_{\mcD, \mcT,\mcZ}\Ebb_{\mx}[\mathrm{TV}(\wt{p}_T(\cdot|\mx), p_0(\cdot|\mx))]$. We have the following decomposition:
$$
\begin{aligned}
\Ebb_{\mcD, \mcT,\mcZ}\Ebb_{\mx}[\mathrm{TV}(\wt{p}_T(\cdot|\mx), p_0(\cdot|\mx))]
 &\leq 
  \Ebb_{\mcD,\mcT,\mcZ}\Ebb_{\mx}[\mathrm{TV}(\wt{p}_T(\cdot|\mx), p_T(\cdot|\mx))] \\
  &~~~~+ \Ebb_{\mx}[\mathrm{TV}(p_T(\cdot|\mx), p_0(\cdot|\mx))].
  \end{aligned}
$$
In the following two subsections, we bound $\Ebb_{\mcD,\mcT,\mcZ}\Ebb_{\mx}[\mathrm{TV}(\wt{p}_T(\cdot|\mx), p_T(\cdot|\mx))]$ and $\Ebb_{\mx}[\mathrm{TV}(p_T(\cdot|\mx), p_0(\cdot|\mx))]$ separately.

\subsection{Bound $\Ebb_{\mx}[\mathrm{TV}(p_T(\cdot|\mx),p_0(\cdot|\mx))]$.}
In this subsection, we bound the term $\Ebb_{\mx}[\mathrm{TV}(p_T(\cdot|\mx),p_0(\cdot|\mx))]$
and prove Lemma \ref{lem: early_stopping}.

\begin{proof}[Proof of Lemma \ref{lem: early_stopping}]
Let $s=\beta\wedge1$, suppose $\beta=r+\theta$ where $\theta\in(0,1]$. We claim first that for any $u,v\in[-1,1]^{d_\mcY}$, there exists $L=\max\{R,\sqrt{d_\mcY}R\}$ such that
\begin{equation}\label{eq:lemma3.5-1}
    |p_0(\mathbf{u}|\mx)-p_0(\mv|\mx)| \le L\|\mathbf{u}-\mv\|^s.  
\end{equation}
Indeed, for $0<\beta\le 1$, the H\"{o}lder continuity of $p_0(\my|\mx)$ implies $|p_0(\mathbf{u}|\mx)-p_0(\mv|\mx)|\le R\|(\mathbf{u},\mx)-(\mv,\mx)\|^\beta=R\|\mathbf{u}-\mv\|^s$. While for $\beta>1$, it holds
\begin{equation*}
\begin{aligned}
|p_0(\mathbf{u}|\mx)-p_0(\mv|\mx)|
&=\left|\int_0^1 \nabla_\my p_0(\mv+t(\mathbf{u}-\mv)|\mx)\cdot(\mathbf{u}-\mv)\,\mrd t\right|\\
&\le\int_0^1 \|\nabla_\my p_0(\mv+t(\mathbf{u}-\mv)|\mx)\|\,\|\mathbf{u}-\mv\|\,\mrd t\\
&\le \sqrt{d_\mcY}\,R\|\mathbf{u}-\mv\|
=\sqrt{d_\mcY}\,R\|\mathbf{u}-\mv\|^s.
\end{aligned}
\end{equation*}
Take $L=\max\{R,\sqrt{d_\mcY}R\}$ and equation \eqref{eq:lemma3.5-1} is proved.\\
\indent Extending the density $p_0(\my|\mx)$ by zero outside $[-1,1]^{d_\mcY}$, we then show that for any $\mh\in\Rbb^{d_\mcY}$, it holds
\begin{equation}\label{eq:lemma3.5-2}
    \|p_0(\cdot-\mh|\mx)-p_0(\cdot|\mx)\|_{L^1(\Rbb^{d_\mcY})}\lesssim\|\mh\|^s.
\end{equation}
Consider first that $\|\mh\|\le1$. Denote by $A_\mh=[-1,1]^{d_\mcY}+\mh$, then
\begin{equation*}
\begin{aligned}
    \hspace{-1.2cm}
    \|p_0(\cdot-\mh|\mx)-p_0(\cdot|\mx)\|_{L^1(\Rbb^{d_\mcY})}
    &= \int_{[-1,1]^{d_\mcY}\cap A_\mh} |p_0(\my-\mh|\mx)-p_0(\my|\mx)| \mrd\my + \int_{[-1,1]^{d_\mcY}\triangle A_\mh} |p_0(\my-\mh|\mx)-p_0(\my|\mx)| \mrd\my\\
    &\le 2^{d_\mcY}L\|\mh\|^s + C_u\left|[-1,1]^{d_\mcY}\triangle A_\mh\right|\\
    &\le 2^{d_\mcY}L\|\mh\|^s + C_u\cdot 2\Big(2^{d_\mcY}-\prod_{j=1}^{d_\mcY}(2-|h_j|)\Big)\\
    &\le 2^{d_\mcY}L\|\mh\|^s + C_u\cdot 2^{d_\mcY}\sum_{j=1}^{d_\mcY}|h_j|\\
    &\le 2^{d_\mcY}L\|\mh\|^s + 2^{d_\mcY}C_u\sqrt{d_\mcY}\|\mh\|\\
    &\le 2^{d_\mcY}\left(L+C_u\sqrt{d_\mcY}\right)\|\mh\|^s \lesssim \|\mh\|^s,
\end{aligned}
\end{equation*}
where the last inequality uses $0<s\le1$ and $\|\mh\|\le1$. For $\|\mh\|\ge1$, we also have
\begin{equation*}
    \|p_0(\cdot-\mh|\mx)-p_0(\cdot|\mx)\|_{L^1(\Rbb^{d_\mcY})}
    \le\|p_0(\cdot-\mh|\mx)\|_{L^1(\Rbb^{d_\mcY})}+\|p_0(\cdot|\mx)\|_{L^1(\Rbb^{d_\mcY})} = 2 \le 2\|\mh\|^s.
\end{equation*}
This proves equation \eqref{eq:lemma3.5-2}.\\
\indent Recall we have $\my_{1-T}\overset{d}{=}\my_T^F\overset{d}{=}(1-T)\my_0^F+\sigma_T\mz$, where $\sigma_T=\sqrt{T(2-T)}$. Then the density of $(1-T)\my_0^F$ can be calculated as $q_{\mx,T}(\my):=(1-T)^{-d_\mcY}p_0(\my/(1-T)|\mx)$, which yields that
\begin{equation*}
    p_T(\my|\mx)=(q_{\mx,T}*\varphi_{\sigma_T})(\my),  
\end{equation*}
where $\varphi_{\sigma_T}$ denotes the density of $\sigma_T\mz$. Let $0<T\le\frac{1}{2}$ so that $a \triangleq 1-T\ge \frac{1}{2}$, then
\begin{equation*}
\begin{aligned}
\|q_{\mx,T}(\cdot)-p_0(\cdot|\mx)\|_1
&=\int_{a[-1,1]^{d_\mcY}}|a^{-d_\mcY}p_0(\my/a|\mx)-p_0(\my|\mx)|\,\mrd\my + \int_{[-1,1]^{d_\mcY}\setminus a[-1,1]^{d_\mcY}}p_0(\my|\mx)\,\mrd\my\\
&\le |a^{-d_\mcY}-1|\int_{a[-1,1]^{d_\mcY}}p_0(\my/a|\mx)\,\mrd\my\\
&~~~~~ +\int_{a[-1,1]^{d_\mcY}}|p_0(\my/a|\mx)-p_0(\my|\mx)|\,\mrd\my+C_u|[-1,1]^{d_\mcY}\setminus a[-1,1]^{d_\mcY}|\\
&\le (a^{-d_\mcY}-1)a^{d_\mcY}\int_{[-1,1]^{d_\mcY}}p_0(\mathbf{u}|\mx)\mrd\mathbf{u} + \int_{a[-1,1]^{d_\mcY}}L\|\my/a-\my\|^s\mrd\my + C_u2^{d_\mcY}(1-a^{d_\mcY})\\
&\le d_\mcY T + L(\sqrt{d_\mcY}T)^s\cdot 2^{d_\mcY} + C_u2^{d_\mcY}d_\mcY T\\
&\le (d_\mcY+2^{d_\mcY}Ld_\mcY^{s/2}+C_u2^{d_\mcY}d_\mcY)T^s
\lesssim T^s.
\end{aligned}
\end{equation*}
And we also have
\begin{equation*}
\begin{aligned}
\|q_{\mx,T}*\varphi_{\sigma_T}(\cdot)-q_{\mx,T}(\cdot)\|_{L^1(\Rbb^{d_\mcY})}
&= \int\left| \int\bigl[q_{\mx,T}(\my-\sigma_T\mz)-q_{\mx,T}(\my)\bigr]\varphi_\mz(\mz)\,\mrd\mz\right|\mrd\my\\
&\le \int\varphi_\mz(\mz) \|q_{\mx,T}(\cdot-\sigma_T\mz)-q_{\mx,T}(\cdot)\|_{L^1(\Rbb^{d_\mcY})}\,\mrd\mz\\
&= \int\varphi_\mz(\mz) \|p_0(\cdot-\sigma_T\mz/a|\mx)-p_0(\cdot|\mx)\|_{L^1(\Rbb^{d_\mcY})}\,\mrd\mz\\
&\lesssim 2^s\sigma_T^s\Ebb[\|\mz\|^s]\\
&\le 2^s(2T)^{s/2}\Ebb[\|\mz\|^2]^{s/2}\\
&\le 2^{3s/2}d_\mcY^{s/2}T^{s/2} \lesssim T^{s/2}.
\end{aligned}
\end{equation*}
With these results, we can bound $\mathrm{TV}(p_T(\cdot|\mx),p_0(\cdot|\mx))$ as follows:
\begin{equation*}
\begin{aligned}
2\mathrm{TV}(p_T(\cdot|\mx),p_0(\cdot|\mx))
&=\|q_{\mx,T}*\varphi_{\sigma_T}(\cdot)-p_0(\cdot|\mx)\|_{L^1(\Rbb^{d_\mcY})}\\
&\le\|q_{\mx,T}*\varphi_{\sigma_T}(\cdot)-q_{\mx,T}(\cdot)\|_{L^1(\Rbb^{d_\mcY})}
    +\|q_{\mx,T}(\cdot)-p_0(\cdot|\mx)\|_{L^1(\Rbb^{d_\mcY})}\\
&\lesssim T^{s/2}+T^s\\
&\lesssim T^{s/2}.
\end{aligned}
\end{equation*}
Taking $C_T=2\beta/s,\ T=\mathcal M^{-C_T}$ and $\mcM=\lfloor n^{\frac{1}{d_\mcX+d_\mcY+2\beta}}\rfloor+1$, the bound becomes
\begin{equation*}
    T^{s/2}
    =\left(\mcM^{-2\beta/s}\right)^{s/2}
    =\mcM^{-\beta}
    \le n^{-\frac{\beta}{d_\mcX+d_\mcY+2\beta}}.
\end{equation*}
Finally, taking supremum over $\mx\in[-1,1]^{d_\mcX}$ proves the lemma.
\end{proof}

\subsection{Bound $\Ebb_{\mcD,\mcT,\mcZ}\Ebb_{\mx}[\mathrm{TV}(\wt{p}_T(\cdot|\mx),p_T(\cdot|\mx))]$}
In this subsection, we bound the term $\Ebb_{\mcD,\mcT,\mcZ}\Ebb_{\mx}[\mathrm{TV}(\wt{p}_T(\cdot|\mx),p_T(\cdot|\mx))]$. We summarize the SDEs we 've defined and introduce a new SDE as follows
\begin{align*}
    \mrd \my_t &= \left[ \frac{\my_t}{t}+\frac{2}{t}\nabla\log p_{1-t}(\my_t|\mx) \right] \mrd t + \sqrt{\frac{2}{t}}\mrd \mw_t,\quad \my_0\sim \mcN(0,\mI_{d_\mcY}),\ \my_1\sim p_0(\my|\mx)\\
    \mrd \ov{\my}_t &= \left[ \frac{\ov{\my}_{t_i}}{t_i}+\frac{2\wh{\mb}(1-t_i,\ov{\my}_{t_i},\mx)}{t_i} \right] \mrd t + \sqrt{\frac{2}{t}}\mrd \mw_t,\quad \ov{\my}_{t_0}=\my_{t_0}=\my_T,\ t\in [t_i,t_{i+1})\\
    \mrd \wt{\my}_t &= \left[ \frac{\wt{\my}_{t_i}}{t_i}+\frac{2\wh{\mb}(1-t_i,\wt{\my}_{t_i},\mx)}{t_i} \right] \mrd t + \sqrt{\frac{2}{t}}\mrd \mw_t,\quad \wt{\my}_{t_0}\sim \mcN(0,\mI_{d_\mcY}),\ t\in [t_i,t_{i+1})
\end{align*}
where $T=t_0<t_1<\cdots<t_K=1-T$ and $t_i=T+i\Delta$ for $0\leq i\leq K-1$.

By the triangle inequality,
\[
\begin{aligned}
    &~~~~\Ebb_{\mcD,\mcT,\mcZ}\Ebb_{\mx}\bigl[\mathrm{TV}(\wt p_T(\cdot\mid\mx),p_T(\cdot|\mx))\bigr]\\
    &\le \Ebb_{\mcD,\mcT,\mcZ}\Ebb_{\mx}\bigl[\mathrm{TV}(p_T(\cdot|\mx),\ov p_T(\cdot|\mx))\bigr] + \Ebb_{\mcD,\mcT,\mcZ}\Ebb_{\mx}\bigl[\mathrm{TV}(\ov p_T(\cdot|\mx),\wt p_T(\cdot|\mx))\bigr]\\
    &\leq \Ebb_{\mcD,\mcT,\mcZ}\Ebb_{\mx}\bigl[\mathrm{TV}(p_T(\cdot|\mx),\ov p_T(\cdot|\mx))\bigr] + \Ebb_{\mcD,\mcT,\mcZ}\Ebb_{\mx}\bigl[\mathrm{TV}(p_{1-T}(\cdot|\mx),\mcN(0,\mI_{d_\mcY}))\bigr].
\end{aligned}
\]
where in the last line, we use the data processing inequality. Therefore, we only need to bound $\Ebb_{\mcD,\mcT,\mcZ}\Ebb_{\mx}[\mathrm{TV}(p_T(\cdot|\mx),\ov{p}_T(\cdot|\mx))]$ and $\Ebb_{\mcD,\mcT,\mcZ}\Ebb_{\mx}[\mathrm{TV}(p_{1-T}(\cdot|\mx),\mcN(0,\mI_{d_\mcY}))]$. 
\par We bound $\Ebb_{\mcD,\mcT,\mcZ}\Ebb_{\mx}[\mathrm{TV}(p_{1-T}(\cdot|\mx),\mcN(0,\mI_{d_\mcY}))]$ firstly and gives the following lemma.
\begin{lemma}
\label{lem: TV bound}
Suppose Assumption~\ref{ass: bounded_density} holds. Taking $\mcM=\lfloor n^\frac{1}{d_\mcX+d_\mcY+2\beta} \rfloor+1\lesssim n^\frac{1}{d_\mcX+d_\mcY+2\beta},T=\mcM^{-C_T}$ and $C_T=2\beta/(\beta\wedge1)$, we have
\begin{equation}
\label{eq: TV bound_1}
    \Ebb_{\mcD,\mcT,\mcZ}\Ebb_{\mx}[\mathrm{TV}(p_{1-T}(\cdot|\mx),\mcN(0,\mI_{d_\mcY}))] \lesssim n^{-\frac{2\beta}{d_\mcX+d_\mcY+2\beta}}.
\end{equation}
\end{lemma}
\begin{proof}[Proof of Lemma \ref{lem: TV bound}]
Using Pinsker' s inequality, we derive that
\begin{align*}
    \mathrm{TV}\left(p_{1-T}(\cdot|\mx),\mcN(0,\mI_{d_\mcY})\right)
    &\lesssim \sqrt{\mathrm{KL}\left( p_{1-T}(\cdot|\mx)\left|\right.\mcN(0,\mI_{d_\mcY}) \right)}\\
    &\lesssim \sqrt{\mathrm{KL}\left( \int p_{1-T}(\cdot|\my_0^F,\mx)p_0(\my_0^F|\mx) \mathrm{d}\my_0^F \left|\right. \mcN(0,\mI_{d_\mcY}) \right)}\\
    &\lesssim \sqrt{ \int \mathrm{KL}\left( p_{1-T}(\cdot|\my_0^F,\mx) \left|\right. \mcN(0,\mI_{d_\mcY}) \right)p_0(\my_0^F|\mx) \mathrm{d}\my_0^F },
\end{align*}
where the last inequality follows from the convexity of KL divergence. Note that $p_{1-T}(\cdot|\my_0^F,\mx)\sim\mcN(m_{1-T}\my_0^F,\sigma_{1-T}^2\mI_{d_\mcY})$ and the KL divergence between two Gaussian distributions has the explicit form of
\begin{equation*}
    \mathrm{KL}(\rho_1|\rho_2)=\frac{1}{2}\left( (\mu_1-\mu_2)^T\Sigma_2^{-1}(\mu_1-\mu_2) + \log\left(\frac{|\Sigma_2|}{|\Sigma_1|}\right) + \mathrm{Tr}(\Sigma_2^{-1}\Sigma_1) - d \right),
\end{equation*}
where $\rho_1\sim\mcN_d(\mu_1,\Sigma_1),~\rho_2\sim\mcN_d(\mu_2,\Sigma_2)$. We thus have
\begin{align*}
    \mathrm{KL}\left( p_{1-T}(\cdot|\my_0^F,\mx) \left|\right. \mcN(0,\mI_{d_\mcY}) \right)
    &= \frac{1}{2}\left( \left\|m_{1-T}\my_0^F\right\|^2-\log\sigma_{1-T}^{2d_{\mcY}}+d_\mcY\sigma_{1-T}^2-d_\mcY \right)\\
    &\lesssim T^2-d_\mcY\log(1-T^2)+d_\mcY((1-T^2)-1)\\
    &\lesssim T^2.
\end{align*}
Finally we get
\begin{equation*}
    \mathrm{TV}\left(p_{1-T}(\cdot|\mx),\mcN(0,\mI_{d_\mcY})\right)
    \lesssim \sqrt{T^2\cdot\int p_0(\my_0^F|\mx)\mathrm{d}\my_0^F}
    \lesssim T\lesssim n^{-\frac{2\beta}{d_\mcX+d_\mcY+2\beta}},
\end{equation*}
and the desired result follows.
\end{proof}

\begin{lemma}[A deterministic grid uniform over the density class]\label{lem:uniform_grid}
Fix the constants in Assumptions~\ref{ass: bounded_density}--\ref{ass: Holder}, the network caps defining $\mathcal C$, and $C_T$. Let $\Lambda_n=(M\kappa)^{L}$. For $\mb\in\mcC$ and $t\in[T,1-T]$, set
\[
    F_{p_0,\mb}(t)=\frac1t\Ebb_{\mx}\int_{\Rbb^{d_{\mcY}}}
    \|\mb(1-t,\my,\mx)-\nabla\log p_{1-t}(\my\mid\mx)\|^2p_{1-t}(\my\mid\mx)\mrd\my.
\]
There is a constant $C$ depending only on these fixed class parameters such that
\[
    |F_{p_0,\mb}(t)-F_{p_0,\mb}(v)|\leq C(1+\Lambda_n)T^{-3}\log\mcM\,|t-v|,
\]
uniformly in $p_0\in\mcP_\beta(C_l,C_u,R),\mb\in\mcC,t,v$ and in the covariate law supported on the prescribed cube. Consequently, on a uniform grid of mesh $\Delta=(1-2T)/K$,
\[
 \left|\sum_{i=0}^{K-1} F_{p_0,\mb}(t_i)\Delta-\int_T^{1-T}F_{p_0,\mb}(t)\mrd t\right|
 \leq C(1+\Lambda_n)T^{-3}\log\mcM\,\Delta.
\]
\end{lemma}
\begin{proof}
We first prove the Lipschitz continuity of networks in $\mcC$. Recall for the weight matrix $\mW_\ell\in\Rbb^{d_{\ell}\times d_{\ell-1}}$, where $d_\ell\le M$ and $\max_{1\le\ell\le L}\left\{ \|\mb_\ell\|_\infty,\|\mW_\ell\|_\infty \right\}\le\kappa$, it holds for any $\mathbf u\in\Rbb^{d_{\ell-1}}$ that
\[
    \|\mathbf W_\ell\mathbf u\|_2^2
    = \sum_{i=1}^{d_\ell} \left(\sum_{j=1}^{d_{\ell-1}}(\mW_\ell)_{ij}u_j\right)^2
    \le \sum_{i=1}^{d_\ell}\left(\sum_{j=1}^{d_{\ell-1}}(\mW_\ell)_{ij}^2\right)\left(\sum_{j=1}^{d_{\ell-1}}u_j^2\right)
    \le d_\ell d_{\ell-1}\kappa^2\|\mathbf u\|_2^2.
\]
Thus $\|\mW_\ell\|_{\rm op}\le\sqrt{d_\ell d_{\ell-1}}\,\kappa\le M\kappa$ for any $1\le\ell\le L$. Combining with the fact that ReLU is $1$-Lipschitz, we derive for any $\mz,\mz'\in\Rbb^{1+d_{\mcY}+d_{\mcX}}$,
\[
    \|\mb(\mz)-\mb(\mz')\|_2
    \le \left(\prod_{\ell=1}^{L}\|\mW_\ell\|_{\rm op}\right)\|\mz-\mz'\|_2
    \le (M\kappa)^L\|\mz-\mz'\|_2
    \triangleq \Lambda_n\|\mz-\mz'\|_2.
\]
Thus every network in $\mcC$ is globally $\Lambda_n$-Lipschitz.

On one probability space choose $\mX$ with its given law, $\mathbf U\mid\mX=\mx\sim p_0(\cdot\mid\mx)$, and an independent standard Gaussian vector $\mZ\in\Rbb^{d_{\mcY}}$. Put
\[
    \mathbf V_t=m_{1-t}\mU+\sigma_{1-t}\mZ=t\mU+\sqrt{1-t^2}\mZ,\qquad t\in[T,1-T].
\]
Thus $\mathbf V_t\mid\mX=\mx\sim p_{1-t}(\cdot\mid\mx)$. This coupling is used only to compare marginal risks. Write $\ms_t(\my,\mx)=\nabla\log p_{1-t}(\my\mid\mx)$ and $\mathbf e_t=\mb(1-t,\mathbf V_t,\mX)-\nabla\log p_{1-t}(\mV_t\mid\mX)$. For a random vector $\mathbf H$, write $\|\mathbf H\|_{L^2}=(\Ebb\|\mathbf H\|^2)^{1/2}$, with the expectation over this coupling.

Since $\|\mU\|_\infty\le 1$, we derive $\|\mathbf V_t\|_\infty \le t+\sqrt{1-t^2}\,\|\mathbf Z\|_\infty$ and then
\[
    \frac{(\|\mathbf V_t\|_\infty-m_{1-t})_+}{\sigma_{1-t}}
    = \frac{(\|\mathbf V_t\|_\infty-t)_+}{\sqrt{1-t^2}}
    \le \|\mathbf Z\|_\infty
    \le \|\mathbf Z\|_2.
\]
Lemma~\ref{lem: derivatives_boundness}, with forward time $1-t$, therefore implies
\[
\begin{aligned}
    \|\nabla\log p_{1-t}(\mV_t\mid\mX)\|&\leq CT^{-1/2}(1+\|\mZ\|),\\
    \|\partial_{\my}\nabla\log p_{1-t}(\mV_t\mid\mX)\|_{\rm op}&\leq CT^{-1}(1+\|\mZ\|)^2,\\
    \|\partial_t\nabla\log p_{1-t}(\mV_t\mid\mX)\|&\leq CT^{-3/2}(1+\|\mZ\|)^3.
\end{aligned}
\]
Note here that the constants in Lemma~\ref{lem: derivatives_boundness} depend only on $d_{\mcY},C_l,C_u$, and are uniform over $p_0\in\mcP_\beta(C_l,C_u,R)$ and $\mx\in[-1,1]^{d_{\mcX}}$. Also $\|\mathbf V_t'\|=\|\mathbf U-t\mZ/\sqrt{1-t^2}\|\leq CT^{-1/2}(1+\|\mZ\|)$. Consider $T\le v\le t\le1-T$, the chain rule gives
\[
\begin{aligned}
    \|\nabla\log p_{1-t}(\mV_t\mid\mX)-\nabla\log p_{1-v}(\mV_v\mid\mX)\|_2
    &\le \int_v^t \left\|\frac{\mrd}{\mrd r}\nabla\log p_{1-r}(\mV_r\mid\mX)\right\|_2\,\mrd r\\
    &\le \int_v^t \left\|\partial_r\nabla\log p_{1-r}(\mV_r\mid\mX)+\partial_\my\nabla\log p_{1-r}(\mV_r\mid\mX)\mV_r'  \right\|_2\,\mrd r\\
    &\le CT^{-3/2}(1+\|\mathbf Z\|_2)^3|t-v|,
\end{aligned}
\]
and thus
\[
    \|\nabla\log p_{1-t}(\mV_t\mid\mX)-\nabla\log p_{1-v}(\mV_v\mid\mX)\|_{L^2}
    \leq CT^{-3/2}|t-v|.
\]
The network Lipschitz bound similarly gives
\[
\begin{aligned}
    \|\mb(1-t,\mathbf V_t,\mX)-\mb(1-v,\mathbf V_v,\mX)\|_{L^2}
    &\le \Lambda_n\left( |t-v|^2+\|\mV_t-\mV_v\|_{L^2}^2 \right)^{1/2}\\
    &\le C\Lambda_nT^{-1/2}|t-v|.
\end{aligned}
\]
The envelope defining $\mathcal C$ and the score bound show $\|\mathbf e_t\|_{L^2}\leq C T^{-1/2}\sqrt{\log\mcM}$. Hence Cauchy--Schwarz yields
\[
\begin{aligned}
    |\Ebb\|\mathbf e_t\|^2-\Ebb\|\mathbf e_v\|^2|
    &\leq(\|\mathbf e_t\|_{L^2}+\|\mathbf e_v\|_{L^2})\|\mathbf e_t-\mathbf e_v\|_{L^2}\\
    &\leq C\sqrt{\log\mcM}(\Lambda_nT^{-1}+T^{-2})|t-v|.
\end{aligned}
\]
Multiplying this estimate by $1/t\leq1/T$ and using $|1/t-1/v|\leq T^{-2}|t-v|$ yields
\[
\begin{aligned}
    |F_{p_0,\mathbf b}(t)-F_{p_0,\mathbf b}(v)|
    &\le \frac1t \left|\Ebb\|\mathbf e_t\|_2^2-\Ebb\|\mathbf e_v\|_2^2\right| + \left|\frac1t-\frac1v\right| \Ebb\|\mathbf e_v\|_2^2\\
    &\le C\left( \sqrt{\log\mcM}\big(\Lambda_nT^{-2}+T^{-3}\big)+T^{-3}\log\mcM \right)|t-v|\\
    &\le C(1+\Lambda_n)T^{-3}\log\mcM\,|t-v|,
\end{aligned}
\]
for sufficiently large $\mcM\ge e$.
Finally, integrate $|F_{p_0,\mb}(t_i)-F_{p_0,\mb}(t)|$ on each $[t_i,t_{i+1}]$ and sum, and we derive
\[
\begin{aligned}
\left| \sum_{i=0}^{K-1} F_{p_0,\mathbf b}(t_i)\Delta-\int_T^{1-T}F_{p_0,\mathbf b}(t)\,\mrd t \right|
&= \left| \sum_{i=0}^{K-1}\int_{t_i}^{t_{i+1}} F_{p_0,\mathbf b}(t_i)-F_{p_0,\mathbf b}(t)\,\mrd t \right|\\
&\le C(1+\Lambda_n)T^{-3}\log\mcM \sum_{i=0}^{K-1}\int_{t_i}^{t_{i+1}}(t-t_i)\,\mrd t\\
&= C(1+\Lambda_n)T^{-3}\log\mcM\,\frac{K\Delta^2}{2}\\
&= C(1+\Lambda_n)T^{-3}\log\mcM\, \frac{(1-2T)\Delta}{2}\\
&\le C(1+\Lambda_n)T^{-3}\log\mcM\,\Delta.
\end{aligned}
\]
The stated Riemann-sum bound follows. It holds pointwise for every admissible network, so it also holds after inserting a data-dependent estimator and averaging over its training randomness.
\end{proof}

\par To bound $\Ebb_{\mcD,\mcT,\mcZ}\Ebb_{\mx}[\mathrm{TV}(p_T(\cdot|\mx),\ov{p}_T(\cdot|\mx))]$, we first introduce the following lemma.
\begin{lemma}[Proposition D.1 in \cite{oko2023diffusion}]
\label{lem: Girsanov}
Let $\pi_0$ be any probability distribution, and $\mz=(\mz_t)_{t\in [0,T]}$, $\mz^{\prime}=(\mz^{\prime}_t)_{t\in [0,T]}$ be two different processes satisfying
$$
\begin{aligned}
& \mrd \mz_t = \mb(t,\mz_t)\mrd t + \sigma(t)\mrd \mw_t, ~ \mz_0 \sim \pi_0, \\
& \mrd \mz^{\prime}_t = \mb^{\prime}(t,\mz^{\prime}_t) \mrd t + \sigma(t)\mrd \mw_t, ~ \mz^{\prime}_0 \sim \pi_0.
\end{aligned}
$$
We define the distributions of $\mz_t$ and $\mz^{\prime}_t$ as $\pi_t$ and $\pi^{\prime}_t$, and the path measures of $\mz$ and $\mz^{\prime}$ as $\Pbb$ and $\Pbb^{\prime}$, respectively. Suppose that the Novikov's condition holds, i.e.,
\begin{equation}\label{eq: novikov}
\Ebb_{\Pbb}\left[
\exp\left(\frac{1}{2}\int_0^T \frac{\Vert \mb(t,\mz_t) - \mb^{\prime}(t,\mz_t)\Vert^2}{\sigma^2(t)}\mrd t \right)
\right] < +\infty.
\end{equation}
Then, the Radon-Nikodym derivative of $\Pbb^{\prime}$ with respect to $\Pbb$ is
$$
\frac{\mrd \Pbb^{\prime}}{\mrd \Pbb} = \exp\left(-\int_{0}^T\frac{\mb(t,\mz_t) - \mb^{\prime}(t,\mz_t)}{\sigma(t)}\mrd \mw_t -\int_0^T \frac{\Vert \mb(t,\mz_t) - \mb^{\prime}(t,\mz_t)\Vert^2}{2\sigma^2(t)}\mrd t\right),
$$
and therefore we have 
$$
\mathrm{KL}(\pi_T|\pi_T^{\prime}) \leq \mathrm{KL}(\Pbb|\Pbb^{\prime}) = \Ebb_{\Pbb} \left[\frac{1}{2}\int_0^T\frac{\Vert \mb(t,\mz_t) - \mb^{\prime}(t,\mz_t)\Vert^2}{\sigma^2(t)} \mrd t \right].
$$
Moreover, if there exists a constant $C>0$ such that for any $t\in [0,T]$ the following finite-energy condition holds
$$ \Ebb_\Pbb\left[ \frac{\Vert \mb(t,\mz_t) - \mb^{\prime}(t,\mz_t)\Vert^2}{\sigma^2(t)} \right] \leq C, 
$$
then even if the Novikov' s condition is not satisfied, we can still derive that
$$
\mathrm{KL}(\pi_T|\pi_T^{\prime}) \leq \Ebb_{\Pbb} \left[\frac{1}{2}\int_0^T\frac{\Vert \mb(t,\mz_t) - \mb^{\prime}(t,\mz_t)\Vert^2}{\sigma^2(t)} \mrd t \right].
$$
\end{lemma}

Based on Lemma \ref{lem: Girsanov}, we can prove Theorem \ref{thm: sampling_error}.

\begin{proof}[Proof of Theorem \ref{thm: sampling_error}]
Fix $(\mcD,\mcT,\mcZ)$ and $\mx\in[-1,1]^{d_\mcX}$. Let $\Pbb$ denote the path law on $[T,1-T]$ of the true reverse process initialized from $p_{1-T}(\cdot|\mx)$. All implicit constants depend only on the fixed dimensions, density-class parameters, and network envelope constant. They are uniform in $p_0$, $\mx$, the training realization, the time points, and the grid.

In each fixed-time moment calculation below, the forward realization of the true process gives $\my_t=m_{1-t}\my_0^F+\sigma_{1-t}\mz$, where $\|\my_0^F\|_\infty\leq1$ and $\mz$ is a standard Gaussian vector independent of $\my_0^F$ and associated with that fixed time.
For any $t\in[T, 1-T]$, we have $\my_t\sim p_{1-t}(\my|\mx)$ and $\Vert\wh{\mb}(1-t,\my_t,\mx)\Vert \lesssim \frac{\sqrt{\log\mcM}}{\sqrt{1-t}}$. Moreover, we claim that $\Ebb[\|\nabla\log p_{1-t}(\my_t|\mx)\|^2]\lesssim\frac{1}{1-t}$. In fact, by Lemma \ref{lem: derivatives_boundness}, it holds
\begin{align*}
    \|\nabla\log p_{1-t}(\my_t|\mx)\|
    & \lesssim \frac{1}{\sigma_{1-t}} \cdot \left( \frac{(\|\my_t\|_\infty-m_{1-t})_+}{\sigma_{1-t}} \vee 1 \right)\\
    & = \frac{1}{\sigma_{1-t}} \cdot \left( \frac{(\|m_{1-t}\my_0^F+\sigma_{1-t}\mz\|_\infty-m_{1-t})_+}{\sigma_{1-t}} \vee 1 \right)\\
    & \lesssim \frac{1}{\sigma_{1-t}} \cdot \left( \frac{(m_{1-t}+\sigma_{1-t}\|\mz\|_\infty-m_{1-t})_+}{\sigma_{1-t}} \vee 1 \right)\\
    & \lesssim \frac{1}{\sqrt{1-t^2}} \cdot \left( \|\mz\|_\infty \vee 1 \right)\\
    & \lesssim \frac{1}{\sqrt{1-t}} \cdot \left( \|\mz\|_\infty+1 \right).
\end{align*}
Taking the square and expectation, we derive the result. 

We want to apply Lemma~\ref{lem: Girsanov} to the process of $(\my_t)_t$ and $(\ov \my_t)_t$ defined before. Since the frozen drift of $(\ov \my_t)_t$ depends on the state at the left endpoint of each interval, we apply Lemma~\ref{lem: Girsanov} conditionally on that state. To do this, we first verify the conditional finite-energy condition.

Fix $\my_{t_i}=\my$ and initialize the frozen process from the same point $\my$. 
Gaussian conditioning in the forward diffusion gives
\[
\begin{aligned}
    \Ebb_\Pbb[\my_t\mid\my_0^F,\my_{t_i}=\my]
    &= t\my_0^F+\frac{t_i(1-t^2)}{t(1-t_i^2)}(\my-t_i\my_0^F),\\
    \operatorname{Cov}_\Pbb(\my_t\mid\my_0^F,\my_{t_i}=\my)
    &= \frac{(1-t^2)(t^2-t_i^2)}{t^2(1-t_i^2)}\mI_{d_\mcY}.
\end{aligned}
\]
Since the coefficients above satisfy
\[
\begin{aligned}
    0\le\frac{t_i(1-t^2)}{t(1-t_i^2)}\le 1,\qquad
    0\le\frac{(1-t^2)(t^2-t_i^2)}{t^2(1-t_i^2)}=\frac{1-t^2}{1-t_i^2}\cdot\left(1-\frac{t_i^2}{t^2}\right)\le 1,
\end{aligned}
\]
combining with $\|\my_0^F\|\leq\sqrt{d_\mcY}$ and we derive
\[
\begin{gathered}
    \left\|\Ebb_{\Pbb}[\my_t\mid\my_0^F,\my_{t_i}=\my]\right\|
    \le t\|\my_0^F\|+\frac{t_i(1-t^2)}{t(1-t_i^2)}\bigl(\|\my\|+t_i\|\my_0^F\|\bigr)
    \le \|\my\|+2\sqrt{d_{\mcY}}\\
    {\rm tr} \operatorname{Cov}_\Pbb(\my_t\mid\my_0^F,\my_{t_i}=\my)\le d_\mcY.
\end{gathered}
\]
Therefore it holds
\[
\begin{aligned}
    \Ebb_{\Pbb}
    [\|\my_t\|^2\mid\my_0^F,\my_{t_i}=\my]
    &= \left\| \Ebb_{\Pbb}[\my_t\mid\my_0^F,\my_{t_i}=\my] \right\|^2 + \operatorname{tr}\operatorname{Cov}_{\Pbb}(\my_t\mid\my_0^F,\my_{t_i}=\my)\\
    &\le 2\|\my\|^2+9d_{\mcY}.
\end{aligned}
\]
Averaging over the conditional law of $\my_0^F$ yields
\[
    \sup_{t\in[t_i,t_{i+1}]}
    \Ebb_\Pbb[\|\my_t\|^2\mid\my_{t_i}=\my]
    \leq C(1+\|\my\|^2).
\]
On $[T,1-T]$, the true score is smooth, locally Lipschitz,
and of at most linear growth by
Lemma~\ref{lem: derivatives_boundness}; the frozen process has
an explicit solution. Thus both conditional processes are
well defined and nonexplosive. The conditional second-moment
bound, the linear growth of the score, and the network envelope
give
\[
\begin{aligned}
&\sup_{t\in[t_i,t_{i+1}]}
\Ebb_{\Pbb}\Bigg[
\frac t2
\left\|
\frac{\my_t}{t}
+\frac{2\nabla\log p_{1-t}(\my_t\mid\mathbf x)}{t}
-\frac{\my}{t_i}
-\frac{2\widehat{\mathbf b}(1-t_i,\my,\mathbf x)}{t_i}
\right\|^2
\,\Bigg|\,\my_{t_i}=\my
\Bigg]
\\
&\lesssim
\sup_{t\in[t_i,t_{i+1}]}\frac t2
\Bigg[
\frac{\Ebb_{\Pbb}[\|\my_t\|^2\mid\my_{t_i}=\my]}{t^2}
+\frac{\|\my\|^2}{t_i^2}
+\frac{\Ebb_{\Pbb}[\|\my_t\|^2\mid\my_{t_i}=\my]}
{t^2(1-t^2)^2}
+\frac1{t^2(1-t^2)}
+\frac{\log\mathcal M}{t_i^2(1-t_i^2)}
\Bigg]\\
&\lesssim
\sup_{t\in[t_i,t_{i+1}]}
\left[
\frac{1+\|\my\|^2}{t}
+\frac{t\|\my\|^2}{t_i^2}
+\frac{1+\|\my\|^2}{t(1-t)^2}
+\frac1{t(1-t)}
+\frac{t\log\mathcal M}{t_i^2(1-t_i)}
\right]\\
&\lesssim
(1+\|\my\|^2)
\sup_{t\in[t_i,t_{i+1}]}
\left[
\frac1t
+\frac1{t(1-t)^2}
+\frac{\log\mathcal M}{t(1-t)}
\right]\\
&\lesssim
(1+\|\my\|^2) \left[ \frac1T+\frac{2}{T^2}+\frac{2\log\mathcal M}{T} \right]\\
&\lesssim
\left(T^{-2}+T^{-1}\log\mathcal M\right)(1+\|\my\|^2)
<\infty,
\end{aligned}
\]
where we used $t/t_i\le t_{i+1}/t_i\le 2$. The finite upper bound may depend on $n$ through $T$ and $\mcM$. For each fixed $n$ and conditioning point $\my$, however, it is uniform over $t\in[t_i,t_{i+1}]$, which is the finite-energy condition required by Lemma~\ref{lem: Girsanov}. Therefore, that lemma applies to the conditional endpoint distributions on each interval.

The true and frozen grid processes $(\my_t)_t$ and $\ov \my_t$ have the same initial density $p_{1-T}(\cdot|\mx)$. Their joint densities factor into this initial density and the successive conditional transition densities. Taking the logarithm of their ratio and integrating gives the sum of the conditional transition KL divergences. Marginalizing at the endpoint and applying the conditional bounds above therefore yields
\[
\begin{aligned}
    &\mathrm{KL}\bigl(
    p_T(\cdot|\mx)\mid\ov p_T(\cdot|\mx)
    \bigr)\\
    &\quad\leq
    \frac14\sum_{i=0}^{K-1}\int_{t_i}^{t_{i+1}}
    t\,\Ebb_\Pbb
    \left\|
    \frac{\my_t}{t}-\frac{\my_{t_i}}{t_i}
    +2\left(
    \frac{\nabla\log p_{1-t}(\my_t|\mx)}{t}
    -\frac{\wh{\mb}(1-t_i,\my_{t_i},\mx)}{t_i}
    \right)
    \right\|^2\mrd t.
\end{aligned}
\]
The conditional energy bound proved above, together with
\[
\begin{aligned}
\Ebb_{\Pbb}\|\my_{t_i}\|^2
&=
\Ebb\left\|
t_i\my_0^F+\sqrt{1-t_i^2}\mz
\right\|^2\\
&=
t_i^2\Ebb\|\my_0^F\|^2
+2t_i\sqrt{1-t_i^2}\,
\Ebb\!\left[(\my_0^F)^\top\mz\right]
+(1-t_i^2)\Ebb\|\mz\|^2\\
&=
t_i^2\Ebb\|\my_0^F\|^2
+(1-t_i^2)d_{\mcY}\\
&\leq
t_i^2d_{\mcY}+(1-t_i^2)d_{\mcY}
=d_{\mcY},
\end{aligned}
\]
imply that the total energy is finite. By Pinsker's inequality, followed by averaging over the training variables and $\mx$, we obtain
\[
\begin{aligned}
&\Ebb_{\mcD,\mcT,\mcZ}\Ebb_{\mx}
\left[
\mathrm{TV}^2\bigl(
p_T(\cdot|\mx),\ov p_T(\cdot|\mx)
\bigr)
\right]\\
&\quad\lesssim
\Ebb_{\mcD,\mcT,\mcZ}\Ebb_{\mx}
\left[
\sum_{i=0}^{K-1}\int_{t_i}^{t_{i+1}}
t\,\Ebb_\Pbb
\left(
\left\|\frac{\my_t}{t}-\frac{\my_{t_i}}{t_i}\right\|^2
+
\left\|
\frac{\nabla\log p_{1-t}(\my_t|\mx)}{t}
-\frac{\wh{\mb}(1-t_i,\my_{t_i},\mx)}{t_i}
\right\|^2
\right)\mrd t
\right].
\end{aligned}
\]

We first bound the term $\Ebb_{\Pbb}\left(
\left\Vert\frac{\my_t}{t} - \frac{\my_{t_i}}{t_i} \right\Vert^2\right)$. Write $s=1-t$ and $s_i=1-t_i$, where $s\le s_i$. Time reversal gives $(\my_{t_i},\my_t) \stackrel{d}{=} (\my^F_{s_i},\my^F_s)$. Integrating the forward SDE \eqref{eq: forward_sde} over $[s,s_i]$ and using Cauchy--Schwarz and It\^o isometry, together with $\sup_{0\leq u\leq1}\Ebb\|\my_u^F\|^2<\infty$, we derive
\[
\begin{aligned}
    \Ebb_{\Pbb}\|\my_t-\my_{t_i}\|^2
    &=\Ebb_{\Pbb}\|\my_s^F-\my_{s_i}^F\|^2\\
    &\lesssim (t-t_i)\int_s^{s_i}\frac{\Ebb\|\my_u^F\|^2}{(1-u)^2}\,\mrd u
    +d_\mcY\int_s^{s_i}\frac{1}{1-u}\,\mrd u\\
    &\lesssim \frac{(t-t_i)^2}{t_i^2}+\log\frac{t}{t_i}
    \lesssim \frac{t-t_i}{t_i},
\end{aligned}
\]
where the last step uses $t-t_i\leq\Delta\leq T/2\leq t_i/2$, which holds for sufficiently large $n$. This also implies $t_i\le t < t_i+\Delta\le 3t_i/2$. It then yields
\[
\begin{aligned}
    \Ebb_{\Pbb}\left\|\frac{\my_t}{t}-\frac{\my_{t_i}}{t_i}\right\|^2
    &\lesssim \frac{1}{t^2}\Ebb_{\Pbb}\|\my_t-\my_{t_i}\|^2
    +\Ebb_{\Pbb}\|\my_{t_i}\|^2\left(\frac{1}{t}-\frac{1}{t_i}\right)^2\\
    &\lesssim \frac{t-t_i}{t_i^3}+\frac{(t-t_i)^2}{t_i^4}
    \lesssim \frac{t-t_i}{t_i^3}.
\end{aligned}
\]
Consequently, the integral of this term has the following bound 
\begin{equation}\label{eq: corrected_bound_V}
\begin{aligned}
    \Ebb_{\mcD,\mcT,\mcZ}\Ebb_{\mx}
    \left[\sum_{i=0}^{K-1}\int_{t_i}^{t_{i+1}} t\cdot
    \Ebb_{\Pbb}\left\|\frac{\my_t}{t}-\frac{\my_{t_i}}{t_i}\right\|^2\mrd t\right]
    &\lesssim \sum_{i=0}^{K-1}\frac{\Delta^2}{t_i^2}
    = \Delta\sum_{i=0}^{K-1}\int_{t_i}^{t_{i+1}}\frac{1}{t_i^2}\mrd t\\
    &\lesssim \Delta\sum_{i=0}^{K-1}\int_{t_i}^{t_{i+1}}\frac{1}{t^2}\mrd t
    \lesssim \frac{\Delta}{T}.
\end{aligned}
\end{equation}

Next, we bound the term $\Ebb_{\Pbb}\left(\left\Vert\frac{\nabla\log p_{1-t}(\my_t|\mx)}{t} - \frac{\wh{\mb}(1-t_i, \my_{t_i},\mx)}{t_i} \right\Vert^2 \right)$. We can decompose this term into following four terms. 
$$
\begin{aligned}
&~~~\Ebb_{\Pbb}\left(\left\Vert\frac{\nabla\log p_{1-t}(\my_t|\mx)}{t} - \frac{\wh{\mb}(1-t_i, \my_{t_i},\mx)}{t_i} \right\Vert^2 \right) \\
&\lesssim \underbrace{\Ebb_{\Pbb}\left(\left\Vert\frac{\nabla\log p_{1-t}(\my_t|\mx)}{t} - \frac{\nabla\log p_{1-t_i}(\my_{t}|\mx)}{t} \right\Vert^2 \right)}_{(\mathrm{I})} \\
& ~~~~ + \underbrace{\Ebb_{\Pbb}\left(\left\Vert\frac{\nabla\log p_{1-t_i}(\my_{t}|\mx)}{t} - \frac{\nabla\log p_{1-t_i}(\my_{t_i}|\mx)}{t} \right\Vert^2 \right)}_{\mathrm{(II)}} \\
& ~~~~ + \underbrace{\Ebb_{\Pbb}\left(\left\Vert\frac{\nabla\log p_{1-t_i}(\my_{t_i}|\mx)}{t} - \frac{\nabla\log p_{1-t_i}(\my_{t_i}|\mx)}{t_i} \right\Vert^2 \right)}_{(\mathrm{III})} \\
& ~~~~ + \underbrace{\Ebb_{\Pbb}\left(\left\Vert\frac{\nabla\log p_{1-t_i}(\my_{t_i}|\mx)}{t_i} - \frac{\wh{\mb}(1-t_i, \my_{t_i},\mx)}{t_i} \right\Vert^2 \right)}_{(\mathrm{IV})}.
\end{aligned}
$$
We bound these four terms separately. Before we begin, let' s give a bound for $\|\nabla\log p_{1-t_i}(\my_t|\mx)\|$ which is similar to $\|\nabla\log p_{1-t}(\my_t|\mx)\|$. Again by Lemma \ref{lem: derivatives_boundness}, we have
\begin{align*}
    \|\nabla\log p_{1-t_i}(\my_t|\mx)\|
    & \lesssim \frac{1}{\sigma_{1-t_i}} \cdot \left( \frac{(\|\my_t\|_\infty-m_{1-t_i})_+}{\sigma_{1-t_i}} \vee 1 \right)\\
    & = \frac{1}{\sigma_{1-t_i}} \cdot \left( \frac{(\|m_{1-t}\my_0^F+\sigma_{1-t}\mz\|_\infty-m_{1-t_i})_+}{\sigma_{1-t_i}} \vee 1 \right)\\
    & \lesssim \frac{1}{\sqrt{1-t_i^2}} \cdot \left( \frac{(t-t_i+\sqrt{1-t^2}\|\mz\|_\infty)_+}{\sqrt{1-t_i^2}} \vee 1 \right)\\
    & \lesssim \frac{1}{\sqrt{1-t_i}} \cdot \left( \left(\frac{t-t_i}{\sqrt{1-t_i}}+\|\mz\|_\infty\right) \vee 1 \right)\\
    & \lesssim \frac{1}{\sqrt{1-t_i}} \cdot \left( \|\mz\|_\infty+1+1 \right),
\end{align*}
where we used $t \geq t_i$ and $t \leq 1$. Now we' ve already derived
\begin{equation*}
    \Ebb_\Pbb\left[\|\nabla\log p_{1-t}(\my_t|\mx)\|^2\right]\lesssim\frac{1}{1-t},\quad
    \Ebb_\Pbb\left[\|\nabla\log p_{1-t_i}(\my_t|\mx)\|^2\right]\lesssim\frac{1}{1-t_i}.
\end{equation*}
The same pointwise estimates also give fourth-moment bounds since $\Ebb[(1+\|\mz\|_\infty)^4]<\infty$ for a standard Gaussian vector. Consequently,
\[
\Ebb_\Pbb \left[\|\nabla\log p_{1-t}(\my_t|\mx)\|^4\right]
\lesssim \frac1{(1-t)^2},\qquad
\Ebb_\Pbb \left[\|\nabla\log p_{1-t_i}(\my_t|\mx)\|^4\right]
\lesssim \frac1{(1-t_i)^2}.
\]

We first bound term $(\mathrm{I})$, which corresponds to
$$
\begin{aligned}
\mathrm{(I)} = \frac{1}{t^2} \cdot \Ebb_{\Pbb}\left[ \Vert\nabla\log p_{1-t}(\my_t|\mx) - \nabla\log p_{1-t_i}(\my_{t}|\mx)\Vert^2 \right].
\end{aligned}
$$
Let $0<\epsilon\leq e^{-1}$. Fix a sufficiently large constant $C$, used throughout the following truncation argument and independent of $\epsilon$, $t$, $i$, $n$, $p_0$, $\mx$, and the training realization. It can be chosen so that, uniformly over $t\in[T,1-T]$,
\[
    \Pbb\left(\|\my_t\|_\infty>m_{1-t}+C\sigma_{1-t}\sqrt{\log\epsilon^{-1}}\right) \lesssim \epsilon^2.
\]
To see this directly, the fixed-time Gaussian representation and $\|\my_0^F\|_\infty\leq1$ imply
\[
\begin{aligned}
\Pbb\left(
\|\my_t\|_\infty>
m_{1-t}+C\sigma_{1-t}\sqrt{\log\epsilon^{-1}}
\right)
&\leq \Pbb\left(\|\mz\|^2>C^2\log\epsilon^{-1}\right)\\
&\leq \epsilon^{C^2/4}\Ebb\left[\exp\left(\tfrac{\|\mz\|^2}{4}\right)\right]
= 2^{d_\mcY/2}\epsilon^{C^2/4}
\lesssim \epsilon^2,
\end{aligned}
\]
where the Gaussian integral gives $\Ebb\left[\exp(\|\mz\|^2/4)\right]=2^{d_\mcY/2}$, and we take $C^2\geq8$. We now split $(\mathrm{I})$ into the central and tail regions:
$$
\begin{aligned}
(\mathrm{I}) &= \frac{1}{t^2} \cdot \Ebb_{\Pbb}\left[\Vert \nabla\log p_{1-t}(\my_t|\mx) - \nabla\log p_{1-t_i}(\my_t|\mx)\Vert^2\right] \\
& \lesssim \frac{1}{t^2} \cdot \Ebb_{\Pbb}\left[\Vert \nabla\log p_{1-t}(\my_t|\mx) - \nabla\log p_{1-t_i}(\my_t|\mx)\Vert^2 \mathbf{I}_{\{\Vert\my_t\Vert_\infty \leq m_{1-t} + C\sigma_{1-t}\sqrt{\log\epsilon^{-1}}\}}\right] \\
& ~~~~ + \frac{1}{t^2} \cdot \Ebb_{\Pbb}\left[\Vert \nabla\log p_{1-t}(\my_t|\mx) - \nabla\log p_{1-t_i}(\my_t|\mx)\Vert^2 \mathbf{I}_{\{\Vert\my_t\Vert_\infty > m_{1-t} + C\sigma_{1-t}\sqrt{\log\epsilon^{-1}}\}}\right].
\end{aligned}
$$
For fixed spatial argument $\my_t$, the fundamental theorem of
calculus gives
\[
    \nabla\log p_{1-t}(\my_t|\mx)-\nabla\log p_{1-t_i}(\my_t|\mx)
    =\int_{t_i}^{t} \partial_u\nabla\log p_{1-u}(\my_t|\mx)\,\mrd u.
\]
On the event
$\|\my_t\|_\infty\leq
m_{1-t}+C\sigma_{1-t}\sqrt{\log\epsilon^{-1}}$,
for every $u\in[t_i,t]$ we have
\[
\begin{aligned}
\frac{(\|\my_t\|_\infty-m_{1-u})_+}{\sigma_{1-u}}
&\leq
\frac{t-u}{\sqrt{1-u^2}}
+
C\frac{\sqrt{1-t^2}}{\sqrt{1-u^2}}
\sqrt{\log\epsilon^{-1}}\\
&\leq
1+C\sqrt{\log\epsilon^{-1}}
\lesssim\sqrt{\log\epsilon^{-1}}.
\end{aligned}
\]
Here $t-u\leq1-u\leq\sqrt{1-u^2}$, $u\leq t$, and $\log\epsilon^{-1}\geq1$. Moreover,
\[
    \frac{|\partial_u\sigma_{1-u}|+|\partial_um_{1-u}|}{\sigma_{1-u}^2}
    = \frac{u/\sqrt{1-u^2}+1}{1-u^2}
    \lesssim (1-u)^{-\frac{3}{2}}
    \leq (1-t)^{-\frac{3}{2}}.
\]
Therefore, Lemma~\ref{lem: derivatives_boundness} and Cauchy--Schwarz in the time integral imply
\[
\begin{aligned}
\|\nabla\log p_{1-t}(\my_t|\mx)
-\nabla\log p_{1-t_i}(\my_t|\mx)\|^2
&\leq
(t-t_i)\int_{t_i}^{t}
\|\partial_u\nabla\log p_{1-u}(\my_t|\mx)\|^2\,\mrd u\\
&\lesssim
\frac{(t-t_i)^2}{(1-t)^3}\log^3\epsilon^{-1}
\end{aligned}
\]
on this central event. For the tail contribution, Cauchy--Schwarz, the fourth-moment bounds above, and the strengthened tail probability give
\[
\begin{aligned}
&~~~~\Ebb_\Pbb\left[ \|\nabla\log p_{1-t}(\my_t|\mx)-\nabla\log p_{1-t_i}(\my_t|\mx)\|^2\cdot\mI_{\{\|\my_t\|_\infty>m_{1-t}+C\sigma_{1-t}\sqrt{\log\epsilon^{-1}}\}} \right]\\
&\leq \Ebb_\Pbb \left[\|\nabla\log p_{1-t}(\my_t|\mx)-\nabla\log p_{1-t_i}(\my_t|\mx)\|^4\right]^{1/2} \cdot \Pbb\left(\|\my_t\|_\infty>m_{1-t}+C\sigma_{1-t}\sqrt{\log\epsilon^{-1}} \right)^{1/2}\\
&\lesssim \left(\frac1{1-t}+\frac1{1-t_i}\right)\epsilon
\lesssim \frac{\epsilon}{1-t}.
\end{aligned}
\]
The fourth moment of the difference is bounded using $\|\mathbf a-\mathbf b\|^4 \le 8(\|\mathbf a\|^4+\|\mathbf b\|^4)$. Combining the central and tail estimates yields
\begin{equation}\label{eq: bound_I}
(\mathrm{I})
\lesssim \frac{(t-t_i)^2}{t^2(1-t)^3}\log^3\epsilon^{-1}+\frac{\epsilon}{t^2(1-t)}
\lesssim \frac1{t^2(1-t)}\left( \frac{\Delta^2}{T^2}\log^3\epsilon^{-1}+\epsilon \right),
\end{equation}
where the last inequality uses $t-t_i\leq\Delta$ and $1-t\geq T$.

Next, we bound term $(\mathrm{II})$, which is
\[
(\mathrm{II})
= \frac1{t^2}\Ebb_\Pbb \left[\|\nabla\log p_{1-t_i}(\my_t|\mx)-\nabla\log p_{1-t_i}(\my_{t_i}|\mx)\|^2\right].
\]
Split the expectation into the contributions from the two tail events and their common complement:
$$
\begin{aligned}
&~~~~\Ebb_{\Pbb}\left[\Vert\nabla\log p_{1-t_i}(\my_t|\mx) - \nabla\log p_{1-t_i}(\my_{t_i}|\mx)\Vert^2\right] \\
&\lesssim \Ebb_{\Pbb}\left[\Vert\nabla\log p_{1-t_i}(\my_t|\mx) - \nabla\log p_{1-t_i}(\my_{t_i}|\mx)\Vert^2 \mathbf{I}_{\{\Vert\my_t\Vert_\infty > m_{1-t} + C\sigma_{1-t}\sqrt{\log\epsilon^{-1}}\}}\right] \\
& ~~~ + \Ebb_{\Pbb}\left[\Vert\nabla\log p_{1-t_i}(\my_t|\mx) - \nabla\log p_{1-t_i}(\my_{t_i}|\mx)\Vert^2 \mathbf{I}_{\{\Vert\my_{t_i}\Vert_\infty > m_{1-t_i} + C\sigma_{1-t_i}\sqrt{\log\epsilon^{-1}}\}}\right] \\
& ~~~ + \Ebb_{\Pbb}\left[\Vert\nabla\log p_{1-t_i}(\my_t|\mx) - \nabla\log p_{1-t_i}(\my_{t_i}|\mx)\Vert^2 \mathbf{I}_{\{\Vert\my_t\Vert_\infty \leq m_{1-t} + C\sigma_{1-t}\sqrt{\log\epsilon^{-1}}, \Vert\my_{t_i}\Vert_\infty \leq m_{1-t_i} + C\sigma_{1-t_i}\sqrt{\log\epsilon^{-1}}\}}\right].
\end{aligned}
$$
The two tail events may overlap, which only increases this upper bound. For the first tail event, Cauchy--Schwarz gives
\[
\begin{aligned}
    &~~~~\Ebb_\Pbb\left[
    \|\nabla\log p_{1-t_i}(\my_t|\mx)-\nabla\log p_{1-t_i}(\my_{t_i}|\mx)\|^2
    \cdot \mI_{\{\|\my_t\|_\infty>m_{1-t}+C\sigma_{1-t}\sqrt{\log\epsilon^{-1}}\}} \right]\\
    &\leq \Ebb_\Pbb \left[\|\nabla\log p_{1-t_i}(\my_t|\mx)-\nabla\log p_{1-t_i}(\my_{t_i}|\mx)\|^4 \right]^{1/2} \cdot \Pbb\left( \|\my_t\|_\infty>m_{1-t}+C\sigma_{1-t}\sqrt{\log\epsilon^{-1}} \right)^{1/2}\\
    &\lesssim \frac{\epsilon}{1-t_i}.
\end{aligned}
\]
Here we used the fourth-moment bounds established before and $\|\mathbf a-\mathbf b\|^4 \le 8(\|\mathbf a\|^4+\|\mathbf b\|^4)$. The same argument applies to the second tail event, so that
\[
\begin{aligned}
    &~~~~\Ebb_\Pbb\left[ \|\nabla\log p_{1-t_i}(\my_t|\mx)-\nabla\log p_{1-t_i}(\my_{t_i}|\mx)\|^2 \cdot \mI_{\{\|\my_{t_i}\|_\infty>m_{1-t_i}+C\sigma_{1-t_i}\sqrt{\log\epsilon^{-1}}\}} \right]\\
    &\leq \Ebb_\Pbb\left[ \|\nabla\log p_{1-t_i}(\my_t|\mx)-\nabla\log p_{1-t_i}(\my_{t_i}|\mx)\|^4 \right]^{1/2} \cdot \Pbb\left( \|\my_{t_i}\|_\infty>m_{1-t_i}+C\sigma_{1-t_i}\sqrt{\log\epsilon^{-1}} \right)^{1/2}\\
    &\lesssim \frac{\epsilon}{1-t_i}.
\end{aligned}
\]
It remains to control the contribution on the common central event
\[
    \|\my_t\|_\infty \le m_{1-t}+C\sigma_{1-t}\sqrt{\log\epsilon^{-1}},\qquad
    \|\my_{t_i}\|_\infty \le m_{1-t_i}+C\sigma_{1-t_i}\sqrt{\log\epsilon^{-1}}.
\]
Likewise, the fundamental theorem of calculus gives
\[
\begin{aligned}
    &\nabla\log p_{1-t_i}(\my_t|\mx)-\nabla\log p_{1-t_i}(\my_{t_i}|\mx)\\
    &\quad\quad=
    \int_0^1 \partial_{\my}\nabla\log p_{1-t_i} \bigl(\my_{t_i}+\theta(\my_t-\my_{t_i})\mid\mx\bigr) (\my_t-\my_{t_i})\,\mrd\theta.
\end{aligned}
\]
On the common central event, for every $\theta\in[0,1]$, convexity of the norm and $\sigma_{1-t}\leq\sigma_{1-t_i}$ imply
\[
\begin{aligned}
    \|\my_{t_i}+\theta(\my_t-\my_{t_i})\|_\infty
    &\leq (1-\theta)\|\my_{t_i}\|_\infty+\theta\|\my_t\|_\infty\\
    &\leq t_i+\theta(t-t_i)+C\sigma_{1-t_i}\sqrt{\log\epsilon^{-1}}.
\end{aligned}
\]
Consequently,
\[
    \frac{\bigl( \|\my_{t_i}+\theta(\my_t-\my_{t_i})\|_\infty-m_{1-t_i} \bigr)_+}{\sigma_{1-t_i}}
    \le \frac{t-t_i}{\sigma_{1-t_i}}+C\sqrt{\log\epsilon^{-1}}
    \lesssim \sqrt{\log\epsilon^{-1}},
\]
where $t-t_i\leq1-t_i\leq\sigma_{1-t_i}$ and $\log\epsilon^{-1}\geq1$. Lemma~\ref{lem: derivatives_boundness} bounds each column of the Jacobian. Summing the squared column bounds and taking the square root therefore gives
\[
    \left\|
    \partial_{\my}\nabla\log p_{1-t_i}
    \bigl(\my_{t_i}+\theta(\my_t-\my_{t_i})\mid\mx\bigr)
    \right\|_{\mathrm{op}}
    \lesssim
    \frac{\log\epsilon^{-1}}{\sigma_{1-t_i}^2}.
\]
Thus, on the common central event,
\[
\begin{aligned}
&\|\nabla\log p_{1-t_i}(\my_t|\mx)
-\nabla\log p_{1-t_i}(\my_{t_i}|\mx)\|^2\\
&\quad\qquad\leq
\|\my_t-\my_{t_i}\|^2
\int_0^1
\left\|
\partial_{\my}\nabla\log p_{1-t_i}
\bigl(\my_{t_i}+\theta(\my_t-\my_{t_i})\mid\mx\bigr)
\right\|_{\mathrm{op}}^2\,\mrd\theta\\
&\quad\qquad\lesssim
\frac{\log^2\epsilon^{-1}}{\sigma_{1-t_i}^4}
\|\my_t-\my_{t_i}\|^2.
\end{aligned}
\]
Then, we have
\[
\begin{aligned}
    &\Ebb_{\Pbb}\left[\Vert\nabla\log p_{1-t_i}(\my_t\mid\mx)-\nabla\log p_{1-t_i}(\my_{t_i}\mid\mx)\Vert^2
    \mI_{\{\Vert\my_t\Vert_\infty\leq m_{1-t}+C\sigma_{1-t}\sqrt{\log\epsilon^{-1}},\,
    \Vert\my_{t_i}\Vert_\infty\leq m_{1-t_i}+C\sigma_{1-t_i}\sqrt{\log\epsilon^{-1}}\}}\right]\\
    \lesssim &\ \frac{\log^2\epsilon^{-1}}{\sigma_{1-t_i}^4}\Ebb_{\Pbb}\left[\|\my_t-\my_{t_i}\|^2\right]\\
    \lesssim &\ \frac{\log^2\epsilon^{-1}}{(1-t_i)^2}\frac{t-t_i}{t_i}.
\end{aligned}
\]
Therefore $(\mathrm{II})$ is bounded by
\begin{equation}\label{eq: bound_II}
    (\mathrm{II})\lesssim \frac{1}{t^2}\left\{\frac{\epsilon}{1-t_i}
    +\frac{\log^2\epsilon^{-1}}{(1-t_i)^2}\frac{t-t_i}{t_i}\right\}.
\end{equation}

Next, consider $(\mathrm{III})$. Since $t/t_i\leq2$,
\[
    \left(\frac1t-\frac1{t_i}\right)^2
    = \frac{(t-t_i)^2}{t^2t_i^2}
    \leq \frac{4(t-t_i)^2}{t^4}.
\]
Same as before, on the central event, Lemma~\ref{lem: derivatives_boundness} bounds the squared score by $C\log\epsilon^{-1}/(1-t_i)$. For the tail contribution, the fourth-moment and tail-probability bounds give
\[
\begin{aligned}
    &~~~~\Ebb_\Pbb\left[ \|\nabla\log p_{1-t_i}(\my_{t_i}|\mx)\|^2 \mI_{\{\|\my_{t_i}\|_\infty>m_{1-t_i}+C\sigma_{1-t_i}\sqrt{\log\epsilon^{-1}}\}} \right]\\
    &\leq \Ebb_\Pbb \left[ \|\nabla\log p_{1-t_i}(\my_{t_i}|\mx)\|^4
    \right]^{1/2} \cdot
    \Pbb\left( \|\my_{t_i}\|_\infty>m_{1-t_i}+C\sigma_{1-t_i}\sqrt{\log\epsilon^{-1}} \right)^{1/2}
    \lesssim \frac{\epsilon}{1-t_i}.
\end{aligned}
\]
Using also $1-t_i\geq1-t$, we obtain
$$
\begin{aligned}
(\mathrm{III}) &= \left(\frac{1}{t} - \frac{1}{t_i}\right)^2 \cdot \Ebb_{\Pbb}\left[\Vert \nabla\log p_{1-t_i}(\my_{t_i}|\mx)\Vert^2\right] \\
& \lesssim \frac{(t-t_i)^2}{t^4} \cdot \Ebb_{\Pbb}\left[\Vert \nabla\log p_{1-t_i}(\my_{t_i}|\mx)\Vert^2 \mathbf{I}_{\{\Vert\my_{t_i}\Vert_\infty \leq m_{1-t_i} + C\sigma_{1-t_i}\sqrt{\log\epsilon^{-1}}\}}\right] \\
& ~~~~ +  \frac{(t-t_i)^2}{t^4} \cdot \Ebb_{\Pbb}\left[\Vert \nabla\log p_{1-t_i}(\my_{t_i}|\mx)\Vert^2 \mathbf{I}_{\{\Vert\my_{t_i}\Vert_\infty > m_{1-t_i} + C\sigma_{1-t_i}\sqrt{\log\epsilon^{-1}}\}}\right] \\
& \lesssim \frac{(t-t_i)^2}{t^4} \cdot \left(\frac{\log\epsilon^{-1}}{1-t} + \frac{\epsilon}{1-t}\right) \\
& \lesssim \frac{(t-t_i)^2}{t^4(1-t)} \cdot (\log\epsilon^{-1} + \epsilon ).
\end{aligned}
$$
Thus, term $(\mathrm{III})$ is bounded by
\begin{equation} \label{eq: bound_III}
(\mathrm{III}) \lesssim \frac{\Delta^2}{T^2}\cdot\frac{\log\epsilon^{-1} + \epsilon}{t^2(1-t)}.
\end{equation}

For the last term $(\mathrm{IV})$, $t/t_i\leq2$ implies $1/t_i^2\le 2/(tt_i)$, and hence
\begin{equation}\label{eq: bound_IV}
\begin{aligned}
    (\mathrm{IV})
    &= \frac1{t_i^2}\cdot\Ebb_\Pbb \left[\|\nabla\log p_{1-t_i}(\my_{t_i}|\mx)-\wh{\mb}(1-t_i,\my_{t_i},\mx)\|^2\right]\\
    &\lesssim
    \frac1{tt_i}\cdot\Ebb_\Pbb \left[\|\nabla\log p_{1-t_i}(\my_{t_i}|\mx)-\wh{\mb}(1-t_i,\my_{t_i},\mx)\|^2\right].
\end{aligned}
\end{equation}

To integrate the four bounds, observe that
\[
    \int_T^{1-T}\frac{\mrd t}{t(1-t)} = 2\log\frac{1-T}{T} \le 2\log(1/T),\qquad
    \int_T^{1-T}\frac{\mrd t}{t^2(1-t)^2} \le \frac8T.
\]
The second bound follows by using $(1-t)^{-2}\leq4$ on $[T,1/2]$ and symmetry on $[1/2,1-T]$. Moreover, $t/t_i\leq2$, $1-t_i\geq1-t$, and $t-t_i\leq\Delta$
give
\[
\begin{gathered}
    \sum_{i=0}^{K-1}\int_{t_i}^{t_{i+1}} \frac{\epsilon}{t(1-t_i)}\,\mrd t
    \le \epsilon\int_T^{1-T}\frac{\mrd t}{t(1-t)}
    \lesssim \epsilon\log\mcM,\\
    \sum_{i=0}^{K-1}\int_{t_i}^{t_{i+1}} \frac{t-t_i}{tt_i(1-t_i)^2}\,\mrd t
    \le 2\Delta\int_T^{1-T}\frac{\mrd t}{t^2(1-t)^2}
    \lesssim \frac{\Delta}{T},
\end{gathered}
\]
where we used $\log(1/T)=C_T\log\mcM$. Combining \eqref{eq: bound_I}--\eqref{eq: bound_IV}, we obtain
\[
\begin{aligned}
    &~~~~\Ebb_{\mcD,\mcT,\mcZ}\Ebb_{\mx}
    \left[\sum_{i=0}^{K-1}\int_{t_i}^{t_{i+1}}t\,
    \Ebb_{\Pbb}\left\|\frac{\nabla\log p_{1-t}(\my_t\mid\mx)}{t}
    -\frac{\wh{\mb}(1-t_i,\my_{t_i},\mx)}{t_i}\right\|^2\mrd t\right]\\
    &\lesssim \frac{\Delta^2}{T^2}\log^3\epsilon^{-1}\cdot\log\mcM + \epsilon \log\mcM\\
    &~~~~ + \frac{\Delta}{T}\log^2\epsilon^{-1} + \epsilon\log\mcM\\
    &~~~~ + \frac{\Delta^2}{T^2}\left( \log\epsilon^{-1} + \epsilon \right)\log\mcM \\
    &~~~~ + \Ebb_{\mcD,\mcT,\mcZ}\Ebb_{\mx}\left(\sum_{i=0}^{K-1}\frac{t_{i+1}-t_i}{t_i}\cdot\Ebb_{\Pbb}\Vert\nabla\log p_{1-t_i}(\my_{t_i}|\mx) - \wh{\mb}(1-t_i,\my_{t_i},\mx)\Vert^2\right)\\
    &\lesssim \frac{\Delta}{T}\log^2\epsilon^{-1}
    +\frac{\Delta^2}{T^2}\left(\log^3\epsilon^{-1}+\log\epsilon^{-1}+\epsilon\right)\log\mcM +\epsilon\log\mcM\\
    &~~~~ + \Ebb_{\mcD,\mcT,\mcZ}\Ebb_{\mx}\left[
    \sum_{i=0}^{K-1}\frac{t_{i+1}-t_i}{t_i}
    \Ebb_{\Pbb}\left\|\nabla\log p_{1-t_i}(\my_{t_i}\mid\mx)-\wh{\mb}(1-t_i,\my_{t_i},\mx)\right\|^2\right].
\end{aligned}
\]

Take $\epsilon=n^{-\frac{2\beta}{d_\mcX+d_\mcY+2\beta}}$, $\mcM=\lfloor n^{\frac1{d_\mcX+d_\mcY+2\beta}}\rfloor+1$, $C_T=2\beta/(\beta\wedge1)$, and $T=\mcM^{-C_T}$. Use the common deterministic grid with $\Delta=\mcO(\min\{\Delta_n,\mcM^{-(C_T+2\beta)}\})$ for which \eqref{eq: discrete_generalization} holds uniformly over the density class. For sufficiently large $n$, these choices give $0<\epsilon\leq e^{-1}$,
$\epsilon\asymp\mcM^{-2\beta}$, $\log\epsilon^{-1}\asymp\log\mcM$, and
$\Delta/T\lesssim\mcM^{-2\beta}$. Consequently,
\[
\begin{aligned}
    &~~~~\frac{\Delta}{T}\log^2\epsilon^{-1}
    +\frac{\Delta^2}{T^2}
    \left(
    \log^3\epsilon^{-1}+\log\epsilon^{-1}+\epsilon
    \right)\log\mcM
    +\epsilon\log\mcM\\
    &\lesssim
    \mcM^{-2\beta}\log^2\mcM
    +\mcM^{-4\beta}\log^4\mcM
    +\mcM^{-2\beta}\log\mcM\\
    &\lesssim
    \mcM^{-2\beta}\log^{19}\mcM.
\end{aligned}
\]
Combining this estimate with \eqref{eq: corrected_bound_V} and \eqref{eq: discrete_generalization}, we obtain
\[
    \Ebb_{\mcD,\mcT,\mcZ}\Ebb_{\mx}
    \left[ \sum_{i=0}^{K-1}\int_{t_i}^{t_{i+1}} t\cdot\Ebb_\Pbb \left\| \frac{\my_t}{t}-\frac{\my_{t_i}}{t_i} \right\|^2\mrd t \right]
    \lesssim n^{-\frac{2\beta}{d_\mcX+d_\mcY+2\beta}},
\]
and
\[
\Ebb_{\mcD,\mcT,\mcZ}\Ebb_{\mx}
\left[ \sum_{i=0}^{K-1}\int_{t_i}^{t_{i+1}} t\cdot\Ebb_\Pbb \left\| \frac{\nabla\log p_{1-t}(\my_t|\mx)}{t}-\frac{\wh{\mb}(1-t_i,\my_{t_i},\mx)}{t_i} \right\|^2\mrd t \right]
\lesssim n^{-\frac{2\beta}{d_\mcX+d_\mcY+2\beta}}\log^{19}n.
\]
Therefore, for all sufficiently large $n$,
\begin{equation*}
\begin{aligned}
\Ebb_{\mcD,\mcT,\mcZ}\Ebb_{\mx}[\mathrm{TV}(p_T(\cdot|\mx),\ov{p}_T(\cdot|\mx))]
&\lesssim \left(\Ebb_{\mcD,\mcT,\mcZ}\Ebb_{\mx}[\mathrm{TV}^2(p_T(\cdot|\mx),\ov{p}_T(\cdot|\mx))]\right)^{\frac{1}{2}} \\
&\lesssim n^{-\frac{\beta}{d_\mcX + d_\mcY + 2\beta}}\log^{\frac{19}{2}}n.
\end{aligned}
\end{equation*}
Combining this with \eqref{eq: TV bound_1}, we have
\begin{align*}
    \Ebb_{\mcD,\mcT,\mcZ}\Ebb_{\mx}[\mathrm{TV}(\wt{p}_T(\cdot|\mx),p_T(\cdot|\mx))]
    &\lesssim \Ebb_{\mcD,\mcT,\mcZ}\Ebb_{\mx}[\mathrm{TV}(p_T(\cdot|\mx),\ov{p}_T(\cdot|\mx))] + \Ebb_{\mcD,\mcT,\mcZ}\Ebb_{\mx}[\mathrm{TV}(p_{1-T}(\cdot|\mx),\mcN(0,\mI_{d_\mcY}))]\\
    &\lesssim n^{-\frac{\beta}{d_\mcX + d_\mcY + 2\beta}}\log^{\frac{19}{2}}n + n^{-\frac{2\beta}{d_\mcX + d_\mcY + 2\beta}}\\
    &\lesssim n^{-\frac{\beta}{d_\mcX + d_\mcY + 2\beta}}\log^{\frac{19}{2}}n.
\end{align*}
The implicit constant and the threshold for sufficiently large $n$ in this final bound are independent of $p_0\in\mcP_\beta(C_l,C_u,R)$, and the same network class and deterministic grid are used throughout the density class. Hence the bound holds uniformly over this class. The proof is complete.
\end{proof}

Based on Theorem \ref{thm: sampling_error} and Lemma \ref{lem: early_stopping}, we can bound $\Ebb_{\mcD,\mcT,\mcZ}\Ebb_{\mx}[\mathrm{TV}(\wt{p}_T(\cdot|\mx), p_0(\cdot|\mx))]$.
\begin{proof}[Proof of Theorem \ref{thm: convergence_rate}]
The bounds in \eqref{eq: sub_bound2} and \eqref{eq: sub_bound1} hold uniformly over $p_0\in\mcP_\beta(C_l,C_u,R)$, since their constants depend only on the fixed class parameters $C_l,C_u,R,\beta$, the dimensions, and the prescribed network caps. Thus, for all sufficiently large $n$, the triangle inequality gives
\[
\begin{aligned}
    &\sup_{p_0\in\mcP_\beta(C_l,C_u,R)}
    \Ebb_{\mcD,\mcT,\mcZ}\Ebb_{\mx}
    \left[ \mathrm{TV}\bigl(\wt{p}_T(\cdot|\mx),p_0(\cdot|\mx)\bigr) \right]\\
    &\le \sup_{p_0\in\mcP_\beta(C_l,C_u,R)}\Ebb_{\mcD,\mcT,\mcZ}\Ebb_{\mx}\left[ \mathrm{TV}\bigl(\wt{p}_T(\cdot|\mx),p_T(\cdot|\mx)\bigr) \right]+\sup_{p_0\in\mcP_\beta(C_l,C_u,R)}\Ebb_{\mx}\left[ \mathrm{TV}\bigl(p_T(\cdot|\mx),p_0(\cdot|\mx)\bigr) \right]\\
    &\lesssim n^{-\frac{\beta}{d_\mcX+d_\mcY+2\beta}}\log^{\frac{19}{2}}n+n^{-\frac{\beta}{d_\mcX+d_\mcY+2\beta}}\\
    &\lesssim n^{-\frac{\beta}{d_\mcX+d_\mcY+2\beta}}\log^{\frac{19}{2}}n.
\end{aligned}
\]
This proves the claimed uniform bound.
\end{proof}

\section{Minimax Lower Bound}\label{sec:minimax}
For fixed constants $0<C_l<2^{-d_\mcY}<C_u<\infty$ and $R>2^{-d_{\mcY}}$, recall that $\mcP_\beta(C_l,C_u,R)$ is the collection of conditional densities satisfying Assumptions \ref{ass: bounded_density}--\ref{ass: Holder} with these constants. For a fixed density $p_\mX$ on $[-1,1]^{d_\mcX}$, define
\[
    \mathfrak R_n(\mcP_\beta,p_\mX)
    :=\inf_{\wh p}\sup_{p_0\in\mcP_\beta(C_l,C_u,R)}
    \Ebb_{p_0,p_{\mX}}^{(n)}\left[
    \int_{[-1,1]^{d_\mcX}}\operatorname{TV}(\wh p(\cdot\mid\mx),p_0(\cdot\mid\mx))p_\mX(\mx)\mrd\mx
    \right].
\]
\begin{theorem}[Minimax lower bound]\label{prop: minimax_lower_bound}
Let $\beta>0$ and suppose that the fixed density satisfies
\[
    p_\mX(\mx)\geq c_\mX>0,
    \qquad \text{for almost every}\ \mx\in[-1,1]^{d_\mcX}.
\]
For every fixed H\"older radius $R>2^{-d_{\mcY}}$ and every pair $0<C_l<2^{-d_\mcY}<C_u<\infty$,
\[
\mathfrak R_n(\mcP_\beta,p_\mX)
\gtrsim n^{-\frac{\beta}{d_\mcX+d_\mcY+2\beta}}.
\]
Consequently, the polynomial part of Theorem \ref{thm: convergence_rate} is minimax optimal under average conditional total-variation loss, up to logarithmic factors.
\end{theorem}
\begin{proof}
Recall $\beta=r+s>0$, where $r$ is an integer and $s\in(0,1]$. Choose bump functions $\psi_{\mcY}\in C_c^\infty((-1,1)^{d_{\mcY}})$ and $\psi_{\mcX}\in C_c^\infty((-1,1)^{d_{\mcX}})$ supported in $(-1,1)^{d_\mcY}$ and $(-1,1)^{d_\mcX}$ such that
\[
    \int_{(-1,1)^{d_{\mcY}}}\psi_{\mcY}(\my)\mrd\my=0,\qquad
    \int_{(-1,1)^{d_{\mcY}}}|\psi_{\mcY}(\my)|\mrd\my>0,\qquad
    \|\psi_{\mcY}\|_{L^\infty((-1,1)^{d_{\mcY}})}\leq1,
\]
\[
    \int_{(-1,1)^{d_{\mcX}}}\psi_{\mcX}(\mx)\mrd\mx>0,\qquad
    0\leq\psi_{\mcX}(\mx)\leq1\ \text{for}\ \mx\in(-1,1)^{d_{\mcX}}.
\]
Let $\mcM\geq2$, $h=2/\mcM$, and partition $[-1,1]^{d_\mcY}$, $[-1,1]^{d_\mcX}$ into half-open cubes of side length $h$. If $\my_{\mathbf j}$ and $\mx_{\mathbf k}$ are the centers of the response and covariate cubes, set
\[
    \psi_{\mcY,\mathbf j}(\my):=\psi_\mcY\!\left(\frac{\my-\my_{\mathbf j}}{h/2}\right),
    \qquad
    \psi_{\mcX,\mathbf k}(\mx):=\psi_\mcX\!\left(\frac{\mx-\mx_{\mathbf k}}{h/2}\right),\qquad
    \my\in\mathbb R^{d_{\mcY}},\ \mx\in\mathbb R^{d_{\mcX}}.
\]
Then we have $\operatorname{supp}(\psi_{\mcY,\mathbf j})\subset\operatorname{int}(I_{\mcY,\mathbf j})$, $\operatorname{supp}(\psi_{\mcX,\mathbf k})\subset\operatorname{int}(I_{\mcX,\mathbf k})$ and moreover,
\[
    \operatorname{supp}(\psi_{\mcY,\mathbf j})
    \cap \operatorname{supp}(\psi_{\mcY,\mathbf j'})=\varnothing,\qquad
    \mathbf j\neq\mathbf j',
\]
\[
    \operatorname{supp}(\psi_{\mcX,\mathbf k})
    \cap \operatorname{supp}(\psi_{\mcX,\mathbf k'})=\varnothing,\qquad
    \mathbf k\neq\mathbf k'.
\]
There are $N=\mcM^{d_\mcY+d_\mcX}$ such pairs of $(\mathbf j,\mathbf k)$.\\
\indent For $\boldsymbol\omega=(\omega_{\mathbf j,\mathbf k})\in\{0,1\}^N$ and $a=c_0\mcM^{-\beta}$, define
\[
    p_{\boldsymbol\omega}(\my\mid\mx)
    :=2^{-d_\mcY}+a\sum_{\mathbf j,\mathbf k}\omega_{\mathbf j,\mathbf k}
    \psi_{\mcY,\mathbf j}(\my)\psi_{\mcX,\mathbf k}(\mx),\qquad
    (\my,\mx)\in[-1,1]^{d_{\mcY}}\times[-1,1]^{d_{\mcX}},
\]
and put $p_{\boldsymbol\omega}(\my\mid\mx)=0$ for
$\my\notin[-1,1]^{d_\mcY}$. Then for every $\mx$,
\[
    \int_{[-1,1]^{d_\mcY}} p_{\boldsymbol\omega}(\my\mid\mx)\mrd\my
    = 1+a\sum_{\mathbf j,\mathbf k}\omega_{\mathbf j,\mathbf k}\psi_{\mcX,\mathbf k}(\mx) \int_{[-1,1]^{d_\mcY}}\psi_{\mcY,\mathbf j}(\my)\mrd\my
    = 1.
\]
By definition of the partition, at every $(\my,\mx)$ at most one summand is nonzero. Hence, for
\[
    0<c_0\leq\frac12\min\{2^{-d_\mcY}-C_l,C_u-2^{-d_\mcY}\},
\]
we have
\[
    C_l\leq2^{-d_\mcY}-a\leq p_{\boldsymbol\omega}(\my\mid\mx) \leq 2^{-d_\mcY}+a\leq C_u.
\]
Here we used the fact that $\mcM\ge 2$ and $\beta>0$ implies $0<\mcM^{-\beta}\leq 1$ and thus $a=c_0\mcM^{-\beta}\le c_0$.\\
\indent We then show that $p_{\boldsymbol{\omega}}\in\mcH^\beta([-1,1]^{d_\mcY}\times[-1,1]^{d_\mcX},R)$ for any $\boldsymbol{\omega}\in\{0,1\}^N$. For notation simplicity, denote by
\[
    G_{\boldsymbol\omega}(\my,\mx)
    :=\sum_{\mj,\mk}\omega_{\mathbf j,\mk} \psi_{\mcY,\mj}(\my)\psi_{\mcX,\mk}(\mx).
\]
Then we can express $p_{\boldsymbol{\omega}}$ as $p_{\boldsymbol{\omega}}(\my\mid\mx)=2^{-d_\mcY}+aG_{\boldsymbol\omega}(\my,\mx)$. Let $\boldsymbol\alpha=(\boldsymbol\alpha_{\mcY}, \boldsymbol\alpha_{\mcX})$ denote the multi-index with respect to $(\my,\mx)$. By definition of the bump functions, it is clear that $\operatorname{supp}\left(\partial^{\boldsymbol\alpha_{\mcY}}\psi_{\mcY,\mj}\right) \subseteq \operatorname{supp}\left(\psi_{\mcY,\mj}\right)$ and $\operatorname{supp}\left(\partial^{\boldsymbol\alpha_{\mcX}}\psi_{\mcX,\mk}\right) \subseteq \operatorname{supp}\left(\psi_{\mcX,\mk}\right)$. Since $h/2=1/\mcM$, the chain rule gives, for any $\|\boldsymbol{\alpha}\|_1\le r$ and some constant $C_\beta$ that,
\[
\begin{aligned}
    \left\| \partial^{\boldsymbol\alpha} G_{\boldsymbol\omega}(\my,\mx) \right\|_\infty
    &= \left\| \sum_{\mj,\mk}\omega_{\mj,\mk}\partial^{\boldsymbol\alpha}\left( \psi_{\mcY,\mj}(\my)\psi_{\mcX,\mk}(\mx) \right) \right\|_\infty\\
    &= \left\| \sum_{\mj,\mk}\omega_{\mj,\mk} \mcM^{\|\boldsymbol\alpha\|_1} \left(\partial^{\boldsymbol\alpha_{\mcY}}\psi_{\mcY}\right)\left(\mcM(\my-\my_{\mj})\right)\left(\partial^{\boldsymbol\alpha_{\mcX}}
    \psi_{\mcX}\right)\left(\mcM(\mx-\mx_{\mk})\right) \right\|_\infty\\
    &\le \mcM^{\|\boldsymbol\alpha\|_1}\left\| \partial^{\boldsymbol\alpha_{\mcY}}
    \psi_{\mcY} \right\|_\infty \left\| \partial^{\boldsymbol\alpha_{\mcX}}
    \psi_{\mcX} \right\|_\infty\\
    &\le C_\beta\mcM^{\|\boldsymbol\alpha\|_1}.
\end{aligned}
\]
It follows, uniformly over $\|\boldsymbol\alpha\|_1\leq r$, that
\[
\begin{aligned}
    \left\| \partial^{\boldsymbol\alpha}p_{\boldsymbol\omega} \right\|_\infty
    &\leq 2^{-d_\mcY} + aC_\beta\mcM^{\|\boldsymbol\alpha\|_1}\\
    &= 2^{-d_\mcY} + C_\beta c_0\mcM^{\|\boldsymbol\alpha\|_1-\beta}\\
    &\leq 2^{-d_\mcY} + C_\beta c_0,
\end{aligned}
\]
where the last inequality follows from $\|\boldsymbol\alpha\|_1\leq r<\beta$. For every multi-index satisfying $\|\boldsymbol\alpha\|_1=r$, the same argument applied to one additional derivative yields
\[
    \left\|\partial^{\boldsymbol\alpha}G_{\boldsymbol\omega}(\my,\mx)\right\|_\infty \leq C_\beta\mcM^r,\qquad
    \operatorname{Lip}
    \left( \partial^{\boldsymbol\alpha}G_{\boldsymbol\omega}(\my,\mx) \right)
    \leq C_\beta\mcM^{r+1}.
\]
These imply that, for $\mz=(\my,\mx)$ and $\mz'=(\my',\mx')$ in
$[-1,1]^{d_\mcY}\times[-1,1]^{d_\mcX}$,
\[
\begin{aligned}
    \left| \partial^{\boldsymbol\alpha}G_{\boldsymbol\omega}(\mz)-\partial^{\boldsymbol\alpha}G_{\boldsymbol\omega}(\mz') \right|
    &\le \min\left\{ 2\|\partial^{\boldsymbol\alpha}G_{\boldsymbol\omega}\|_\infty,C_\beta\mcM^{r+1}\|\mz-\mz'\| \right\}\\
    &\le \min\left\{ 2C_\beta\mcM^{r},2C_\beta\mcM^{r+1}\|\mz-\mz'\| \right\}\\
    &\le C_\beta\mcM^r \min\left\{ 1,\mcM\|\mz-\mz'\| \right\},
\end{aligned}
\]
where in the last inequality, we enlarge $C_\beta$ to absorb the extra constant $2$. Since $0<s\leq1$, the inequality $\min\{1,u\}\leq u^s$ holds for every $u\geq0$. Hence
\[
\begin{aligned}
    \left| \partial^{\boldsymbol\alpha}p_{\boldsymbol\omega}(\mz)-\partial^{\boldsymbol\alpha}p_{\boldsymbol\omega}(\mz') \right|
    &\leq aC_\beta\mcM^r \min\left\{1,\mcM\|\mz-\mz'\|\right\}\\
    &\leq C_\beta c_0\mcM^{-(r+s)}\mcM^r\left( \mcM\|\mz-\mz'\| \right)^s\\
    &= C_\beta c_0 \|\mz-\mz'\|^s.
\end{aligned}
\]
And therefore,
\[
    \max_{\|\boldsymbol\alpha\|_1=r}\sup_{\mz\neq\mz'}
    \frac{\left| \partial^{\boldsymbol\alpha}p_{\boldsymbol\omega}(\mz)-\partial^{\boldsymbol\alpha}p_{\boldsymbol\omega}(\mz') \right|}{\|\mz-\mz'\|^s}
    \leq C_\beta c_0.
\]
Combining the derivative and H\"older-seminorm bounds gives
\[
    \|p_{\boldsymbol\omega}\|_{\mcH^\beta}
    \leq 2^{-d_\mcY}+C_\beta c_0.
\]
Since $R>2^{-d_\mcY}$, taking $c_0$ sufficiently small such that $c_0 \leq \frac{R-2^{-d_\mcY}}{C_\beta}$ yields that $\|p_{\boldsymbol\omega}\|_{\mcH^\beta}\leq R$ and thus
\[
\{p_{\boldsymbol\omega}:\boldsymbol\omega\in\{0,1\}^{N}\}\subset\mcP_\beta(C_l,C_u,R).
\]
\indent Denoting by
\[
    \rho(p,q)
    := \int_{[-1,1]^{d_{\mathcal X}}} \operatorname{TV} \bigl(p(\cdot\mid\mx),q(\cdot\mid\mx)\bigr) p_{\mX}(\mx)\,\mrd\mx.
\]
By the Varshamov--Gilbert bound (see \cite[Lemma 2.9]{tsybakov2009introduction}), there exists a set $\Omega\subset\{0,1\}^N$ containing $\boldsymbol0$ such that
\[
    |\Omega|\geq 2^{\frac{N}{8}},\qquad
    d_H(\boldsymbol\omega,\boldsymbol\omega')\geq \frac{N}{8},\quad \forall\ \boldsymbol\omega,\boldsymbol\omega'\in\Omega,\ \boldsymbol\omega\neq\boldsymbol\omega',
\]
where $d_H$ is the Hamming distance defined by $d_H(\boldsymbol\omega,\boldsymbol\omega'):=\sum_{\ell=1}^N \mathbf 1_{\{\omega_\ell\neq\omega'_\ell\}}$. We will next give the lower bound of
$
    \sup_{\boldsymbol\omega\in\Omega} \Ebb_{\boldsymbol\omega}
    \bigl[\rho(\wh p,p_{\boldsymbol\omega})\bigr]
$
for any density estimator $\wh p$. For $\boldsymbol\omega\neq\boldsymbol\omega'$, disjointness of the bump supports and $p_\mX\geq c_\mX$ imply
\[
\begin{aligned}
    \rho(p_{\boldsymbol{\omega}},p_{\boldsymbol{\omega}'})
    &= \int p_\mX(\mx)\operatorname{TV}(p_{\boldsymbol\omega}(\cdot\mid\mx),p_{\boldsymbol\omega'}(\cdot\mid\mx))\mrd\mx\\
    &= \frac a2\sum_{\mathbf j,\mathbf k}|\omega_{\mathbf j,\mathbf k}-\omega'_{\mathbf j,\mathbf k}|
    \int p_\mX(\mx)\psi_{\mcX,\mathbf k}(\mx)\mrd\mx
    \int|\psi_{\mcY,\mathbf j}(\my)|\mrd\my\\
    &\geq \frac{a c_\mX}{2}\mcM^{-(d_\mcY+d_\mcX)}
    \left(\int \psi_\mcX(\mathbf u)\mrd\mathbf u \right)\left(\int|\psi_\mcY(\mv)|\mrd\mv\right)
    d_H(\boldsymbol\omega,\boldsymbol\omega')\\
    &\gtrsim a\mcM^{-(d_\mcY+d_\mcX)}N
    \asymp a.
\end{aligned}
\]
Let $P_{\boldsymbol\omega}$ be the joint distribution with density
$p_\mX(\mx)p_{\boldsymbol\omega}(\my\mid\mx)$. Since
$p_{\boldsymbol0}(\my\mid\mx)=2^{-d_\mcY}$ and
$\int(p_{\boldsymbol\omega}-p_{\boldsymbol0})\mrd\my=0$,
\[
\begin{aligned}
    \operatorname{KL}(P_{\boldsymbol\omega}^{\otimes n},P_{\boldsymbol0}^{\otimes n})
    &= n\int p_\mX(\mx)p_{\boldsymbol\omega}(\my\mid\mx)
    \log\frac{p_{\boldsymbol\omega}(\my\mid\mx)}{p_{\boldsymbol0}(\my\mid\mx)}\mrd\my\mrd\mx\\
    &\leq n\int p_\mX(\mx)p_{\boldsymbol\omega}(\my\mid\mx)
    \left(\frac{p_{\boldsymbol\omega}(\my\mid\mx)}{p_{\boldsymbol0}(\my\mid\mx)}-1\right)\mrd\my\mrd\mx\\
    &= n\int p_\mX(\mx)
    \frac{(p_{\boldsymbol\omega}(\my\mid\mx)-p_{\boldsymbol0}(\my\mid\mx))^2}{p_{\boldsymbol0}(\my\mid\mx)}\mrd\my\mrd\mx\\
    &\leq 2^{d_\mcY}na^2\sum_{\mathbf j,\mathbf k}\omega_{\mathbf j,\mathbf k}
    \int p_\mX(\mx)\psi_{\mcX,\mathbf k}^2(\mx)\mrd\mx
    \int\psi_{\mcY,\mathbf j}^2(\my)\mrd\my\\
    &\leq Cna^2
    \sum_{\mathbf k}\int p_\mX(\mx)\psi_{\mcX,\mathbf k}^2(\mx)\mrd\mx
    \sum_{\mathbf j}\left(\frac h2\right)^{d_\mcY}\int\psi_\mcY^2(\mv)\mrd\mv\\
    &\leq Cna^2\cdot\int p_\mX(\mx)\mrd\mx \cdot \mcM^{d_\mcY}\frac{1}{\mcM^{d_\mcY}}\int\psi_{\mcY}^2(\mv)\mrd\mv\\
    &\lesssim na^2,
\end{aligned}
\]
Choose $\mcM=\left\lceil n^{\frac{1}{d_\mcY+d_\mcX+2\beta}}\right\rceil$, we have
\[
    \operatorname{KL}(P_{\boldsymbol\omega}^{\otimes n},P_{\boldsymbol0}^{\otimes n})
    \lesssim na^2
    = c_0^2n\mcM^{-2\beta}
    \le c_0^2\mcM^{d_\mcY+d_\mcX} = c_0^2N.
\]
Since $\log|\Omega|\geq\frac{N\log 2}{8}$, decreasing $c_0$ once more gives
\[
    \frac1{|\Omega|}\sum_{\boldsymbol\omega\in\Omega}
    \operatorname{KL}(P_{\boldsymbol\omega}^{\otimes n},P_{\boldsymbol0}^{\otimes n})
    \leq \kappa\log|\Omega|,
\]
where $\kappa\in(0,1)$ is a constant.\\
\indent Let $W$ be uniformly distributed on $\Omega$. Conditional on $W=\boldsymbol\omega$, let the sample $\mathcal Z_n$ follow $P_{\boldsymbol\omega}^{\otimes n}$. Define the corresponding mixture distribution and the mutual information by
\[
    \overline P^{(n)}
    := \frac{1}{|\Omega|}\sum_{\boldsymbol\omega\in\Omega} P_{\boldsymbol\omega}^{\otimes n}\qquad
    I(W;\mathcal Z_n)
    := \frac{1}{|\Omega|}\sum_{\boldsymbol\omega\in\Omega}
    \operatorname{KL}\!\left( P_{\boldsymbol\omega}^{\otimes n},\overline P^{(n)} \right).
\]
The standard upper bound of mutual information gives
\begin{equation}\label{eq:mutual-information-bound}
    I(W;\mathcal Z_n) \leq
    \frac{1}{|\Omega|} \sum_{\boldsymbol\omega\in\Omega} \operatorname{KL}\!\left( P_{\boldsymbol\omega}^{\otimes n},P_{\boldsymbol 0}^{\otimes n} \right)
    \leq \kappa\log|\Omega|.
\end{equation}
Now let $\widehat p$ be an arbitrary conditional-density estimator and define the nearest-neighbor decoder
\[
    \widehat W
    \in \operatorname*{arg\,min}_{\boldsymbol\nu\in\Omega}
    \rho\!\left(\widehat p,p_{\boldsymbol\nu}\right),
\]
where ties are resolved according to an arbitrary fixed deterministic
rule. Recall that $\rho\left(p_{\boldsymbol\omega},p_{\boldsymbol\omega'}\right)\gtrsim a$ for any $\boldsymbol\omega\neq\boldsymbol\omega'$, and thus there exists $\wt{s}\asymp a$ such that
\[
    \rho\!\left(p_{\boldsymbol\omega},p_{\boldsymbol\omega'}\right)
    \geq 2\wt{s}, \qquad
    \boldsymbol\omega\neq\boldsymbol\omega'.
\]
While on the event $\{\widehat W\neq W\}$, the pairwise separation and the triangle inequality imply
\begin{align*}
    2\wt{s}
    \leq \rho\left(p_W,p_{\widehat W}\right)
    \leq \rho\left(p_W,\widehat p\right) + \rho\left(\widehat p,p_{\widehat W}\right)
    \leq \rho\left(p_W,\widehat p\right) + \rho\left(\widehat p,p_W\right)
    = 2\rho\left(\widehat p,p_W\right).
\end{align*}
Therefore, $\rho\left(\widehat p,p_W\right)\geq \wt{s}$ on the event $\{\widehat W\neq W\}$ and
\begin{equation}\label{eq:risk-testing-reduction}
    \Ebb\left[ \rho\left(\wh p,p_W\right) \right]
    \geq \wt{s}\,\Pbb(\wh W \neq W).
\end{equation}
\indent Since $\widehat W$ is constructed from the sample $\mathcal Z_n$, we have the Markov chain
\[
    W\longrightarrow\mathcal Z_n\longrightarrow\widehat W.
\]
Therefore, the data-processing inequality gives $I(W;\widehat W)\leq I(W;\mathcal Z_n)$. By Fano's inequality (see \cite[Theorem 1]{scarlett2019introductory}) and \eqref{eq:mutual-information-bound}, we thus have
\[
\begin{aligned}
    \Pbb(\widehat W\neq W)
    &\geq
    1-\frac{I(W;\widehat W)+\log 2}{\log|\Omega|}\\
    &\geq
    1-\frac{I(W;\mathcal Z_n)+\log 2}{\log|\Omega|}\\
    &\geq
    1-\kappa-\frac{\log 2}{\log|\Omega|}.
\end{aligned}
\]
Since $\log|\Omega|\gtrsim N$ and
$N=\mathcal M^{d_{\mathcal X}+d_{\mcY}}\to\infty$, we have $\frac{\log 2}{\log|\Omega|}=o(1)$. Because $\kappa<1$ is fixed, for all sufficiently large $n$,
\[
    \Pbb(\widehat W\neq W)
    \geq \frac{1-\kappa}{2}
    \gtrsim 1.
\]
Combining this inequality with \eqref{eq:risk-testing-reduction} and
$\wt{s}\asymp a$ yields $\Ebb\left[ \rho\left(\wh p,p_W\right) \right] \gtrsim a$. And since $W$ is uniformly distributed on $\Omega$,
\[
    \Ebb\left[ \rho\left(\wh p,p_W\right) \right]
    = \frac{1}{|\Omega|}\sum_{\boldsymbol\omega\in\Omega}
    \Ebb_{\boldsymbol\omega}\left[ \rho\left( \wh p,p_{\boldsymbol\omega} \right) \right].
\]
It follows that
\[
    \sup_{\boldsymbol\omega\in\Omega}
    \Ebb_{\boldsymbol\omega}\left[ \rho\left( \widehat p,p_{\boldsymbol\omega} \right) \right]
    \geq \frac{1}{|\Omega|}\sum_{\boldsymbol\omega\in\Omega}
    \Ebb_{\boldsymbol\omega}\left[ \rho\left( \widehat p,p_{\boldsymbol\omega} \right) \right]
    \gtrsim a.
\]
\indent Because $\left\{ p_{\boldsymbol\omega}:\boldsymbol\omega\in\Omega \right\}\subset\{ p_{\boldsymbol\omega}:\boldsymbol\omega\in\{0,1\}^N\} \subset\mcP_\beta(C_l,C_u,R)$, we obtain
\begin{align*}
    \sup_{p_0\in\mcP_\beta(C_l,C_u,R)}\Ebb_{p_0,p_{\mX}}^{(n)}
    \left[ \rho(\wh p,p_0) \right]
    \geq \sup_{\boldsymbol\omega\in\Omega}\Ebb_{\boldsymbol\omega}
    \left[ \rho\left( \wh p,p_{\boldsymbol\omega} \right) \right]
    \gtrsim a.
\end{align*}
Taking the infimum over all estimators $\widehat p$ gives
\[
    \mathfrak R_n(\mcP_\beta,p_{\boldsymbol X})
    = \inf_{\wh p} \sup_{p_0\in\mcP_\beta(C_l,C_u,R)}\Ebb_{p_0,p_{\mX}}^{(n)}
    \left[ \rho(\wh p,p_0) \right]
    \gtrsim a.
\]
Since $a \asymp \mcM^{-\beta} \asymp n^{-\frac{\beta}{d_{\mathcal X}+d_{\mcY}+2\beta}}$, we finally derive the minimax lower bound
\[
    \mathfrak R_n(\mcP_\beta,p_{\boldsymbol X})
    \gtrsim n^{-\frac{\beta}{d_{\mathcal X}+d_{\mcY}+2\beta}}.
\]
This completes the proof.
\end{proof}

To relate the randomized estimator in Theorem~\ref{thm: convergence_rate} to the definition of $\mathfrak R_n(\mcP_\beta,p_{\mX})$, define, for each training sample $\mcD$,
\[
    \widehat p_{\mcD}^{\,\mathrm{av}}(\my\mid\mx):= \Ebb_{\mcT,\mcZ}
    \left[ \wt p_T(\my\mid\mx)\mid\mcD \right],\qquad
    \my\in\Rbb^{d_{\mcY}},\quad
    \mx\in[-1,1]^{d_{\mcX}}.
\]
The expectation integrates out the auxiliary training randomness under its prescribed distribution, which is independent of $\mcD$ and of the unknown conditional density. Thus $\widehat p_{\mcD}^{\,\mathrm{av}}$ is an estimator based only on $\mcD$. It is nonnegative, and integration of the nonnegative random density gives $\int_{\Rbb^{d_{\mcY}}}\wh p_{\mcD}^{\,\mathrm{av}}(\my\mid\mx)\,\mrd\my=1$.

For the average conditional total-variation distance $\rho$ defined above, the inequality $|\Ebb V|\le\Ebb|V|$ and interchange of nonnegative integrals yield
\[
\begin{aligned}
    \rho(\wh p_{\mcD}^{\,\mathrm{av}},p_0)
    &= \frac12\int_{[-1,1]^{d_{\mcX}}}p_{\mX}(\mx) \int_{\Rbb^{d_{\mcY}}} \left| \Ebb_{\mcT,\mcZ}\left[ \wt p_T(\my\mid\mx)-p_0(\my\mid\mx)\mid\mcD \right] \right|\,\mrd\my\,\mrd\mx\\
    &\le \Ebb_{\mcT,\mcZ}\left[ \rho(\wt p_T,p_0)\mid\mcD \right].
\end{aligned}
\]
Taking expectation over $\mcD$, taking the supremum over $p_0\in\mcP_\beta(C_l,C_u,R)$, and applying Theorem~\ref{thm: convergence_rate}, we obtain
\[
\begin{aligned}
    \sup_{p_0\in\mcP_\beta(C_l,C_u,R)} \Ebb_{p_0,p_{\mX}}^{(n)}\left[ \rho(\widehat p_{\mcD}^{\,\mathrm{av}},p_0) \right]
    &\le \sup_{p_0\in\mcP_\beta(C_l,C_u,R)} \Ebb_{\mcD,\mcT,\mcZ}\left[ \rho(\wt p_T,p_0) \right]\\
    &\lesssim n^{-\frac{\beta}{d_{\mcX}+d_{\mcY}+2\beta}}\log^{19/2}n.
\end{aligned}
\]
The averaging rule is the same for every density in the fixed class. Since
$\widehat p_{\mcD}^{\,\mathrm{av}}$ is admissible in the infimum defining $\mathfrak R_n(\mcP_\beta,p_{\mX})$, combining this uniform upper bound with the lower bound just proved gives
\[
    c\,n^{-\frac{\beta}{d_{\mcX}+d_{\mcY}+2\beta}}
    \leq \mathfrak R_n(\mcP_\beta,p_{\mX})
    \leq C\,n^{-\frac{\beta}{d_{\mcX}+d_{\mcY}+2\beta}}\log^{19/2}n.
\]
Here $c,C>0$ are independent of $n$ and $p_0$; the lower constant can depend on $c_{\mX}$ and the strict gaps in the class constraints. This is statistical minimax optimality up to logarithmic factors.

\section{Proofs of Bootstrap Convergence}\label{sec:prbc}

\begin{proof}[Proof of Theorem \ref{thm: bootstrap_consistency}]
Let $(\bar{\mcD},\bar{\mcT},\bar{\mcZ})$ be an independent copy of $(\mcD,\mcT,\mcZ)$. Denote the corresponding learned function as $\bar{f}$, an independent counterpart of $\wh{f}$, and $\bar{R}(\mx):=\bar{f}(\mx)-f_0(\mx)$. Then, we have
\begin{equation*}
    \Law(\bar{R}(\mx)\mid\mx)=\Law(\wh{R}(\mx)\mid\mx).
\end{equation*}
Here, $\bar{R}(\mx)$ is a proof device representing the true repeated-sampling law of the original residual. Thus the target of our proof can be represented as
\begin{equation*}
\begin{aligned}
    \mcW_1\left(\Law(\wh{R}(\mx)\mid\mx),\Law^*(\wh{R}^*(\mx)\mid\mathscr F_{n,\mx})\right)
    &= \mcW_1\left(\Law(\bar{R}(\mx)\mid\mx),\Law^*(\wh{R}^*(\mx)\mid\mathscr F_{n,\mx})\right)\\
    &\le \Ebb^{*,-}\left[|\wh{R}^*(\mx)-\bar{R}(\mx)|\mid\mathscr F_{n,\mx}\right],
\end{aligned}
\end{equation*}
where $\Ebb^{-}$ averages over the entire independent copy, including its training and final sample-generation randomness. The expectation $\Ebb^{*,-}$ averages jointly over this independent-copy randomness and the bootstrap randomness, conditionally on $\mathscr F_{n,\mx}$. Decompose the residual difference into two parts, we get
\begin{align*}
    \Ebb^{*,-}\left[|\wh{R}^*(\mx)-\bar{R}(\mx)|\mid\mathscr F_{n,\mx}\right]
    &\le \Ebb^{*,-}\left[|\wh{R}^*(\mx)| + |\bar{R}(\mx)|\mid\mathscr F_{n,\mx}\right]\\
    &=\Ebb^*\left[|\wh{R}^*(\mx)|\mid\mathscr F_{n,\mx}\right] + \Ebb^{-}\left[|\bar{R}(\mx)|\mid\mx\right],
\end{align*}
where the second equality holds because, conditional on $\mx$, $\bar R(\mx)$ is independent of $\mathscr F_{n,\mx}$. We then bound the two terms separately.

Before we start to bound the two terms, we introduce a preliminary lemma below.
\begin{lemma}[Sample averages are Wasserstein Lipschitz]\label{lem:thm 3.7 mean-lip}
Let \(P\) and \(Q\) be probability laws on \(\Rbb\) with finite second moments. Let \(S_n(P)\) and \(S_n(Q)\) be sample averages based on \(n\) i.i.d. draws from \(P\) and \(Q\), respectively. Then
\[
  \mcW_1\left(\Law\{S_n(P)\},\Law\{S_n(Q)\}\right)
  \le \mcW_1(P,Q).
\]
Moreover, if \(m_Q:=\int z\,dQ(z)\), then
\[
  \Ebb\left|S_n(P)-m_Q\right|
  \le \mcW_1(P,Q) + \Ebb\left|S_n(Q)-m_Q\right|.
\]
If $Z \sim Q$ has finite variance, then
\begin{equation*}
    \Ebb\left|S_n(Q)-m_Q\right| \le n^{-1/2}\{\mathrm{Var}_Q(Z)\}^{1/2}.
\end{equation*}
\end{lemma}
\begin{proof}
Let \(\varepsilon>0\).  Choose a coupling \(\pi_\varepsilon\in\Pi(P,Q)\) such that, for \((Y,Z)\sim\pi_\varepsilon\),
\[
  \Ebb|Y-Z|\le \mcW_1(P,Q)+\varepsilon.
\]
Let \((Y_j,Z_j)\), \(j=1,\ldots,n\), be i.i.d. copies of \((Y,Z)\).  Then
\(n^{-1}\sum_jY_j\) has law \(\Law\{S_n(P)\}\), and
\(n^{-1}\sum_jZ_j\) has law \(\Law\{S_n(Q)\}\).  Hence this construction is
an admissible coupling of the two sample-average laws.  Therefore,
\[
\begin{aligned}
\mcW_1\left(\Law\{S_n(P)\},\Law\{S_n(Q)\}\right) 
&\le
\Ebb\left|\frac1n\sum_{j=1}^n(Y_j-Z_j)\right| \\
&\quad\le
\frac1n\sum_{j=1}^n\Ebb|Y_j-Z_j| \\
&\quad\le \mcW_1(P,Q)+\varepsilon.
\end{aligned}
\]
Letting \(\varepsilon\downarrow0\) gives the first inequality.

For the centered bound, write $S_n(Q):=n^{-1}\sum_j Z_j$, the triangle inequality gives
\begin{equation*}
\begin{aligned}
    \Ebb\left|S_n(P)-m_Q\right|
    &\le \Ebb\left|S_n(P)-S_n(Q)\right| + \Ebb\left|S_n(Q)-m_Q\right|\\
    &\le \mcW_1(P,Q) + \varepsilon + \Ebb\left|S_n(Q)-m_Q\right|.
\end{aligned}
\end{equation*}
Again let $\varepsilon\downarrow 0$. If $Z \sim Q$ has finite variance, then Jensen's inequality yields
\begin{equation*}
    \Ebb|S_n(Q)-m_Q|
    \le \Ebb\left[(S_n(Q)-m_Q)^2\right]^{1/2}
    = n^{-1/2}\{\mathrm{Var}_Q(Z)\}^{1/2}.
\end{equation*}
\end{proof}

Now we are ready for the main proof. For the first part, let
$\wh{\mu}_T^B(\mx):=\int y\,d\wt{p}_T^B(y|\mx)$.
Let $\wt{Y}_{\mx}^{(1)},\ldots,\wt{Y}_{\mx}^{(n)}$ be i.i.d. draws from
$\wt{p}_T^B(\cdot|\mx)$ conditionally on $\mathscr F_{n,\mx}$, and define
\[
    \zeta_n(\mx)
    :=
    \Ebb\left[
    \left|
    \frac1n\sum_{j=1}^n \wt{Y}_{\mx}^{(j)}
    -\wh{\mu}_T^B(\mx)
    \right|
    \mid \mathscr F_{n,\mx}
    \right].
\]
By the triangle inequality and the centered part of Lemma
\ref{lem:thm 3.7 mean-lip}, conditionally on
$\mathscr F_{n,\mx},\mcD^*,\mcT^*,\mcZ^*$,
\begin{equation*}
\begin{aligned}
    \Ebb\left[|\wh{R}^*(\mx)|\mid\mathscr F_{n,\mx},\mcD^*,\mcT^*,\mcZ^*\right]
    &= \Ebb\left[|\wh{f}^*(\mx)-\wh{f}(\mx)|\mid\mathscr F_{n,\mx},\mcD^*,\mcT^*,\mcZ^*\right]\\
    &\le
    \Ebb\left[|\wh{f}^*(\mx)-\wh{\mu}_T^B(\mx)|
    \mid\mathscr F_{n,\mx},\mcD^*,\mcT^*,\mcZ^*\right]
    +|\wh{f}(\mx)-\wh{\mu}_T^B(\mx)|\\
    &\le
    \mcW_1\left(\wt{p}_T^{B,*}(\cdot|\mx),\wt{p}_T^B(\cdot|\mx)\right)
    +\zeta_n(\mx)
    +|\wh{f}(\mx)-\wh{\mu}_T^B(\mx)|.
\end{aligned}
\end{equation*}
Here $\wt{p}_T^{B,*}(\cdot\mid\mx)$ denotes the postprocessed terminal law of the refitted model. Thus we obtain the first part bound
\begin{equation}
\begin{aligned}\label{eq:thm3.7 first bound}
    \Ebb^*\left[|\wh{R}^*(\mx)|\mid\mathscr F_{n,\mx}\right]
    &\le
    \Ebb^*\left[
    \mcW_1\left(\wt{p}_T^B(\cdot|\mx),\wt{p}_T^{B,*}(\cdot|\mx)\right)
    \mid\mathscr F_{n,\mx}
    \right]\\
    &\quad
    +\zeta_n(\mx)
    +|\wh{f}(\mx)-\wh{\mu}_T^B(\mx)|.
\end{aligned}
\end{equation}
For the second part, use the same method and we derive
\begin{equation*}
    \Ebb\left[|\bar{R}(\mx)| \mid \mx,\bar{\mcD},\bar{\mcT},\bar{\mcZ}\right]
    \le \mcW_1\left(\wt{\bar{p}}_T^B(\cdot|\mx), p_0(\cdot|\mx)\right)+\eta_n(\mx),
\end{equation*}
where we denote $\eta_n(\mx):=\Ebb\left[\left|\frac1n\sum_{j=1}^nY_{\mx}^{(j)}-f_0(\mx)\right| \mid \mx\right]$.
Thus we get
\begin{equation}\label{eq:thm3.7 second bound}
\begin{aligned}
    \Ebb^{-}\left[ |\bar{R}(\mx)| \mid \mx \right]
    &\le \Ebb_{\bar{\mcD},\bar{\mcT},\bar{\mcZ}}\left[
    \mcW_1\left(\wt{\bar{p}}_T^B(\cdot|\mx), p_0(\cdot|\mx)\right)
    \mid \mx\right]+\eta_n(\mx)\\
    &= \Ebb_{{\mcD},{\mcT},{\mcZ}}\left[
    \mcW_1\left(\wt{p}_T^B(\cdot|\mx), p_0(\cdot|\mx)\right)
    \mid \mx\right]+\eta_n(\mx),
\end{aligned}
\end{equation}
where the last equality is because $(\bar{\mcD},\bar{\mcT},\bar{\mcZ})$ is an independent copy of $({\mcD},{\mcT},{\mcZ})$.\\
\indent Combining equations \eqref{eq:thm3.7 first bound} and
\eqref{eq:thm3.7 second bound} gives
\begin{align*}
&\Ebb_{\mcD, \mcT,\mcZ} \Ebb_{\mx}\left[
\mcW_1\left(\Law\left(\wh{R}(\mx)\mid\mx\right),
\Law^*(\wh{R}^*(\mx)\mid\mathscr F_{n,\mx})\right)
\right]\\
&\leq
\Ebb_{\mcD,\mcT,\mcZ} \Ebb_{\mx}
\left[\Ebb^*\left[
\mcW_1(\wt{p}_T^B(\cdot|\mx), \wt{p}_T^{B,*}(\cdot|\mx))
\mid\mathscr F_{n,\mx}\right]\right]+
\Ebb_{\mcD,\mcT,\mcZ}\Ebb_{\mx}
\left[\mcW_1(\wt{p}_T^B(\cdot|\mx), p_0(\cdot|\mx))\right]\\
&\quad +
\Ebb_{\mcD,\mcT,\mcZ}\Ebb_{\mx}[\zeta_n(\mx)]
+
\Ebb_{\mcD,\mcT,\mcZ}\Ebb_{\mx}
\left[|\wh{f}(\mx)-\wh{\mu}_T^B(\mx)|\right]
+\Ebb_\mx[\eta_n(\mx)].
\end{align*}
The distribution errors follow from Theorem~\ref{thm: convergence_rate} and its bootstrap analogue under the additional target-density assumption. The common randomized postprocessing kernel contracts TV. Since the resulting laws are supported on $[-B,B]^{d_\mcY}$, the comparison $\mcW_1(P,Q)\leq2B\sqrt{d_\mcY}\,\mathrm{TV}(P,Q)$ gives the required Wasserstein bounds. We therefore obtain
\begin{align*}
&\Ebb_{\mcD, \mcT,\mcZ} \Ebb_{\mx}\left[
\mcW_1\left(\Law(\wh{R}(\mx)\mid\mx),
\Law^*(\wh{R}^*(\mx)\mid\mathscr F_{n,\mx})\right)
\right]\\
&\lesssim
n^{-\frac{\beta}{d_\mcX + d_\mcY + 2\beta}} \log^{\frac{19}{2}}n
+
\Ebb_{\mcD,\mcT,\mcZ}\Ebb_{\mx}[\zeta_n(\mx)]
+
\Ebb_{\mcD,\mcT,\mcZ}\Ebb_{\mx}
\left[|\wh{f}(\mx)-\wh{\mu}_T^B(\mx)|\right]
+\Ebb_\mx[\eta_n(\mx)].
\end{align*}

Conditional on the original fitted model, the postprocessed responses defining $\wh f$ are i.i.d. and supported on $[-B,B]$, so their conditional variance is at most $B^2$. The same bound applies to the fresh samples defining $\zeta_n(\mx)$, while the variance under $p_0$ is at most $1$.
The sample-average bound in Lemma~\ref{lem:thm 3.7 mean-lip} therefore gives

\[
    \Ebb_{\mcD,\mcT,\mcZ}\Ebb_{\mx}[\zeta_n(\mx)]\lesssim n^{-1/2},
    \quad
    \Ebb_{\mcD,\mcT,\mcZ}\Ebb_{\mx}
    \left[|\wh{f}(\mx)-\wh{\mu}_T^B(\mx)|\right]\lesssim n^{-1/2},
    \quad
    \Ebb_\mx[\eta_n(\mx)]\lesssim n^{-1/2}.
\]
Moreover,
\[
    n^{-1/2}\lesssim
    n^{-\frac{\beta}{d_\mcX+d_\mcY+2\beta}}\log^{\frac{19}{2}}n
    =\delta_n,
\]
because $\beta/(d_\mcX+d_\mcY+2\beta)<1/2$. Therefore, the additional sample-average fluctuation terms are absorbed by $\delta_n$, proving the stated averaged $\mcW_1$ bound. Since the test-function class defining $d_{\mathrm{BL}}$ is contained in the $1$-Lipschitz class defining $\mcW_1$, we have
\begin{equation*}
    d_{\mathrm{BL}}(\Law(\wh R(\mx)\mid\mx),\Law^*(\wh R^*(\mx)\mid\mathscr F_{n,\mx}))
    \leq \mcW_1(\Law(\wh R(\mx)\mid\mx),\Law^*(\wh R^*(\mx)\mid\mathscr F_{n,\mx})),
\end{equation*}
for every $\mx$ and every realization of the data. Taking expectations yields
\[
\begin{aligned}
&\Ebb_{\mcD,\mcT,\mcZ}\Ebb_{\mx}\bigl[d_{\mathrm{BL}}(\Law(\wh R(\mx)\mid\mx),\Law^*(\wh R^*(\mx)\mid\mathscr F_{n,\mx}))\bigr]\\
&~~~~~~\leq\Ebb_{\mcD,\mcT,\mcZ}\Ebb_{\mx}\bigl[\mcW_1(\Law(\wh R(\mx)\mid\mx),\Law^*(\wh R^*(\mx)\mid\mathscr F_{n,\mx}))\bigr]
\lesssim n^{-\frac{\beta}{d_\mcX+d_\mcY+2\beta}}\log^{\frac{19}{2}}n.
\end{aligned}
\]
The proof is complete.
\end{proof}

{
We next give the proof of Theorem \ref{thm: bootstrap_confidence_interval}. The proof uses bounded-Lipschitz convergence of the centered statistics after normalization by their stochastic scale.

\begin{lemma}[From bounded-Lipschitz closeness to the same weak limit]\label{lem: scaled_wasserstein_weak_limit}
Let $G_n$ be deterministic distribution functions such that $G_n\Rightarrow F$.  Let $K_n$ be possibly random distribution functions. If
\[
        d_\mathrm{BL}(K_n,G_n)\xrightarrow{\Pbb}0,
\]
then
\[
        d_{\mathrm{BL}}(K_n,F)\xrightarrow{\Pbb}0,
\]
where
\[
        d_{\mathrm{BL}}(K,F)
        :=
        \sup_{\|h\|_\infty\leq1,\ \mathrm{Lip}(h)\leq1}
        \left|\int h\,dK-\int h\,dF\right|.
\]
In particular, $K_n\Rightarrow F$ in probability in the bounded-Lipschitz metric.
\end{lemma}

\begin{proof}
Since $G_n\Rightarrow F$ and $d_{\mathrm{BL}}$ metrizes weak convergence on $\Rbb$, we have
\[
        d_{\mathrm{BL}}(G_n,F)\to0.
\]
By the triangle inequality,
\[
        d_{\mathrm{BL}}(K_n,F)
        \leq
        d_{\mathrm{BL}}(K_n,G_n)+d_{\mathrm{BL}}(G_n,F)
        \xrightarrow{\Pbb}0.
\]
\end{proof}

\begin{lemma}[Random quantile-continuity lemma]\label{lem: random_quantile_continuity}
Let $K_n$ be possibly random distribution functions and let $F$ be a distribution function.  If
\[
        d_{\mathrm{BL}}(K_n,F)\xrightarrow{\Pbb}0
\]
and $F^{-1}$ is continuous at $u\in(0,1)$, then
\[
        K_n^{-1}(u)\xrightarrow{\Pbb}F^{-1}(u),
\]
where $K^{-1}(u)=\inf\{t:K(t)\geq u\}$.
\end{lemma}

\begin{proof}
Let $q=F^{-1}(u)$.  It is enough to show that every subsequence has a further subsequence along which $K_n^{-1}(u)\to q$ almost surely.  Take an arbitrary subsequence.  Since $d_{\mathrm{BL}}(K_n,F)\to0$ in probability along this subsequence, there is a further subsequence, still denoted by $n$ for simplicity, such that
\[
        d_{\mathrm{BL}}(K_n,F)\to0
        \quad\text{almost surely}.
\]
On this almost sure event, $K_n\Rightarrow F$ as deterministic distribution functions. The deterministic quantile-continuity lemma for weak convergence, see Lemma 21.2 of  \cite{van2000asymptotic}, then gives
\[
        K_n^{-1}(u)\to F^{-1}(u)=q,
\]
because $F^{-1}$ is continuous at $u$.  Hence, every subsequence has a further subsequence converging almost surely to $q$, which is equivalent to convergence in probability.
\end{proof}

\begin{lemma}[Empirical bootstrap approximation]\label{lem: empirical_bootstrap_cdf}
Assume $\mcB\to\infty$. Conditional on $\mathscr F_{n,\mx}$, let $\wh R_1^*(\mx),\ldots,\wh R_\mcB^*(\mx)$ be independent draws from the conditional law $\Law^*(\wh R^*(\mx)\mid\mathscr F_{n,\mx})$.  Define
\[
        G^*(t)=\Pbb^*\{a_n^{-1}\wh R^*(\mx)\leq t\mid\mathscr F_{n,\mx}\},
        \qquad
        \widehat G^*(t)=\frac{1}{\mcB}\sum_{b=1}^{\mcB}
        \mI\{a_n^{-1}\wh R_b^*(\mx)\leq t\}.
\]
Then
\[
        \sup_{t\in\Rbb}|\widehat G^*(t)-G^*(t)|\xrightarrow{\Pbb}0.
\]
Consequently, if $d_{\mathrm{BL}}(G^*,F)\xrightarrow{\Pbb}0$, then $d_{\mathrm{BL}}(\widehat G^*,F)\xrightarrow{\Pbb}0$.
\end{lemma}

\begin{proof}
By the conditional Dvoretzky--Kiefer--Wolfowitz inequality, for every $\epsilon>0$,
\[
        \Pbb^*\left\{
        \sup_t|\widehat G^*(t)-G^*(t)|>\epsilon \mid\mathscr F_{n,\mx}
        \right\}
        \leq
        2\exp(-2\mcB\epsilon^2).
\]
Taking expectation over all randomness in the original estimation procedure gives
\[
        \Pbb\left\{
        \sup_t|\widehat G^*(t)-G^*(t)|>\epsilon
        \right\}
        \leq
        2\exp(-2\mcB\epsilon^2)\to0.
\]
It remains to prove the final claim.  Take an arbitrary subsequence.  Since
\[
        \sup_t|\widehat G^*(t)-G^*(t)|\xrightarrow{\Pbb}0
        \quad\text{and}\quad
        d_{\mathrm{BL}}(G^*,F)\xrightarrow{\Pbb}0,
\]
there is a further subsequence along which both convergences hold almost surely.  On this almost sure event, $G^*\Rightarrow F$ and
\[
        \sup_t|\widehat G^*(t)-G^*(t)|\to0.
\]
For every continuity point $t$ of $F$,
\[
        |\widehat G^*(t)-F(t)|
        \leq
        |\widehat G^*(t)-G^*(t)|+|G^*(t)-F(t)|\to0.
\]
Thus $\widehat G^*\Rightarrow F$ along this further subsequence, or equivalently $d_{\mathrm{BL}}(\widehat G^*,F)\to0$ almost surely.  Since every subsequence has such a further subsequence, $d_{\mathrm{BL}}(\widehat G^*,F)\xrightarrow{\Pbb}0$.
\end{proof}

\begin{proof}[Proof of Theorem \ref{thm: bootstrap_confidence_interval}]
Let
\[
        G(t)=\Pbb\{a_n^{-1}\wh R(\mx)\leq t\mid\mx\},
        \qquad
        G^*(t)=\Pbb^*\{a_n^{-1}\wh R^*(\mx)\leq t\mid\mathscr F_{n,\mx}\}.
\]
By Assumption \ref{ass: stochastic_scale}, $G\Rightarrow F_\mx$ and $d_{\mathrm{BL}}(G^*,G)\xrightarrow{\Pbb}0$.
Lemma \ref{lem: scaled_wasserstein_weak_limit} therefore implies
\[
        d_{\mathrm{BL}}(G^*,F_\mx)\xrightarrow{\Pbb}0.
\]
By Lemma \ref{lem: empirical_bootstrap_cdf},
\[
        d_{\mathrm{BL}}(\widehat G^*,F_\mx)\xrightarrow{\Pbb}0.
\]
Assumption \ref{ass: quantile_regular} and Lemma \ref{lem: random_quantile_continuity} yield
\[
        (\widehat G^*)^{-1}(u)\xrightarrow{\Pbb}q_\mx(u),
        \qquad
        u\in\{\alpha/2,1-\alpha/2\}.
\]
Since
\[
        \widehat G^*(t)=\widehat H^*(a_n t),
        \qquad
        (\widehat H^*)^{-1}(u)=a_n(\widehat G^*)^{-1}(u),
\]
the event $f_0(\mx)\in[\wh f(\mx)-(\wh H^*)^{-1}(1-\alpha/2),\wh f(\mx)-(\wh H^*)^{-1}(\alpha/2)]$ is equivalent to
\[
        (\widehat G^*)^{-1}(\alpha/2)
        \leq
        a_n^{-1}\{\wh f(\mx)-f_0(\mx)\}
        \leq
        (\widehat G^*)^{-1}(1-\alpha/2).
\]
By Assumption \ref{ass: stochastic_scale},
\[
        a_n^{-1}\{\wh f(\mx)-f_0(\mx)\}=a_n^{-1}\wh R(\mx)\xrightarrow{d}Z_\mx.
\]
Combining this convergence with the two bootstrap quantile convergences gives, by Slutsky's theorem,
\[
\left(
(\widehat G^*)^{-1}(\alpha/2),
 a_n^{-1}\wh R(\mx),
(\widehat G^*)^{-1}(1-\alpha/2)
\right)
\xrightarrow{d}
\left(q_\mx(\alpha/2),Z_\mx,q_\mx(1-\alpha/2)\right).
\]
Because $F_\mx$ is continuous at $q_\mx(\alpha/2)$ and $q_\mx(1-\alpha/2)$, the limiting boundary has probability zero.  The portmanteau theorem gives
\[
\begin{aligned}
&\Pbb\left(
 f_0(\mx)\in\left[\wh f(\mx)-(\wh H^*)^{-1}(1-\alpha/2),
 \wh f(\mx)-(\wh H^*)^{-1}(\alpha/2)\right]
\right)\\
&\qquad\longrightarrow
\Pbb\left(q_\mx(\alpha/2)\leq Z_\mx\leq q_\mx(1-\alpha/2)\right).
\end{aligned}
\]
Finally, the continuity of $F_\mx^{-1}$ at $\alpha/2$ and $1-\alpha/2$, together with the continuity of $F_\mx$ at the two corresponding quantiles, implies
\[
        F_\mx(q_\mx(\alpha/2))=\alpha/2,
        \qquad
        F_\mx(q_\mx(1-\alpha/2))=1-\alpha/2.
\]
Thus the last probability equals $1-\alpha$, which proves the theorem.
\end{proof}
}

\section{Auxiliary Lemmas}\label{sec:al}
\subsection{Several High-Probability Bounds}\label{sec: shb}
Following \cite{oko2023diffusion},
we provide several high-probability bounds in this section.  In the following, 
we denote $m_t:=1-t$ and $\sigma_t := \sqrt{t(2-t)}$.

\subsubsection{Bounds on $p_t(\my|\mx)$}

In this section, we give the upper and lower bounds 
on $p_t(\my|\mx)$.

\begin{lemma}\label{lem: bound_for_density}
Suppose Assumption~\ref{ass: bounded_density} holds. The following upper and lower bounds hold uniformly over $0<t<1$, $\mx\in[-1,1]^{d_\mcX}$ and $\my\in\Rbb^{d_\mcY}$, with implicit constants depending only on $d_\mcY,C_l,C_u$:
\begin{equation}\label{eq: upper_lower_bound_p}
\exp\left(-\frac{d_\mcY(\Vert \my \Vert_{\infty} - m_t)_{+}^{2}}{\sigma_t^2}\right) \lesssim p_t(\my|\mx) \lesssim \exp\left(-\frac{(\Vert \my \Vert_{\infty} - m_t)_{+}^{2}}{2\sigma_t^2}\right).
\end{equation}
\end{lemma}

\begin{proof}
This proof can be divided into two cases: $\my\in [-m_t, m_t]^{d_\mcY}$ and  $\my\notin [-m_t,m_t]^{d_\mcY}$.\\
\textbf{
Case I ($\my\in [-m_t, m_t]^{d_\mcY}$): 
}
Given that $p_0(\my_1|\mx) \leq C_u$, 
we have
$$
\begin{aligned}
p_t(\my|\mx) &= \int_{\Rbb^{d_\mcY}}\frac{1}{\sigma_t^{d_\mcY}(2\pi)^{d_\mcY/2}} p_0(\my_1|\mx)\exp\left(-\frac{\Vert \my - m_t\my_1\Vert^2}{2\sigma_t^2}\right)\mrd\my_1\\
&\leq C_u \int_{\Rbb^{d_\mcY}}\frac{\mathbf{I}_{\{\my_1\in [-1,1]^{d_\mcY}\}}}{\sigma_t^{d_\mcY}(2\pi)^{d_\mcY/2}} \exp\left(-\frac{\Vert \my - m_t\my_1\Vert^2}{2\sigma_t^2}\right)\mrd\my_1\\
&\leq \frac{C_u 2^{d_\mcY}}{\sigma_t^{d_\mcY}(2\pi)^{d_\mcY/2}}.
\end{aligned}
$$
Additionally, we also have
$$
p_t(\my|\mx) \leq C_u \int_{\Rbb^{d_\mcY}}\frac{1}{\sigma_t^{d_\mcY}(2\pi)^{d_\mcY/2}} \exp\left(-\frac{\Vert \my - m_t\my_1\Vert^2}{2\sigma_t^2}\right)\mrd\my_1 \leq \frac{C_u}{m_t^{d_\mcY}}.
$$
Thus, $p_t(\my|\mx)$ is bounded by $\min\left\{\frac{C_u2^{d_\mcY}}{\sigma_t^{d_\mcY}(2\pi)^{d_\mcY/2}}, \frac{C_u}{m_t^{d_\mcY}}\right\} = \mcO(1)$. 

The lower bound can be derived as follows:
$$
p_t(\my|\mx) \geq C_l\int_{\Rbb^{d_\mcY}}\frac{\mathbf{I}_{\{\my_1\in[-1,1]^{d_\mcY}\}}}{\sigma_t^{d_\mcY}(2\pi)^{d_\mcY/2}}\exp\left(-\frac{\Vert\my-m_t\my_1\Vert^2}{2\sigma_t^2}\right)\mrd\my_1.
$$
Let $\mz = \frac{\my - m_t\my_1}{\sigma_t}$, then we have
$$
\begin{aligned}
p_t(\my|\mx) &\geq C_l\cdot\frac{1}{m_t^{d_\mcY}}\int_{\Rbb^{d_\mcY}}\frac{\mathbf{I}_{\left\{\mz\in\left[\frac{\my - m_t}{\sigma_t},\frac{\my + m_t}{\sigma_t}\right]\right\}}}{(2\pi)^{d_\mcY/2}}\exp\left(-\frac{\Vert\mz\Vert^2}{2}\right)\mrd\mz \\
&\geq C_l\cdot\frac{1}{m_t^{d_\mcY}}\int_{\Rbb^{d_\mcY}}\frac{\mathbf{I}_{\left\{\mz\in [\my - m_t,\my + m_t]\right\}}}{(2\pi)^{d_\mcY/2}}\exp\left(-\frac{\Vert\mz\Vert^2}{2}\right)\mrd\mz \\
& \geq C_l\cdot\frac{1}{m_t^{d_\mcY}} \frac{(2m_t)^{d_\mcY}}{(2\pi)^{d_\mcY/2}}\exp\left(-2d_\mcY\right) \\
& \gtrsim 1.
\end{aligned}
$$

\noindent
\textbf{
Case II ($\my\notin [-m_t,m_t]^{d_\mcY}$): 
}
Let $r = \frac{\Vert \my \Vert_{\infty} - m_t}{\sigma_t}$ and $|y_{i^*}| = \Vert \my \Vert_{\infty}$. For $i\neq i^*$, if $y_i > m_t$, then
$$
\begin{aligned}
\int_{\Rbb}\frac{\mathbf{I}_{\{y_{1,i}\in[-1,1]\}}}{\sigma_t(2\pi)^{1/2}}\exp\left(-\frac{(y_i - m_t y_{1,i})^2}{2\sigma_t^2}\right)\mrd y_{1,i}
&\leq \int_{-1}^{1}\frac{1}{\sigma_t(2\pi)^{1/2}}\exp\left(-\frac{(m_t-m_t y_{1,i})^2}{2\sigma_t^2}\right)\mrd y_{1,i} \\
&\leq \frac{1}{m_t} \cdot \int_{0}^{\infty}\frac{1}{\sigma_t(2\pi)^{1/2}}\exp\left(-\frac{y_{1,i}^2}{2\sigma_t^2}\right)\mrd y_{1,i} \leq \frac{1}{2m_t}. 
\end{aligned}
$$
If $y_i < -m_t$, then
$$
\begin{aligned}
\int_{\Rbb}\frac{\mathbf{I}_{\{y_{1,i}\in[-1,1]\}}}{\sigma_t(2\pi)^{1/2}}\exp\left(-\frac{(y_i - m_t y_{1,i})^2}{2\sigma_t^2}\right)\mrd y_{1,i}
&\leq \int_{-1}^{1}\frac{1}{\sigma_t(2\pi)^{1/2}}\exp\left(-\frac{(-m_t-m_t y_{1,i})^2}{2\sigma_t^2}\right)\mrd y_{1,i} \\
&\leq \frac{1}{m_t} \cdot \int_{0}^{\infty}\frac{1}{\sigma_t(2\pi)^{1/2}}\exp\left(-\frac{y_{1,i}^2}{2\sigma_t^2}\right)\mrd y_{1,i} \leq \frac{1}{2m_t}.
\end{aligned}
$$
Also, we have
$$
\int_{\Rbb}\frac{\mathbf{I}_{\{y_{1,i}\in[-1,1]\}}}{\sigma_t(2\pi)^{1/2}}\exp\left(-\frac{(y_i - m_t y_{1,i})^2}{2\sigma_t^2}\right)\mrd y_{1,i} \leq \frac{2}{\sigma_t(2\pi)^{1/2}}.
$$
Therefore, we have
$$
p_t(\my|\mx) \lesssim \frac{1}{m_t^{d_\mcY-1}}\int_{-1}^{1}\frac{1}{\sigma_t(2\pi)^{1/2}}\exp\left(-\frac{(y_{i^*} - m_ty_{1,i^*})^2}{2\sigma_t^2}\right)\mrd y_{1,i^*},
$$
and
$$
\begin{aligned}
p_t(\my|\mx) &\lesssim \frac{1}{\sigma_t^{d_\mcY-1}}\int_{-1}^{1}\frac{1}{\sigma_t(2\pi)^{1/2}}\exp\left(-\frac{(y_{i^*} - m_ty_{1,i^*})^2}{2\sigma_t^2}\right)\mrd y_{1,i^*} \\
& \lesssim \frac{1}{\sigma_t^{d_\mcY}} \exp\left(-\frac{(\|\my\|_\infty - m_t)^2}{2\sigma_t^2}\right).
\end{aligned}
$$
Let $z_{i^*} = \frac{y_{i^*} - m_t y_{1,i^*}}{\sqrt{2}\sigma_t}$, then
$$
p_t(\my|\mx) \lesssim \frac{1}{m_t^{d_\mcY}}\cdot \int_{\frac{y_{i^*} - m_t}{\sqrt{2}\sigma_t}}^{\frac{y_{i^*} + m_t}{\sqrt{2}\sigma_t}}\frac{1}{\sqrt{\pi}}\exp(-z_{i^*}^2)\mathrm{d}z_{i^*}.
$$
If $y_{i^*} > m_t$, then
$$
p_t(\my|\mx) \lesssim \frac{1}{m_t^{d_\mcY}} \cdot \int_{\frac{\Vert\my\Vert_{\infty} - m_t}{\sqrt{2}\sigma_t}}^{\infty}\frac{1}{\sqrt{\pi}}\exp(-z_{i^*}^2)\mrd z_{i^*}\lesssim \frac{1}{m_t^{d_\mcY}}\exp\left(-\frac{(\Vert \my \Vert_{\infty} - m_t)^2}{2\sigma_t^2}\right).
$$
If $y_i^* < -m_t$, $y_i^* = -\Vert \my\Vert_{\infty}$, then
$$
\begin{aligned}
p_t(\my|\mx) &\lesssim \frac{1}{m_t^{d_\mcY}} \cdot \int_{-\frac{y_{i^*} + m_t}{\sqrt{2}\sigma_t}}^{-\frac{y_{i^*} - m_t}{\sqrt{2}\sigma_t}}\frac{1}{\sqrt{\pi}}\exp(-z_{i^*}^2)\mrd z_{i^*}\\
&\lesssim \frac{1}{m_t^{d_\mcY}}\cdot\int_{\frac{\Vert\my\Vert_{\infty} - m_t}{\sqrt{2}\sigma_t}}^{\infty}\frac{1}{\sqrt{\pi}}\exp(-z_{i^*}^2)\mrd z_{i^*}\\
&\lesssim \frac{1}{m_t^{d_\mcY}}\exp\left(-\frac{(\Vert \my \Vert_{\infty} - m_t)^2}{2\sigma_t^2}\right).
\end{aligned}
$$
Therefore, $p_t(\my|\mx)$ is bounded by
$$
\begin{aligned}
p_t(\my|\mx) & \lesssim \min\left\{\frac{1}{m_t^{d_\mcY}}, \frac{1}{\sigma_t^{d_\mcY}}\right\}\exp\left(-\frac{(\Vert \my \Vert_{\infty} - m_t)^2}{2\sigma_t^2}\right) \\
& \lesssim \exp\left(-\frac{(\Vert \my \Vert_{\infty} - m_t)^2}{2\sigma_t^2}\right).
\end{aligned}
$$

On the other hand, for $1\leq i\leq d_\mcY$, let $z_i = \frac{y_i-m_t y_{1,i}}{\sqrt{2}\sigma_t}$, then we define
$$
\begin{aligned}
g(y_i) &:= \int_{-1}^{1}\frac{1}{\sigma_t(2\pi)^{1/2}}\exp\left(-\frac{(y_i - m_t y_{1,i})^2}{2\sigma_t^2}\right)\mrd y_{1,i} \\
&= \frac{1}{m_t}\int_{\frac{y_{i} - m_t}{\sqrt{2}\sigma_t}}^{\frac{y_{i} + m_t}{\sqrt{2}\sigma_t}}\frac{1}{\sqrt{\pi}}\exp(-z_{i}^2)\mathrm{d}z_{i}.
\end{aligned}
$$
If $y_i > m_t$, then $g(y_i)$ is a decreasing function. Therefore,
\begin{align*}
g(y_i)\geq g(\Vert\my\Vert_{\infty}) &= \frac{1}{m_t}\int_{\frac{\Vert\my\Vert_{\infty} - m_t}{\sqrt{2}\sigma_t}}^{\frac{\Vert\my\Vert_{\infty} + m_t}{\sqrt{2}\sigma_t}}\frac{1}{\sqrt{\pi}}\exp(-z_{i}^2)\mrd z_{i} \\
&\geq \frac{1}{m_t} \int_{\frac{r}{\sqrt{2}}}^{\frac{r}{\sqrt{2}} + \sqrt{2} m_t}\frac{1}{\sqrt{\pi}}\exp(-z_i^2)\mrd z_i \\
&\gtrsim  \exp(-r^2 - 4m_t^2)\\
&\gtrsim  \exp\left(-\frac{(\Vert\my\Vert_{\infty} - m_t)^2}{\sigma_t^2}\right),
\end{align*}
where we used equality $(a + b)^2 \leq 2(a^2 + b^2)$ and $\exp(-4m_t^2) \geq \exp(-4)$.
If $y_i<-m_t$, then
$$
g(y_i) = \frac{1}{m_t}\int_{\frac{-y_{i} - m_t}{\sqrt{2}\sigma_t}}^{\frac{-y_{i} + m_t}{\sqrt{2}\sigma_t}}\frac{1}{\sqrt{\pi}}\exp(-z_{i}^2)\mathrm{d}z_{i}.
$$
It is easy to check that $g(y_i)$ is a increasing function. Therefore,
$$
\begin{aligned}
g(y_i)\geq g(-\Vert\my\Vert_{\infty}) &= \frac{1}{m_t}\int_{\frac{\Vert\my\Vert_{\infty} - m_t}{\sqrt{2}\sigma_t}}^{\frac{\Vert\my\Vert_{\infty} + m_t}{\sqrt{2}\sigma_t}}\frac{1}{\sqrt{\pi}}\exp(-z_{i}^2)\mrd z_{i} \\
& \gtrsim  \exp\left(-\frac{(\Vert\my\Vert_{\infty} - m_t)^2}{\sigma_t^2}\right).
\end{aligned}
$$
For the remaining coordinates with $|y_i|\leq m_t$, the standardized integration interval $[(y_i-m_t)/\sigma_t,(y_i+m_t)/\sigma_t]$ contains $[y_i-m_t,y_i+m_t]$ because $\sigma_t\leq1$. Hence
\begin{equation*}
\begin{aligned}
    g(y_i)
    &=\frac1{m_t}
    \int_{(y_i-m_t)/\sigma_t}^{(y_i+m_t)/\sigma_t}
    \frac{e^{-z^2/2}}{\sqrt{2\pi}}\,dz\\
    &\ge\frac1{m_t}
    \int_{y_i-m_t}^{y_i+m_t}
    \frac{e^{-z^2/2}}{\sqrt{2\pi}}\,dz\\
    &\ge\frac1{m_t}(2m_t)\frac{e^{-2}}{\sqrt{2\pi}}
    =\frac{2e^{-2}}{\sqrt{2\pi}}>0.
\end{aligned}
\end{equation*}
And therefore
\[
g(y_i)\gtrsim1\geq \exp\left(-\frac{(\|\my\|_\infty-m_t)^2}{\sigma_t^2}\right).
\]
The above discussion implies that
$$
p_t(\my|\mx) \gtrsim \exp\left(-\frac{d_\mcY(\Vert\my\Vert_{\infty} - m_t)^2}{\sigma_t^2}\right).
$$
Based on the above discussion of the two cases,
we finally obtain
$$
\exp\left(-\frac{d_\mcY(\Vert \my \Vert_{\infty} - m_t)_{+}^{2}}{\sigma_t^2}\right) \lesssim p_t(\my|\mx) \lesssim \exp\left(-\frac{(\Vert \my \Vert_{\infty} - m_t)_{+}^{2}}{2\sigma_t^2}\right).
$$
The proof is complete.
\end{proof}

\subsubsection{Bounds on the derivatives of $p_t(\my|\mx)$ and $\nabla\log p_t(\my|\mx)$}
In this section, we give the upper bounds on the derivatives of $p_t(\my|\mx)$ and $\nabla\log p_t(\my|\mx)$. We first state the following lemma.
\begin{lemma}[Integral Clipping]\label{lem: integral_clipping}
Suppose Assumption~\ref{ass: bounded_density} holds. Let $\boldsymbol{\alpha}\in\{0,1,2,\ldots\}^{d_\mcY}$. For any $0<\epsilon<e^{-1}$, there exists a constant $C>0$, depending only on $\boldsymbol{\alpha}$ and $d_\mcY$, such that the following bound holds uniformly over $0<t<1$, $\mx\in[-1,1]^{d_\mcX}$ and $\my\in\Rbb^{d_\mcY}$.
The implicit constant in the bound depends only on $\boldsymbol{\alpha},d_\mcY,C_u$:
$$
\begin{aligned}
&\Bigg|
\int_{\Rbb^{d_\mcY}}\prod_{i=1}^{d_\mcY}\left(\frac{y_i-m_ty_{1,i}}{\sigma_t}\right)^{\alpha_i}\frac{1}{\sigma_t^{d_\mcY}(2\pi)^{d_\mcY/2}}p_0(\my_1|\mx)\exp\left(-\frac{\Vert\my-m_t\my_1\Vert^2}{2\sigma_t^2}\right)\mrd\my_1 \\ & ~~~~~~~~~ -  \int_{A_{\my}}\prod_{i=1}^{d_\mcY}\left(\frac{y_i-m_t y_{1,i}}{\sigma_t}\right)^{\alpha_i}\frac{1}{\sigma_t^{d_\mcY}(2\pi)^{d_\mcY/2}}p_0(\my_1|\mx)\exp\left(-\frac{\Vert\my-m_t\my_1\Vert^2}{2\sigma_t^2}\right)\mrd\my_1
\Bigg| \lesssim \epsilon,
\end{aligned}
$$
where $A_{\my} = \prod_{i=1}^{d_\mcY}a_{i,\my}$ with $a_{i,\my} = \bigg[\frac{y_i - C\sigma_t\sqrt{\log\epsilon^{-1}}}{m_t}, \frac{y_i + C\sigma_t\sqrt{\log\epsilon^{-1}}}{m_t} \bigg].$ 
    
\end{lemma}

\begin{proof}
It follows that
$$
\begin{aligned}
&\Bigg|
\int_{\Rbb^{d_\mcY}\backslash A_{\my}} \prod_{i=1}^{d_\mcY}\left(\frac{y_i-m_ty_{1,i}}{\sigma_t}\right)^{\alpha_i}\frac{1}{\sigma_t^{d_\mcY}(2\pi)^{d_\mcY/2}}p_0(\my_1|\mx)\exp\left(-\frac{\Vert\my-m_t\my_1\Vert^2}{2\sigma_t^2}\right)\mrd\my_1
\Bigg| \\
\leq & \frac{C_u}{\sigma_t^{d_\mcY}(2\pi)^{d_\mcY/2}}\int_{\Rbb^{d_\mcY}\backslash A_{\my}}\prod_{i=1}^{d_\mcY} \left|\frac{y_i-m_t y_{i,1}}{\sigma_t}\right|^{\alpha_i} \mathbf{I}_{\{\Vert\my_1\Vert_\infty\leq 1\}}\exp\left(-\frac{\Vert\my-m_t\my_1\Vert^2}{2\sigma_t^2}\right)\mrd\my_1 \\
\leq & \frac{C_u}{\sigma_t^{d_\mcY}(2\pi)^{d_\mcY/2}} \cdot \sum_{j=1}^{d_\mcY}
\int_{\Rbb\times\cdots\times\Rbb\times (\Rbb\backslash a_{j,\my})\times\Rbb\times\cdots\times\Rbb} \prod_{i=1}^{d_\mcY} \left|\frac{y_i-m_t y_{1,i}}{\sigma_t}\right|^{\alpha_i} \mathbf{I}_{\{|y_{1,i}|\leq 1\}}\exp\left(-\frac{\Vert\my-m_t\my_1\Vert^2}{2\sigma_t^2}\right)\mrd\my_1 \\
\leq & C_u\cdot\sum_{j=1}^{d_\mcY} \Bigg[ \left(\prod_{i=1, i\neq j}^{d_\mcY} \frac{1}{\sigma_t\sqrt{2\pi}} \int_{\Rbb}\left|\frac{y_i-m_t y_{1,i}}{\sigma_t}\right|^{\alpha_i} \mathbf{I}_{\{|y_{1,i}|\leq 1\}} \exp\left(-\frac{(y_i-m_t y_{1,i})^2}{2\sigma_t^2}\right) \mrd y_{1,i}
\right) \\
& ~~~~~~~~~~~~~~~~~~~~~~~~~~ \cdot  \frac{1}{\sigma_t\sqrt{2\pi}} \int_{\Rbb\backslash a_{j,\my}}\left|\frac{y_j-m_t y_{1,j}}{\sigma_t}\right|^{\alpha_j} \mathbf{I}_{\{|y_{1,j}|\leq 1\}} \exp\left(-\frac{(y_j-m_t y_{1,j})^2}{2\sigma_t^2}\right) \mrd y_{1,j} \Bigg].
\end{aligned}
$$
For $i \neq j$, let $z_i = \frac{y_i - m_t y_{1,i}}{\sigma_t}$.
Then, it holds that
$$
\begin{aligned}
&~\frac{1}{\sigma_t\sqrt{2\pi}}\int_{\Rbb}\left|\frac{y_i-m_t y_{1,i}}{\sigma_t}\right|^{\alpha_i} \mathbf{I}_{\{|y_{1,i}|\leq 1\}} \exp\left(-\frac{(y_i-m_t y_{1,i})^2}{2\sigma_t^2}\right) \mrd y_{1,i} \\
\leq & ~\frac{1}{m_t} \cdot \frac{1}{\sqrt{2\pi}}\int_{\Rbb} |z_i|^{\alpha_i}\exp\left(-\frac{z_i^2}{2}\right)\mrd z_i \\
\lesssim & ~ \frac{1}{m_t}.
\end{aligned}
$$ 
Additionally, since for any $\alpha > 0$, $|z|^{\alpha}\exp(-z^2/2) = \mcO(1)$, it also have
$$
\frac{1}{\sigma_t\sqrt{2\pi}}\int_{\Rbb}\left|\frac{y_i-m_t y_{1,i}}{\sigma_t}\right|^{\alpha_i} \mathbf{I}_{\{|y_{1,i}|\leq 1\}} \exp\left(-\frac{(y_i-m_t y_{1,i})^2}{2\sigma_t^2}\right) \mrd y_{1,i} \lesssim \frac{1}{\sigma_t}.
$$
Thus, we obtain
$$
\frac{1}{\sigma_t\sqrt{2\pi}}\int_{\Rbb}\left|\frac{y_i-m_t y_{1,i}}{\sigma_t}\right|^{\alpha_i} \mathbf{I}_{\{|y_{1,i}|\leq 1\}} \exp\left(-\frac{(y_i-m_t y_{1,i})^2}{2\sigma_t^2}\right) \mrd y_{1,i} \lesssim \min\left\{\frac{1}{m_t}, \frac{1}{\sigma_t}\right\} = \mcO(1).
$$
For $i=j$, we have
$$
\begin{aligned}
&~\frac{1}{\sigma_t\sqrt{2\pi}}\int_{\Rbb\backslash a_{j,\my}}\left|\frac{y_j-m_t y_{1,j}}{\sigma_t}\right|^{\alpha_j} \mathbf{I}_{\{|y_{1,j}|\leq 1\}} \exp\left(-\frac{(y_j-m_t y_{1,j})^2}{2\sigma_t^2}\right) \mrd y_{1,j} \\ \lesssim &~\frac{1}{m_t} \cdot \int_{|z_j|\geq C\sqrt{\log\epsilon^{-1}}} |z_j|^{\alpha_j}\exp\left(-\frac{z_j^2}{2}\right)\mathrm{d}z_j \\
\lesssim &~ \frac{1}{m_t} \cdot 2\int_{z_j\geq C\sqrt{\log\epsilon^{-1}}} z_j^{\alpha_j}\exp\left(-\frac{z_j^2}{2}\right)\mathrm{d}z_j \\
\lesssim &~ \frac{1}{m_t} \cdot \epsilon^{\frac{C^2}{2}}(C^2\log\epsilon^{-1})^{\frac{\alpha_j}{2}}.
\end{aligned}
$$
Since for any $\alpha > 0$, $z^{\alpha}\exp(-z^2/2)$ is decreasing for $|z| > \sqrt{\alpha}$, we can choose $C \geq \sqrt{\alpha}$ and $\epsilon < e^{-1}$, then it holds that
$$
\begin{aligned}
&~\frac{1}{\sigma_t\sqrt{2\pi}}\int_{\Rbb\backslash a_{j,\my}}\left|\frac{y_j-m_t y_{1,j}}{\sigma_t}\right|^{\alpha_j} \mathbf{I}_{\{|y_{1,j}|\leq 1\}} \exp\left(-\frac{(y_j-m_t y_{1,j})^2}{2\sigma_t^2}\right) \mrd y_{1,j} \\ \lesssim &~\frac{1}{\sigma_t} \cdot \epsilon^{\frac{C^2}{2}}(C^2\log\epsilon^{-1})^{\frac{\alpha_j}{2}} \cdot \int_{\Rbb\backslash a_{j,\my}}\mathbf{I}_{\{|y_{1,j}| \leq 1\}}  \mrd y_{1,j} \\
\lesssim &~\frac{1}{\sigma_t} \cdot \epsilon^{\frac{C^2}{2}}(C^2\log\epsilon^{-1})^{\frac{\alpha_j}{2}} \cdot 2^{d_\mcY} \\
\lesssim & ~\frac{1}{\sigma_t} \cdot \epsilon^{\frac{C^2}{2}}(C^2\log\epsilon^{-1})^{\frac{\alpha_j}{2}}.
\end{aligned}
$$
Taking $C \geq \max\{2, \max_{1\leq i \leq d_\mcY} \sqrt{\alpha_i}\}$, we obtain
$$
\frac{1}{\sigma_t\sqrt{2\pi}}\int_{\Rbb\backslash a_{j,\my}}\left|\frac{y_j-m_t y_{1,j}}{\sigma_t}\right|^{\alpha_j} \mathbf{I}_{\{|y_{1,j}|\leq 1\}} \exp\left(-\frac{(y_j-m_t y_{1,j})^2}{2\sigma_t^2}\right) \mrd y_{1,j} \lesssim \epsilon \cdot \min\left\{\frac{1}{m_t}, \frac{1}{\sigma_t}\right\} \lesssim \epsilon,
$$
which implies that
$$
\Bigg|
\int_{\Rbb^{d_\mcY}\backslash A_{\my}} \prod_{i=1}^{d_\mcY}\left(\frac{y_i-m_t y_{1,i}}{\sigma_t}\right)^{\alpha_i}\frac{1}{\sigma_t^{d_\mcY}(2\pi)^{d_\mcY/2}}p_0(\my_1|\mx)\exp\left(-\frac{\Vert\my-m_t\my_1\Vert^2}{2\sigma_t^2}\right)\mrd\my_1
\Bigg| \lesssim \sum_{j=1}^{d_\mcY} \epsilon \lesssim \epsilon.
$$
The proof is complete.
\end{proof}

\begin{lemma}[Boundedness of Derivatives]\label{lem: derivatives_boundness}
Suppose Assumption \ref{ass: bounded_density} holds. For any $k\in\mathbb N$ and $i_1,\ldots,i_k\in\{1,\ldots,d_\mcY\}$, uniformly over $0<t<1$, $\my\in\Rbb^{d_\mcY}$, and $\mx\in[-1,1]^{d_\mcX}$,
\begin{equation}\label{eq: derivative_q} 
\left|\partial_{y_{i_1}}\cdots\partial_{y_{i_k}}p_t(\my|\mx)\right|\lesssim \frac{1}{\sigma_t^k}.
\end{equation}
Moreover, we have 
\begin{equation}\label{eq: bound_score}
\Vert \nabla\log p_t(\my|\mx) \Vert \lesssim \frac{1}{\sigma_t}\cdot\left(\frac{(\Vert\my\Vert_{\infty} - m_t)_{+}}{\sigma_t}  \vee 1\right),
\end{equation}
\begin{equation}\label{eq: derivative_y_score}
\Vert\partial_{y_i}\nabla\log p_t(\my|\mx)\Vert \lesssim \frac{1}{\sigma_t^2}\cdot\left(\frac{(\Vert\my\Vert_{\infty} - m_t)_{+}^2}{\sigma_t^2}\vee 1\right), ~~~ 1\leq i\leq d_\mcY,
\end{equation}
and 
\begin{equation}\label{eq: derivative_t_score}
\Vert\partial_t\nabla\log p_t(\my|\mx)\Vert \lesssim \frac{|\partial_t\sigma_t| + |\partial_t m_t|}{\sigma_t^2}\cdot\left(\frac{(\Vert\my\Vert_{\infty} - m_t)_{+}^2}{\sigma_t^2} \vee 1\right)^{\frac{3}{2}}.
\end{equation}
\end{lemma}
\begin{proof}
We first prove \eqref{eq: derivative_q}. Let $f_1(\my|\mx)=p_t(\my|\mx)$. For multi-index $\boldsymbol{\alpha}\in\mathbb{N}^{d_\mcY}$, we denote
\[
f_1^{\boldsymbol{\alpha}}(\my|\mx)
:= \partial_{y_1}^{\alpha_1} \cdots \partial_{y_{d_\mcY}}^{\alpha_{d_\mcY}} f_1(\my|\mx),
\]
and define
\[
B_{\boldsymbol{\alpha}}
:= \left\{ \mathbf{u}\in\mathbb{N}^{d_\mcY}: u_i\leq\alpha_i,\ i=1,\ldots,d_\mcY \right\}.
\]
By Assumptions \ref{ass: bounded_density}, $\operatorname{supp}p_0(\cdot|\mx)
\subseteq[-1,1]^{d_\mcY}$. Recall we have
\[
\begin{aligned}
f_1(\my|\mx)
&= \int_{[-1,1]^{d_\mcY}} p_0(\my_1|\mx)\frac{1}{\sigma_t^{d_\mcY}(2\pi)^{d_\mcY/2}}\exp\left( -\frac{\|\my-m_t\my_1\|^2}{2\sigma_t^2}
\right) \,\mrd\my_1,
\end{aligned}
\]
where $m_t=1-t>0$, $\sigma_t=\sqrt{t(2-t)}>0$ and $m_t^2+\sigma_t^2=1$ for any $0<t<1$. Since
\[
\frac{\mrd}{\mrd z}e^{-z^2/2}
=-ze^{-z^2/2},
\qquad
\frac{\mrd}{\mrd z}
\left(z^r e^{-z^2/2}\right)
=
\left(rz^{r-1}-z^{r+1}\right)e^{-z^2/2},
\quad r\geq1,
\]
repeated differentiation gives
\[
\partial_{y_1}^{\alpha_1}\cdots\partial_{y_{d_\mcY}}^{\alpha_{d_\mcY}}e^{-\|\my\|^2/2}
= \prod_{i=1}^{d_\mcY} \left( \frac{\mrd^{\alpha_i}}{\mrd y_i^{\alpha_i}}
e^{-y_i^2/2} \right)
= \sum_{\mathbf{u}\in B_{\boldsymbol{\alpha}}} C_{\mathbf{u}}\prod_{i=1}^{d_\mcY}y_i^{u_i}e^{-\|\my\|^2/2},
\]
with some constants $C_{\mathbf{u}}$ depending only on $\boldsymbol{\alpha}$. Consequently,
\[
\begin{aligned}
&\partial_{\my}^{\boldsymbol{\alpha}}
\left[
\frac{1}{\sigma_t^{d_\mcY}(2\pi)^{d_\mcY/2}}
\exp\left(
-\frac{\|\my-m_t\my_1\|^2}{2\sigma_t^2}
\right)
\right]\\
&~~~~~~~~ =
\frac{1}{\sigma_t^{d_\mcY+\sum_{i=1}^{d_\mcY}\alpha_i}(2\pi)^{d_\mcY/2}}
\sum_{\mathbf{u}\in B_{\boldsymbol{\alpha}}}
C_{\mathbf{u}} \prod_{i=1}^{d_\mcY}\left( \frac{y_i-m_ty_{1,i}}{\sigma_t} \right)^{u_i} \exp\left( -\frac{\|\my-m_t\my_1\|^2}{2\sigma_t^2} \right).
\end{aligned}
\]
For $r\in\mathbb{N}$ and $z>0$, one can calculate that
\[
\sup_{z\in\mathbb{R}}|z|^r e^{-z^2/2} =
\begin{cases}
1, & r=0,\\
r^{r/2}e^{-r/2}, & r\geq1.
\end{cases}
\]
It follows that
\[
\sup_{\mathbf{z}\in\Rbb^{d_\mcY}} \prod_{i=1}^{d_\mcY}|z_i|^{u_i}e^{-\|\mathbf{z}\|^2/2}
= \prod_{i=1}^{d_\mcY}\left( \sup_{z_i\in\mathbb{R}}|z_i|^{u_i}e^{-z_i^2/2} \right)
= \prod_{i=1}^{d_\mcY}\left( \max\{1,u_i\}\right)^{u_i/2}e^{-u_i/2}
< \infty.
\]
Thus,
\[
\begin{aligned}
&\int_{[-1,1]^{d_\mcY}}
p_0(\my_1|\mx)
\sup_{\my\in\Rbb^{d_\mcY}} \left| \partial_{\my}^{\boldsymbol{\alpha}}
\left[ \frac{1}{\sigma_t^{d_\mcY}(2\pi)^{d_\mcY/2}}\exp\left(-\frac{\|\my-m_t\my_1\|^2}{2\sigma_t^2}\right) \right] \right| \,\mrd\my_1\\
&~~~~~~~~\leq
\frac{
2^{d_\mcY}C_u
}{
\sigma_t^{d_\mcY+\sum_{i=1}^{d_\mcY}\alpha_i}
(2\pi)^{d_\mcY/2}}
\sum_{\mathbf{u}\in B_{\boldsymbol{\alpha}}}
|C_{\mathbf{u}}|
\prod_{i=1}^{d_\mcY}
\left(\max\{1,u_i\}\right)^{u_i/2}
e^{-u_i/2}
<\infty.
\end{aligned}
\]
These bounds also hold for every required lower derivative order and justify differentiation under the integral sign. Therefore,
\[
f_1^{\boldsymbol{\alpha}}(\my|\mx)
= \frac{1}{\sigma_t^{\sum_{i=1}^{d_\mcY}\alpha_i}}\sum_{\mathbf{u}\in B_{\boldsymbol{\alpha}}}C_{\mathbf{u}}\int_{[-1,1]^{d_\mcY}}\prod_{i=1}^{d_\mcY} \left( \frac{y_i-m_ty_{1,i}}{\sigma_t} \right)^{u_i}
\cdot
\frac{p_0(\my_1|\mx)}{
\sigma_t^{d_\mcY}(2\pi)^{d_\mcY/2}}
\exp\left( -\frac{\|\my-m_t\my_1\|^2}{2\sigma_t^2} \right) \,\mrd\my_1.
\]
Taking
$\mathbf{z}=(\my-m_t\my_1)/\sigma_t$, we have
\[
\my_1 = \frac{\my-\sigma_t\mathbf{z}}{m_t}, \qquad
\mrd\my_1 = \left(\frac{\sigma_t}{m_t}\right)^{d_\mcY} \mrd\mathbf{z}.
\]
Hence, for each $\mathbf{u}\in B_{\boldsymbol{\alpha}}$,
\[
\begin{aligned}
&\int_{[-1,1]^{d_\mcY}}
\prod_{i=1}^{d_\mcY}
\left| \frac{y_i-m_ty_{1,i}}{\sigma_t} \right|^{u_i}
\frac{1}{\sigma_t^{d_\mcY}(2\pi)^{d_\mcY/2}}\exp\left( -\frac{\|\my-m_t\my_1\|^2}{2\sigma_t^2} \right)
\,\mrd\my_1\\
&\quad=
\frac{1}{m_t^{d_\mcY}(2\pi)^{d_\mcY/2}}
\int_{
\prod_{i=1}^{d_\mcY}
\left[
(y_i-m_t)/\sigma_t,\,
(y_i+m_t)/\sigma_t
\right]
}
\prod_{i=1}^{d_\mcY}|z_i|^{u_i}
e^{-\|\mathbf{z}\|^2/2}
\,\mrd\mathbf{z}\\
&\quad\leq
\frac{1}{m_t^{d_\mcY}(2\pi)^{d_\mcY/2}}
\int_{\Rbb^{d_\mcY}}
\prod_{i=1}^{d_\mcY}|z_i|^{u_i}
e^{-\|\mathbf{z}\|^2/2}
\,\mrd\mathbf{z}\\
&\quad=
\frac{1}{m_t^{d_\mcY}}
\prod_{i=1}^{d_\mcY}
\left(
\frac{2}{\sqrt{2\pi}}
\int_0^\infty z^{u_i}e^{-z^2/2}\,\mrd z
\right)\\
&\quad=
\frac{1}{m_t^{d_\mcY}}
\prod_{i=1}^{d_\mcY}
\left(
\frac{2^{u_i/2}}{\sqrt{\pi}}
\int_0^\infty
w^{(u_i-1)/2}e^{-w}\,\mrd w
\right)\\
&\quad=
\frac{1}{m_t^{d_\mcY}}
\prod_{i=1}^{d_\mcY}
\frac{
2^{u_i/2}\Gamma\left((u_i+1)/2\right)
}{\sqrt{\pi}}
\lesssim
\frac{1}{m_t^{d_\mcY}}.
\end{aligned}
\]
On the other hand,
\[
\begin{aligned}
&\int_{[-1,1]^{d_\mcY}}
\prod_{i=1}^{d_\mcY}
\left| \frac{y_i-m_ty_{1,i}}{\sigma_t} \right|^{u_i} \frac{1}{\sigma_t^{d_\mcY}(2\pi)^{d_\mcY/2}}\exp\left(-\frac{\|\my-m_t\my_1\|^2}{2\sigma_t^2}\right) \,\mrd\my_1\\
&\quad\leq
\frac{2^{d_\mcY}}{
\sigma_t^{d_\mcY}(2\pi)^{d_\mcY/2}}
\sup_{\mathbf{z}\in\Rbb^{d_\mcY}}
\left(
\prod_{i=1}^{d_\mcY}|z_i|^{u_i}
e^{-\|\mathbf{z}\|^2/2}
\right)\\
&\quad=
\frac{2^{d_\mcY}}{
\sigma_t^{d_\mcY}(2\pi)^{d_\mcY/2}}
\prod_{i=1}^{d_\mcY}
\left(\max\{1,u_i\}\right)^{u_i/2}
e^{-u_i/2}
\lesssim
\frac{1}{\sigma_t^{d_\mcY}}.
\end{aligned}
\]
As a result, we have
\begin{equation*}
\begin{aligned}
\left|f_1^{\boldsymbol{\alpha}}(\my|\mx)\right|
&\le \frac{1}{\sigma_t^{\sum_{i=1}^{d_\mcY}\alpha_i}}\sum_{\mathbf{u}\in B_{\boldsymbol{\alpha}}}|C_{\mathbf{u}}|\int_{[-1,1]^{d_\mcY}}\prod_{i=1}^{d_\mcY} \left| \frac{y_i-m_ty_{1,i}}{\sigma_t} \right|^{u_i}
\cdot \frac{p_0(\my_1|\mx)}{
\sigma_t^{d_\mcY}(2\pi)^{d_\mcY/2}} \exp\left( -\frac{\|\my-m_t\my_1\|^2}{2\sigma_t^2} \right) \,\mrd\my_1\\
&\le \frac{C_u}{\sigma_t^{\sum_{i=1}^{d_\mcY}\alpha_i}}\sum_{\mathbf{u}\in B_{\boldsymbol{\alpha}}}|C_{\mathbf{u}}|\cdot\min\left\{\frac{1}{m_t^{d_\mcY}}, \frac{1}{\sigma_t^{d_\mcY}}\right\}\\
&\lesssim \frac{1}{\sigma_t^{\sum_{i=1}^{d_\mcY}\alpha_i}},
\end{aligned}
\end{equation*}
where we used $m_t^2+\sigma_t^2=1$, which gives $\max\{m_t,\sigma_t\}\geq 1/\sqrt{2}$, and
\[
\min\left\{ \frac{1}{m_t^{d_\mcY}}, \frac{1}{\sigma_t^{d_\mcY}} \right\}
= \frac{1}{\max\{m_t,\sigma_t\}^{d_\mcY}}
\leq 2^{d_\mcY/2} = \mcO(1).
\]
For $k\in\mathbb{N}$, take
\[
\alpha_i
= \sum_{a=1}^k\mathbf{1}_{\{i_a=i\}}, \quad i=1,\ldots,d_\mcY.
\]
Then $\sum_{i=1}^{d_\mcY}\alpha_i=k$ and
\[
\left| \partial_{y_{i_1}}\cdots\partial_{y_{i_k}} p_t(\my|\mx) \right|
= \left| f_1^{\boldsymbol{\alpha}}(\my|\mx) \right|
\lesssim \frac{1}{\sigma_t^k}.
\]

Next, we prove \eqref{eq: bound_score} and \eqref{eq: derivative_y_score}. For convenience, we focus on the first coordinate of $\nabla\log p_t(\my|\mx)$, and all  other coordinates of $\nabla\log p_t(\my|\mx)$ are bounded in the same manner. Let $f_2(\my|\mx):= [\sigma_t\nabla p_t(\my|\mx)]_1$. 
Then, we have
$$
[\nabla\log p_t(\my|\mx)]_1 = \frac{1}{\sigma_t}\cdot\frac{f_2(\my|\mx)}{f_1(\my|\mx)}, ~ [\partial_{y_i}\nabla\log p_t(\my|\mx)]_1 = \frac{1}{\sigma_t}\cdot\left(\frac{\partial_{y_i}f_2(\my|\mx)}{f_1(\my|\mx)} - \frac{f_2(\my|\mx)(\partial_{y_i}f_1(\my|\mx))}{f_1^2(\my|\mx)}\right).
$$
Moreover,
$$
\frac{f_2(\my|\mx)}{f_1(\my|\mx)} = \frac{
-\int_{\Rbb^{d_\mcY}}\left(\frac{y_1-m_t y_{1,1}}{\sigma_t}\right)\cdot\frac{1}{\sigma_t^{d_\mcY}(2\pi)^{d_\mcY/2}}p_0(\my_1|\mx)\exp\left(-\frac{\Vert\my-m_t\my_1\Vert^2}{2\sigma_t^2}\right)\mrd\my_1}
{\int_{\Rbb^{d_\mcY}}\frac{1}{\sigma_t^{d_\mcY}(2\pi)^{d_\mcY/2}}p_0(\my_1|\mx)\exp\left(-\frac{\Vert\my-m_t\my_1\Vert^2}{2\sigma_t^2}\right)\mrd\my_1},
$$
$$
\frac{\partial_{y_i}f_1(\my|\mx)}{f_1(\my|\mx)} = \frac{1}{\sigma_t}\cdot\frac{
-\int_{\Rbb^{d_\mcY}}\left(\frac{y_i-m_ty_{1,i}}{\sigma_t}\right)\cdot\frac{1}{\sigma_t^{d_\mcY}(2\pi)^{d_\mcY/2}}p_0(\my_1|\mx)\exp\left(-\frac{\Vert\my-m_t\my_1\Vert^2}{2\sigma_t^2}\right)\mrd\my_1}
{\int_{\Rbb^{d_\mcY}}\frac{1}{\sigma_t^{d_\mcY}(2\pi)^{d_\mcY/2}}p_0(\my_1|\mx)\exp\left(-\frac{\Vert\my-m_t\my_1\Vert^2}{2\sigma_t^2}\right)\mrd\my_1},
$$
$$
\frac{\partial_{y_i}f_2(\my|\mx)}{f_1(\my|\mx)} = \frac{1}{\sigma_t} \cdot \frac{
-\int_{\Rbb^{d_\mcY}} \left(\mathbf{I}_{\{i=1\}} - \frac{y_1-m_t y_{1,1}}{\sigma_t}\cdot\frac{y_i-m_t y_{1,i}}{\sigma_t}\right)\cdot\frac{1}{\sigma_t^{d_\mcY}(2\pi)^{d_\mcY/2}}p_0(\my_1|\mx)\exp\left(-\frac{\Vert\my-m_t \my_1\Vert^2}{2\sigma_t^2}\right)\mrd\my_1}
{\int_{\Rbb^{d_\mcY}}\frac{1}{\sigma_t^{d_\mcY}(2\pi)^{d_\mcY/2}}p_0(\my_1|\mx)\exp\left(-\frac{\Vert\my-m_t\my_1\Vert^2}{2\sigma_t^2}\right)\mrd\my_1}.
$$
The integral terms in the three equations above can be expressed in the following unified form:
\begin{equation}\label{eq: eq_unified}
\frac{
\int_{\Rbb^{d_\mcY}}\prod_{i=1}^{d_\mcY}\left(\frac{y_i-m_t y_{1,i}}{\sigma_t}\right)^{\alpha_i}\frac{1}{\sigma_t^{d_\mcY}(2\pi)^{d_\mcY/2}}p_0(\my_1|\mx)\exp\left(-\frac{\Vert\my-m_t\my_1\Vert^2}{2\sigma_t^2}\right)\mrd\my_1}
{\int_{\Rbb^{d_\mcY}}\frac{1}{\sigma_t^{d_\mcY}(2\pi)^{d_\mcY/2}}p_0(\my_1|\mx)\exp\left(-\frac{\Vert\my-m_t\my_1\Vert^2}{2\sigma_t^2}\right)\mrd\my_1}
\end{equation}
with $\sum_{i=1}^{d_\mcY}\alpha_i\leq 2$. To also control the time derivative below, we establish a bound for normalized absolute moments for every $\boldsymbol{\alpha}\in\{0,1,2,\ldots\}^{d_\mcY}$ satisfying $\sum_{i=1}^{d_\mcY}\alpha_i\leq3$.

Fix $0<\epsilon<e^{-1}$ and define
\[
    A_{\my}
    := \prod_{i=1}^{d_\mcY} \left[ \frac{y_i-C\sigma_t\sqrt{\log\epsilon^{-1}}}{m_t}, \frac{y_i+C\sigma_t\sqrt{\log\epsilon^{-1}}}{m_t} \right].
\]
The proof of Lemma~\ref{lem: integral_clipping} bounds the integral of the absolute integrand over $A_{\my}^c$. Since there are only finitely many multi-indices with $\sum_i\alpha_i\leq3$, the constant $C$ defining $A_{\my}$
can be chosen uniformly over all such multi-indices. In particular,
\[
    \int_{\Rbb^{d_\mcY}\setminus A_{\my}}
    \prod_{i=1}^{d_\mcY} \left| \frac{y_i-m_ty_{1,i}}{\sigma_t} \right|^{\alpha_i}
    \frac{p_0(\my_1\mid\mx)}{\sigma_t^{d_\mcY}(2\pi)^{d_\mcY/2}}
    \exp\!\left( -\frac{\|\my-m_t\my_1\|^2}{2\sigma_t^2} \right)\mrd\my_1
    \lesssim \epsilon.
\]
On $A_{\my}$, the product of the absolute powers is bounded by $(C\sqrt{\log\epsilon^{-1}})^{\sum_i\alpha_i}$. Moreover, the integral of the remaining nonnegative integrand over $A_{\my}$ is at most $p_t(\my\mid\mx)$.
Consequently, for any fixed $c>0$, whenever $p_t(\my\mid\mx)\geq c\epsilon$, we obtain
\begin{equation}\label{eq: eq_unified_bound}
\begin{aligned}
    &~~~~\frac{1}{p_t(\my\mid\mx)}
    \int_{\Rbb^{d_\mcY}} \prod_{i=1}^{d_\mcY}
    \left|\frac{y_i-m_ty_{1,i}}{\sigma_t}\right|^{\alpha_i}
    \frac{p_0(\my_1\mid\mx)}{\sigma_t^{d_\mcY}(2\pi)^{d_\mcY/2}}
    \exp\!\left( -\frac{\|\my-m_t\my_1\|^2}{2\sigma_t^2} \right)\mrd\my_1\\
    &\lesssim (\log\epsilon^{-1})^{\frac12\sum_i\alpha_i}+\frac{\epsilon}{p_t(\my\mid\mx)}
    \lesssim (\log\epsilon^{-1})^{\frac12\sum_i\alpha_i}.
\end{aligned}
\end{equation}
The last inequality uses $p_t(\my\mid\mx)\geq c\epsilon$ and $\log\epsilon^{-1}>1$. The implicit constants depend only on $d_\mcY,C_u,c$,
and are uniform over $0<t<1$, $\my\in\Rbb^{d_\mcY}$, $\mx\in[-1,1]^{d_\mcX}$, $\epsilon$, and the indicated multi-indices.

In particular, the preceding three ratios satisfy
$$
\left| \frac{f_2(\my|\mx)}{f_1(\my|\mx)} \right| \lesssim \sqrt{\log\epsilon^{-1}},~
\left| \frac{\partial_{y_i}f_1(\my|\mx)}{f_1(\my|\mx)} \right| \lesssim \frac{\sqrt{\log\epsilon^{-1}}}{\sigma_t},~
\left| \frac{\partial_{y_i}f_2(\my|\mx)}{f_1(\my|\mx)} \right| \lesssim \frac{\log\epsilon^{-1}}{\sigma_t}.
$$
Then, we obtain
$$
\Vert\nabla\log p_t(\my|\mx)\Vert \lesssim \frac{\sqrt{\log\epsilon^{-1}}}{\sigma_t}, ~ \Vert\partial_{y_i}\nabla\log p_t(\my|\mx)\Vert \lesssim \frac{\log\epsilon^{-1}}{\sigma_t^2}.
$$
By Lemma~\ref{lem: bound_for_density}, $p_t(\my|\mx) \gtrsim \exp\left(-\frac{d_\mcY(\Vert\my\Vert_{\infty}-m_t)_{+}^2}{\sigma_t^2}\right)$. Choose
\[
    \epsilon
    := \exp\!\left( -2-\frac{d_\mcY(\|\my\|_\infty-m_t)_+^2}{\sigma_t^2} \right).
\]
Then $0<\epsilon\leq e^{-2}<e^{-1}$, and there exists $c>0$, depending only on the fixed constants in Lemma~\ref{lem: bound_for_density}, such that $p_t(\my\mid\mx)\geq c\epsilon$ uniformly in $t,\my,\mx$. Moreover,
\[
    \log\epsilon^{-1}
    = 2+\frac{d_\mcY(\|\my\|_\infty-m_t)_+^2}{\sigma_t^2}
    \asymp 1\vee\frac{(\|\my\|_\infty-m_t)_+^2}{\sigma_t^2},
\]
where the comparison constants depend only on $d_\mcY$. Substituting this choice into the preceding estimates, we obtain
$$
\Vert \nabla\log p_t(\my|\mx) \Vert \lesssim \frac{1}{\sigma_t}\cdot\left(\frac{(\Vert\my\Vert_{\infty} - m_t)_{+}}{\sigma_t} \vee 1\right)
$$
and
$$
\Vert\partial_{y_i}\nabla\log p_t(\my|\mx)\Vert \lesssim \frac{1}{\sigma_t^2}\cdot \left(\frac{(\Vert\my\Vert_{\infty} - m_t)_{+}^2}{\sigma_t^2} \vee 1\right), ~~~ 1\leq i\leq d_\mcY.
$$

Finally, we prove \eqref{eq: derivative_t_score}. By a simple calculation, we have
$$
\partial_t[\nabla\log p_t(\my|\mx)]_1  = \left(\partial_t\frac{1}{\sigma_t}\right) \cdot \frac{f_2(\my|\mx)}{f_1(\my|\mx)} - \frac{1}{\sigma_t} \cdot \frac{\partial_t f_1(\my|\mx)}{f_1(\my|\mx)} \cdot \frac{f_2(\my|\mx)}{f_1(\my|\mx)} + \frac{1}{\sigma_t} \cdot \frac{\partial_t f_2(\my|\mx)}{f_1(\my|\mx)}.
$$
We calculate these three terms separately.
$$
\left(\partial_t\frac{1}{\sigma_t}\right) \cdot \frac{f_2(\my|\mx)}{f_1(\my|\mx)} = -\frac{\partial_t\sigma_t}{\sigma_t}\cdot[\nabla\log p_t(\my|\mx)]_1,
$$

$$
\frac{\partial_t f_1(\my|\mx)}{f_1(\my|\mx)} = 
\frac{
\int_{\Rbb^{d_\mcY}}\frac{
\partial_t\sigma_t\left(\sigma_t^{-3}\Vert\my - m_t\my_1\Vert^2 - d_\mcY\sigma_t^{-1}\right) +\sigma_t^{-2} \langle \my - m_t\my_1, \my_1 \rangle \partial_t m_t}
{\sigma_t^{d_\mcY}(2\pi)^{d_\mcY/2}}
p_0(\my_1|\mx)\exp\left(-\frac{\Vert\my-m_t\my_1\Vert^2}{2\sigma_t^2}\right)\mrd\my_1}
{\int_{\Rbb^{d_\mcY}}\frac{1}{\sigma_t^{d_\mcY}(2\pi)^{d_\mcY/2}}p_0(\my_1|\mx)\exp\left(-\frac{\Vert\my-m_t\my_1\Vert^2}{2\sigma_t^2}\right)\mrd\my_1},
$$

$$
\begin{aligned}
\frac{\partial_t f_2(\my|\mx)}{f_1(\my|\mx)} &= \partial_t\sigma_t \cdot \frac{
\int_{\Rbb^{d_\mcY}}\frac{1}{\sigma_t^{d_\mcY}(2\pi)^{d_\mcY/2}}p_0(\my_1|\mx)\exp\left(-\frac{\Vert\my-m_t\my_1\Vert^2}{2\sigma_t^2}\right)\frac{(y_1-m_t y_{1,1})}{\sigma_t}\cdot\left(\frac{1+d_\mcY}{\sigma_t} - \frac{\Vert\my - m_t\my_1\Vert^2}{\sigma_t^3} 
\right)\mrd\my_1}
{\int_{\Rbb^{d_\mcY}}\frac{1}{\sigma_t^{d_\mcY}(2\pi)^{d_\mcY/2}}p_0(\my_1|\mx)\exp\left(-\frac{\Vert\my-m_t\my_1\Vert^2}{2\sigma_t^2}\right)\mrd\my_1} \\
& + \partial_t m_t \cdot \frac{\int_{\Rbb^{d_\mcY}}\frac{1}{\sigma_t^{d_\mcY}(2\pi)^{d_\mcY/2}}p_0(\my_1|\mx)\exp\left(-\frac{\Vert\my-m_t\my_1\Vert^2}{2\sigma_t^2}\right)\left(\frac{y_{1,1}}{\sigma_t} - \frac{y_1-m_t y_{1,1}}{\sigma_t} \cdot \frac{\langle \my - m_t\my_1, \my_1\rangle}{\sigma_t^2}\right)\mrd\my_1}{\int\frac{1}{\sigma_t^{d_\mcY}(2\pi)^{d_\mcY/2}}p_0(\my_1|\mx)\exp\left(-\frac{\Vert\my-m_t\my_1\Vert^2}{2\sigma_t^2}\right)\mrd\my_1}
\end{aligned}
$$
We apply the absolute-moment estimate \eqref{eq: eq_unified_bound} with total degrees up to three. Since $|y_{1,i}|\leq1$ on the support of $p_0$, the formula for $\partial_t f_1$ gives
\[
\left| \frac{\partial_t f_1(\my\mid\mx)}{f_1(\my\mid\mx)} \right|
\lesssim \frac{|\partial_t\sigma_t|}{\sigma_t}\log\epsilon^{-1} + \frac{|\partial_t m_t|}{\sigma_t} (\log\epsilon^{-1})^{1/2}.
\]
Similarly, the formula for $\partial_t f_2$ gives
\[
\begin{aligned}
\left| \frac{\partial_t f_2(\my\mid\mx)}{f_1(\my\mid\mx)} \right|
&\lesssim
\frac{|\partial_t\sigma_t|}{\sigma_t}
\left[ (\log\epsilon^{-1})^{1/2}+(\log\epsilon^{-1})^{3/2} \right] + \frac{|\partial_t m_t|}{\sigma_t} \left(1+\log\epsilon^{-1}\right)\\
&\lesssim
\frac{|\partial_t\sigma_t|+|\partial_t m_t|}{\sigma_t} (\log\epsilon^{-1})^{3/2}.
\end{aligned}
\]

Combining these estimates with $|f_2/f_1|\lesssim(\log\epsilon^{-1})^{1/2}$ and $\partial_t(\sigma_t^{-1})=-(\partial_t\sigma_t)/\sigma_t^2$, and using the preceding choice of $\epsilon$, we obtain
$$
\left|\left(\partial_t\frac{1}{\sigma_t}\right) \cdot \frac{f_2(\my|\mx)}{f_1(\my|\mx)}\right| \lesssim \frac{|\partial_t\sigma_t|}{\sigma_t^2}\cdot\left(\frac{(\Vert\my\Vert_{\infty} - m_t)_{+}}{\sigma_t} \vee 1\right),
$$
$$
\left|\frac{1}{\sigma_t} \cdot \frac{\partial_t f_1(\my|\mx)}{f_1(\my|\mx)} \cdot \frac{f_2(\my|\mx)}{f_1(\my|\mx)}\right| \lesssim 
\frac{|\partial_t\sigma_t| + |\partial_t m_t|}{\sigma_t^2}\cdot \left(\frac{(\Vert\my\Vert_{\infty} - m_t)_{+}^2}{\sigma_t^2} \vee 1\right)^{\frac{3}{2}},
$$
$$
\left|\frac{1}{\sigma_t} \cdot \frac{\partial_t f_2(\my|\mx)}{f_1(\my|\mx)}\right| \lesssim \frac{|\partial_t\sigma_t| + |\partial_t m_t|}{\sigma_t^2}\cdot \left(\frac{(\Vert\my\Vert_{\infty} - m_t)_{+}^2}{\sigma_t^2} \vee 1\right)^{\frac{3}{2}}.
$$
Combining the above three inequalities, we finally obtain
$$
\Vert\partial_t\nabla\log p_t(\my|\mx)\Vert \lesssim \frac{|\partial_t\sigma_t| + |\partial_t m_t|}{\sigma_t^2}\cdot\left(\frac{(\Vert\my\Vert_{\infty} - m_t)_{+}^2}{\sigma_t^2}\vee 1\right)^{\frac{3}{2}}.
$$
The proof is complete.
\end{proof}

\begin{lemma}[Error Bounds for Clipping]\label{lem: bound_for_clipping}
Suppose that Assumption~\ref{ass: bounded_density} holds. Let $0<\epsilon\le e^{-1}$ and $0<\epsilon_1<1$. There exists a constant $C>0$, depending only on
$d_{\mcY},C_l,C_u$, such that the following bounds hold uniformly over $0<t\le1$ and $\mx\in[-1,1]^{d_{\mcX}}$:
\begin{equation}\label{eq: clipping_bound_1}
\int_{\Vert\my\Vert_{\infty} > m_t + C\sigma_t\sqrt{\log\epsilon^{-1}}} p_t(\my|\mx)\Vert\nabla\log p_t(\my|\mx)\Vert^2\mrd\my \lesssim \frac{\epsilon}{\sigma_t},
\end{equation}

\begin{equation}\label{eq: clipping_bound_2}
\int_{\Vert\my\Vert_{\infty} > m_t + C\sigma_t\sqrt{\log\epsilon^{-1}}} p_t(\my|\mx)\mrd \my \lesssim \sigma_t\epsilon.
\end{equation}
Moreover, for $\Vert\my\Vert_{\infty} \leq m_t + C\sigma_t\sqrt{\log\epsilon^{-1}}$, it holds that
\begin{equation}\label{eq: clipping_bound_3}
\int_{\Vert\my\Vert_{\infty}\leq m_t + C\sigma_t\sqrt{\log\epsilon^{-1}}} p_t(\my|\mx)\mathbf{I}_{\{p_t(\my|\mx)\leq\epsilon_1\}}\Vert\nabla\log p_t(\my|\mx)\Vert^2\mrd\my \lesssim \frac{\epsilon_1}{\sigma_t^2}(\log\epsilon^{-1})^{\frac{d_\mcY+2}{2}},
\end{equation}

\begin{equation}\label{eq: clipping_bound_4}
\int_{\Vert\my\Vert_{\infty}\leq m_t + C\sigma_t\sqrt{\log\epsilon^{-1}}} p_t(\my|\mx)\mathbf{I}_{\{p_t(\my|\mx)\leq\epsilon_1\}}\mrd\my \lesssim \epsilon_1 (\log\epsilon^{-1})^{\frac{d_\mcY}{2}}.
\end{equation}
\end{lemma}

\begin{proof}
We first prove \eqref{eq: clipping_bound_1}. For $\Vert\my\Vert_\infty \geq m_t + C\sigma_t\sqrt{\log\epsilon^{-1}}$, according to Lemma \ref{lem: bound_for_density} and Lemma \ref{lem: derivatives_boundness}, we have
$$
p_t(\my|\mx)\Vert\nabla\log p_t(\my|\mx)\Vert^2 \lesssim \frac{1}{\sigma_t^2}\exp\left(-\frac{(\Vert\my\Vert_\infty - m_t)^2}{2\sigma_t^2}\right)  \frac{(\Vert\my\Vert_\infty - m_t)^2}{\sigma_t^2} .
$$
Without loss of generality, we set $|y_1| = \Vert\my\Vert_\infty$. Then, we have
$$
\begin{aligned}
&\int_{\Vert\my\Vert_{\infty} > m_t + C\sigma_t\sqrt{\log\epsilon^{-1}}} p_t(\my|\mx)\Vert\nabla\log p_t(\my|\mx)\Vert^2\mrd\my \\
\lesssim &\int_{|y_1| > m_t + C\sigma_t\sqrt{\log\epsilon^{-1}}}\int_{|y_2|\leq |y_1|}\cdots\int_{|y_{d_\mcY}|\leq |y_1|}\frac{1}{\sigma_t^2}\exp\left(-\frac{(|y_1| - m_t)^2}{2\sigma_t^2}\right) \frac{(|y_1| - m_t)^2}{\sigma_t^2} \mrd y_2\cdots\mrd y_{d_\mcY}\mrd y_1\\
\lesssim &\int_{|y_1| > m_t + C\sigma_t\sqrt{\log\epsilon^{-1}}}\frac{1}{\sigma_t^2}\exp\left(-\frac{(|y_1| - m_t)^2}{2\sigma_t^2}\right) \frac{(|y_1| - m_t)^2}{\sigma_t^2} \cdot |y_1|^{d_\mcY-1}\mrd y_1\\
\lesssim &\int_{y_1 > m_t + C\sigma_t\sqrt{\log\epsilon^{-1}}}\frac{1}{\sigma_t^2}\exp\left(-\frac{(y_1 - m_t)^2}{2\sigma_t^2}\right)  \frac{(y_1 - m_t)^2}{\sigma_t^2}\cdot y_1^{d_\mcY-1}\mrd y_1\\
\lesssim &\int_{C\sqrt{\log\epsilon^{-1}}}^{\infty}\frac{1}{\sigma_t}\exp\left(-\frac{r^2}{2}\right) r^2 (m_t + r\sigma_t)^{d_\mcY-1}\mathrm{d}r\\
\lesssim & \int_{C\sqrt{\log\epsilon^{-1}}}^{\infty}\frac{1}{\sigma_t}\exp\left(-\frac{r^2}{2}\right)r^{d_\mcY+1}\mrd r  \\
\lesssim & \frac{1}{\sigma_t}\epsilon^{\frac{C^2}{2}}(\log\epsilon^{-1})^{\frac{d_\mcY+1}{2}},
\end{aligned}
$$
where we let $r = \frac{y_1 - m_t}{\sigma_t}$. Taking $C \geq 2$ and use $\epsilon(\log \epsilon^{-1})^{\frac{d_\mcY+1}{2}} = \mcO(1)$, it holds that
$$
\int_{\Vert\my\Vert_{\infty} > m_t + C\sigma_t\sqrt{\log\epsilon^{-1}}} p_t(\my|\mx)\Vert\nabla\log p_t(\my|\mx)\Vert^2\mrd \my 
\lesssim \frac{\epsilon}{\sigma_t}.
$$
Additionally,
\eqref{eq: clipping_bound_2} can be derived in the same manner,  that is,
$$
\begin{aligned}
\int_{\Vert\my\Vert_{\infty} > m_t + C\sigma_t\sqrt{\log\epsilon^{-1}}} p_t(\my|\mx)\mrd\my 
& \lesssim  \int_{C\sqrt{\log\epsilon^{-1}}}^{\infty}\sigma_t\exp\left(-\frac{r^2}{2}\right)\cdot r^{d_\mcY-1}\mrd r \\
& \lesssim \sigma_t\epsilon.
\end{aligned}
$$

Now, we consider the second part of this lemma. For $\Vert\my\Vert_\infty \leq m_t + C\sigma_t\sqrt{\log\epsilon^{-1}}$, by Lemma \ref{lem: bound_for_density} and Lemma \ref{lem: derivatives_boundness},  we have
$$
p_t(\my|\mx)\mathbf{I}_{\{p_t(\my|\mx)\leq\epsilon_1\}}\Vert\nabla\log p_t(\my|\mx)\Vert^2 \lesssim \epsilon_1\cdot\frac{\log\epsilon^{-1}}{\sigma_t^2}.
$$
Therefore, we have
$$
\begin{aligned}
\int_{\Vert\my\Vert_{\infty}\leq m_t + C\sigma_t\sqrt{\log\epsilon^{-1}}} p_t(\my|\mx)\mathbf{I}_{\{p_t(\my|\mx)\leq\epsilon_1\}}\Vert\nabla\log p_t(\my|\mx)\Vert^2\mrd\my &\lesssim \left(m_t + C\sigma_t\sqrt{\log\epsilon^{-1}}\right)^{d_\mcY} \cdot \frac{\epsilon_1\log\epsilon^{-1}}{\sigma_t^2}\\
&\lesssim \frac{\epsilon_1}{\sigma_t^2}(\log\epsilon^{-1})^{\frac{d_\mcY+2}{2}}.
\end{aligned}
$$
In the same manner, we also have
$$
\begin{aligned}
\int_{\Vert\my\Vert_{\infty}\leq m_t + C\sigma_t\sqrt{\log\epsilon^{-1}}} p_t(\my|\mx)\mathbf{I}_{\{p_t(\my|\mx)\leq\epsilon_1\}}\mrd \my &\lesssim \left(m_t + C\sigma_t\sqrt{\log\epsilon^{-1}}\right)^{d_\mcY} \cdot \epsilon_1 \\ 
&\lesssim \epsilon_1 (\log\epsilon^{-1})^{\frac{d_\mcY}{2}}.
\end{aligned}
$$
The proof is complete.
\end{proof}
\subsection{Auxiliary Lemmas on ReLU Network Approximation}\label{sec: app}
In this section, we summarize existing results and fundamental tools for function approximation using neural networks. See \cite{oko2023diffusion,fu2024unveil} for more details.

\subsubsection{Construction of a large ReLU network}

\begin{lemma}[Neural Network Concatenation]\label{lem: concatenation}
For a series of ReLU networks $\mb_1: \Rbb^{d_1}\rightarrow\Rbb^{d_2}$, $\mb_2:\Rbb^{d_2}\rightarrow\Rbb^{d_3}$, $\cdots$, $\mb_k:\Rbb^{d_k}\rightarrow\Rbb^{d_{k+1}}$ with $\mb_i\in\mathrm{NN}(L_i,M_i,J_i,\kappa_i)~(i=1,2,\cdots,k)$, there exists a neural network $\mb\in\mathrm{NN}(L,M,J,\kappa)$ satisfying $\mb(\mx) = \mb_k\circ\mb_{k-1}\circ\cdots\circ\mb_1(\mx)$ for all $\mx\in\Rbb^{d_1}$, with
$$
L = \sum_{i=1}^{k}L_i, ~~ M\leq 2\sum_{i=1}^{k}M_i, ~~ J\leq 2\sum_{i=1}^{k}J_i, ~~\text{and} ~ \kappa\leq\max_{1\leq i \leq k}\kappa_i.
$$

\end{lemma}

\begin{lemma}[Identity Function]\label{lem: identity_function}
Given $d\in\mathbb{N}_{+}$ and $L\geq 2$, there exists a neural network $\mb_{\mathrm{Id}, L} \in \mathrm{NN}(L,M,J,\kappa)$ that realizes $d$-dimensional identity function $\mb_{\mathrm{Id},L}(\mx) = \mx$, $\mathbf{x}\in\Rbb^{d}$. Here 
$$
M = 2d, ~~ J = 2dL, ~~ \kappa = 1.
$$
\end{lemma}

\begin{lemma}[Neural Network Parallelization]\label{lem: parallelization}
For a series of ReLU networks $\mb_i: \Rbb^{d_i}\rightarrow\Rbb^{d_i^{\prime}}$ with $\mb_i\in\mathrm{NN}(L_i,M_i,J_i,\kappa_i)~(i=1,2,\cdots,k)$, there exists a neural network $\mb\in\mathrm{NN}(L,M,J,\kappa)$ satisfying $\mb(\mx) = [\mb_1^{\top}(\mx_1), \mb_2^{\top}(\mx_2),\cdots,\mb_{k}^{\top}(\mx_k)]:\rightarrow \Rbb^{d_1+d_2+\cdots+d_k}\rightarrow\Rbb^{d_1^{\prime} + d_2^{\prime} + \cdots + d_k^{\prime}}$ for all $\mx = (\mx_1^{\top},\mx_2^{\top},\cdots,\mx_k^{\top})^{\top}\in\Rbb^{d_1+d_2+\cdots+d_k}$($\mx_i$ can be shared), with
$$
\begin{aligned}
L = L, ~~ M\leq 2\sum_{i=1}^{k}M_i, ~~ J\leq 2\sum_{i=1}^{k}J_i, ~~\text{and} ~ \kappa\leq\max_{1\leq i \leq k}\kappa_i ~~ (\text{when} ~ L = L_i ~ \text{holds  for all} ~ i), \\
L = \max_{1\leq i\leq k} L_i, ~~ M \leq 2\sum_{i=1}^{k} M_i, ~~ J \leq 2\sum_{i=1}^{k} (J_i + Ld_i^{\prime}),~~ \text{and} ~ \kappa\leq \max{\{\max_{1\leq i\leq k}\kappa_i, 1\}} ~ (\text{otherwise}). 
\end{aligned}
$$    
Moreover, for $\mx_1 = \mx_2 = \cdots = \mx_k = \mx\in\Rbb^d$ and $d_1^{\prime} = d_2^{\prime} = \cdots = d_k^{\prime} = d^{\prime}$, there exists a neural network $\mb_{\mathrm{sum}}\in\mathrm{NN}(L,M,J,\kappa)$ that realizes $\mb_{\mathrm{sum}}(\mx) = \sum_{i=1}^{k}\mb_i(\mx)$ with
$$
L = \max_{1\leq i\leq k} L_i + 1, ~~ M \leq 4\sum_{i=1}^{k} M_i, ~~ J \leq 4\sum_{i=1}^{k} (J_i + Ld_i^{\prime}) + 2M,~~ \text{and} ~ \kappa\leq \max{\{\max_{1\leq i\leq k}\kappa_i, 1\}}.
$$
\end{lemma}

\subsubsection{Approximation of basic functions with ReLU network}

\begin{lemma}[Approximating the Multiple Products]\label{lem: product}
Let $d\geq 2$, $C\geq 1$. For any $0<\epsilon<1$, there exists a neural network $\mathrm{b}_{\mathrm{prod}}\in\mathrm{NN}(L,M,J,\kappa)$ with $L = \mathcal{O}(\log d(\log d + \log\epsilon^{-1} + d\log C))$, $M = 48d$, $J = \mathcal{O}(d(\log d + \log\epsilon^{-1} + d\log C))$, $\kappa = C^d$ such that
$$
\left|\mathrm{b}_{\mathrm{prod}}(x_1^{\prime},x_2^{\prime},\cdots,x_d^{\prime}) - \prod_{i=1}^d x_i\right| \leq \epsilon + dC^{d-1}\epsilon_0,
$$
for all $\mx\in[-C,C]^d$ and $\mx^{\prime}\in\Rbb^d$ with $\Vert\mathbf{x} - \mx^{\prime}\Vert_{\infty} \leq \epsilon_0$. Moreover, $|\mathrm{b}_{\mathrm{prod}}(\mx^{\prime})|\leq C^d$ for all $\mx^{\prime}\in\Rbb^d$, and $\mathrm{b}_{\mathrm{prod}}(x_1^{\prime},x_2^{\prime},\cdots,x_d^{\prime}) = 0$ if at least one of $x_i^{\prime}$ is 0.
\end{lemma}

\begin{remark}
We note that some of $x_i$, $x_j$ $(i\neq j)$ can be shared. For $\prod_{i=1}^{I}x_i^{u_i}$ with $u_i\in\mathbb{N}_{+} (i=1,2,\cdots,I)$ and $\sum_{i=1}^{I}u_i = d$, there exists a neural network satisfying the same bounds as above.    
\end{remark}

\begin{lemma}[Approximating the Reciprocal Function]\label{lem: reciprocal}
For any $0 < \epsilon < 1$, there exists a neural network $\mathrm{b}_{\mathrm{rec}}\in\mathrm{NN}(L,M,J,\kappa)$ with $L = \mathcal{O}(\log^2\epsilon^{-1})$, $M = \mathcal{O}(\log^3\epsilon^{-1})$, $J = \mathcal{O}(\log^4\epsilon^{-1})$ and $\kappa = \mathcal{O}(\epsilon^{-2})$ such that
$$
\left|\mathrm{b}_{\mathrm{rec}}(x^{\prime}) - \frac{1}{x}\right| \leq \epsilon + \frac{|x^{\prime} - x|}{\epsilon^2},
$$
for all $x \in [\epsilon, \epsilon^{-1}]$ and $x^{\prime}\in\mathbb{R}$.
\end{lemma}

\begin{lemma}[Approximating the Square Root Function]\label{lem: square_root}
For any $0 < \epsilon < 1$, there exists a neural network $\mathrm{b}_{\mathrm{root}}\in\mathrm{NN}(L,M,J,\kappa)$ with $L = \mathcal{O}(\log^2\epsilon^{-1})$, $M = \mathcal{O}(\log^3\epsilon^{-1})$, $J = \mathcal{O}(\log^4\epsilon^{-1})$ and $\kappa = \mathcal{O}(\epsilon^{-1})$ such that
$$
|\mathrm{b}_{\mathrm{root}}(x^{\prime}) - \sqrt{x}| \leq \epsilon + \frac{|x^{\prime} - x|}{\sqrt{\epsilon}},
$$
for all $x\in[\epsilon, \epsilon^{-1}]$ and $x^{\prime}\in\Rbb$.
\end{lemma}

\subsubsection{Clipping functions}
\begin{lemma}[Clipping Function]
\label{lem: clipping}
For any $\mathbf{a}$, $\mathbf{b}\in\Rbb^d$ with $a_i \leq b_i$ 
$(i=1,2,\cdots,d)$, there exists a clipping function $\mathrm{b}_{\mathrm{clip}}(\mx, \mathbf{a},\mathbf{b})\in\mathrm{NN}(L,M,J,\kappa)$ with
$$
L = 2, ~ M = 2d, ~ J = 7d, ~ \kappa = \left(\max_{1\leq i \leq d}\max\{|a_i|, |b_i|\}\right) \vee 1, 
$$
such that
$$
\mathrm{b}_{\mathrm{clip}}(\mx, \mathbf{a}, \mathbf{b})_i = \min\{b_i, \max\{x_i, a_i\}\} ~~ (i=1,2,\cdots,d).
$$
When $a_i=c_{min}$ and $b_i=c_{max}$ for all $i$, 
we sometimes denote $\mathrm{b}_{\mathrm{clip}}(\mx,\mathbf{a},\mathbf{b})$ as $\mathrm{b}_{\mathrm{clip}}(\mx,c_{min},c_{max})$ using scalar values $c_{min}$ and $c_{max}$.
\end{lemma}

\bibliographystyle{spbasic}
\bibliography{biblio}

\end{document}